\pdfoutput=1
\documentclass{article}

\usepackage{iclr2027_conference,times}
\iclrfinalcopy
\usepackage{etoolbox}

\usepackage[utf8]{inputenc}
\usepackage[T1]{fontenc}
\usepackage{amsmath,amssymb,amsthm}
\usepackage{graphicx}
\usepackage{booktabs}
\usepackage{xcolor}
\usepackage{enumitem,needspace}
\usepackage{hyperref}
\usepackage{url}
\usepackage{caption}
\usepackage{subcaption}
\usepackage{microtype}
\usepackage{longtable}
\usepackage{array,mathtools,multicol,placeins}
\definecolor{PaperLink}{HTML}{59616A}
\definecolor{PaperCite}{HTML}{235C59}
\definecolor{PaperExternal}{HTML}{254C6F}
\hypersetup{
  colorlinks=true,
  linkcolor=PaperLink,
  citecolor=PaperCite,
  urlcolor=PaperExternal,
  filecolor=PaperExternal,
  pdfborder={0 0 0},
  breaklinks=true,
  bookmarksnumbered=true,
  bookmarksopen=false,
  pdfdisplaydoctitle=true,
  pdftitle={Can Representation Learning Decouple from Loss Minimization? Polar Updates Have an Answer},
  pdfauthor={Akash Kumar}
}
\graphicspath{{figures/}{code/figs/}}

\newtheorem{theorem}{Theorem}
\newtheorem{proposition}[theorem]{Proposition}
\newtheorem{lemma}[theorem]{Lemma}
\newtheorem{corollary}[theorem]{Corollary}
\newtheorem{definition}[theorem]{Definition}
\newtheorem{remark}[theorem]{Remark}

\newtheorem{conjecture}[theorem]{Conjecture}

\newcommand{\R}{\mathbb{R}}
\newcommand{\E}{\mathbb{E}}
\newcommand{\sgn}{\operatorname{sgn}}
\newcommand{\Polar}{\operatorname{Polar}}

\newcommand{\Tr}{\operatorname{tr}}
\newcommand{\diag}{\operatorname{diag}}
\providecommand{\linspan}{\operatorname{span}}
\newcommand{\rank}{\operatorname{rank}}

\let\origeqref\eqref
\renewcommand{\eqref}[1]{Eq.~\origeqref{#1}}

\title{Can Representation Learning Decouple\\ from Loss Minimization?\\ Polar Updates Have an Answer}
\author{Akash Kumar\\
Department of Computer Science and Engineering\\
University of California San Diego}
\date{}

\begin{document}
\maketitle
\lhead{}
\renewcommand{\headrulewidth}{0pt}

\begin{abstract}
Does representation learning stop when the training loss stops improving? We study this question for matrix Muon, whose polar-normalised updates have a step length set by the gradient's rank rather than its norm. Near the edge of stability, full-batch Muon on teacher-student problems enters approximately period-$2$ loss oscillations that persist for thousands of steps: the cycle-mean loss stays flat or rises, yet the weights keep moving and the learned features continue to align with the teacher subspace. For linear teacher-student learning toys, we derive explicit cycle and alignment formulas and conditional plateau and decay bounds. For a population mean-field ReLU model, we prove that, under stated dimension, initialisation and small-head conditions, the leading eigenspace of the average gradient outer product (AGOP) recovers the teacher subspace exactly during a loss plateau, before the loss later drops. In all $33$ ReLU, GELU and SiLU teacher configurations we study, direction-only alignment metrics show the student AGOP aligned with, or still aligning to, the teacher subspace during the period-$2$ oscillations; projected head refitting on selected configurations shows that the learned directions are useful for prediction, and further measurements distinguish AGOP alignment from weight-mass concentration. In deep residual ReLU students, freezing the downstream layers while the first layer trains with full-batch exact polar updates recreates a nearly flat cycle-mean loss with improving input-AGOP alignment; freezing and unfreezing switch between this plateau and loss decrease, and the effect is sensitive to momentum and to the choice of orthogonaliser.
\end{abstract}

\section{Introduction}\label{sec:intro}

A training curve is the most common readout of learning, yet large-step training is rarely
monotone. Gradient descent near the edge of stability oscillates around sharp
directions~\citep{cohen2021eos,arora2022eos,li2022progressive,damian2023selfstabilization}, and
large-learning-rate stochastic gradient descent produces transient loss spikes, or
\emph{catapults}~\citep{lewkowycz2020catapult,zhu2024catapults}. Such curves raise a basic question
of interpretation: when the training loss stops improving, has the network also stopped learning
useful representations? The loss measures the error of the current predictor, whereas representation
(or feature) learning concerns the input directions on which that predictor depends. Nothing forces
the two to move together. Kernel alignment can precede appreciable loss reduction from small
initialisation~\citep{atanasov2022silent}, teacher directions can be recovered during oscillatory
training in a single-neuron model~\citep{chen2023twostep}, and catapults coincide with improved
alignment of the student's and the teacher's average gradient outer products
(AGOPs)~\citep{zhu2024catapults}, the matrices that summarise which input directions a predictor
depends on and that also underlie the Neural Feature
Ansatz~\citep{radhakrishnan2024science,beaglehole2024alignment}. In these works, however, the separation
is tied to early training, a single neuron or a loss spike. We ask \textit{whether representation learning can continue through a
\emph{sustained} plateau, thousands of steps during which the loss is flat or even rising, and what in
the training dynamics produces such a separation.}

We answer this question for \emph{polar updates} of matrix
Muon~\citep{jordan2024muon}. For a trainable hidden-weight matrix $W$, training loss $L$ and step
size $\eta>0$, the update is
\[
W_{t+1}=W_t-\eta\,\Polar(\nabla_W L(W_t)),
\qquad
\Polar(g)=AB^\top
\quad\text{for the compact SVD }g=A\Sigma B^\top.
\]
Polar normalisation sets every nonzero singular value of the gradient to one, so a full-rank
$p\times r_s$ gradient, with input dimension $p$ and hidden width $r_s$, gives a step of Frobenius
length $\eta\sqrt{M}$, $M=\min(p,r_s)$, whatever its norm. Near a loss floor the update therefore does
not shrink: the weights keep moving by a fixed amount. A fixed step length alone neither creates an
oscillation nor implies learning, but it makes polar updates a natural place to look for learning that
continues while the loss stalls.

\paragraph{Our answer.}
For polar updates the answer the posed questionis yes, and we establish it at three levels. In \emph{learning toys},
linear regression trained by normalised gradient descent or by matrix Muon, we prove that under
explicit phase and margin conditions the dominant residual approaches a nearly symmetric two-cycle
whose loss floor scales as $\eta^2$, while smaller components drain and teacher alignment can keep
improving; for matrix Muon we give exact weight cycles, their loss floors and teacher-alignment
formulas (Section~\ref{sec:toys}). In a \emph{population mean-field model} of a one-hidden-layer ReLU
network, we prove that, under stated dimension, initialisation and small-head conditions, the leading
AGOP eigenspace recovers the teacher subspace exactly while the loss stays within $1\%$ of its initial
value, and that the loss drops only later (Section~\ref{sec:1NN-mf}). In \emph{finite networks} we
measure the separation directly. In all $33$ ReLU, GELU and SiLU teacher configurations
(Table~\ref{tab:all33}), full-batch polar updates produce thousands of steps of approximately
period-$2$ loss oscillations, during which the student AGOP is aligned with, or keeps aligning to, the
teacher subspace (Figures~\ref{fig:1hl-main} and~\ref{fig:theory-vs-exp}); in Figure~\ref{fig:1hl-main}
the cycle-mean loss stays nearly constant or rises while the alignment grows, and
Section~\ref{sec:1NN-phenomenon} gives the per-configuration outcomes. The period-$2$ signature concerns the loss: consecutive loss increments have a
second-half seed-median correlation of $-1.0000$ in each of the $15$ configurations of our diagnostic
suite, while the weights keep moving, which is what allows the alignment to continue.

\paragraph{Measuring learning without the loss.}
Because the loss hides this progress, we measure the representation directly. We consider the \textit{principal angles}
between the leading student-AGOP eigenspace (average gradient outer product) and the teacher subspace to track the recovery of directions
separately from changes in scale. A no-oracle diagnostic, $L_{\mathrm{AGOP\text{-}opt}}$, asks whether
the recovered directions support a better predictor: it projects the hidden weights onto their own
leading AGOP eigenspace and refits the head, without using the teacher. On the six cleanest ReLU
configurations the cycle-mean training loss is about $5$--$11$ times this diagnostic over the last
$30\%$ of training (seed medians). Because the projection also changes the features, this shows that
the learned directions are useful, not that the head alone limits the loss. Separate measurements show
that AGOP alignment need not come with concentration of weight mass in the teacher subspace.

\paragraph{Where the separation comes from.}
The learning toys isolate a simple mechanism. With a fixed Hessian the plateau is a floor set by the
step size rather than by the problem: plain gradient descent with a stable step reaches zero loss on
the same objectives, whereas the normalised updates settle on nonzero cycles, and misalignment carried
by small-eigenvalue directions costs so little loss that it drains slowly while the loss hardly changes
(Section~\ref{sec:toys}). In deep residual ReLU students the same separation appears when the first
layer trains against a fixed downstream map: freezing every later layer while the first layer takes
exact polar steps keeps the cycle-mean loss nearly flat as input-AGOP alignment rises, and freezing or
unfreezing mid-run switches between this regime and loss reduction (Figure~\ref{fig:freeze-switch}).
The effect depends on the optimiser. Muon's quintic Newton--Schulz approximation and the tested
momentum settings disrupt it in deep students, while the small-head plateau of the one-hidden-layer
student is more robust to both (Appendix~\ref{app:fp-ns}).

\paragraph{Contributions.}
\begin{enumerate}[leftmargin=*]
\item \emph{Measurements and interventions:} we use principal-angle AGOP metrics, a
no-oracle head-refit diagnostic, ReLU weight-mass measurements, and freezing
interventions whose effect depends on momentum and the orthogonaliser
(Sections~\ref{sec:1NN-phenomenon} and~\ref{sec:deep-main};
Appendix~\ref{app:fp-scope}).
\item \emph{Learning toys with closed-form theory:} for normalised-GD linear
regression (M1/M2) and matrix-Muon multi-output regression (M3), we derive conditional
plateau and drain bounds, exact weight cycles, loss floors and
teacher-alignment formulas (Section~\ref{sec:toys}). Symmetric-cycle rates are
linearised references, not convergence guarantees from arbitrary starts.
\item \emph{A population mean-field separation:} we prove that, under stated dimension,
initialisation and small-head conditions, a Gaussian ReLU mean-field model
recovers the teacher subspace exactly in its leading AGOP eigenspace during a
$1\%$ loss plateau, followed by a quantified loss decrease
(Section~\ref{sec:1NN-mf}).
\item \emph{Structural and conditional results:} we present exact polar step-length and
update-mass identities, a plateau-height conjecture and a conditional
drift proposition whose assumptions remain unverified for our trajectories and
which does not derive the fitted exponent
(Sections~\ref{sec:1NN-identities} and~\ref{sec:1NN-claims}).
\end{enumerate}

\paragraph{Organisation.}
Section~\ref{sec:related} reviews related work and Section~\ref{sec:setup} sets up the models, the
optimiser and the alignment metrics. Section~\ref{sec:toys} treats the learning toys, and
Section~\ref{sec:1NN} the one-hidden-layer and deep residual networks, including the mean-field
theorem. Section~\ref{sec:discuss} discusses the results. Proofs, additional
experiments and the numerical certificates of the mean-field proof are in the appendices.

\section{Related work}\label{sec:related}
\paragraph{Feature learning and oscillatory dynamics.}
Silent alignment describes early NTK alignment before appreciable loss
reduction from small initialisation~\citep{atanasov2022silent}, with related
directional changes in ReLU gradient flow~\citep{boursier2025early}.
Grokking separates training-data fit from later
generalisation~\citep{kumar2024grokking}.
Edge-of-stability analyses study non-monotone GD
dynamics~\citep{cohen2021eos,li2022progressive,damian2023selfstabilization};
normalised GD admits quadratic two-cycles and, under additional assumptions,
slow sharpness-reducing drift~\citep{arora2022eos}.
\citet{chen2023twostep} prove teacher-direction recovery with convergence to an
oscillatory orbit in a single-neuron ReLU model under specified initialisation and
step-size conditions. Catapult experiments link loss spikes to improved
student--teacher AGOP alignment~\citep{zhu2024catapults}.
We examine AGOP alignment during sustained exact-polar loss plateaus;
our supporting mean-field theory uses distinct population and small-head
assumptions and does not assert period-two dynamics.

\paragraph{AGOP and weight geometry.}
Gradient-based dimension reduction has statistical
antecedents~\citep{hardle1989ade,trivedi2014egop}.
The Neural Feature Ansatz relates weight geometry to a predictor's own
AGOP~\citep{radhakrishnan2024science,beaglehole2024alignment}.
Exact linear-network identities and nonlinear counterexamples delimit this
relationship~\citep{tansley2026nfa}; FACT instead derives a weight-geometry
identity at stationary points with positive weight
decay~\citep{boixadsera2026fact}.
We measure alignment to the teacher subspace along the training trajectory.
Its separation from weight-mass concentration distinguishes directional
from spectral information, but does not itself contradict an identity
between weight geometry and the student's own AGOP.

\paragraph{Muon and spectral learning.}
Practical Muon orthogonalises momentum approximately~\citep{jordan2024muon};
we isolate its exact-polar, zero-momentum update.
Recent work studies balanced spectral acquisition~\citep{vasudeva2026spectral},
feature diversity in a stylised sequential model~\citep{tian2026li2},
and robustness and transfer of learned representations~\citep{ruan2026features}.
AMUSE includes an exact-polar quadratic cycle~\citep{kim2026amuse}, while
river-valley theory characterises rapid signal progress and local
overshoot~\citep{shen2026river}.
Our focus is the timing of AGOP alignment relative to loss reduction.
The polar map and its norm-constrained optimisation interpretation are
background~\citep{higham1986polar,gawlik2017frechet,pethick2025scion};
fixed step length alone does not establish a cycle.

\section{Problem setup}\label{sec:setup}

\paragraph{Conventions.}
Expectations use $\E[Z]$ or $\E[Z\mid B]$, with variable/law subscripts.
Vector $\|\cdot\|$ is Euclidean; matrix $\|\cdot\|_{\mathrm{op}},\|\cdot\|_F$ are spectral/Frobenius norms.
$\Tr,\rank,\linspan,\diag$ mean trace, rank, column span and diagonal matrix;
$I_k$ is the identity and $e_k$ a coordinate vector. $P_S$ projects orthogonally onto $S$,
$P_S^\perp=I-P_S$; $A\succ B$ means $A-B$ positive definite.
Singular values $\sigma_i$ decrease; $\cos^2(v,w)=(v^\top w)^2/(\|v\|^2\|w\|^2)$ for $v,w\ne0$.

We study learning toys (linear regression trained by (normalized) gradient descent: \textbf{M1}, $p=2$;
\textbf{M2}, general $p$; or matrix Muon: \textbf{M3}, multi-output) and a 1-HL
teacher--student network (\textbf{M4}). The toys isolate a fixed-Hessian mechanism;
M4 adds an activation-dependent Hessian. Throughout, $\eta$ is fixed near the edge of
stability, $p$ is the input dimension, $r_t$ the teacher rank, $r_s$ the student width, and
\begin{equation}
M := \min(p,r_s)
\label{eq:M-defn}
\end{equation}
is the rank scale that sets the Frobenius length of a full-rank matrix-Muon step.

\subsection{Data and teacher--student model}\label{sec:data}

For M4, the $n$ training inputs are i.i.d.\ Gaussian, $x_i\sim\mathcal N(0,I_p)$,
with labels $y_i=f^\star(x_i)$ and $y=(y_1,\ldots,y_n)^\top$. The
teacher is a one-hidden-layer network
\begin{equation}
f^\star(x) = \sum_{j=1}^{r_t}\sigma(u_j^\top x),
\qquad
U=[u_1,\ldots,u_{r_t}]\in\mathbb R^{p\times r_t},
\qquad
U^\top U=I_{r_t},
\label{eq:teacher}
\end{equation}
with teacher subspace $V:=\linspan(U)$, $P_V:=UU^\top$, $P_\perp:=I_p-P_V$,
and $\sigma\in\{\mathrm{ReLU},\mathrm{GELU},\mathrm{SiLU}\}$.
The student is a 1-HL (hidden-layer) ReLU network with trainable hidden weights
$W=[w_1|\cdots|w_{r_s}]\in\mathbb R^{p\times r_s}$ and \emph{fixed} head
$a_s=p^{-1/2}\mathbf 1_{r_s}$:
\begin{equation}
f_W(x) = \sum_{k=1}^{r_s} a_{s,k}\,\mathrm{ReLU}(w_k^\top x),
\qquad
L_{\mathrm{train}}(W) = \tfrac1n\sum_{i=1}^n\bigl(f_W(x_i)-f^\star(x_i)\bigr)^2.
\label{eq:student-loss}
\end{equation}
Freezing $a_s$ forces all representation changes through $W$. We use
full-batch gradients throughout, and write $\mathcal L:=L_{\mathrm{train}}$ in M4.

\subsection{Optimiser: matrix Muon}\label{sec:optim}

Matrix Muon updates the hidden weights by the polar factor of the
gradient:
\begin{equation}
W_{t+1} = W_t-\eta\,\Polar(\nabla_W L(W_t)),
\qquad
\Polar(g):=AB^\top \text{ for the compact SVD } g=A\Sigma B^\top
\label{eq:muon}
\end{equation}
($\Polar(g)=g(g^\top g)^{-1/2}$ when $g$ has full column rank). The
practical optimiser approximates this via Newton--Schulz iterations; we
analyse the exact polar update, with $\Polar(0):=0$. Its M4 iteration map is
\begin{equation}
\Phi(W):=W-\eta\,\Polar(\nabla\mathcal L(W)).
\label{eq:Phi-def}
\end{equation}
For $g:=\nabla\mathcal L(W)$ of rank $M$, let $\ell_1,\ldots,\ell_M$ be its left singular vectors.
The leakage quantities are $\varepsilon_V^2:=\sum_{i\le r_t}\|P_\perp\ell_i\|^2$ and
$\varepsilon_\perp^2:=\sum_{i=r_t+1}^{M}\|P_V\ell_i\|^2$ (assuming $r_t\le M$).
Two properties drive the analysis. All
nonzero singular values of $\Polar(g)$ are one, so
\begin{equation}
\|\Polar(g)\|_F^2=\rank(g),
\qquad
\|W_{t+1}-W_t\|_F=\eta\sqrt{M}
\text{ when }\rank(\nabla_W L(W_t))=M,
\label{eq:polar-frob}
\end{equation}
i.e.\ step length is independent of the gradient norm. And $\Polar$ is
left/right orthogonal-equivariant: $\Polar(QgR^\top)=Q\,\Polar(g)\,R^\top$
for orthogonal $Q,R$.
When $r_s=1$ the polar factor reduces to vector normalization
$\Polar(g)=g/\|g\|_2$, and matrix Muon becomes NGD (normalized gradient descent),
\begin{equation}
\theta_{t+1} = \theta_t-\eta\,\nabla L(\theta_t)/\|\nabla L(\theta_t)\|_2,
\label{eq:ngd}
\end{equation}
which is the update used in M1 and M2.

\subsection{Alignment metrics for M4}\label{sec:metrics}

We measure feature alignment through the student AGOP
\begin{equation}
G_s(W) := \tfrac1n\sum_{i=1}^n \nabla_x f_W(x_i)\nabla_x f_W(x_i)^\top
\in\mathbb R^{p\times p}.
\label{eq:student-agop}
\end{equation}
Let $\bar V_k=[v_1,\ldots,v_k]$ collect the top $k$ orthonormal eigenvectors
of $G_s(W)$, and let $\theta_1,\ldots,\theta_{r_t}$ be the principal
angles, in increasing order, between $\linspan(\bar V_{r_t})$ and $V$ ($\cos\theta_i =
\sigma_i(\bar V_{r_t}^\top U)$). We use the direction-only summaries
\begin{equation}
\overline{\cos^2}(W) = \tfrac1{r_t}\sum_{i=1}^{r_t}\cos^2\theta_i,
\qquad
\cos^2_{\min}(W) = \min_{i\le r_t}\cos^2\theta_i ,
\label{eq:cos-metrics}
\end{equation}
which separate subspace recovery from AGOP spectrum, and the weight
Frobenius mass-in-$V$
\begin{equation}
A_{\mathrm{sub}}(W) = \|P_VW\|_F^2 / \|W\|_F^2.
\label{eq:Asub}
\end{equation}
Our no-oracle head-refit diagnostic projects $W$ onto its top-$\tilde r$ AGOP eigenspace:
$\widehat W:=\bar V_{\tilde r}\bar V_{\tilde r}^\top W$, with the rank rule in
Appendix~\ref{app:additional}. For projected features
$\Psi_{ik}:=a_{s,k}\mathrm{ReLU}(\widehat w_k^\top x_i)$ and
$a^\star:=\arg\min_a\|\Psi a-y\|_2^2$, define
\begin{equation}
L_{\mathrm{AGOP\text{-}opt}}(W) := \tfrac1n\|\Psi a^\star-y\|_2^2.
\label{eq:LAGOPopt}
\end{equation}
This minimizes training loss over heads on the student's AGOP-projected features;
the oracle $L_{V,\mathrm{opt}}$ replaces $\bar V_{\tilde r}$ by $U$.
For a loss sequence $L_t$, its cycle mean is $(L_t+L_{t+1})/2$;
$\Delta L_t:=L_{t+1}-L_t$ and $\rho_2:=-\mathrm{Corr}(\Delta L_t,\Delta L_{t+1})$,
with Pearson correlation over the stated time window.

For the conditional drift result, set $A_t:=\cos^2_{\min}(W_t)$,
$q:=r_t/M$, $e_k:=1-A_{2k}$ and
$v_k:=(A_{2k+2}-A_{2k})/2$; $k$ indexes pairs of steps, so $v_k$
is the alignment gain per optimisation step along the even subsequence.
The symbol $\theta$ collects the remaining fixed family parameters,
including the learning-rate rule.
We use $\nu>0$ for the expansion scale $q^\nu$ and an integer $j\ge1$
for its first active dissipative order, with local drift exponent
$\alpha=\nu j$ (distinct from the mean-field head scale denoted by
$\alpha$ below).
For positive quantities, $f\asymp_\theta g$ means
$c_\theta g\le f\le C_\theta g$ for positive constants depending only
on $\theta$, uniformly over the stated parameter range and trajectory window.

\subsection{Learning toys}\label{sec:toy-setup}

The learning toys keep samples, a teacher and the empirical MSE, but use a
linear student, so the Hessian is the fixed data matrix $2\hat\Sigma$. Draw
$x_1,\ldots,x_n\overset{\text{iid}}{\sim}\mathcal N(0,\Sigma_x)$ with
$\Sigma_x\succ0$ and $n\ge p\ge2$, let $X\in\R^{n\times p}$ have rows $x_i^\top$, and let
$\hat\Sigma:=X^\top X/n=Q\,\diag(\mu_1,\ldots,\mu_p)\,Q^\top$ with
$\mu_1>\mu_2>\cdots>\mu_p>0$ (almost surely) and eigenvectors $q_k=Qe_k$.

\paragraph{M2 (M1 for $p=2$): NGD on linear regression.}
Fix $\eta>0$. Teacher $u\in\R^p\setminus\{0\}$, labels $y_i=u^\top x_i$, student $f_w(x)=w^\top x$,
trained by NGD (cf.~\eqref{eq:ngd}) on
\begin{equation}
\hat{\mathcal L}(w)=\tfrac1n\sum_{i=1}^n\bigl(w^\top x_i-y_i\bigr)^2
=(w-u)^\top\hat\Sigma(w-u)=\sum_{k=1}^p\mu_kz_k^2,
\qquad z:=Q^\top(w-u).
\label{eq:lin-loss}
\end{equation}
We track the teacher alignment $\cos^2(w_t,u)$ -- the input AGOP of a
linear student is $G_s=ww^\top$, so this is the $r_t=1$ case of
$\cos^2_{\min}$ in~\eqref{eq:cos-metrics} -- and the residual alignment
$\cos^2(w_t-u,q_1)$; the remaining toy notation is in Appendix~\ref{app:lin-proofs}.

\paragraph{M3: matrix Muon on multi-output regression.}
Teacher $U_\star\in\R^{p\times r}$ with orthonormal columns, labels
$y_i=U_\star^\top x_i\in\R^r$, student $f_W(x)=W^\top x$ with
$W\in\R^{p\times r}$ with $1\le r<p$ (so $M=r$), trained by~\eqref{eq:muon} on
$\hat{\mathcal L}(W)=\tfrac1n\|XW-XU_\star\|_F^2=\Tr(R^\top\hat\Sigma R)$,
$R:=W-U_\star$. The data have a planted block,
$\Sigma_x=\lambda_DP_D+\lambda_\perp P_D^\perp$ with $\dim D=r$ and
$\lambda_D\gg\lambda_\perp$. The AGOP of $f_W$ (summed over outputs) is
$WW^\top$, so the M4 metric becomes
$\cos^2_{\min}(\mathrm{range}\,W_t,\mathrm{range}\,U_\star)$. We place the
teacher either in the high-variance block, (A) $\mathrm{range}\,U_\star=D$,
or adversarially in the low-variance directions, (B)
$\mathrm{range}\,U_\star\subset D^\perp$ (requiring $2r\le p$). These are population blocks; their empirical counterparts need not coincide at finite $n$.
Write $\hat D$ for the top-$r$ eigenspace of $\hat\Sigma$,
$R_{\hat D}:=P_{\hat D}R$ and $R_\perp:=(I-P_{\hat D})R$.

\begin{figure}[t]
\centering
\makebox[\linewidth][c]{\includegraphics[trim=0 5bp 0 0,clip]{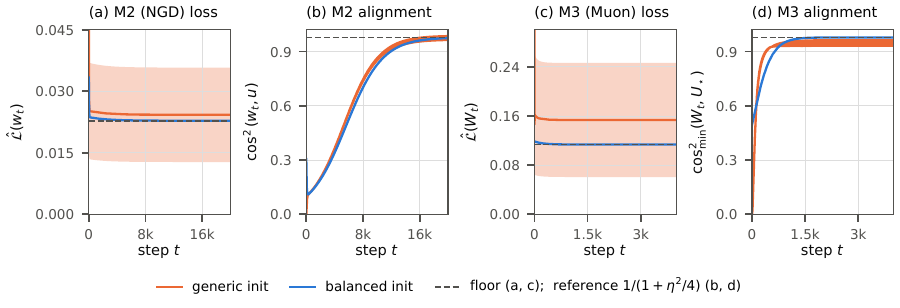}}%
\caption{\textbf{Learning toys: nearly flat cycle-mean loss with improving teacher alignment.}
(a,b) M2: NGD; (c,d) M3: matrix Muon; $\eta=0.3$, low-variance teachers.
Blue: balanced leading components with nonzero off-mode residuals; orange: generic starts.
Bands show adjacent-step loss ranges; dashed lines give symmetric-cycle references.
These finite-sample illustrations do not verify all theorem hypotheses.
Settings and endpoints: Appendix~\ref{app:lin-verify}; optimiser comparison: Figure~\ref{fig:lin-optimizer}.}
\label{fig:lin-main}
\end{figure}

\subsection{Population mean-field model}\label{sec:variants}
We use the Gaussian inputs and teacher directions of Section~\ref{sec:data},
set $r=r_t$, and take $\sigma=\mathrm{ReLU}$ in~\eqref{eq:teacher}.
For a fresh input $X$, let $Y=\gamma f^\star(X)$, where $\gamma>0$ makes
$\E[Y^2]=1$, and write $\sigma_Y:=\sqrt{\mathrm{Var}(Y)}$.
The student in~\eqref{eq:student-loss} now uses fixed head
$a_{s,j}=\alpha/r_s$, $\alpha>0$, and an added intercept $b$ refitted at every
step so that $\E[f_W(X)]=\E[Y]$.
It takes the polar update~\eqref{eq:muon} on the population loss
$L(W):=\E[(Y-f_W(X))^2]$, starting from
$w_j(0)\overset{\mathrm{iid}}{\sim}\nu_0:=\mathcal N(0,\sigma_0^2I_p/p)$,
with $\sigma_0,\bar\eta>0$, $\eta=\bar\eta\sqrt{r_s}$ and
$\kappa:=\sigma_0/(\bar\eta\sqrt p)$.

At infinite width $r_s\to\infty$ with $\alpha,\bar\eta,\sigma_0$ fixed,
the \emph{mean-field model} describes the hidden weights by a probability
law $\nu_t$. Starting from $\nu_0$, it follows the distributional polar
update in Appendix~\ref{app:mf}, with output
$f_t(x)=b_t+\alpha\E_{w\sim\nu_t}[\mathrm{ReLU}(w^\top x)]$,
intercept refitted as above, and loss $L(t):=\E[(Y-f_t(X))^2]$.
Let $G_s(t)$ be the population version of~\eqref{eq:student-agop} for $f_t$,
with decreasing eigenvalues $\lambda_i(t)$ and any top-$r$ spectral
projector $P_r(t)$. Then $A_r(t):=r^{-1}\Tr(P_VP_r(t))$ is the population
analogue of $\overline{\cos^2}$ in~\eqref{eq:cos-metrics}.

\paragraph{Deep residual students.}
For depth $L\ge2$ and width $m$, the bias-free residual ReLU student is
\[
\begin{aligned}
 (2\le\ell\le L),\,\, h_1&=\mathrm{ReLU}(x^\top W_1),\,\, h_\ell = h_{\ell-1}+\frac{\mathrm{ReLU}(h_{\ell-1}W_\ell)}{L-1},\quad f_W(x)=h_La.
\end{aligned}
\]
Here $W_1\in\R^{p\times m}$, $W_\ell\in\R^{m\times m}$ for $\ell\ge2$, and the head
$a=(c_a/\sqrt p)\mathbf1_m$ is frozen, with $c_a>0$.
Every unfrozen hidden matrix takes the full-batch update
$W_\ell\leftarrow W_\ell-\eta_\ell\Polar(\nabla_{W_\ell}\mathcal L)$, where
$\eta_\ell=\eta_{\mathrm{deep}}$ for $\ell\ge2$ and $\mathcal L$ is the squared training loss.
These runs use all thin-SVD columns in the polar factor; Appendix~\ref{app:deep} gives
that convention and the settings. Input-AGOP alignment uses~\eqref{eq:cos-metrics};
Appendices~\ref{app:freeze}--\ref{app:freeze-principle} specify the freezing interventions.

\section{Learning toys with closed-form theorems}\label{sec:toys}

The toys' empirical losses are shifted quadratics. Appendix~\ref{app:lin-proofs} gives
proofs, exact residual recursions (Lemma~\ref{lem:lin-reduction}) and the period-$2$
theory with teacher alignment; Appendix~\ref{app:lin-verify} gives numerical checks.

\paragraph{Scalar regression (M2).}
The scalar normalised map has a continuum of period-$2$ cycles $\{a,a-\eta\}$, so the M2 theorem is
stated from a phase-selected handoff time $T_1$. Write $z_{t,\ge2}:=(z_{t,2},\ldots,z_{t,p})$,
$\rho_t^2:=\sum_{k\ge2}\mu_k^2z_{t,k}^2/(\mu_1\eta)^2$ and $q:=\max_{k\ge2}|1-2\mu_k/\mu_1|<1$, and fix
$C_H>0$. With a constant $c_\star>0$ that depends only on $C_H$ and the ratios $\mu_k/\mu_1$, the
hypotheses are phase selection and a contraction margin,
\begin{equation}
\textnormal{(H)}\ \ \bigl|z_{T_1,1}-\tfrac{\eta}{2}\bigr|\le C_H\eta\,\varepsilon_{\mathrm{phase}},
\qquad
\textnormal{(C)}\ \ \varepsilon_{\mathrm{phase}}+\rho_{T_1}^2+\rho_t^2\le c_\star(1-q)\quad(t\ge T_1),
\label{eq:M2-hyp-main}
\end{equation}
for some $\varepsilon_{\mathrm{phase}}\ge0$. These are the hypotheses of Lemma~\ref{thm:single}
(Appendix~\ref{app:lin-ngd}) with $\lambda_k=\mu_k$, and the $O(\cdot)$ constants below depend only on
$C_H$ and the ratios $\mu_k/\mu_1$.

\begin{theorem}[M2: loss plateau with teacher alignment]\label{thm:M2prime}
Assume (H) and (C) in~\eqref{eq:M2-hyp-main}. Then for all $t\ge T_1$:
\begin{enumerate}[leftmargin=*, label=\textnormal{(\roman*)}]
\item \emph{Cycle:} $z_{t,1}=(-1)^{t-T_1}\frac{\eta}{2}+\xi_\infty+O(\eta\rho_t^2)$
for a constant $\xi_\infty=O\bigl(\eta(\varepsilon_{\mathrm{phase}}+\rho_{T_1}^2)\bigr)$.
\item \emph{Plateau:}
$\hat{\mathcal L}(w_t)=\frac{\mu_1\eta^2}{4}+\sum_{k\ge2}\mu_kz_{t,k}^2
+O\bigl(\mu_1\eta^2(\varepsilon_{\mathrm{phase}}+\rho_{T_1}^2+\rho_t^2)\bigr)$.
\item \emph{Drain:} for $k\ge2$,
$z_{t+1,k}=\bigl(1-\frac{2\mu_k}{\mu_1}\bigr)z_{t,k}
+O\bigl(\frac{\mu_k}{\mu_1}(\varepsilon_{\mathrm{phase}}+\rho_{T_1}^2+\rho_t^2)|z_{t,k}|\bigr)$;
thus $\|z_{t,\ge2}\|$ decays geometrically with some factor $\bar q\in(q,1)$; $q$ is the symmetric-cycle linearised reference.
\item \emph{Teacher alignment:} with the cycle point
$w_t^{\mathrm{cyc}}:=u+z_{t,1}q_1$ (assumed nonzero, as is $w_t$),
\begin{equation}
\bigl|\cos^2(w_t,u)-\cos^2(w_t^{\mathrm{cyc}},u)\bigr|
\;\le\;\|z_{t,\ge2}\|/\|w_t^{\mathrm{cyc}}\|,
\label{eq:M2p-align}
\end{equation}
where $\cos^2(w_t^{\mathrm{cyc}},u)=1$ if $u\parallel q_1$, and
$\cos^2(w_t^{\mathrm{cyc}},u)=\|u\|^2/(\|u\|^2+z_{t,1}^2)$ if $u\perp q_1$,
which for a unit teacher is the \emph{cycle cap} $(1+\eta^2/4)^{-1}$ up to
$O(\eta|\xi_\infty|+\eta^2\rho_t^2)$.
\end{enumerate}
\end{theorem}

The leading residual coordinate thus settles into a nearly symmetric two-cycle whose floor
$\mu_1\eta^2/4$ dominates the loss, while the other coordinates drain geometrically and teacher
alignment keeps improving; the proof is in Appendix~\ref{app:lin-M2p}. Small-eigenvalue modes can
carry substantial misalignment but little loss: at $\eta=0.3$, a unit residual hidden below $5\%$ of
the floor requires at least $443$ steps per e-fold to leading order (Corollary~\ref{cor:invisible}).
Generic starts can select asymmetric cycles; their sustained plateaus in Figure~\ref{fig:lin-main} are
numerical observations, not consequences of the conditional theorem.

\begin{figure}[t]
\centering
\includegraphics{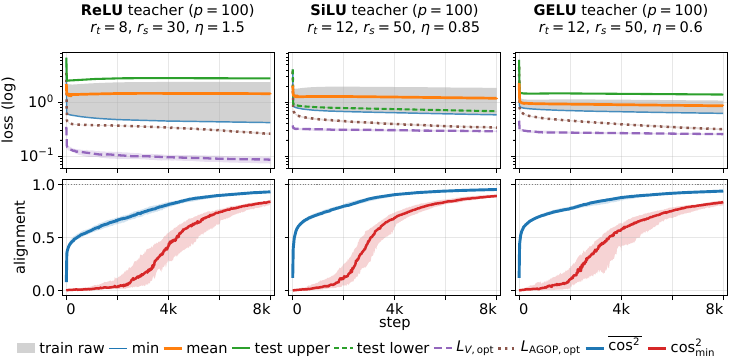}%

\caption{\textbf{Loss plateaus while AGOP alignment improves.}
One ReLU, SiLU and GELU teacher configuration; $8{,}000$ matrix-Muon steps.
Top: training/test losses and head-refit diagnostics. Bottom: direction-only AGOP alignment.
Medians over $53$ seeds, shaded interquartile ranges.
Full metrics: Figure~\ref{fig:three-best}; selection: Appendix~\ref{app:selection}.}
\label{fig:1hl-main}
\end{figure}

\begin{proposition}[M3: exact weight $2$-cycles, floor, drain and teacher alignment]\label{prop:M3prime}\label{prop:M3-main}
In M3, let $U_\star^\top U_\star=I_r$, $R_t:=W_t-U_\star$ and $\eta>0$, and let $\hat\Sigma\succ0$
have eigenvalues $\mu_1\ge\cdots\ge\mu_p>0$ with $\mu_r>\mu_{r+1}$. Then:
\begin{enumerate}[leftmargin=*, label=\textnormal{(\roman*)}]
\item \emph{Step and loss:} $\|W_{t+1}-W_t\|_F=\eta\sqrt r$ whenever $R_t$ has
full column rank, and $\hat{\mathcal L}(W_t)=\Tr(R_{\hat D,t}^\top\hat\Sigma
R_{\hat D,t})+\Tr(R_{\perp,t}^\top\hat\Sigma R_{\perp,t})$.
\item \emph{Exact weight $2$-cycles:} suppose $R_T$ has full column rank and columns in
$\hat D$; write $\hat\Sigma R_T=OP$ with $O^\top O=I_r$ and $P\succ0$, and set
$K:=O^\top\hat\Sigma O$ and $P':=\eta K-P$. If $P'\succ0$, then $W_{t+2}=W_t$ for all
$t\ge T$, with $W_{T+1}=W_T-\eta O$. The two phase losses are $\Tr(PK^{-1}P)$ and
$\Tr(P'K^{-1}P')$ (possibly equal); their mean is at least
$\frac{\eta^2}{4}\sum_{i\le r}\mu_i$, with equality iff
$W_T=U_\star+\frac{\eta}{2}O$ and $W_{T+1}=U_\star-\frac{\eta}{2}O$
(the symmetric residual cycle, with all $\sigma_i(R_T)=\eta/2$).
\item \emph{Drain of the complement (linearised):} to first order in $R_\perp$,
around a cycle of (ii) each row $j>r$ of $R_\perp$ in the eigenbasis of
$\hat\Sigma$ is multiplied per two steps by
$(I-\eta\mu_jP^{-1})(I-\eta\mu_jP'^{-1})$, a strict Euclidean contraction on that row when
$P,P'\succ\frac{\eta\mu_{r+1}}{2}I$, while the cycle moves only at
$O(\|R_\perp\|_F^2)$. On the symmetric cycle the per-step map is
$R_\perp\mapsto R_\perp-2\hat\Sigma R_\perp K^{-1}$: the $(j,i)$ coefficient in
eigenbases ($j>r\ge i$) is multiplied by $1-2\mu_j/\mu_i\in(-1,1)$.
\item \emph{Teacher alignment:} let $R$ have columns in $\hat D$. If
$\mathrm{range}\,U_\star\perp\hat D$, the multiset of principal-angle cosines between
$\mathrm{range}\,W$ and $\mathrm{range}\,U_\star$ is exactly
$\{(1+\sigma_i(R)^2)^{-1/2}:1\le i\le r\}$, so $\cos^2_{\min}=(1+\eta^2/4)^{-1}$ on the
symmetric cycle; if $\mathrm{range}\,U_\star=\hat D$ and $\|R\|_{\mathrm{op}}<1$, then
$\cos^2_{\min}=1$.
\end{enumerate}
\end{proposition}

Part~(ii) holds on any $r$-dimensional $\hat\Sigma$-invariant subspace $S$, with floor
$\frac{\eta^2}{4}\Tr(\hat\Sigma|_S)$, for every positive-definite quadratic.
For block-isotropic $\hat\Sigma$ ($K=\lambda_DI_r$), the cycle condition is
$\sigma_i(R_T)\in(0,\eta)$ and the floor is $\lambda_D\eta^2r/4$.
On an exact weight cycle, loss and teacher alignment (when defined) repeat every two steps.
Continued alignment change concerns off-cycle trajectories; part~(iii) gives only
their local transverse linearisation, not a convergence theorem.
The proof is in Appendix~\ref{app:lin-M3p}; numerical checks are in Appendix~\ref{app:lin-verify}.

\paragraph{The plateau belongs to the optimiser.}
All three losses have minimum $0$, which plain gradient descent with a stable step reaches
(Figure~\ref{fig:lin-optimizer}); in our runs the normalised updates approach nonzero cycles instead, with floors
that scale as $\eta^2$ (M2 means within $0$--$4\%$ of $\mu_1\eta^2/4$ for $\eta\in[0.01,0.2]$; M3
floors $1.05$--$1.39\times\lambda_D\eta^2r/4$, reflecting sampling and the asymmetric-cycle excess of
Proposition~\ref{prop:M3prime}(ii)). In M3 the fixed Hessian $2\hat\Sigma$ determines which invariant
subspaces can support cycles; in M4 the Hessian also depends on the activation pattern and hence on
the teacher.

\begin{figure}[!t]
\centering
\includegraphics[width=\linewidth]{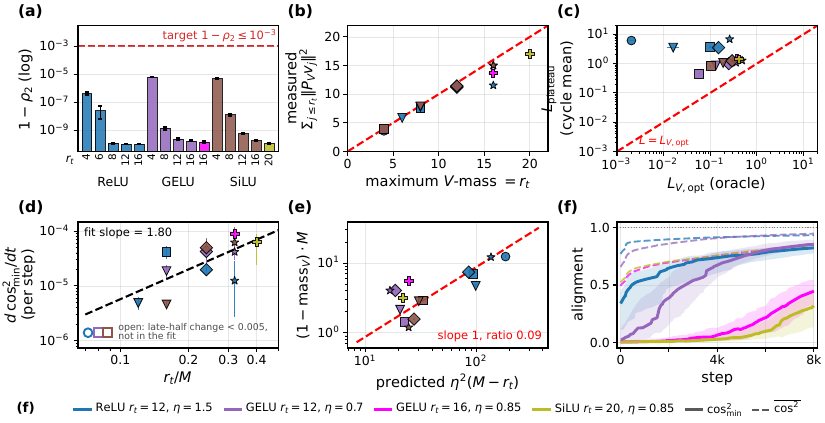}
\caption{\textbf{Theory--experiment diagnostics across 15 ReLU, GELU and SiLU configurations.}
Teachers: ReLU (blue), GELU (purple; magenta $r_t{=}16$, $\eta{=}0.85$) and SiLU (brown; olive $r_t{=}20$);
$p=100$, $r_s=50$, $r_t\in\{4,6,8,12,16,20\}$, $\eta\in\{0.7,0.85,1.5,2\}$; marker shape indexes $r_t$
(full list: Appendix~\ref{app:tests}). Medians and interquartile ranges over $53$ seeds.
(a) Period-$2$ loss signature; (b) top-$r_t$ AGOP $V$-mass versus its maximum;
(c) plateau loss versus head oracle; (d) late-half alignment rate versus $r_t/M$;
(e) off-$V$ AGOP-mass proxy with a unit-slope reference line; (f) alignment trajectories.
Panel (b) does not measure gradient leakage; (d) is descriptive, not a test of
Proposition~\ref{claim:drift}. Definitions, fits and ranges: Appendix~\ref{app:tests}.}
\label{fig:theory-vs-exp}
\end{figure}

\section{The one-hidden-layer ReLU NN under matrix Muon}\label{sec:1NN}

We study the one-hidden-layer ReLU student defined in Section~\ref{sec:data}.

\definecolor{mplblue}{HTML}{1F77B4}
\definecolor{mplorange}{HTML}{FF7F0E}
\definecolor{mplgreen}{HTML}{2CA02C}
\definecolor{mplred}{HTML}{D62728}
\definecolor{mplpurple}{HTML}{9467BD}
\definecolor{mplgrey}{HTML}{555555}
\newcommand{\legline}[1]{{\color{#1}\rule[0.55ex]{9pt}{0.9pt}}}
\newcommand{\legdash}[1]{{\color{#1}\rule[0.55ex]{2.4pt}{0.9pt}\hspace{1.1pt}\rule[0.55ex]{2.4pt}{0.9pt}\hspace{1.1pt}\rule[0.55ex]{2.4pt}{0.9pt}}}
\newcommand{\legdot}[1]{{\color{#1}\rule[0.55ex]{1pt}{0.9pt}\hspace{1pt}\rule[0.55ex]{1pt}{0.9pt}\hspace{1pt}\rule[0.55ex]{1pt}{0.9pt}\hspace{1pt}\rule[0.55ex]{1pt}{0.9pt}\hspace{1pt}\rule[0.55ex]{1pt}{0.9pt}}}
\newcommand{\legband}{{\color{black!22}\rule[0.05ex]{9pt}{1.5ex}}}
\newcommand{\legsep}{\hspace{5.5pt}}
\newcommand{\legbar}{\hspace{5pt}{\color{black!45}\rule[-0.25ex]{0.4pt}{1.7ex}}\hspace{5pt}}
\newcommand{\legtext}[1]{\hspace{2.5pt}#1}
\newcommand{\legendOneHLLegacy}{%
\par\vspace{2pt}\centerline{\fontsize{6.4}{7.5}\selectfont
\legband\legtext{$L_{\mathrm{train}}$ raw}\legsep
\legline{mplblue}\legtext{min}\legsep
\legline{mplorange}\legtext{$L_{\mathrm{train}}$ mean}\legsep
\legline{mplgreen}\legtext{test (upper)}\legsep
\legdash{mplgreen}\legtext{test (lower)}\legsep
\legdash{mplpurple}\legtext{$L_{V,\mathrm{opt}}$}\legsep
\legdot{mplgrey}\legtext{$L_{\mathrm{AGOP},\mathrm{opt}}$}\legbar
\legline{mplblue}\legtext{$\overline{\cos^2}$}\legsep
\legline{mplred}\legtext{$\cos^2_{\min}$}}\par}
\newcommand{\legendDeepLegacy}{%
\par\vspace{2pt}\centerline{\fontsize{6.4}{7.5}\selectfont
\legband\legtext{$L_{\mathrm{train}}$ raw}\legsep
\legline{mplblue}\legtext{min}\legsep
\legline{mplorange}\legtext{$L_{\mathrm{train}}$ mean}\legbar
\legline{mplblue}\legtext{$\overline{\cos^2}$}\legsep
\legline{mplred}\legtext{$\cos^2_{\min}$}}\par}

\newcommand{\legendOneHL}{%
\par\vspace{3pt}{\centering\fontsize{7}{9}\selectfont
\legband\legtext{$L_{\mathrm{train}}$ raw}\legsep
\legline{mplblue}\legtext{min}\legsep
\legline{mplorange}\legtext{$L_{\mathrm{train}}$ mean}\legsep
\legline{mplgreen}\legtext{test (upper)}\legsep
\legdash{mplgreen}\legtext{test (lower)}\legsep
\legdash{mplpurple}\legtext{$L_{V,\mathrm{opt}}$}\legsep
\legdot{mplgrey}\legtext{$L_{\mathrm{AGOP},\mathrm{opt}}$}\legbar
\legline{mplblue}\legtext{$\overline{\cos^2}$}\legsep
\legline{mplred}\legtext{$\cos^2_{\min}$}\par}%
}
\newcommand{\legendDeep}{%
\par\vspace{3pt}{\centering\fontsize{8}{10}\selectfont
\legband\legtext{$L_{\mathrm{train}}$ raw}\legsep
\legline{mplblue}\legtext{min}\legsep
\legline{mplorange}\legtext{$L_{\mathrm{train}}$ mean}\legbar
\legline{mplblue}\legtext{$\overline{\cos^2}$}\legsep
\legline{mplred}\legtext{$\cos^2_{\min}$}\par}%
}

\newcommand{\legendCommonSteps}{%
\begingroup\setlength{\unitlength}{\linewidth}\fontfamily{phv}\fontsize{6}{7}\selectfont
\begin{picture}(1,0.025)
\put(0.0731679,0.009){\makebox(0,0){0}}
\put(0.2148245,0.009){\makebox(0,0){4k}}
\put(0.3526516,0.009){\makebox(0,0){8k}}
\put(0.3869245,0.009){\makebox(0,0){0}}
\put(0.5285811,0.009){\makebox(0,0){4k}}
\put(0.6664083,0.009){\makebox(0,0){8k}}
\put(0.7006811,0.009){\makebox(0,0){0}}
\put(0.8423377,0.009){\makebox(0,0){4k}}
\put(0.9801649,0.009){\makebox(0,0){8k}}
\end{picture}\par
{\fontsize{7}{8}\selectfont step\par}\endgroup}

\subsection{Empirical phenomenology}\label{sec:1NN-phenomenon}

Figure~\ref{fig:1hl-main} illustrates the behaviour.
Across the $33$ ReLU/SiLU/GELU teacher configurations listed in Table~\ref{tab:all33}
(Appendix~\ref{app:grid}), with step sizes tuned near the edge of stability, we observe two features:

\begin{enumerate}[leftmargin=*, label=(P\arabic*)]
\item \label{P:cycle} \emph{Period-$2$ loss signature with fixed step length.}
After a short transient, loss increments alternate almost perfectly:
$\rho_2=1.0000$ to four decimal places (second-half seed medians, all $15$ configurations
of Figure~\ref{fig:theory-vs-exp}(a)). Each full-rank step has length $\eta\sqrt M$;
the weights keep moving, allowing continued alignment rather than an exact weight two-cycle.
\item \label{P:align} \emph{AGOP-into-$V$ rotation during the plateau.}
Both $\overline{\cos^2}$ and $\cos^2_{\min}$ show large net increases, including configurations
with increasing training loss (Figure~\ref{fig:1hl-main}). Over the second half of training
$\cos^2_{\min}$ still rises in $29$ of the $33$ configurations (Table~\ref{tab:all33}); it has
settled by the midpoint in the other four, and in the three ReLU configurations with $r_t\ge16$
it rises but ends below $0.2$, so the weakest teacher direction is not recovered.

\end{enumerate}

\subsection{Universal polar-update identities}\label{sec:1NN-identities}\label{sec:1NN-theory}

\begin{proposition}[polar-update identities]\label{prop:polar-identities}\label{prop:polar}\label{lem:displ}\label{prop:split-correction}\label{thm:split}
For any iterate $W$ with full-column-rank gradient $g$ and $r_t\le M$, the matrix
Muon update satisfies the three exact identities below
(proofs in Appendix~\ref{app:1NN-proofs}):
\begin{enumerate}[leftmargin=*, label=\textnormal{(\roman*)}]
\item \emph{Polar Frobenius identity.}
$\|\Polar(g)\|_F^2 = M$.
\item \emph{Per-step displacement.}
$\|\Phi(W) - W\|_F = \eta\,\sqrt{M}$, independent of the loss
landscape and of the iterate.
\item \emph{$V/V^\perp$ leakage decomposition.}
$\|P_V \Polar(g)\|_F^2 = r_t - \varepsilon_V^2 + \varepsilon_\perp^2$
and
$\|P_\perp \Polar(g)\|_F^2 = (M - r_t) + \varepsilon_V^2 -
\varepsilon_\perp^2$.
In the small-leakage limit, the polar update splits its squared
Frobenius mass between $V$ and $V^\perp$ in the ratio
$r_t : (M - r_t)$.
\end{enumerate}
\end{proposition}

These identities give the fixed step length in~\ref{P:cycle} and the update-mass split;
they do not imply AGOP alignment or a periodic network trajectory.
The network experiments exhibit approximately period-$2$ loss oscillations while the
weights keep moving. The separate existence result for exact weight cycles concerns
only an isotropic quadratic (Theorem~\ref{thm:p2existence}, Appendix~\ref{app:thm1-proof}).
Figure~\ref{fig:theory-vs-exp} summarises the $15$-configuration diagnostics;
Appendix~\ref{app:tests} separates the exact identities from the empirical comparisons.

\subsection{Plateau height and alignment drift}\label{sec:1NN-claims}

Proposition~\ref{prop:polar-identities} establishes the update length and
its decomposition between the teacher subspace and its orthogonal
complement. Theorem~\ref{thm:p2existence} supplies a separate example
of an exact two-cycle on an isotropic quadratic. These results do not
establish the observed network oscillations or AGOP alignment.
The plateau-height statement below is an untested conjecture;
Proposition~\ref{claim:drift} gives sufficient conditions for alignment
drift, with its proof in Appendix~\ref{app:drift-floquet}.

\begin{conjecture}[plateau-height decomposition]\label{claim:plateau}
During the period-$2$ loss plateau of matrix Muon on the 1-HL ReLU
teacher--student loss, the plateau loss admits the leading-order
decomposition
\begin{equation}
\mathcal{L}(W_t) \;\approx\;
\mathcal{L}\bigl(P_V W_t;\,a_s\,\mathrm{fixed}\bigr)
\;+\; c_2\,\eta^2\,(M - r_t),
\label{eq:plateau-height}
\end{equation}
with a configuration-dependent constant $c_2>0$. Test~C (Appendix~\ref{app:tests}) compares the plateau loss with the head-refitted oracle
$L_{V,\mathrm{opt}}$ rather than with the fixed-head term in~\eqref{eq:plateau-height}, so it does not
test the decomposition directly; the polar identities alone do not establish a loss law.
\end{conjecture}

But we can give a conditional statement about the alignment drift. The following proposition
establishes a uniform bound on the alignment drift rate under the stated assumptions.

\begin{proposition}[Conditional secular drift]
\label{prop:conditional-secular-drift}\label{claim:drift}
Let $A_t=\cos^2_{\min}(W_t)$, $q=r_t/M$, $e_k=1-A_{2k}$ and
$v_k=\tfrac12(A_{2k+2}-A_{2k})$, and fix the remaining family
parameters $\theta$, including the learning-rate rule. Suppose that, on
a trajectory window, the leading AGOP subspace is spectrally separated
and near the teacher subspace, and its two-step transverse dynamics
satisfy the uniform expansion of Appendix~\ref{app:drift-floquet}:
positive operator-norm dissipation of finite order $q^{\nu j}$ with a
correspondingly scaled nonlinear remainder. Then, for sufficiently
small $q>0$ and $e_k>0$, uniformly on the window,
\begin{equation}
 v_k\asymp_\theta q^{\alpha}e_k,\qquad \alpha=\nu j>0.
 \label{eq:drift-rate-main}
\end{equation}
Under the additional stationary slow-mode conditions there,
$v_k=C(\theta)q^\alpha e_k(1+o(1))$ with $C(\theta)>0$ as $q,e_k\to0$ and
$kq^\alpha\to\infty$ within the window.
\end{proposition}

\subsection{A mean-field theorem: plateau, then feature learning, then loss drop}\label{sec:1NN-mf}

With the mean-field model as defined in Section~\ref{sec:variants}, ordinary initialisation means $\kappa\gg1$:
the initial weight scale is large relative to the effective step size.
The theorem separates the feature-learning time $t_F$, a chosen plateau horizon $t_P$,
and the later loss-drop time $t_2$. All times count discrete updates.

\begin{theorem}[Plateau, then feature learning, then loss drop; informal version of
Theorem~\ref{mfl:intro-thm:main}]\label{thm:mf}
Let $r\ge2$, $p\ge64r$ and $p-r\ge400$, with head scale $\alpha>0$. There are $\tau_F\in(0,0.101)$, $\kappa_{\min}<\infty$
and $\delta>0$, depending only on $(r,p)$, such that the following holds for every
$\kappa\ge\kappa_{\min}$, with $t_F:=\lfloor\tau_F\kappa\rfloor$, any integer $t_P\ge t_F$ and
$\alpha^*(t_P):=0.00845\,\sigma_Y/(\bar\eta(\kappa+t_P))$.
There is a threshold $a_0=a_0(r,p,\kappa)>0$ such that:
\begin{itemize}
\item[(P)] \emph{Plateau.} If $\alpha\le\alpha^*(t_P)$, the loss stays within $1\%$ of its
initial value, and above $0.99\,\sigma_Y^2$, for $0\le t\le t_P$.
\item[(F)] \emph{Feature learning.} The AGOP is isotropic at $t=0$. If $\alpha\bar\eta\le a_0$, then at $t_F$ its top-$r$ eigenspace is exactly $V$, so $A_r(t_F)=1$, with
$\lambda_r(t_F)\ge(1+\delta/2)\lambda_{r+1}(t_F)$.
\item[(R)] \emph{Loss drop.} If $\alpha\le\alpha^*(t_P)$, then
$L(0)-L(t_2)\ge0.03\,\sigma_Y^2$ at
$t_2:=\lceil0.05086\,\sigma_Y/(\alpha\bar\eta)\rceil>t_P$; at
$\alpha=\alpha^*(t_P)$, $t_2\le6.02(\kappa+t_P)+1$.
\end{itemize}
All three hold together for $0<\alpha\le\min\{\alpha^*(t_P),\,a_0/\bar\eta\}$.
\end{theorem}

The plateau and loss-drop constants are explicit; $a_0$ is not, since (F) transfers from
zero to small positive heads by continuity of the normalized AGOP $G_s/\alpha^2$. Its proof is analytic except on a finite $(r,p)$ region
covered by $80$ outward-rounded interval certificates. Appendix~\ref{app:mf} gives the
precise theorem and complete proof.

\subsection{Deep residual students}\label{sec:deep-main}

Depth-$6$--$16$ residual ReLU students show the same pattern with a large first-layer and small
downstream exact-polar steps (Figure~\ref{fig:deep-depth}; Appendix~\ref{app:deep}).

\begin{figure}[t]
\centering
\includegraphics[width=\linewidth]{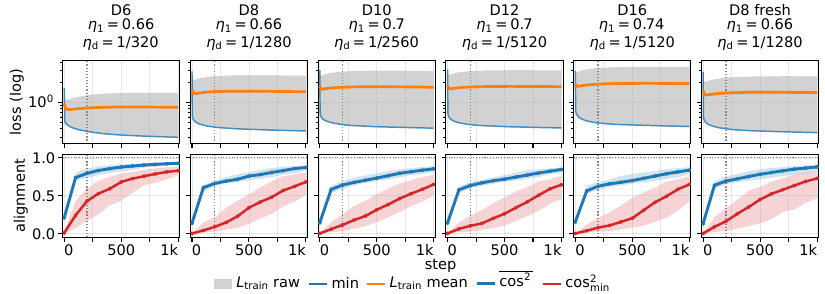}%

\caption{\textbf{Deep residual students: oscillating loss and rising alignment.}
Medians/IQRs over $100$ runs per column ($20$ problems $\times$ $5$ initialisations).
Dotted lines mark scoring windows; across columns, the median per-run
cycle-mean loss change is at most about $2.6\%$.
Settings and metrics: Appendix~\ref{app:deep}.}
\label{fig:deep-depth}
\end{figure}

\paragraph{A freezing principle.} In these students, continued feature learning on a loss plateau
appears when $W_1$ trains against a fixed downstream map
(Appendices~\ref{app:freeze} and~\ref{app:freeze-principle}). Across depths $4$, $8$, $16$ and $24$,
freezing every layer after $W_1$ lets exact polar steps keep the cycle-mean loss within $5\%$ in
almost all runs while input-AGOP $\cos^2_{\min}$ rises by $0.38$--$0.53$ over the window.
At depth $8$, training the downstream layers at $\eta_1/53$ to $\eta_1/3.3$ lowers the loss by $8$--$14\%$,
and at the full step $\eta_1$ by a median $4.7\%$; mid-run freezing or unfreezing switches the plateau
on or off (Figure~\ref{fig:freeze-switch}), with alignment still rising at the slower downstream step.
After partial pretraining with slow GD or Adam, training $W_1$ alone usually aligns features rapidly;
GD or Adam then lower the cycle-mean loss, whereas a large polar step holds it near a floor (Appendix~\ref{app:fp-retrofit}).
In the deep student, Muon's Newton--Schulz quintic and the tested momentum settings disrupt the
combination of flat loss and continued alignment, whereas a convergent iteration preserves it. The one-hidden-layer
student's small-head plateau is more robust to both (Appendix~\ref{app:fp-ns}). Outcomes for
transformers, vision transformers and a CNN remain limited (Appendix~\ref{app:fp-scope}).

\begin{figure}[t]
\centering
\includegraphics[width=\linewidth,trim=0 2bp 0 4bp,clip]{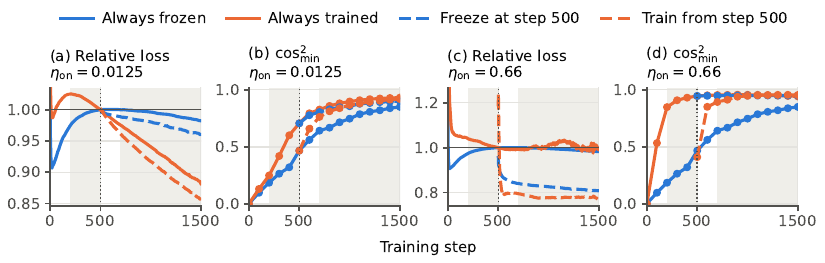}
\caption{\textbf{Freezing downstream layers can restore a loss plateau.}
Depth $8$; the first layer always trains. Dashed curves switch the deep layers at step $500$;
solid controls keep their state. Their step while training is $\eta_{\mathrm{on}}$.
Panels (a,c): relative cycle-mean loss; (b,d): $\cos^2_{\min}$.
Pooled medians over $48$ runs; protocol: Appendix~\ref{app:freeze}.}
\label{fig:freeze-switch}
\end{figure}

\section{Discussion and conclusion}\label{sec:discuss}\label{sec:1NN-discussion}\label{sec:concl}

\paragraph{A plateau is not a pause.}
The training loss and the learned representation answer different questions. The loss reports how
well the current predictor fits the data; AGOP alignment reports which input directions the predictor
has come to depend on. Under polar updates the two can separate for thousands of steps: the loss sits
on a floor set by the step size while the weights keep moving and the leading AGOP eigenspace rotates
into the teacher subspace. The learning toys show that no nonlinearity is needed for this. With a
fixed step length the dominant residual approaches a two-cycle whose loss floor scales as $\eta^2$,
and misalignment carried by small-eigenvalue directions can be nearly invisible in the loss: at
$\eta=0.3$, a unit residual hidden below $5\%$ of the floor needs at least $443$ steps per e-fold to
drain (Corollary~\ref{cor:invisible}). The mean-field theorem gives the same ordering in a nonlinear
network: under its conditions, exact recovery of the teacher subspace precedes the drop in the loss.

\paragraph{What the plateau is made of.}
In the one-hidden-layer experiments the plateau loss stays well above what the learned directions
could support. The no-oracle diagnostic $L_{\mathrm{AGOP\text{-}opt}}$ lies about $5$--$11$ times
below the cycle-mean training loss on the six cleanest ReLU configurations (seed medians over the last
$30\%$ of training; Appendix~\ref{app:additional}), and the teacher-subspace oracle
$L_{V,\mathrm{opt}}$ lies well below the plateau loss across the diagnostic suite
(Figure~\ref{fig:theory-vs-exp}(c)). Both diagnostics change the features as well as the head, so they
show that useful directions are present, not that the head alone limits the loss. How the remaining
loss splits is open: Conjecture~\ref{claim:plateau} proposes a fixed-head projection term plus an
off-subspace term of order $\eta^2(M-r_t)$, but testing it requires the fixed-head projected loss,
which we have not measured.

\paragraph{Directions without weight mass.}
At long-run ReLU checkpoints, $83$--$90\%$ of the AGOP eigenvalue mass lies in $V$
($\overline{\cos^2}\in[0.85,0.96]$), but only $29$--$50\%$ of the weight Frobenius mass does
(Appendix~\ref{app:nfa}). The weight Gram matrix is therefore not proportional to the AGOP, the
simplest form of the Neural Feature Ansatz~\citep{radhakrishnan2022nfa}; power-law variants are
untested. In the small-leakage limit, polar updates split their squared mass between $V$ and its
complement in the ratio $r_t:(M-r_t)$ (Proposition~\ref{prop:polar-identities}). This is consistent
with directions aligning while weight mass does not concentrate, although it does not by itself account
for the measured fractions.

\paragraph{Relation to catapults and edge-of-stability dynamics.}
Catapult experiments connect loss spikes to improvements in AGOP
alignment~\citep{lewkowycz2020catapult,zhu2024catapults}; here alignment improves over thousands of
plateau steps, including intervals in which the loss rises. Gradient descent at the edge of stability
self-stabilises around sharp directions~\citep{cohen2021eos,damian2023selfstabilization}; polar updates
differ in that their step length does not shrink with the gradient, which sustains motion at the
floor. The optimiser matters: in the deep students, gradient descent and Adam on the first layer lower
the loss instead of holding it at a floor (Appendix~\ref{app:fp-retrofit}), and the Newton--Schulz
quintic and the tested momentum settings disrupt the combination of a flat loss and continued alignment
(Appendix~\ref{app:fp-ns}).

\paragraph{Practical implications.}
A flat training curve might hide a latent representation (feature) learning process. In our settings,
direction-only AGOP metrics and head-refit probes reveal progress that the loss hides, and the
freezing interventions show that when the loss falls can be controlled separately from when features
are learned (Figure~\ref{fig:freeze-switch}). Whether such probes are informative for practical Muon
training at scale is an open question.

\paragraph{Open directions.}
Three questions seem most pressing: a finite-width, finite-sample version of the mean-field
separation; an account of how Newton--Schulz approximation, momentum and minibatch noise interact with
the plateau in practical Muon; and a quantitative theory of the plateau height and the alignment drift,
starting with a test of Conjecture~\ref{claim:plateau}.

\paragraph{Conclusion.}
Representation learning can decouple from loss minimization. Under polar updates, a loss plateau need
not be a pause in learning: the separation is proved, under stated conditions, for learning toys and a
mean-field model, measured in one-hidden-layer and deep residual students, and switched on and off by
freezing downstream layers. Training loss alone can therefore understate how much a network has
learned, and direction-only measurements of the representation are a useful complement.

\subsection*{AI use statement}
We used large language model assistants to support the development of the
learning-toy models, the formulation of mathematical claims and hypotheses,
and aspects of proof development and writing. These tools also assisted with
experimental design, implementation (including synthetic-data sampling),
and interpretation of results.
Additionally, we used generative AI tools to create and edit software code and
scientific figures, to check literature and references, and to draft and edit parts of the paper. We have reviewed
all AI-assisted work.
Numerical consistency checks for the learning toys accompany the analytic
proofs (Appendix~\ref{app:lin-verify}); they do not substitute for those proofs.
We take responsibility for the final content of this work, including text,
claims or artifacts produced with the aid of generative AI.

\bibliographystyle{iclr2027_conference}
\bibliography{references}

\appendix
\clearpage

\section*{Appendix: Contents}
\addcontentsline{toc}{section}{Appendix Contents}
\noindent\begin{minipage}{\linewidth}\small
\renewcommand{\arraystretch}{1.15}
\begin{tabular}{@{}p{0.85\linewidth}@{\hfill}p{0.10\linewidth}@{}}
\textbf{A.\ Learning toys: proofs, cycle classification and additional results}
  & p.~\pageref{app:lin-proofs}\\
\quad A.1\ Proof of Lemma~\ref{lem:lin-reduction} (exact reduction)
  & p.~\pageref{app:lin-reduction}\\
\quad A.2\ Proof of Theorem~\ref{thm:M2prime} and Corollary~\ref{cor:invisible}
  & p.~\pageref{app:lin-M2p}\\
\quad A.3\ The normalised-gradient recursion on a quadratic
  & p.~\pageref{app:lin-ngd}\\
\quad A.4\ The scalar cycle family and generic initialisation
  & p.~\pageref{app:lin-family}\\
\quad A.5\ Proof of Proposition~\ref{prop:M3prime} (matrix Muon)
  & p.~\pageref{app:lin-M3p}\\
\quad A.6\ Loss envelopes
  & p.~\pageref{app:lin-envelope}\\
\quad A.7\ Numerical verification of the statements
  & p.~\pageref{app:lin-verify}\\
\quad A.8\ Additional figures
  & p.~\pageref{app:lin-figs}\\[2pt]
\textbf{B.\ Proofs and supporting derivations for Section~\ref{sec:1NN}}
  & p.~\pageref{app:1NN-proofs}\\
\quad B.1\ Proof of Proposition~\ref{prop:polar-identities} (polar-update identities)
  & p.~\pageref{app:prop1-proof}\\
\quad B.2\ Period-$2$ orbits on the isotropic quadratic (Theorem~\ref{thm:p2existence})
  & p.~\pageref{app:thm1-proof}\\
\quad B.3\ Block-orthogonal equivariance
  & p.~\pageref{app:propnoV-proof}\\
\quad B.4\ Assumptions and proof of conditional secular drift
  & p.~\pageref{app:drift-floquet}\\[2pt]
\textbf{C.\ Additional experiments and theory--experiment alignment}
  & p.~\pageref{app:additional}\\
\quad C.1\ Plateau-search grid
  & p.~\pageref{app:grid}\\
\quad C.2\ Per-teacher configurations and the full ReLU grid
  & p.~\pageref{app:per-teacher-extras}\\
\quad C.3\ Configuration selection
  & p.~\pageref{app:selection}\\
\quad C.4\ Theory--experiment overlay and test-by-test detail
  & p.~\pageref{app:tests}\\
\quad C.5\ How theory and experiment align
  & p.~\pageref{app:align-discussion}\\[2pt]
\textbf{D.\ Deep Gaussian teacher--student experiments}
  & p.~\pageref{app:deep}\\
\quad D.1\ Data, teacher and student
  & p.~\pageref{app:deep-model}\\
\quad D.2\ Metrics and scoring protocol
  & p.~\pageref{app:deep-protocol}\\
\quad D.3\ Configurations and seeds
  & p.~\pageref{app:deep-configs}\\
\quad D.4\ Depth results
  & p.~\pageref{app:deep-results}\\[2pt]
\textbf{E.\ Freezing the deep layers switches the plateau on}
  & p.~\pageref{app:freeze}\\
\quad E.1\ Setup
  & p.~\pageref{app:freeze-setup}\\
\quad E.2\ Results
  & p.~\pageref{app:freeze-results}\\
\quad E.3\ Mechanism
  & p.~\pageref{app:freeze-mechanism}\\[2pt]
\textbf{F.\ A freezing principle for the plateau with feature learning}
  & p.~\pageref{app:freeze-principle}\\
\quad F.1\ Depth and width
  & p.~\pageref{app:fp-depth}\\
\quad F.2\ Training only $W_1$ after other schemes
  & p.~\pageref{app:fp-retrofit}\\
\quad F.3\ Newton--Schulz and momentum: two regimes
  & p.~\pageref{app:fp-ns}\\
\quad F.4\ Other architectures and scope
  & p.~\pageref{app:fp-scope}\\[2pt]
\textbf{G.\ The mean-field theorem: complete proof}
  & p.~\pageref{app:mf}\\
\quad G.1\ Main result
  & p.~\pageref{mfl:sec:intro}\\
\quad G.2\ The mean-field map
  & p.~\pageref{mfl:sec:pc}\\
\quad G.3\ The plateau and the loss drop
  & p.~\pageref{mfl:sec:tools}\\
\quad G.4\ The \texorpdfstring{$\alpha=0$}{alpha = 0} trajectory from the Gaussian initialization
  & p.~\pageref{mfl:sec:rb}\\
\quad G.5\ The good-set certificate at \texorpdfstring{$\alpha=0$}{alpha = 0}
  & p.~\pageref{mfl:sec:mff}\\
\quad G.6\ Evaluating the certificate
  & p.~\pageref{mfl:sec:eval}\\
\quad G.7\ Proof of the main theorem
  & p.~\pageref{mfl:sec:main}\\
\quad G.8\ Finite verification details
  & p.~\pageref{mfl:sec:repro}\\
\quad G.9\ Complete certificate data
  & p.~\pageref{mfl:sec:app}\\
\end{tabular}
\end{minipage}
\bigskip
\newpage

\section{Learning toys: proofs, cycle classification and additional results}\label{app:lin-proofs}

\paragraph{Notation.}
We use the learning models of Sections~\ref{sec:toy-setup} and~\ref{sec:toys}.
The sample matrix $X\in\mathbb R^{n\times p}$ has rows $x_i^\top$, and
$\hat\Sigma:=X^\top X/n=Q\diag(\mu_1,\ldots,\mu_p)Q^\top$, where
$Q=[q_1,\ldots,q_p]$ is orthogonal and $\mu_1>\cdots>\mu_p>0$.
The step size is $\eta>0$. For M1/M2 the teacher is $u\ne0$, the label
vector is $y=Xu$, and $z_t:=Q^\top(w_t-u)$;
$z_{t,\ge2}:=(z_{t,2},\ldots,z_{t,p})$ and
$\hat G_t:=(\sum_j\mu_j^2z_{t,j}^2)^{1/2}$.
For M3, $1\le r<p$ and $W_t,U_\star\in\mathbb R^{p\times r}$ with
$U_\star^\top U_\star=I_r$, $Y=XU_\star$, and $R_t:=W_t-U_\star$.
Write $\hat D:=\mathrm{span}(q_1,\ldots,q_r)$,
$R_{\hat D}:=P_{\hat D}R$ and $R_\perp:=(I_p-P_{\hat D})R$, where $P_S$
is the orthogonal projector onto a subspace $S$.
When $R$ has full column rank with columns in $\hat D$, write
$\hat\Sigma R=OP$ ($O^\top O=I_r$, $P\succ0$) and $K:=O^\top\hat\Sigma O$;
$K$ has eigenvalues $\mu_1,\ldots,\mu_r$.
The empirical losses are $\hat{\mathcal L}(w)=n^{-1}\|Xw-y\|_2^2$ and
$\hat{\mathcal L}(W)=n^{-1}\|XW-Y\|_F^2$; their symmetric-cycle floors are
$\mu_1\eta^2/4$ and $\frac{\eta^2}{4}\sum_{i\le r}\mu_i$, respectively.
Here $\|\cdot\|_2$ on vectors is the Euclidean norm, $\|\cdot\|_{\mathrm{op}}$
on matrices is the spectral norm, and $\|\cdot\|_F$ is the Frobenius norm.
We write $\sigma_i(B)$ for the singular values of $B$, $\Tr(B)$ for its trace,
and $B\succ0$ for positive definiteness; $B\succ C$ means $B-C\succ0$.
For a compact SVD $B=U_B\Sigma_BV_B^\top$, the polar factor is
$\Polar(B):=U_BV_B^\top$.
For nonzero vectors, $\cos^2(a,b):=(a^\top b)^2/(\|a\|_2^2\|b\|_2^2)$;
for equal-dimensional subspaces, $\cos^2_{\min}$ is the smallest squared
principal-angle cosine, as in~\eqref{eq:cos-metrics}.

\subsection{Exact reduction and its proof}\label{app:lin-reduction}

\begin{lemma}[exact reduction]\label{lem:lin-reduction}
In M2 the residual coordinates $z_t=Q^\top(w_t-u)$ obey, under NGD,
\begin{equation}
z_{t+1,k}=z_{t,k}\Bigl(1-\frac{\eta\mu_k}{\hat G_t}\Bigr),
\qquad
\hat G_t=\Bigl(\textstyle\sum_{j}\mu_j^2z_{t,j}^2\Bigr)^{1/2},
\label{eq:lin-recursion}
\end{equation}
i.e.\ the normalised-gradient recursion~\eqref{eq:M2-update-setup} on a diagonal
quadratic with $\lambda_k=\mu_k$, and $\hat{\mathcal L}(w_t)=\sum_k\mu_kz_{t,k}^2$ at every
step. In M3 the residual obeys $R_{t+1}=R_t-\eta\,\Polar(\hat\Sigma R_t)$.
\end{lemma}

Positive gradient rescaling leaves NGD and $\Polar$ unchanged, so the gradient's factor $2$
only doubles the reference quadratic loss $\tfrac12\sum_k\lambda_ky_k^2$.

\begin{proof}
Since $y_i=u^\top x_i$, $Xw-y=X(w-u)$, hence
$\hat{\mathcal L}(w)=\frac1n\|X(w-u)\|_2^2=(w-u)^\top\hat\Sigma(w-u)
=\sum_k\mu_kz_k^2$ and $\nabla\hat{\mathcal L}(w)=2\hat\Sigma(w-u)$. The NGD
step is therefore $w_{t+1}-u=(w_t-u)-\eta\,\hat\Sigma(w_t-u)/\|\hat\Sigma(w_t-u)\|_2$;
multiplying by the orthogonal $Q^\top$ gives
$z_{t+1,k}=z_{t,k}-\eta\mu_kz_{t,k}/(\sum_j\mu_j^2z_{t,j}^2)^{1/2}$,
which is~\eqref{eq:lin-recursion} and coincides with~\eqref{eq:M2-update-setup}
for $\lambda_k=\mu_k$. The loss
$\frac12\sum_k\lambda_ky_k^2$ evaluated at $y=z$ equals
$\frac12\hat{\mathcal L}$. For M3, $XW-Y=X(W-U_\star)$, so
$\hat{\mathcal L}(W)=\Tr(R^\top\hat\Sigma R)$, $\nabla\hat{\mathcal L}=2\hat\Sigma R$
and $\Polar(2\hat\Sigma R)=\Polar(\hat\Sigma R)$.
\end{proof}

\subsection{Proof of Theorem~\ref{thm:M2prime} (scalar regression)}\label{app:lin-M2p}

Theorem~\ref{thm:M2prime} is stated in Section~\ref{sec:toys}. Items (i)--(iii) follow from
Lemmas~\ref{thm:single} and~\ref{lem:lin-reduction}; (iv) converts residual drain into teacher
alignment. A small-$\mu_k$ mode can carry substantial misalignment but little loss, at the cost of
slow decay:

\begin{corollary}[invisible misalignment drains slowly]\label{cor:invisible}
In the setting of Theorem~\ref{thm:M2prime}, let $k\ge2$ and $f\in(0,1)$. If
off-mode $k$ carries amplitude $|z_{t,k}|=Z>0$ while contributing at most a
fraction $f$ of the floor, $\mu_kZ^2\le f\mu_1\eta^2/4$, then its
leading-order contraction factor in (iii) is at least $1-f\eta^2/(2Z^2)$, so
to leading order it needs at least $2Z^2/(f\eta^2)-1$ steps to shrink by a
factor $e$ (a vacuous bound when $f\eta^2\ge2Z^2$).
\end{corollary}

At $\eta=0.3$, a unit misalignment hidden below $5\%$ of the floor needs at
least $443$ steps per e-fold: on the symmetric cycle, feature learning that the
loss cannot see is necessarily a long phase, and it happens \emph{inside} the
plateau (Figure~\ref{fig:lin-main}a,b). Without (H) the cycle can settle on
another member $\{a,a-\eta\}$ of the period-$2$ family:
Proposition~\ref{prop:lin-family} gives its plateau, raised by
$2(a^2+(\eta-a)^2)/\eta^2$, and its first-order drain (modes with
$\eta\mu_k/\mu_1<\min(a,\eta-a)$ drain faster, e-fold ratio
$4a(\eta-a)/\eta^2$), but that the plateau then stays flat is a numerical
finding (orange in Figure~\ref{fig:lin-main}).

Here $\rho_t^2:=\sum_{k\ge2}\mu_k^2z_{t,k}^2/(\mu_1\eta)^2$ and
$q:=\max_{k\ge2}|1-2\mu_k/\mu_1|$. The handoff time $T_1$ and
phase tolerance $\varepsilon_{\mathrm{phase}}$ are those of hypotheses
\eqref{eq:M2-phase-hyp} and~\eqref{eq:M2-margin-hyp} in
Lemma~\ref{thm:single}; $\xi_\infty$ is its limiting phase shift.
The comparison point is $w_t^{\mathrm{cyc}}:=u+z_{t,1}q_1$.

\begin{proof}[Proof of Theorem~\ref{thm:M2prime}]
By Lemma~\ref{lem:lin-reduction}, $(z_t)$ is a trajectory of~\eqref{eq:M2-update-setup} for the
diagonal quadratic with $\lambda_k=\mu_k$ (the rotation $Q$ is irrelevant to
Lemma~\ref{thm:single}, which is stated in eigen-coordinates), and
$\hat{\mathcal L}(w_t)=2L(z_t)$. Items (i) and (iii) are
Lemma~\ref{thm:single}(i),(iii) verbatim; (ii) is
Lemma~\ref{thm:single}(ii) multiplied by $2$. By (iii) and the margin
hypothesis (C), $|z_{t+1,k}|\le\bar q|z_{t,k}|$ for some
$\bar q\in(q,1)$, uniformly over $k\ge2$ and $t\ge T_1$.
The factor $q$ describes the symmetric-cycle linearisation, not an
exact asymptotic rate for every phase-shifted trajectory.

For (iv), $w_t-w_t^{\mathrm{cyc}}=\sum_{k\ge2}z_{t,k}q_k$ has norm
$\|z_{t,\ge2}\|_2$. For nonzero $a,b,u$ write $\hat a=a/\|a\|_2$ and let
$\Pi_a^\perp$ be the orthogonal projector onto $a^\perp$. Since
$\cos^2(a,u)=\hat u^\top\hat a\hat a^\top\hat u$,
\[
\begin{aligned}
|\cos^2(a,u)-\cos^2(b,u)|
&\le\|\hat a\hat a^\top-\hat b\hat b^\top\|_{\mathrm{op}}=\sin\angle(a,b)\\
&=\frac{\|\Pi_a^\perp b\|_2}{\|b\|_2}
=\frac{\|\Pi_a^\perp(b-a)\|_2}{\|b\|_2}
\le\frac{\|a-b\|_2}{\|b\|_2}.
\end{aligned}
\]
which gives~\eqref{eq:M2p-align} with $a=w_t$, $b=w_t^{\mathrm{cyc}}$. If
$u\parallel q_1$ then $w_t^{\mathrm{cyc}}\parallel u$ and the cosine is $1$
(when $w_t^{\mathrm{cyc}}\neq0$). If $u\perp q_1$ then
$\langle w_t^{\mathrm{cyc}},u\rangle=\|u\|_2^2$ and
$\|w_t^{\mathrm{cyc}}\|_2^2=\|u\|_2^2+z_{t,1}^2$, so
$\cos^2(w_t^{\mathrm{cyc}},u)=\|u\|_2^2/(\|u\|_2^2+z_{t,1}^2)$; inserting (i),
$z_{t,1}^2=\eta^2/4+O(\eta|\xi_\infty|+\eta^2\rho_t^2)$.
\end{proof}

\begin{proof}[Proof of Corollary~\ref{cor:invisible}]
The hypothesis gives $\delta:=2\mu_k/\mu_1\le f\eta^2/(2Z^2)$, so the
leading-order factor $1-2\mu_k/\mu_1=1-\delta$ of Theorem~\ref{thm:M2prime}(iii)
is at least $1-f\eta^2/(2Z^2)$. If $\delta<1$, the number of steps to shrink by
$e$ at the constant factor $1-\delta\in(0,1)$ is $1/(-\log(1-\delta))$, and
$-\log(1-\delta)\le\delta/(1-\delta)$ gives
$1/(-\log(1-\delta))\ge1/\delta-1\ge2Z^2/(f\eta^2)-1$; if $\delta\ge1$ the bound
is $\le0$ and there is nothing to prove.
\end{proof}

\subsection{The normalised-gradient recursion on a quadratic}\label{app:lin-ngd}

Theorem~\ref{thm:M2prime} is proved by transporting Lemma~\ref{thm:single} below through
Lemma~\ref{lem:lin-reduction}. This appendix states the lemma and proves it.

\paragraph{Setting.} Let
$w\in\mathbb R^p$, $L(w)=\tfrac12 w^\top A w$ with $A=QDQ^\top$,
$D=\diag(\lambda_1,\ldots,\lambda_p)$, $\lambda_1>\lambda_2\ge\cdots\ge\lambda_p>0$.
Here $Q$ is orthogonal, $q_k:=Qe_k$ is the $k$th eigenvector, and $e_k$
is the $k$th standard coordinate vector. The step size is $\eta>0$;
all vector norms in this subsection are Euclidean. This local $D$ is the
diagonal eigenvalue matrix, not the population data subspace used for M3.
In eigenbasis coordinates $y_t:=Q^\top w_t$, NGD decouples component-wise:
\begin{equation}
y_{t+1,k} = y_{t,k}\bigl(1-\tfrac{\eta\lambda_k}{G_t}\bigr),
\qquad G_t = \bigl(\textstyle\sum_{j=1}^p\lambda_j^2y_{t,j}^2\bigr)^{1/2}.
\label{eq:M2-update-setup}
\end{equation}
The gradient norm is $G_t=\|\nabla L(w_t)\|_2=\|Dy_t\|_2$.
We write $y_{t,\ge2}:=(y_{t,2},\ldots,y_{t,p})$ for the off-leading coordinates.
The target direction is the leading eigenvector $q_1=Qe_1$, and
\begin{equation}
\mathrm{align}_t = \cos^2\angle(w_t,q_1) = y_{t,1}^2/\|y_t\|_2^2.
\label{eq:M2-align}
\end{equation}
Lemma~\ref{thm:single} is stated conditionally from a phase-selected
handoff $T_1$, because the scalar normalized map has a continuum of
period-$2$ cycles $\{a,a-\eta\}$.

\paragraph{Statement.}\label{sec:M2}

With this notation, define
\begin{equation}
\rho_t^2
:=
\frac{\sum_{k\ge2}\lambda_k^2y_{t,k}^2}{(\lambda_1\eta)^2}.
\label{eq:rho-M2}
\end{equation}
We state the result conditionally from a phase-selected
handoff time \(T_1\). This is necessary because the scalar normalized
map has a continuum of period-$2$ cycles.

\begin{lemma}[Conditional plateau with alignment for normalised gradient descent on a quadratic]\label{thm:single}
Let
\begin{equation}
q \;:=\; \max_{k\ge 2}\left|1-\tfrac{2\lambda_k}{\lambda_1}\right|
\;<\; 1
\label{eq:M2-q-defn}
\end{equation}
(strict from the spectral gap $\lambda_1>\lambda_2$ and $\lambda_k>0$
for $k\ge 2$). Fix $C_H>0$. There exists $c_\star>0$, depending only
on $C_H$ and the fixed eigenvalue ratios $\lambda_k/\lambda_1$, for
which the following holds. Suppose at some time $T_1$, for
$\varepsilon_{\mathrm{phase}}\ge0$,
\begin{equation}
\textnormal{(H)}\qquad
\left|y_{T_1,1}-\frac{\eta}{2}\right|
\le C_H\eta\varepsilon_{\mathrm{phase}},
\label{eq:M2-phase-hyp}
\end{equation}
and that the perturbation lies inside the contraction margin,
\begin{equation}
\textnormal{(C)}\qquad
\varepsilon_{\mathrm{phase}}+\rho_{T_1}^2+\rho_t^2
\;\le\;
c_\star(1-q)
\qquad\text{for all }t\ge T_1.
\label{eq:M2-margin-hyp}
\end{equation}
All $O(\cdot)$ constants below may depend on $C_H$ and these eigenvalue
ratios, but not on $t$ or $\eta$. Then for all $t\ge T_1$:
\begin{enumerate}[leftmargin=*, label=\textnormal{(\roman*)}]
\item There exists a constant phase shift $\xi_\infty
=
O\!\left(\eta(\varepsilon_{\mathrm{phase}}+\rho_{T_1}^2)\right)$ such that
\begin{equation}
y_{t,1}
=
(-1)^{t-T_1}\frac{\eta}{2}
+
\xi_\infty
+
O(\eta\rho_t^2).
\label{eq:M2-thm-cycle}
\end{equation}

\item The loss satisfies
\begin{equation}
L(w_t)
=
\frac{\lambda_1\eta^2}{8}
+
\frac12\sum_{k\ge2}\lambda_k y_{t,k}^2
+
O\!\left(
\lambda_1\eta^2
(\varepsilon_{\mathrm{phase}}+\rho_{T_1}^2+\rho_t^2)
\right).
\label{eq:M2-thm-loss}
\end{equation}

\item For each $k\ge2$,
\begin{equation}
y_{t+1,k}
=
y_{t,k}
\left(1-\frac{2\lambda_k}{\lambda_1}\right)
+
O\!\left(
(\lambda_k/\lambda_1)
(\varepsilon_{\mathrm{phase}}+\rho_{T_1}^2+\rho_t^2)
|y_{t,k}|
\right).
\label{eq:M2-thm-offmode}
\end{equation}

\item For some $\bar q\in(q,1)$,
\[
\|y_{t,\ge2}\|_2\le \bar q^{\,t-T_1}\|y_{T_1,\ge2}\|_2,\qquad
1-\frac{y_{t,1}^2}{\|y_t\|_2^2}
\le
\frac{16\|y_{T_1,\ge2}\|_2^2}{\eta^2}\,
\bar q^{\,2(t-T_1)}.
\]
The worst squared multiplier of the symmetric-cycle linearisation is
$q^2$; this need not be the exact asymptotic rate of an individual
trajectory.
\end{enumerate}
\end{lemma}

The contraction-margin hypothesis~\eqref{eq:M2-margin-hyp} ensures the
\emph{perturbed} per-step factor $|1-2\lambda_k/\lambda_1| +
O(\varepsilon_{\mathrm{phase}}+\rho_{T_1}^2+\rho_t^2)$ stays uniformly
below $1$; the strict spectral gap $q<1$ is necessary, the
margin-relative smallness of the perturbation in (C) is what makes the
geometric off-mode decay (used both in the bootstrap of (i)--(iii) and
the conclusion of (iv)) uniform in $t$. The absolute value in the definition of $q$
accounts for sign-flipping off-modes with $\lambda_k>\lambda_1/2$.
The handoff hypothesis (H) is assumed; no entry-time result is asserted.

\paragraph{Proof of Lemma~\ref{thm:single}.}\label{app:M2}

\paragraph{Setup.}
Work in the eigenbasis of \(A\). With \(y_t:=Q^\top w_t\), the NGD update is
\begin{equation}
y_{t+1,k}
=
y_{t,k}
\left(1-\frac{\eta\lambda_k}{G_t}\right),
\qquad
G_t
:=
\|Dy_t\|_2
=
\left(\sum_j \lambda_j^2 y_{t,j}^2\right)^{1/2}.
\label{eq:M2-app-update}
\end{equation}
Recall that
\begin{equation}
\rho_t^2
=
\frac{\sum_{k\ge 2}\lambda_k^2 y_{t,k}^2}{(\lambda_1\eta)^2}.
\label{eq:M2-rho-proof}
\end{equation}
The proof starts from the handoff hypothesis~\eqref{eq:M2-phase-hyp};
it does not require or assert convergence into this neighbourhood from
arbitrary initialisation.

\paragraph{Remark on phase selection.}
If \(\rho_t=0\), the leading-mode map reduces to
\begin{equation}
y_{t+1,1}=y_{t,1}-\eta\,\sgn(y_{t,1}).
\label{eq:M2-scalar-map}
\end{equation}
This scalar map has the continuum of period-\(2\) cycles
\(\{a,a-\eta\}\), \(a\in(0,\eta)\). Only the symmetric cycle \(a=\eta/2\)
has constant leading loss \(\lambda_1\eta^2/8\). For a generic cycle
\(\{a,a-\eta\}\), the leading loss alternates between
\begin{equation}
\frac12\lambda_1a^2
\qquad\text{and}\qquad
\frac12\lambda_1(a-\eta)^2.
\label{eq:M2-generic-cycle-loss}
\end{equation}
Thus hypothesis~\eqref{eq:M2-phase-hyp} is a phase-selection condition.

\paragraph{Step A: Leading-mode residual recurrence.}
Let
\begin{equation}
s_t:=(-1)^{t-T_1},
\qquad
y_{t,1}=s_t\frac{\eta}{2}+\sigma_t.
\label{eq:M2-sigma-def}
\end{equation}
By~\eqref{eq:M2-phase-hyp},
\begin{equation}
\sigma_{T_1}
=
O(\eta\varepsilon_{\mathrm{phase}}).
\label{eq:M2-sigma-init}
\end{equation}
Assume temporarily that \(|\sigma_t|/\eta\ll1\), so that \(\sgn(y_{t,1})=s_t\). Then
\begin{equation}
y_{t,1}
=
s_t\frac{\eta}{2}
\left(1+2s_t\sigma_t/\eta\right).
\label{eq:M2-leading-factor}
\end{equation}
Using~\eqref{eq:M2-rho-proof}, factor \(G_t\) as
\begin{equation}
G_t
=
\frac{\lambda_1\eta}{2}
\left(1+2s_t\sigma_t/\eta\right)
\sqrt{1+v_t},
\qquad
v_t
:=
\frac{4\rho_t^2}{(1+2s_t\sigma_t/\eta)^2}.
\label{eq:M2-Gt-factored}
\end{equation}
The cancellation used below is exact:
\begin{equation}
\frac{y_{t,1}}{1+2s_t\sigma_t/\eta}
=
\frac{s_t\eta}{2}.
\label{eq:M2-exact-identity}
\end{equation}
Hence
\begin{equation}
\frac{\eta\lambda_1 y_{t,1}}{G_t}
=
\frac{s_t\eta}{\sqrt{1+v_t}}.
\label{eq:M2-normalized-leading-step}
\end{equation}
Substituting into the leading-coordinate update gives
\begin{align}
y_{t+1,1}
&=
y_{t,1}
-
\frac{s_t\eta}{\sqrt{1+v_t}}
\notag\\
&=
s_t\frac{\eta}{2}+\sigma_t
-
s_t\eta
\left(1-\frac12v_t+O(v_t^2)\right)
\notag\\
&=
-s_t\frac{\eta}{2}
+
\sigma_t
+
\frac{s_t\eta}{2}v_t
+
O(\eta v_t^2).
\label{eq:M2-leading-one-step}
\end{align}
Since
\begin{equation}
v_t
=
4\rho_t^2
+
O\!\left(\rho_t^2|\sigma_t|/\eta\right),
\label{eq:M2-v-expansion}
\end{equation}
and \(s_{t+1}=-s_t\), identifying
\(y_{t+1,1}=s_{t+1}\eta/2+\sigma_{t+1}\) yields
\begin{equation}
\sigma_{t+1}
=
\sigma_t
+
2s_t\eta\rho_t^2
+
E_t,
\qquad
E_t
=
O\!\left(\eta\rho_t^4+\rho_t^2|\sigma_t|\right).
\label{eq:M2-sigma-recurrence}
\end{equation}
Here $E_t$ denotes the deterministic remainder in this recurrence.
Every error term carries a factor of \(\rho_t^2\). In particular, if \(\rho_t=0\), then
\(\sigma_{t+1}=\sigma_t\) exactly.

\paragraph{Step B: Bootstrap and convergence of the phase shift.}
Set
\begin{equation}
\varepsilon_\star
:=
\varepsilon_{\mathrm{phase}}+\rho_{T_1}^2.
\label{eq:M2-eps-star}
\end{equation}
We show that \(|\sigma_t|=O(\eta\varepsilon_\star)\). Assume on an interval
\(T_1\le s\le T\) that
\begin{equation}
|\sigma_s|\le B\eta\varepsilon_\star.
\label{eq:M2-bootstrap-assumption}
\end{equation}
Then
\begin{equation}
G_s
=
\frac{\lambda_1\eta}{2}
\left(1+O(\varepsilon_\star+\rho_s^2)\right).
\label{eq:M2-Gt-bootstrap}
\end{equation}
Therefore, for each \(k\ge2\),
\begin{align}
y_{s+1,k}
&=
y_{s,k}
\left(
1-\frac{\eta\lambda_k}{G_s}
\right)
\notag\\
&=
y_{s,k}
\left(
1-\frac{2\lambda_k}{\lambda_1}
+
O\!\left((\lambda_k/\lambda_1)(\varepsilon_\star+\rho_s^2)\right)
\right).
\label{eq:M2-offmode-bootstrap}
\end{align}
Under the contraction-margin
hypothesis~\eqref{eq:M2-margin-hyp}, the perturbed per-step factor in
the off-mode update satisfies $|1 - 2\lambda_k/\lambda_1| +
O(\varepsilon_\star + \rho_s^2) \le q'$ for some $q' \in (q, 1)$
uniformly over $k \ge 2$ and $s \ge T_1$, hence
\begin{equation}
\rho_{s+1}\le q'\rho_s.
\label{eq:M2-rho-contraction}
\end{equation}
Consequently,
\begin{equation}
\sum_{s=T_1}^{T}\rho_s^2
\le
C\rho_{T_1}^2,
\qquad
\sum_{s=T_1}^{T}\rho_s^4
\le
C\rho_{T_1}^4,
\label{eq:M2-rho-sums}
\end{equation}
where \(C\) may depend on the spectral contraction margin but not on \(T\).

Telescoping~\eqref{eq:M2-sigma-recurrence},
\begin{equation}
\sigma_t
=
\sigma_{T_1}
+
2\eta\sum_{s=T_1}^{t-1}s_s\rho_s^2
+
\sum_{s=T_1}^{t-1}E_s.
\label{eq:M2-sigma-telescope}
\end{equation}
The main sum is alternating with geometrically decreasing magnitudes, hence
\begin{equation}
\left|
\sum_{s=T_1}^{t-1}s_s\rho_s^2
\right|
\le
C\rho_{T_1}^2.
\label{eq:M2-alternating-bound}
\end{equation}
Using~\eqref{eq:M2-bootstrap-assumption} and~\eqref{eq:M2-rho-sums},
\begin{align}
\sum_{s=T_1}^{t-1}|E_s|
&\le
C\eta\sum_{s=T_1}^{t-1}\rho_s^4
+
C\sum_{s=T_1}^{t-1}\rho_s^2|\sigma_s|
\notag\\
&\le
C\eta\rho_{T_1}^4
+
CB\eta\varepsilon_\star\rho_{T_1}^2
\notag\\
&\le
C_B\eta\varepsilon_\star^2.
\label{eq:M2-error-sum-bound}
\end{align}
Combining
\eqref{eq:M2-sigma-init},
\eqref{eq:M2-sigma-telescope},
\eqref{eq:M2-alternating-bound}, and
\eqref{eq:M2-error-sum-bound},
\begin{equation}
|\sigma_t|
\le
C_1\eta\varepsilon_\star
+
C_2B\eta\varepsilon_\star^2.
\label{eq:M2-bootstrap-close-pre}
\end{equation}
Choose \(B>2C_1\), and then take \(\varepsilon_\star\) sufficiently small
so that \(C_2\varepsilon_\star\le 1/2\). The usual first-exit argument
closes the bootstrap:
\begin{equation}
|\sigma_t|
\le
B\eta\varepsilon_\star
\qquad
\text{for all }t\ge T_1.
\label{eq:M2-sigma-bound}
\end{equation}

The same estimates show that \(\sum_{s\ge T_1}E_s\) converges absolutely.
Define
\begin{equation}
\xi_\infty
:=
\sigma_{T_1}
+
2\eta\sum_{s=T_1}^{\infty}s_s\rho_s^2
+
\sum_{s=T_1}^{\infty}E_s.
\label{eq:M2-xi-infty}
\end{equation}
Then
\begin{equation}
\xi_\infty
=
O(\eta\varepsilon_{\mathrm{phase}})
+
O(\eta\rho_{T_1}^2)
=
O(\eta\varepsilon_\star).
\label{eq:M2-xi-bound}
\end{equation}
Applying the same argument to the tail from \(t\) to \(\infty\) gives
\begin{equation}
|\sigma_t-\xi_\infty|
=
O(\eta\rho_t^2).
\label{eq:M2-tail-bound}
\end{equation}
Therefore
\begin{equation}
y_{t,1}
=
(-1)^{t-T_1}\frac{\eta}{2}
+
\xi_\infty
+
O(\eta\rho_t^2),
\label{eq:M2-cycle-final}
\end{equation}
which proves~\eqref{eq:M2-thm-cycle}.

\paragraph{Step C: Off-mode update.}
From~\eqref{eq:M2-cycle-final},
\begin{equation}
G_t
=
\frac{\lambda_1\eta}{2}
\left(1+O(\varepsilon_\star+\rho_t^2)\right).
\label{eq:M2-Gt-cycle}
\end{equation}
Substituting into~\eqref{eq:M2-app-update}, for \(k\ge2\),
\begin{align}
y_{t+1,k}
&=
y_{t,k}
\left(
1-\frac{\eta\lambda_k}{G_t}
\right)
\notag\\
&=
y_{t,k}
\left(
1-\frac{2\lambda_k}{\lambda_1}
\right)
+
O\!\left(
(\lambda_k/\lambda_1)(\varepsilon_\star+\rho_t^2)|y_{t,k}|
\right).
\label{eq:M2-offmode-final}
\end{align}
This proves~\eqref{eq:M2-thm-offmode}.

\paragraph{Step D: Loss expansion.}
Using~\eqref{eq:M2-cycle-final},
\begin{equation}
y_{t,1}^2
=
\frac{\eta^2}{4}
+
O\!\left(\eta^2(\varepsilon_\star+\rho_t^2)\right).
\label{eq:M2-leading-square}
\end{equation}
Hence
\begin{align}
L(w_t)
&=
\frac12\lambda_1y_{t,1}^2
+
\frac12\sum_{k\ge2}\lambda_k y_{t,k}^2
\notag\\
&=
\frac{\lambda_1\eta^2}{8}
+
\frac12\sum_{k\ge2}\lambda_k y_{t,k}^2
+
O\!\left(\lambda_1\eta^2(\varepsilon_\star+\rho_t^2)\right).
\label{eq:M2-loss-final}
\end{align}
This proves~\eqref{eq:M2-thm-loss}.

\paragraph{Step E: Alignment bound and linearised rate.}
By~\eqref{eq:M2-sigma-bound}, taking $c_\star$ sufficiently small ensures
$|y_{t,1}|\ge\eta/4$ for all $t\ge T_1$. The componentwise bound used
in~\eqref{eq:M2-rho-contraction} also gives
$\|y_{t,\ge2}\|_2\le\bar q^{\,t-T_1}\|y_{T_1,\ge2}\|_2$
for some $\bar q\in(q,1)$. Hence
\begin{equation}
1-\frac{y_{t,1}^2}{\|y_t\|_2^2}
=
\frac{\|y_{t,\ge2}\|_2^2}{y_{t,1}^2+\|y_{t,\ge2}\|_2^2}
\le
\frac{16}{\eta^2}\,
\bar q^{\,2(t-T_1)}\|y_{T_1,\ge2}\|_2^2.
\label{eq:M2-align-error}
\end{equation}
This proves (iv). On the exactly symmetric limiting cycle, the
linearised $k$-th off-mode multiplier is
$1-2\lambda_k/\lambda_1$, whose worst squared magnitude is $q^2$.
A nonzero limiting phase shift instead gives alternating multipliers;
even on the symmetric cycle, a mode absent from the trajectory need not
determine its asymptotic rate. The rigorous conclusion here is the
uniform geometric bound above. \qed

\paragraph{Numerical verification.}
\(p = 8\), \(\lambda_1 = 1\), off-mode tail
\([0.30, 0.10, 0.05, 0.02, 0.01, 0.005, 0.001]\),
symmetric init \(y_{0,1} = (k + \tfrac12)\eta\),
off-modes \(y_{0,j} = c\,\eta\) for \(j \ge 2\),
\(n_{\mathrm{steps}} = 4000\).
The measured loss is the mean of the final $200$ iterates; decay$_8$
is $\exp(b)$, where $b$ is the least-squares slope of
$\log|y_{t,8}|$ over $t=2000,\ldots,4000$.

\begin{table}[h]
\caption{Finite-run comparison with the symmetric-cycle reference:
the averaged loss is within $1.1\%$ of the reference floor, and the fitted
slow-mode multiplier is within $10^{-4}$ of $0.998$.
These initialisations give $\rho_0^2\le0.092$ and the displayed starting
alignments; they do not certify the conditional handoff hypotheses.
Final residual alignment rounds to $1.000$. Small nonzero phase shifts
remain, so $0.998^2$ is a linearised squared-rate reference, not an exact
rate established by this table.}
\label{tab:M2-verify}
\centering
\makebox[\linewidth][c]{\small
\begin{tabular}{cccc|cc|cc}
\toprule
$\eta$ & $k$ & $c$ & $\mathrm{align}(0)$
& $\lambda_1\eta^2/8$ pred & $L_{\mathrm{late}}$ measured
& $1 - 2\lambda_8/\lambda_1$ pred & decay$_8$ measured \\
\midrule
0.10 & 50 & 0.70 & 0.999 & 0.001250 & 0.001251 & 0.998 & 0.998 \\
0.20 &  8 & 0.94 & 0.921 & 0.005000 & 0.005033 & 0.998 & 0.998 \\
0.50 &  2 & 0.94 & 0.503 & 0.031250 & 0.031583 & 0.998 & 0.998 \\
\bottomrule
\end{tabular}}
\end{table}

\subsection{The scalar cycle family and generic initialisation}\label{app:lin-family}

\begin{proposition}[the scalar cycle family]\label{prop:lin-family}
In M2 (or M1):
\begin{enumerate}[leftmargin=*, label=\textnormal{(\alph*)}]
\item If $z_{T,k}=0$ for all $k\ge2$, the off-modes stay zero and
$z_{t+1,1}=z_{t,1}-\eta\,\sgn(z_{t,1})$. Every orbit that does not hit $0$ is,
after at most $\lceil|z_{T,1}|/\eta\rceil$ steps, the period-$2$ orbit
$\{a,a-\eta\}$ for some $a\in(0,\eta)$. The loss alternates between
$\mu_1a^2$ and $\mu_1(\eta-a)^2$, and its two-step mean is the floor
$\mu_1\eta^2/4$ times $2(a^2+(\eta-a)^2)/\eta^2\ge1$, with equality iff
$a=\eta/2$.
\item Once mode $1$ is cycling, i.e.\ along any stretch of trajectory with
$a_t:=|z_{t,1}|\in[c\eta,(1-c)\eta]$ for some $c>0$, write
$\chi_k:=\eta\mu_k/\mu_1$ and
$\tilde\rho_t:=(\sum_{j\ge2}\mu_j^2z_{t,j}^2)^{1/2}/(\mu_1a_t)$.
As $\tilde\rho_t\to0$ with $c$ fixed,
\[
z_{t+2,k}=\Bigl(1-\frac{\chi_k}{a_t}\Bigr)\Bigl(1-\frac{\chi_k}{\eta-a_t}\Bigr)z_{t,k}
+O_c\bigl(\tilde\rho_t^2|z_{t,k}|\bigr),\qquad k\ge2 .
\]
If $\chi_k<\min(a_t,\eta-a_t)$, the displayed leading-order factor lies in
$\bigl(0,(1-2\mu_k/\mu_1)^2\bigr]$, with equality iff $a_t=\eta/2$.
Thus the linearised map at fixed cycle amplitude drains these modes at
least as fast as the symmetric-cycle linearisation; the e-fold time is
shortened by $4a_t(\eta-a_t)/\eta^2$ to leading order in $\chi_k$. If $\chi_k\ge\min(a_t,\eta-a_t)$ the factor can be
negative, or exceed $1$ in absolute value.
\item (Selection at frozen off-mode content.) For a fixed
$g\in(0,\mu_1\eta/2)$, the map
$T_g(y)=y-\eta\mu_1y/\sqrt{\mu_1^2y^2+g^2}$ (the mode-$1$ update with the
off-mode gradient norm frozen at $g$) has the unique symmetric $2$-cycle
$\{a_g,-a_g\}$, $a_g^2=(\eta/2)^2-g^2/\mu_1^2$, which is locally attracting
with two-step multiplier $(1-s'(a_g))^2$, $s'(a_g)=8g^2/(\mu_1^2\eta^2)\in(0,2)$.
\end{enumerate}
\end{proposition}

\begin{proof}
(a) With $z_{t,\ge2}=0$ the gradient is $2\mu_1z_{t,1}q_1$ and its
normalisation is $\sgn(z_{t,1})q_1$; the off-mode coordinates are multiplied by
$1-\eta\mu_k/\hat G_t$ and stay $0$. The sign map decreases $|z|$ by $\eta$
until $|z|<\eta$ (unless it lands on $0$), after which $a\mapsto a-\eta\mapsto a$.
The loss values follow from $\hat{\mathcal L}=\mu_1z_1^2$, and
$a^2+(\eta-a)^2=\eta^2/2+2(a-\eta/2)^2$.
(b) Exactly, $z_{t+1,k}=z_{t,k}(1-\eta\mu_k/\hat G_t)$ with
$\hat G_t=\mu_1a_t(1+\tilde\rho_t^2)^{1/2}$, and
$z_{t+1,1}=z_{t,1}-\eta\,\sgn(z_{t,1})(1+\tilde\rho_t^2)^{-1/2}$; for
$a_t<\eta(1+\tilde\rho_t^2)^{-1/2}$ the sign flips, which gives the exact
\emph{sum rule} $|z_{t,1}|+|z_{t+1,1}|=\eta(1+\tilde\rho_t^2)^{-1/2}$, hence
$a_{t+1}=\eta-a_t+O(\eta\tilde\rho_t^2)$. Multiplying the two one-step factors
gives the display. For sufficiently small $\tilde\rho_t$ depending on
$c$, the sign-flip condition holds and the next amplitude remains
bounded below by $c\eta/2$; the off-mode multipliers are then bounded in
terms of $c$, so $\tilde\rho_{t+1}=O_c(\tilde\rho_t)$. This also justifies
the second-step expansion. Writing $\chi=\chi_k$ and $b=\eta-a_t$,
$(1-\frac{\chi}{a_t})(1-\frac{\chi}{b})=1-\frac{\chi(\eta-\chi)}{a_tb}$ because
$a_t+b=\eta$; since $a_tb\le\eta^2/4$ with equality iff $a_t=b$, this is at
most $1-4\chi(\eta-\chi)/\eta^2=(1-2\chi/\eta)^2$, and it is positive when
$\chi<\min(a_t,b)$. For small $\chi$ its logarithm is
$-\chi\eta/(a_tb)+O(\chi^2)$ against $-4\chi/\eta+O(\chi^2)$ on the symmetric
cycle, which gives the e-fold ratio $4a_tb/\eta^2$.
(c) A symmetric $2$-cycle requires $T_g(a)=-a$, i.e.\ $s(a)=2a$ with
$s(y)=\eta\mu_1y/\sqrt{\mu_1^2y^2+g^2}$; $s(a)/a$ decreases strictly from
$\eta\mu_1/g>2$ to $0$, so the solution is unique and
$\mu_1^2a_g^2+g^2=\mu_1^2\eta^2/4$. Since $s$ is odd, $T_g'(\pm a_g)=1-s'(a_g)$
with $s'(y)=\eta\mu_1g^2/(\mu_1^2y^2+g^2)^{3/2}$, which at $a_g$ equals
$8g^2/(\mu_1^2\eta^2)$; it lies in $(0,2)$ iff $g<\mu_1\eta/2$.
\end{proof}

Part~(c) explains the two behaviours of Figure~\ref{fig:lin-main}. The off-mode
gradient norm $g_t^2=\sum_{k\ge2}\mu_k^2z_{t,k}^2$ carries a \emph{second}
power of $\mu_k$, so a mode that is invisible in the loss (small $\mu_kz_k^2$)
exerts almost no symmetrising pressure, and a generic start freezes an
asymmetric member of the family of~(a); loss-visible slow modes, by contrast,
symmetrise the cycle on their way out (the $\eta=0.1$ and M1 ``slow'' runs,
Figure~\ref{fig:lin-M1p}). Writing $m_t=(z_{t,1}+z_{t+1,1})/2$ for the cycle midpoint, summing its frozen-gradient
linearised contraction $1-s'_t$ over a symmetric-rate drain gives, in the
additional slow-mode regime $\mu_k/\mu_1\ll1$ for the modes carrying the
off-mode loss, the heuristic \emph{energy budget}
$-\log(m_\infty/m_{t_0})\approx\sum_{t\ge t_0}s'_t\approx
E_{\mathrm{off}}(t_0)/(2L_{\mathrm{floor}})$, with
$E_{\mathrm{off}}(t_0)=\sum_{k\ge2}\mu_kz_{t_0,k}^2$ and
$L_{\mathrm{floor}}=\mu_1\eta^2/4$: the cycle symmetrises by the off-mode loss it
still has to dissipate, in floor units. Here the logarithm concerns nonzero midpoints of the same sign.
This uses the symmetric-cycle rate, the frozen-$g$ multiplier and the
slow-mode approximation; a hidden-eigenvalue sweep gives fitted slope
$1.09$, a numerical comparison rather than a proof.

\subsection{Proof of Proposition~\ref{prop:M3prime}}\label{app:lin-M3p}

At the cycle point $R_T$, write the full-column-rank polar decomposition
$\hat\Sigma R_T=OP$, with $O^\top O=I_r$ and
$P=(R_T^\top\hat\Sigma^2R_T)^{1/2}\succ0$.
Set $K:=O^\top\hat\Sigma O$ and $P':=\eta K-P$.
In part~(iii), a prime on a residual denotes its one-step update, and
$e_i$ denotes a standard coordinate vector. All matrix remainders there
are measured in Frobenius norm with the cycle point fixed.

\begin{proof}
(i) $\rank(\hat\Sigma R_t)=\rank(R_t)=r$ because $\hat\Sigma\succ0$, so
Proposition~\ref{prop:polar-identities}(ii) gives
$\|W_{t+1}-W_t\|_F=\eta\sqrt r$. Since $\hat D$ and $\hat D^\perp$ are
$\hat\Sigma$-invariant and orthogonal, $R_{\hat D}^\top\hat\Sigma R_\perp=0$ and
the loss splits.

(ii) The columns of $\hat\Sigma R_T$ lie in $\hat D$ and $P\succ0$, so $O$ is an
orthonormal basis of $\hat D$ (here $\dim\hat D=r$ is used), $OO^\top$ is the
projector onto $\hat D$, and invariance gives $\hat\Sigma O=OO^\top\hat\Sigma O=OK$.
Hence $R_{T+1}=R_T-\eta O$ has columns in $\hat D$ and
$\hat\Sigma R_{T+1}=OP-\eta OK=-O(\eta K-P)=-OP'$ with $P'\succ0$. By uniqueness
of the polar factor of a full-column-rank matrix, $\Polar(\hat\Sigma R_{T+1})=-O$,
so $R_{T+2}=R_{T+1}+\eta O=R_T$. The hypothesis is symmetric in $(P,P')$ and
$(O,-O)$ (note $(-O)^\top\hat\Sigma(-O)=K$), so the argument repeats and the
residual orbit is $2$-periodic. Since $W_t=U_\star+R_t$ with fixed $U_\star$,
$W_{t+2}=W_t$ as well. Its phases are distinct because
$\|W_{T+1}-W_T\|_F=\eta\|O\|_F=\eta\sqrt r>0$.
From $O^\top\hat\Sigma R_T=P$ and
$O^\top\hat\Sigma R_T=O^\top\hat\Sigma OO^\top R_T=KO^\top R_T$ we get
$R_T=OK^{-1}P$, hence $\hat{\mathcal L}(W_T)=\Tr(PK^{-1}O^\top\hat\Sigma OK^{-1}P)
=\Tr(PK^{-1}P)$, and likewise $R_{T+1}=-OK^{-1}P'$ with loss $\Tr(P'K^{-1}P')$.
The map $X\mapsto\Tr(XK^{-1}X)=\|K^{-1/2}X\|_F^2$ is strictly convex on
symmetric matrices and $P+P'=\eta K$, so the two-step mean is at least its
value at $X=\eta K/2$, namely $\frac{\eta^2}{4}\Tr(K)=\frac{\eta^2}{4}
\sum_{i\le r}\mu_i$ ($O$ spans the eigenspace of $\mu_1,\ldots,\mu_r$), with
equality iff $P=P'=\eta K/2$, i.e.\ $R_T=\frac{\eta}{2}O$, in which case
$R_{T+1}=-\frac{\eta}{2}O=-R_T$. Thus the symmetric weight cycle is
$W_T=U_\star+\frac{\eta}{2}O$, $W_{T+1}=U_\star-\frac{\eta}{2}O$:
it is centred at the teacher, and its loss is constant at the floor.
The singular values $\eta/2$ belong to $R_T$, not generally to $W_T$.
For $K=\lambda_DI_r$,
$P=\lambda_D(R_T^\top R_T)^{1/2}$ and the condition is
$\sigma_i(R_T)\in(0,\eta)$. The argument uses only that $\hat D$ is an
$r$-dimensional $\hat\Sigma$-invariant subspace; for an invariant subspace of
dimension larger than $r$, $\mathrm{range}\,O$ need not be invariant and the
conclusion can fail.

(iii) Write $G=\hat\Sigma R=G_{\hat D}+G_\perp$ with $G_{\hat D}=\hat\Sigma R_{\hat D}$
and $G_\perp=\hat\Sigma R_\perp$; invariance gives $G_{\hat D}^\top G_\perp=0$,
so $G^\top G=G_{\hat D}^\top G_{\hat D}+G_\perp^\top G_\perp$. On the phase of the
cycle where $R_{\hat D}=R_T$, $G_{\hat D}=OP$ and $G_{\hat D}^\top G_{\hat D}=P^2\succ0$.
The map $S\mapsto S^{-1/2}$ is smooth at positive-definite $S$, so
$(G^\top G)^{-1/2}=P^{-1}+O(\|R_\perp\|_F^2)$ and $\Polar(G)=G(G^\top G)^{-1/2}$
gives
\[
R_\perp'=R_\perp-\eta\hat\Sigma R_\perp P^{-1}+O(\|R_\perp\|_F^3),\qquad
R_{\hat D}'=R_{\hat D}-\eta O+O(\|R_\perp\|_F^2).
\]
On the next phase $G_{\hat D}=-OP'$ and the same computation gives the factor
$P'^{-1}$. In the eigenbasis of $\hat\Sigma|_{\hat D^\perp}=\sum_{j>r}\mu_jq_jq_j^\top$
the row $q_j^\top R_\perp$ is therefore multiplied by $I-\eta\mu_jP^{-1}$ and then
by $I-\eta\mu_jP'^{-1}$. Both are symmetric with eigenvalues
$1-\eta\mu_j/\pi$, $\pi$ an eigenvalue of $P$ (resp.\ $P'$), which lie in
$(-1,1)$ when $\pi>\eta\mu_j/2$; since $\mu_j\le\mu_{r+1}$, the condition
$P,P'\succ\frac{\eta\mu_{r+1}}{2}I$ makes each row multiplier have spectral norm below one. Thus the
full linearised complement map contracts in Frobenius norm. On the symmetric cycle $P=P'=\frac{\eta}{2}K$, the per-step
map is $R_\perp\mapsto R_\perp-2\hat\Sigma R_\perp K^{-1}$, and with
$K=V\diag(\mu_1,\ldots,\mu_r)V^\top$ the coefficients $c_{ji}=q_j^\top R_\perp Ve_i$
are multiplied by $1-2\mu_j/\mu_i$, which lies in $(-1,1)$ because
$0<\mu_j<\mu_i$.

(iv) Let $R$ have columns in $\hat D$. If $\hat D\perp\mathrm{range}\,U_\star$,
then $U_\star^\top R=0$, so $W^\top W=I_r+R^\top R$ and $U_\star^\top W=I_r$. The
columns of $W(W^\top W)^{-1/2}$ are an orthonormal basis of
$\mathrm{range}\,W$, and the principal-angle cosines are the singular values
of $U_\star^\top W(W^\top W)^{-1/2}=(I_r+R^\top R)^{-1/2}$, i.e.\
the multiset $\{(1+\sigma_i(R)^2)^{-1/2}:1\le i\le r\}$.
In particular, $\cos^2_{\min}=1/(1+\|R\|_{\mathrm{op}}^2)$; on the symmetric cycle
all $\sigma_i=\eta/2$. If
$\mathrm{range}\,U_\star=\hat D$, then $R=U_\star U_\star^\top R$ and
$W=U_\star(I_r+U_\star^\top R)$ with $\|U_\star^\top R\|_{\mathrm{op}}=\|R\|_{\mathrm{op}}<1$, so
$\mathrm{range}\,W=\hat D$ and all cosines equal $1$.
\end{proof}

\paragraph{Exact cycles and changing alignment.}
The exact cycles in part~(ii) repeat the full weight matrix, hence any fixed
function of it, every two steps. In particular, teacher alignment cannot
continue improving along either parity subsequence after exact cycle entry.
Part~(iii) instead describes perturbations off the cycle family, where a
nonzero complementary residual can decay while alignment changes. Its local
linearisation does not prove convergence from arbitrary initialisation.

\paragraph{The minimum is not the plateau.}
$\hat{\mathcal L}\ge0$ with equality at $W=U_\star$, and gradient descent with
step $0<\ell<1/\mu_1$ contracts $R$ by $\max_k|1-2\ell\mu_k|<1$ per step. The
floor of part~(ii) is therefore a property of the polar update: every
full-column-rank matrix-Muon $2$-cycle in $\hat D$ has two-step mean loss at least $\frac{\eta^2}{4}\sum_{i\le r}\mu_i$.
For completeness, every such full-rank $2$-cycle has the form in
part~(ii): returning after two steps forces its polar factors to be
$O$ and $-O$. Writing
$\hat\Sigma R_T=OP$ and $\hat\Sigma R_{T+1}=-OP'$ with $P,P'\succ0$,
the update and $\hat\Sigma O=OK$ give $P+P'=\eta K$.
Thus $0\prec P\prec\eta K$. Rank-deficient $2$-cycles use fewer than
$r$ directions and can sit lower.

\subsection{Loss envelopes: where the feature-learning signal sits in the loss}\label{app:lin-envelope}

Write $H_t:=\sum_{k\ge2}\mu_kz_{t,k}^2$ for the loss held by the draining modes.
The lower phase means the parity class with the smaller leading-coordinate
loss in the limiting scalar cycle; the exact identity below holds for any two times.
Lemma~\ref{lem:lin-reduction} gives, for any two times $t,T$ on the lower phase
of the cycle, the exact identity
\begin{equation}
\hat{\mathcal L}(w_t)-\hat{\mathcal L}(w_T)=(H_t-H_T)+\mu_1\bigl(z_{t,1}^2-z_{T,1}^2\bigr).
\label{eq:lin-envelope}
\end{equation}
For the frozen-gradient symmetric cycle of
Proposition~\ref{prop:lin-family}(c), one has exactly
$\mu_1a_g^2=\mu_1\eta^2/4-g^2/\mu_1$.
That identity does not hold instantaneously for a trajectory with
changing off-mode content: the sum rule constrains two adjacent
amplitudes, not either amplitude separately.

Under the hypotheses of Lemma~\ref{thm:single}, its phase-convergence
estimate instead gives, for $t,T\ge T_1$ of the same parity,
\[
\left|\mu_1\bigl(z_{t,1}^2-z_{T,1}^2\bigr)\right|
\le C\,\frac{g_t^2+g_T^2}{\mu_1},
\qquad
g_t^2:=\sum_{k\ge2}\mu_k^2z_{t,k}^2,
\]
where $C$ may depend on the fixed spectral contraction margin.
For a single off-mode $k$, $g_t^2/\mu_1=(\mu_k/\mu_1)H_t$.
Thus slow modes can make the offset small compared with the hidden-loss
level, but this estimate does not control the ratio of changes between
arbitrarily close times.

For the auxiliary run initialised with $z_0=1.5\eta e_1+3e_8$,
the measured ratio of the lower-envelope
loss change to the hidden-loss change is $1.000$.
For the generic asymmetric showcase it is $0.797$--$0.798$ from
$t=300$ to $6000$: in that run the changing cycle offset
cancels about $20\%$ of the hidden-loss change. These are
configuration-specific observations. The apparent plateau can therefore
retain a small, measurable loss signal while alignment improves.

\subsection{Numerical consistency checks}\label{app:lin-verify}

\paragraph{Reading Figure~\ref{fig:lin-main}.}
\textbf{Learning toys: nearly flat cycle-averaged loss while teacher alignment improves.}
\textbf{(a,b)} M2: NGD, $p=8$, $n=6000$, $\eta=0.3$, low-variance teacher.
\textbf{(c,d)} M3: matrix Muon, $p=30$, $r=5$, $n=5000$, $\lambda_D=1$,
$\lambda_\perp=10^{-3}$, $\eta=0.3$, $\mathrm{range}\,U_\star\subset D^\perp$.
Loss panels: adjacent-step bands and means, symmetric-cycle floors dashed; alignment panels:
orthogonal-teacher reference $(1+\eta^2/4)^{-1}=0.978$ dashed. \emph{Blue}: balanced leading
components with nonzero off-mode residuals; the final means match the floors to relative
$10^{-4}$ while alignment rises $0.13\to0.97$ (M2) and $0.49\to0.978$ (M3). \emph{Orange}: generic
starts; final means $1.06$ and $1.35$ times the floors. These finite-sample runs are numerical
illustrations, not verified instances of all theorem hypotheses (details in
Appendix~\ref{app:lin-verify}; Figure~\ref{fig:lin-optimizer} compares optimisers).

We checked the formulas proved above against finite runs: the coordinate
reduction and loss identities of Lemma~\ref{lem:lin-reduction}, the scalar
cycle relations and drain rates of Proposition~\ref{prop:lin-family}, the
alignment bounds of Theorem~\ref{thm:M2prime} and
Corollary~\ref{cor:invisible}, the envelope identity of
Section~\ref{app:lin-envelope}, and the matrix cycle relations of
Proposition~\ref{prop:M3prime}. The coordinate/loss identities, constructed matrix cycles and exact
alignment formulas have numerical residuals below $10^{-12}$.
Approximate comparisons use check-specific tolerances, including $2\%$
for the slow-mode log-drain rates and $5\%$ for the sample-to-population
floor comparison; all $50$ checks meet their numerical tolerances.
These tolerances are not interval-arithmetic certificates. These checks illustrate the proved statements and their
approximations rather than supply the proofs, and they describe particular
trajectories: they do not establish convergence from arbitrary
initialisation. The configurations are as follows.

\paragraph{Configurations.}
Here $\Sigma_x$ denotes the population input covariance. In M3, $D$ is its
planted $r$-dimensional high-variance block, with covariance eigenvalues
$\lambda_D$ on $D$ and $\lambda_\perp$ on $D^\perp$. Thus $D$ and the empirical
subspace $\hat D$ need not coincide. Each $e_k$ is a standard coordinate vector.
For M2, $p=8$, $n=6000$, $\eta=0.3$, and
$\Sigma_x=\diag(1,0.6,0.3,0.15,0.07,0.03,0.01,10^{-4})$, with independent
Gaussian samples. Case (A) uses $u=e_1$ and $w_0=-e_1+3e_8$; case (B) uses
$u=e_7$ and $w_0=2e_1+3e_8$. Both are run for $20000$ steps. The symmetric
initialisation check instead uses $z_0=1.5\eta e_1+3e_8$ and $6001$ steps.
In these finite samples, alignment of the teacher with a population
eigendirection need not imply exact alignment with an empirical eigenvector.

For M3, $p=30$, $r=5$, $n=5000$, $\eta=0.3$, and
$\Sigma_x=P_D+10^{-3}P_{D^\perp}$. The constructed-cycle checks use an
orthonormal basis $O$ of $\hat D$, $K=O^\top\hat\Sigma O$, and
\[
P=\frac\eta2K^{1/2}(I+\tfrac12T)K^{1/2},\qquad
R_0=OK^{-1}P,\qquad T=T^\top,\quad\|T\|_{\mathrm{op}}=1.
\]
Thus $\eta K/4\preceq P\preceq3\eta K/4$, so the exact cycle condition is
satisfied. The cycle is followed for $2000$ steps. Complement perturbations
are of the form $\varepsilon F\Xi$, with $F$ an orthonormal basis of
$\hat D^\perp$ and $\Xi$ a Gaussian matrix: $\varepsilon=10^{-6}$ for the
one- and two-step checks, and $10^{-3}$ for the $1200$-step comparison.
The M3 trajectory checks use a teacher in $D^\perp$, a Gaussian initial
matrix projected onto $D$, and $4000$ steps, matching the orange run in
Figure~\ref{fig:lin-main}. For the blue M2 run, instead,
$z_0=(\eta/2)e_1-e_7+3e_8$ in empirical coordinates.
The blue M3 run uses
$W_0=U_\star+(\eta/2)U_D+U_{\rm junk}$, where $U_D$ is an orthonormal
basis of the population block $D$ and $U_{\rm junk}$ is an orthonormal
$r$-frame perpendicular to both $D$ and the teacher. Thus this
initialisation has a nonzero complement residual; it is not an exact cycle.
The plotted blue runs do not establish the conditional NGD margin or
nonlinear convergence of the matrix complement. The M1 runs use the configurations of
Figure~\ref{fig:lin-M1p}. The frozen-gradient scalar check uses
$\mu_1=1$, $\eta=0.3$, $g=0.05$, and $4000$ iterations; its two-step
multiplier is estimated by a centred difference with increment $10^{-7}$.

For Figure~\ref{fig:lin-main}, the settling time is the first reported
step after the last adjacent-step loss mean that differs by more than
$5\%$ from the average of the final $200$ adjacent-step means.
In Figure~\ref{fig:lin-main}(a,b) the hidden residual contributes about $4\%$ of the floor
(symmetric-cycle e-fold $\approx5152$ steps). The final adjacent-step bands of the blue runs are
below $1\%$, and the blue M2 alignment is quoted from step $90$. From the settled steps $89$ (M2)
and $17$ (M3) of the orange runs, alignment rises $0.07\to0.96$--$0.99$ and
$0.04\to0.93$--$0.96$.
The population teacher directions need not be orthogonal to the
empirical cycle subspace, so the dashed alignment level is a reference.

\subsection{Additional figures}\label{app:lin-figs}

The additional figures use the following notation: $\lambda_k$
for the eigenvalues of $\hat\Sigma$ (our $\mu_k$; ``EoS floor
$2\lambda_1\eta^2/8$'' is our $\mu_1\eta^2/4$), $y_{t,k}$ and $\tilde w$ for the
residual coordinates and residual (our $z_{t,k}$, $w-u$), $v_1$ for $q_1$, $V$
for the planted data block (our $D$), $U_t$ for the teacher (our $U_\star$),
$M=r$, and ``$V$-block floor'' for the population value $\lambda_D\eta^2r/4$.

\begin{figure}[htbp]
\centering
\includegraphics[width=\linewidth]{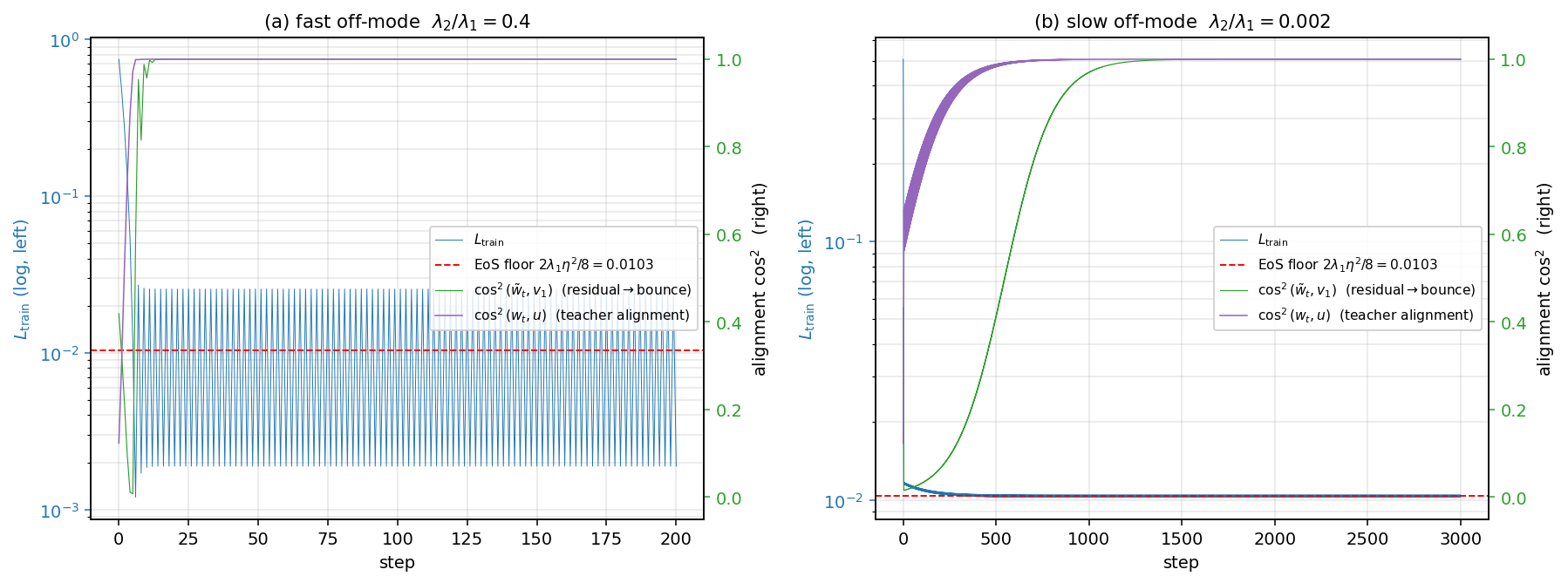}
\caption{\textbf{M1: two-dimensional NGD on linear regression.} $n=3000$,
$\Sigma_x=\diag(1,\lambda_2)$ with $\lambda_2=0.4$ (a) or $0.002$ (b), teacher
$u=e_1$, $\eta=0.2$, $w_0=(0.3,0.8)$. Loss (log, left axis) and
alignments (right axis). (a) The fast off-mode leaves the
cycle on the asymmetric member $\{0.157,-0.043\}$, so the plateau is
$1.33\times$ the floor, as Proposition~\ref{prop:lin-family}(a) predicts
($2(a^2+(\eta-a)^2)/\eta^2=1.328$). (b) The slow, loss-visible off-mode
symmetrises the cycle onto $\pm\eta/2$ (plateau/floor $=1.000$): the loss
settles at step $120$ and stays on the floor for $2900$ further steps while the
residual alignment climbs $0.04\to1.00$ and the teacher alignment
$0.82\to1.00$. Here the misalignment sits in a loss-visible mode; the
invisible regime is shown with a dedicated hidden mode in M2.}
\label{fig:lin-M1p}
\end{figure}
\begin{figure}[p]
\centering
\includegraphics[width=0.9\linewidth]{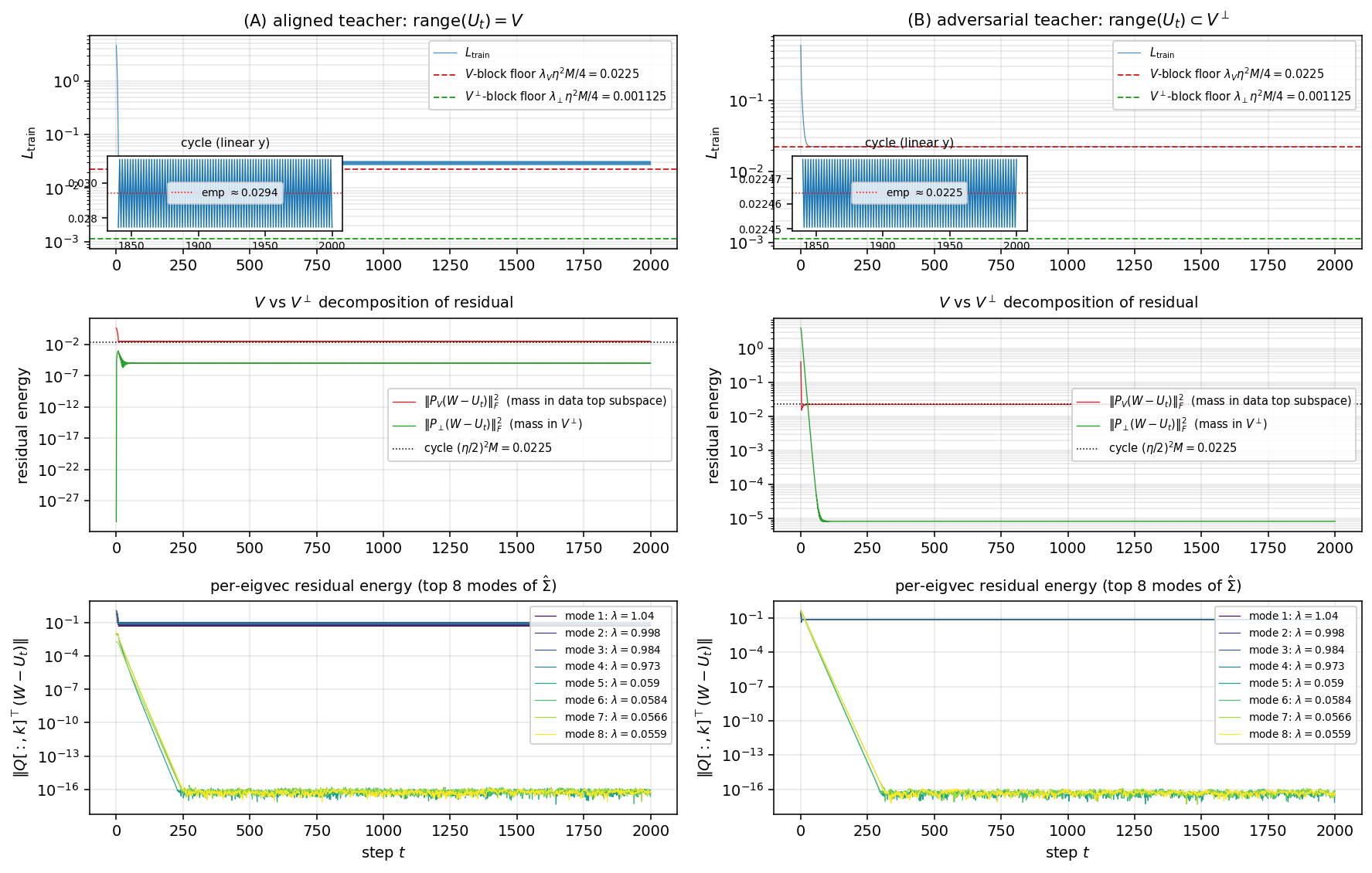}
\caption{\textbf{M3: matrix Muon on multi-output regression.} $p=30$,
$r=4$, $n=4000$, $\eta=0.15$, $\lambda_D=1$, $\lambda_\perp=0.05$, both runs
initialised in $D$; teacher (A) $=D$, (B) $\subset D^\perp$. \textbf{Top:} loss with the
population floors $\lambda_D\eta^2r/4=0.0225$ and $\lambda_\perp\eta^2r/4$
(insets: the cycle; (A) sits on an asymmetric member, $1.31\times$, (B) nearly symmetric).
\textbf{Middle:} residual energy in the population blocks $D$, $D^\perp$ (the
$D^\perp$ part levels off near $10^{-5}$ only because $\hat D$ is slightly
rotated from $D$). \textbf{Bottom:} residual energy in the eigenbasis of
$\hat\Sigma$: the four displayed complement modes
($\mu\approx0.056$--$0.059$) decay toward numerical precision; the factor
$1-2\mu_j/\mu_i\approx0.89$ is the symmetric-cycle linearised
reference from Proposition~\ref{prop:M3prime}(iii), and asymmetric cycles
have the two-phase multipliers stated there. The late teacher alignment is approximately $1$ (A) and
$0.9944$ (B), close to the ideal symmetric-cycle reference
$(1+\eta^2/4)^{-1}$ (the exact formula in
Proposition~\ref{prop:M3prime}(iv) requires empirical-subspace
orthogonality), and the step length $\eta\sqrt r$ holds to four
digits (Proposition~\ref{prop:M3prime}(i)).}
\label{fig:lin-M3p-basic}
\vspace{8pt}
\includegraphics[width=0.9\linewidth]{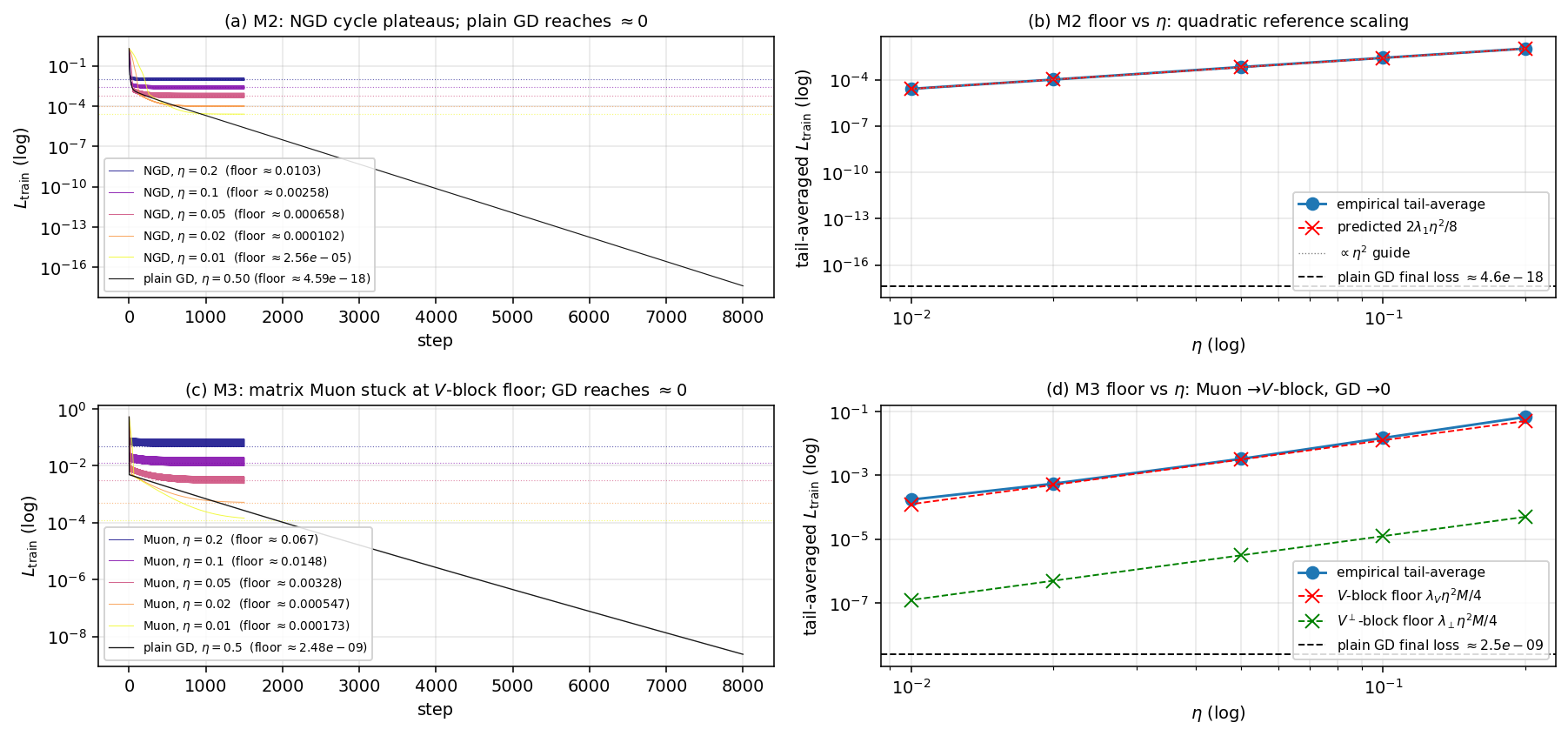}
\caption{\textbf{The plateau belongs to the optimiser.} (a,b) M2 ($p=9$,
$n=2000$, $\Sigma_x=\diag(1,0.6,0.3,0.15,0.07,0.03,0.012,0.005,0.002)$, teacher
$e_1$): NGD at $\eta\in\{0.2,0.1,0.05,0.02,0.01\}$ stops at floors within
$0$--$4\%$ of $\mu_1\eta^2/4$; plain GD (step $0.5$) reaches
$4.6\times10^{-18}$. (c,d) M3 ($p=30$, $r=5$, $n=2000$, $\lambda_D=1$,
$\lambda_\perp=10^{-3}$, adversarial teacher): matrix Muon stops at
$1.05$--$1.39\times\lambda_D\eta^2r/4$ (a population reference;
both sampling and the asymmetric-cycle excess of
Proposition~\ref{prop:M3prime}(ii) affect the ratio), far above the $D^\perp$-scale value
(green); plain GD reaches $2.5\times10^{-9}$.}
\label{fig:lin-optimizer}
\end{figure}
\clearpage

\section{Proofs and supporting derivations for
Section~\ref{sec:1NN}}\label{app:1NN-proofs}

This appendix collects the proofs of
Proposition~\ref{prop:polar-identities},
Theorem~\ref{thm:p2existence}, and
Proposition~\ref{prop:matrix-muon-noV} (Sections~\ref{app:prop1-proof},
\ref{app:thm1-proof}, \ref{app:propnoV-proof}), together with the full assumptions
and proof of the conditional drift result, Proposition~\ref{claim:drift},
stated in Section~\ref{sec:1NN-claims}
(Section~\ref{app:drift-floquet} below).

\paragraph{Notation for these proofs.}
We use $W\in\R^{p\times r_s}$, input dimension $p$, student width $r_s$,
teacher subspace $V$ of dimension $r_t$, and $M=\min(p,r_s)$.
The matrices $P_V$ and $P_\perp=I_p-P_V$ are the orthogonal projectors onto
$V$ and $V^\perp$; $I_k$ is the $k\times k$ identity.
For vectors, $\|\cdot\|=\|\cdot\|_2$; for matrices, $\|\cdot\|_F$ and
$\|\cdot\|_{\rm op}$ are the Frobenius and operator norms.
The map $\Phi(W)=W-\eta\Polar(\nabla\mathcal L(W))$ uses step coefficient $\eta>0$.
Here $\mathcal L$ is the loss specified in each statement, and
$r^{\mathrm{grad}}=\rank(\nabla\mathcal L(W))$.
We write $A_{\mathrm{sub}}(W)=\|P_VW\|_F^2/\|W\|_F^2$ for $W\ne0$.
For an empirical student AGOP $G_s(W)$, $\cos^2_{\min}(W)$ denotes the smallest
squared cosine of its leading $r_t$-dimensional eigenspace with $V$
(\eqref{eq:student-agop}--\eqref{eq:cos-metrics}).
Singular values are ordered from largest to smallest.
The symbols $\Tr$ and $\rank$ denote trace and matrix rank;
$O(k)=\{Q\in\R^{k\times k}:Q^\top Q=I_k\}$, and
$\mathrm{St}(p,k)=\{Q\in\R^{p\times k}:Q^\top Q=I_k\}$.

In the first three subsections, $O(\cdot)$ denotes the standard asymptotic notation:
$f(\eta) = O(g(\eta))$ as $\eta \to 0$ means
$\limsup_{\eta\to 0}|f(\eta)/g(\eta)| < \infty$. Constants in
$O(\cdot)$ may depend on $p$, $r_t$, $r_s$, the spectrum of the
Hessian along the trajectory, and the teacher activation; they
never depend on $\eta$ itself. Section~\ref{app:drift-floquet} instead uses
bounds uniform across its stated family and trajectory window, with the
learning-rate rule included among the fixed family parameters.

\subsection{Proof of Proposition~\ref{prop:polar-identities}
(polar-update identities)}\label{app:prop1-proof}

We prove parts (i), (ii), (iii) of the proposition, using only the
orthonormality of the SVD factors and the definition of the polar
factor $\Polar(g) := UV_g^\top$ for a compact SVD $g = U\Sigma
V_g^\top$. The leakage scalars used in part~(iii) are recorded
formally for reference:

\begin{definition}[leakage scalars]\label{def:leakage}
Let $g \in \R^{p \times r_s}$ have $\rank(g) = r\ge r_t$ and compact SVD
$g = U\Sigma V_g^\top$ with $U \in \R^{p\times r}$ orthonormal columns
$u_1, \ldots, u_r$. Given the teacher subspace
$V \subseteq \R^p$ from~\eqref{eq:teacher},
\begin{equation}
\varepsilon_V^2 \;:=\; \sum_{i=1}^{r_t} \|P_\perp u_i\|^2,
\qquad
\varepsilon_\perp^2 \;:=\; \sum_{i=r_t+1}^{r} \|P_V u_i\|^2.
\label{eq:leakage}
\end{equation}
$\varepsilon_V^2$ measures how much of the squared norm of the first
$r_t$ left singular vectors of $g$ falls outside $V$;
$\varepsilon_\perp^2$ measures how much of the squared norm of the
remaining left singular vectors falls inside $V$. Both vanish iff the
left singular vectors of $g$ split exactly into a $V$-block and a
$V^\perp$-block.
\end{definition}

\paragraph{Part (i): $\|\Polar(g)\|_F^2 = r$.}
Since $U \in \R^{p\times r}$ has orthonormal columns
($U^\top U = I_r$),
\begin{equation}
\|UV_g^\top\|_F^2 \;=\; \Tr\bigl(V_g\,U^\top U\,V_g^\top\bigr)
\;=\; \Tr(V_g V_g^\top) \;=\; \Tr(I_r) \;=\; r.
\end{equation}

\paragraph{Part (ii): $\|\Phi(W) - W\|_F = \eta\sqrt{r^{\mathrm{grad}}}$.}
By definition, $\Phi(W) - W = -\eta\,\Polar(\nabla\mathcal{L}(W))$,
so $\|\Phi(W) - W\|_F = \eta\,\|\Polar(\nabla\mathcal{L}(W))\|_F =
\eta\sqrt{r^{\mathrm{grad}}}$ by part (i). Generically along a
trajectory $r^{\mathrm{grad}} = M$, giving the stated $\eta\sqrt{M}$.

\paragraph{Part (iii): leakage decomposition.}
Here $r=M\ge r_t$, as in the full-rank statement of the proposition.
Writing $\Polar(g) = UV_g^\top$ and using
$V_g^\top V_g = I_r$,
\begin{equation}
\|P_V \Polar(g)\|_F^2 \;=\;
\Tr\bigl(V_g (P_V U)^\top (P_V U) V_g^\top\bigr)
\;=\; \Tr\bigl((P_V U)^\top (P_V U)\bigr)
\;=\; \sum_{i=1}^{r} \|P_V u_i\|^2,
\end{equation}
where $u_i$ is the $i$-th column of $U$. The $V/V^\perp$ split
$\|P_V u_i\|^2 = 1 - \|P_\perp u_i\|^2$ then gives
\begin{align}
\|P_V\Polar(g)\|_F^2
&= \sum_{i=1}^{r_t}\bigl(1 - \|P_\perp u_i\|^2\bigr)
+ \sum_{i=r_t+1}^{r}\|P_V u_i\|^2 \nonumber\\
&= r_t - \varepsilon_V^2 + \varepsilon_\perp^2,
\end{align}
with $\varepsilon_V^2$, $\varepsilon_\perp^2$ as in
Definition~\ref{def:leakage}. The complementary statement
$\|P_\perp \Polar(g)\|_F^2 = (M - r_t) + \varepsilon_V^2 -
\varepsilon_\perp^2$ follows from
$\|\Polar(g)\|_F^2 = \|P_V\Polar(g)\|_F^2 +
\|P_\perp\Polar(g)\|_F^2 = M$ (part (i) at $r = M$). The identities
are exact under no further smallness assumption on the leakage.

\paragraph{Connection to data.}
The identities hold exactly for every gradient of full column rank; how close
$\|P_V\Polar(g)\|_F^2$ is to $r_t$ depends on the leakage scalars, which we did
not measure along the trajectories. The polar-mass check of
Appendix~\ref{app:tests} (Table~\ref{tab:polar-mass}) finds the $V$-mass of the
update approximately $2$--$5\%$ below $r_t$ on three ReLU checkpoints.
Figure~\ref{fig:theory-vs-exp}(b) reports a different quantity, the $V$-mass of
the top AGOP eigenvectors.

\subsection{Period-$2$ orbits on the isotropic quadratic}\label{app:thm1-proof}\label{sec:1NN-toys}

Proposition~\ref{prop:polar-identities} constrains the \emph{shape} of any period-$2$ orbit of the
polar map but not its existence. On the isotropic quadratic there is an explicit family.

\begin{theorem}[period-$2$ cycle on the isotropic quadratic]\label{thm:p2existence}
Let $\mathcal{L}(W) = \tfrac{\lambda}{2}\,\|W\|_F^2$ on
$\R^{p \times r_s}$ with $\lambda > 0$ and $r_s \le p$. The matrix
Muon iteration $\Phi$ admits a continuum of \emph{symmetric}
period-$2$ orbits $\{W^*, -W^*\}$ of Frobenius radius
$\|W^*\|_F = \tfrac{\eta}{2}\sqrt{M}$. Concretely,
$W^* = \tfrac{\eta}{2}\,U_0 V_0^\top$ for any orthonormal columns
$U_0 \in \R^{p\times r_s}$, $V_0 \in \R^{r_s \times r_s}$.
\end{theorem}

Existence on arbitrary $A\succ0$ holds on every $r_s$-dimensional $A$-invariant subspace
(Proposition~\ref{prop:M3prime}(ii)). On the 1-HL ReLU loss we observe only approximately period-$2$
loss oscillations, while the parameters continue to evolve (Appendix~\ref{app:tests}, Test~A).

\begin{proof}[Proof of Theorem~\ref{thm:p2existence}]
We use the symmetric ansatz $W_B = -W_A$ for the period-$2$
condition $\Phi(W_A) = W_B$:
\begin{equation}
W_A - \eta\,\Polar(\lambda W_A) \;=\; -W_A
\quad\Longleftrightarrow\quad
2 W_A \;=\; \eta\,\Polar(W_A),
\label{eq:p2-symmetric-ansatz}
\end{equation}
where we used $\Polar(c X) = \Polar(X)$ for any $c > 0$.

We verify that $W^* := (\eta/2)\,U_0 V_0^\top$, for any orthonormal
$U_0 \in \R^{p \times r_s}$ and $V_0 \in \R^{r_s \times r_s}$,
satisfies~\eqref{eq:p2-symmetric-ansatz}. The compact SVD of $W^*$
is $W^* = (\eta/2)\,U_0 \Sigma V_0^\top$ with $\Sigma = I_{r_s}$
(both $U_0, V_0$ are orthonormal), so all singular values equal
$\eta/2$ and $\Polar(W^*) = U_0 V_0^\top$. Therefore
$\eta\,\Polar(W^*) = \eta\,U_0 V_0^\top = 2 W^*$, which is the
right-hand side of~\eqref{eq:p2-symmetric-ansatz}.

The orbit is genuinely period-$2$ because a fixed point of $\Phi$
would require $\Polar(\lambda W) = 0$, contradicting
$\|\Polar(g)\|_F^2 = M > 0$ at full column rank
(Proposition~\ref{prop:polar-identities}(i)). The continuum of
solutions is parameterised by the product of Stiefel manifolds
$(U_0, V_0) \in \mathrm{St}(p, r_s) \times O(r_s)$, modulo the
right-action of $O(r_s)$.
\end{proof}

\begin{remark}[extension to general $A \succ 0$]\label{rem:p2-perturbation}
For $\mathcal{L}(W) = \tfrac12\Tr(W^\top A W)$ with $A \succ 0$
general, the symmetric ansatz becomes $2W = \eta\,\Polar(AW)$.
For arbitrary $A \succ 0$ the symmetric ansatz is solved
exactly by $W^* = \tfrac{\eta}{2}O$ whenever the columns of $O$ are an
orthonormal basis of an $r_s$-dimensional $A$-invariant subspace, since then
$AW^* = \tfrac{\eta}{2}O(O^\top AO)$ with $O^\top AO \succ 0$ and
$\Polar(AW^*) = O$ (Proposition~\ref{prop:M3prime}(ii) gives the whole
family of such orbits); stability for general $A$ and orbits outside invariant
subspaces remain open. In the network experiments only the loss has a
period-$2$ signature (Test~A, Appendix~\ref{app:tests}); the weights keep
moving, so they do not sit on an orbit of this kind.
\end{remark}

\subsection{Block-orthogonal equivariance on a block-isotropic
quadratic}\label{app:propnoV-proof}\label{app:matrix-muon-noV}

\begin{proposition}[block-orthogonal equivariance]\label{prop:matrix-muon-noV}
Let $\mathcal{L}(W) = \tfrac12\Tr(W^\top A W)$ with
$A = \diag(\lambda I_{r_t}, \mu I_{p - r_t})$ for some
$\lambda, \mu > 0$. The matrix Muon iteration map $\Phi$
defined in~\eqref{eq:Phi-def} commutes with the block-orthogonal
group $G = \mathrm{blk\text{-}diag}(Q, R) \in O(r_t) \times O(p - r_t)$
acting on $W$ by left multiplication: $\Phi(GW) = G\,\Phi(W)$ for
every $W$ at which $\Polar(AW)$ is defined. Consequently, if the
law of $W_0$ is invariant under $O(r_t) \times O(p - r_t)$, the
law of $W_t$ is invariant under the same group action for every
$t \ge 0$.
\end{proposition}

\begin{proof}[Proof of Proposition~\ref{prop:matrix-muon-noV}]
Let $G \in O(r_t) \times O(p - r_t) \subset O(p)$, written in block
form as $G = \mathrm{blk\text{-}diag}(Q, R)$ with $Q \in O(r_t)$,
$R \in O(p - r_t)$. Since $A = \diag(\lambda I_{r_t},
\mu I_{p - r_t})$ has the same block structure, $G^\top A G = A$
(equivalently $A G = G A$). For any $W \in \R^{p \times r_s}$,
\begin{equation}
\nabla \mathcal{L}(GW) \;=\; A\,GW \;=\; G\,AW \;=\;
G\,\nabla\mathcal{L}(W).
\end{equation}
The polar factor is left-equivariant under orthogonal
transformations: for any $G \in O(p)$ and full-column-rank
$X \in \R^{p \times r_s}$, the compact SVD of $GX$ is
$GX = (GU)\Sigma V^\top$, so $\Polar(GX) = (GU) V^\top =
G\,\Polar(X)$. Combining the two equivariances gives
\begin{equation}
\Phi(GW) \;=\; GW - \eta\,\Polar(A G W) \;=\; GW - \eta\,G\,\Polar(AW)
\;=\; G\,\Phi(W),
\end{equation}
which is the equivariance claim. Iterating, $\Phi^t(GW_0) =
G\,\Phi^t(W_0)$ for every $t$. If the law of $W_0$ is
$O(r_t)\times O(p-r_t)$-invariant, the law of $W_t = \Phi^t(W_0)$
is invariant under the same group action.
\end{proof}

\paragraph{Caveat.}
Block-orthogonal invariance does not pin down the relative energy
of $W_t$ between the $V$- and $V^\perp$-blocks: an
$O(r_t)\times O(p-r_t)$-invariant distribution can place arbitrary
relative mass in the two blocks. In particular,
$\E_{W_0}[A_{\mathrm{sub}}(W_t)] = r_t/p$, where the expectation is over
initialisation, does not follow from the
symmetry alone. Block-orthogonal equivariance neither establishes
nor excludes a preference for the teacher subspace under a fixed
Hessian. Identifying the cause of the observed preference in the
1-HL ReLU network requires an additional dynamical argument; the
proposition does not attribute it to changing activation patterns.

\phantomsection\label{fig:matrix-muon-noV}\label{obs:matrix-muon-noV-empirical}

\subsection{Assumptions and proof of conditional secular drift}
\label{app:drift-floquet}\label{app:conditional-secular-drift}
This subsection proves Proposition~\ref{prop:conditional-secular-drift}.
Write $M=\min(p,r_s)$, $q=r_t/M$, $A_t=\cos^2_{\min}(W_t)$,
$e_k=1-A_{2k}$ and $v_k=(A_{2k+2}-A_{2k})/2$.
Fix the remaining family parameters $\theta$, including the learning-rate rule.
Here $k$ indexes pairs of optimisation steps, $\asymp_\theta$ means upper and
lower bounds by positive constants depending only on $\theta$, and
$\lesssim_\theta$ has the corresponding one-sided meaning. All matrix norms
in this subsection are operator norms, including those in remainder bounds.
Unadorned constants $C,K$ in bounds are finite and uniform over the stated
family and window, and may change between estimates; $C(\theta)$ denotes the
fixed leading prefactor specified below.
The base assumptions yield the two-sided bound
\eqref{eq:drift-rate-main}. Under the additional stationary
slow-mode conditions specified below, the sharper conclusion is
\begin{equation}
 v_k=C(\theta)q^\alpha e_k(1+o(1)),
 \label{eq:drift-mode-main}
\end{equation}
with explicit uniform control
\begin{equation}
 \left|\frac{v_k}{C(\theta)q^\alpha e_k}-1\right|
 \le K_\theta\left(q^\nu+\sqrt{e_k}+e^{-\omega kq^\alpha}\right).
 \label{eq:drift-mode-error}
\end{equation}
The exponent and prefactor depend on the specified family, including
its learning-rate convention. These statements concern same-parity
alignment drift; an exact repeating weight orbit is not assumed.

\paragraph{Feature coordinates and residence.}
Assume that, on a trajectory window \(W_0,\ldots,W_{2N}\), the leading
rank-\(r_t\) student AGOP subspace is separated from the remaining
spectrum. Near the teacher subspace it has a graph representation
\[
 \mathcal S(W_{2k})=\operatorname{range}(U+U_\perp Z_k),
\]
All two-step conclusions below apply for \(0\le k<N\).
Here \(U\in\R^{p\times r_t}\) and \(U_\perp\in\R^{p\times(p-r_t)}\)
are orthonormal frames for the teacher subspace and its complement, so
\(Z_k\in\R^{(p-r_t)\times r_t}\). For the paper's worst-principal-angle metric,
\begin{equation}
 A_{2k}=\frac1{1+\|Z_k\|_{\rm op}^2},
 \qquad e_k=\frac{\|Z_k\|_{\rm op}^2}{1+\|Z_k\|_{\rm op}^2}.
 \label{eq:drift-mild-graph}
\end{equation}
Orthogonal changes of these frames preserve the formula. Residence in
the stated window and any near period-two loss condition are separate
assumptions; the alignment calculation does not establish them.

\paragraph{Finite-order transverse dynamics.}
Put \(z=q^\nu\), with a prescribed \(\nu>0\).
For an integer \(j\ge1\), suppose that \(\varepsilon=z^j\) and
\begin{equation}
 Z_{k+1}=Q_k\{Z_k-\varepsilon(L_kZ_k+Z_kR_k)+E_k\}S_k^\top ,
 \label{eq:drift-mild-nf}
\end{equation}
where \(Q_k,L_k\in\R^{(p-r_t)\times(p-r_t)}\),
\(S_k,R_k\in\R^{r_t\times r_t}\), and \(E_k\) has the same shape as \(Z_k\).
The matrices \(Q_k,S_k\) are orthogonal, \(L_k,R_k\) are symmetric positive
semidefinite, and, uniformly across the family and window,
\begin{align}
 \lambda_{\min}(L_k)+\lambda_{\min}(R_k)&\ge a>0,\\
 \|L_k\|_{\rm op}+\|R_k\|_{\rm op}&\le b<\infty,\\
 \|E_k\|_{\rm op}
 &\le K\varepsilon\bigl(z\|Z_k\|_{\rm op}+\|Z_k\|_{\rm op}^2\bigr).
 \label{eq:drift-mild-rem}
\end{align}
These are sufficient conditions for dissipation in the operator norm,
which is the norm relevant to the worst principal angle. Positive
definiteness of an arbitrary generator in Frobenius norm alone does
not imply them.

\begin{proof}[Proof of the two-sided law]
Write \(x=\|Z_k\|_{\rm op}\) and \(x_+=\|Z_{k+1}\|_{\rm op}\).
The linear term can be replaced by
\((I-\varepsilon L_k)Z_k(I-\varepsilon R_k)\), with an additional
error at most \(b^2\varepsilon^2x\).
This is absorbed in~\eqref{eq:drift-mild-rem}, since
\(\varepsilon=z^j\le z\) for \(z\le1\).
For positive definite \(P,T\),
\[
 \sigma_{\min}(P)\sigma_{\min}(T)\|Z\|_{\rm op}
 \le\|PZT\|_{\rm op}\le\|P\|_{\rm op}\|T\|_{\rm op}\|Z\|_{\rm op}.
\]
It follows that
\[
 [1-\varepsilon b-C\varepsilon(z+x)]x
 \le x_+\le
 [1-\varepsilon a+C\varepsilon(z+x)]x .
\]
The exact graph identity gives
\[
 \frac{v_k}{\varepsilon e_k}
 =\frac{x^2-x_+^2}{2\varepsilon x^2(1+x_+^2)}.
\]
Hence
\begin{equation}
 a-C(z+\sqrt{e_k})
 \le \frac{v_k}{q^{\nu j}e_k}
 \le b+C(z+\sqrt{e_k}).
 \label{eq:drift-mild-bounds}
\end{equation}
For sufficiently small \(q,e_k\), this proves
\eqref{eq:drift-rate-main} with \(\alpha=\nu j\).
The upper bound on \(x_+\) also keeps the graph error small.
\end{proof}

\paragraph{Obtaining the exponent without analyticity.}
The expansion above can be obtained from a finite Taylor expansion of
the entire two-step graph map in moving coordinates.
Existence of such a graph map is an additional closure hypothesis:
the AGOP subspace alone need not determine its next iterate.
Other state coordinates may be included as parameters if the stated
estimates hold uniformly in them. Derivative bounds below use norms
induced by the operator norm of graph matrices, not dimension-dependent
conversions from the Frobenius norm.
When dimensions vary, this means either a common reduced coordinate
space or a smooth parameter extension on each fixed-dimensional graph
space, evaluated at its actual \(z=q^\nu\); all stated derivative and
remainder bounds must then be uniform across those extensions.
It is not a derivative of matrix dimensions.
Specifically, suppose \(F_z(0)=0\), \(F_0(Z)=Z\), and
\(\partial_z^\ell F_z(Z)|_{z=0}=0\) for \(1\le\ell<j\), throughout
the local graph neighborhood. If
\[
 \frac1{j!}\partial_z^jF_z(Z)\big|_{z=0}
 =-LZ-ZR+O(\|Z\|_{\rm op}^2)
\]
and the mixed derivative \(D_Z\partial_z^{j+1}F\) is uniformly bounded,
Taylor's theorem gives~\eqref{eq:drift-mild-nf}--\eqref{eq:drift-mild-rem}.
Only finite smoothness and a nonzero dissipative term are used.
The integer \(j\) is the first active damping order; the rank-ratio
exponent is \(\alpha=\nu j\).

Checking these derivatives only for the linearization at \(Z=0\) is
weaker. Ordinary uniform \(C^2\) control then gives an unscaled
\(O(\|Z\|_{\rm op}^2)\) nonlinear remainder. Writing \(e=e_k\) and \(v=v_k\),
in that case the same proof gives
\[
 a-C\left(z+\frac{\sqrt e}{q^\alpha}\right)
 \le\frac{v}{q^\alpha e}\le
 b+C\left(z+\frac{\sqrt e}{q^\alpha}\right).
\]
Thus a small fixed relative-error bound holds only in a shrinking
neighborhood \(e\lesssim q^{2\alpha}\).
A uniform small-error neighborhood requires the scaled nonlinear
control above or another condition that supplies it.

\paragraph{One sufficient set of slow-mode conditions.}
For the sharper result, write $Y_k=O_k^\top Z_k T_k$ for orthogonal
changes of the left and right frames, so $\|Y_k\|_{\rm op}=\|Z_k\|_{\rm op}$.
Suppose that in these coordinates the exact transverse dynamics are
stationary and have the form
\begin{equation}
 Y_{k+1}=(I-\varepsilon L(q))Y_k(I-\varepsilon R(q))+N_q(Y_k),
 \qquad
 \|N_q(Y)\|_{\rm op}\le K\varepsilon\|Y\|_{\rm op}^2 .
 \label{eq:drift-mild-stationary}
\end{equation}
Assume the preceding bounds on \(L,R\). Write their smallest eigenvalues
as \(\lambda_0(q),\rho_0(q)\), and their corresponding orthogonal
projectors as \(P_L(q),P_R(q)\).
Assume that these slow eigenspaces are separated from the other
eigenvalues by a uniform gap \(\gamma>0\), and
\[
 \lambda_0(q)+\rho_0(q)=C(\theta)+O(q^\nu),\qquad C(\theta)>0.
\]
Finally, for a fixed \(\delta>0\), assume
\[
 \|P_LY_0P_R\|_{\rm op}\ge\delta\|Y_0\|_{\rm op},
\]
with \(\|Y_0\|_{\rm op}\) sufficiently small depending on \(\delta\).
The quantitative projection condition prevents nonlinear forcing
from canceling the slow component.
These conditions are additional to the finite-order expansion.

\begin{proof}[Proof of the shared-prefactor law]
Set \(x_0=\|Y_0\|_{\rm op}\),
\(\mathcal P Y=P_LYP_R\), and
\[
 \mu_0=(1-\varepsilon\lambda_0)(1-\varepsilon\rho_0)
       =1-C(\theta)\varepsilon+O(\varepsilon q^\nu).
\]
The linear map equals \(\mu_0 I\) on the slow subspace.
Its complement splits into the three blocks
\((I-P_L)YP_R\), \(P_LY(I-P_R)\), and
\((I-P_L)Y(I-P_R)\).
The \(n\)-step linear action on the complement has norm at most
\(3\mu_0^n e^{-c\gamma\varepsilon n}\); this estimate is independent
of matrix dimensions.

For sufficiently small \(x_0\), global contraction gives
\(\|Y_k\|_{\rm op}\le x_0r^k\), with
\(r=\mu_0+K\varepsilon x_0<1\) and
\(r^2/\mu_0\le1-c\varepsilon\).
For every \(k\) in the window, write
\[
 D_k=\mathcal P Y_0+
 \sum_{n=0}^{k-1}\mu_0^{-n-1}\mathcal P N_q(Y_n),
 \qquad \mathcal P Y_k=\mu_0^kD_k.
\]
Each nonlinear term carries \(\varepsilon\), which cancels the
\(O(\varepsilon^{-1})\) geometric sum, so that
\(\|D_k-\mathcal P Y_0\|_{\rm op}\le Cx_0^2\), uniformly in \(k\).
Choosing \(x_0\) small relative to \(\delta\) gives
\(\|D_k\|_{\rm op}\ge\delta x_0/2\).
The same sums, projected onto the complementary blocks, give
\[
 \|(I-\mathcal P)Y_k\|_{\rm op}
 \le Cx_0\mu_0^k e^{-\omega\varepsilon k},
 \qquad
 \|Y_k\|_{\rm op}\ge c_\delta x_0\mu_0^k .
\]
These estimates use only iterates within the stated window.

Let \(\mathcal T_qY=(I-\varepsilon L)Y(I-\varepsilon R)\).
The operator \(\mathcal T_q-\mu_0I\) vanishes on the slow subspace
and has operator-norm-to-operator-norm bound \(O(\varepsilon)\).
Consequently
\[
 \left|\|Y_{k+1}\|_{\rm op}-\mu_0\|Y_k\|_{\rm op}\right|
 \le C\varepsilon\|Y_k\|_{\rm op}
 \left(e^{-\omega\varepsilon k}+\|Y_k\|_{\rm op}\right).
\]
This extra factor of \(\varepsilon\) is essential when converting
mode selection to a normalized drift rate.
Substitution in the exact graph identity, together with
\[
 \frac{1-\mu_0^2}{2}
 =C(\theta)\varepsilon+O(\varepsilon q^\nu),
\]
proves~\eqref{eq:drift-mode-main}.
\end{proof}

\paragraph{Interpretation and limits.}
A fixed-factor error reduction takes
\(\Theta_\theta(q^{-\alpha})\) optimization steps, provided the plateau
window lasts long enough. The shared-prefactor relative error vanishes
when \(q\to0\), \(e_k\to0\), and \(kq^\alpha\to\infty\).
The slow-mode hypotheses need not follow from a leading Taylor expansion:
they must hold for the actual linearization.
Likewise, arbitrary perturbations of a multiple slow eigenvalue need
not preserve its gap.

Finite smoothness alone is insufficient: a positive damping factor
\(\exp(-1/q^2)\), extended by zero at \(q=0\), is smooth but has no
finite nonzero power-law term. A finite-order nondegeneracy condition
is therefore necessary for this route.
If contraction remains nonzero as \(q\to0\), an exponent-zero variant
uses the exact discrete coefficient \((1-\mu(0)^2)/2\).
None of these statements identifies the exponent for the full empirical
Muon family without checking its transverse dynamics.

\phantomsection\label{app:open-claims}\label{app:status-summary}

\section{Additional experiments and theory-vs-experiment alignment}\label{app:additional}

This appendix collects the multi-activation extension of the experiments
in Section~\ref{sec:1NN-phenomenon} (the GELU and SiLU teacher grids) and the
per-config quantitative detail for each of the seven theory tests on the
long-run ReLU checkpoints. The tables in Section~\ref{app:tests} are
computed from the saved long-run ReLU checkpoints. Figures~\ref{fig:silu-top6} and~\ref{fig:gelu-top6} show the smooth-teacher grids, and
Table~\ref{tab:all33} summarises all $33$ ReLU, GELU and SiLU configurations over five seeds.

\paragraph{Notation and diagnostics.}
Here $p,r_t,r_s,n,\eta$ denote input dimension, teacher rank, student width,
training-set size and the polar-update coefficient; $M=\min(p,r_s)$.
The symbol $\sigma_t$ in table headers names the teacher activation (the
subscript means ``teacher'', not optimisation time).
Write $V=\linspan(U)$, $P_V=UU^\top$, $P_\perp=I_p-P_V$, and
$L_t=L_{\mathrm{train}}(W_t)$. For the empirical input AGOP
$G_s(W)=n^{-1}\sum_{i=1}^n\nabla_xf_W(x_i)\nabla_xf_W(x_i)^\top$,
let $\lambda_1\ge\cdots\ge\lambda_p\ge0$ and $v_1,\ldots,v_p$ be its
ordered eigenvalues and orthonormal eigenvectors.
The two alignment summaries are the mean and minimum squared cosines of the
principal angles between $\linspan(v_1,\ldots,v_{r_t})$ and $V$
(\eqref{eq:cos-metrics}). We distinguish the direction-only mass
$\sum_{j\le r_t}\|P_Vv_j\|_2^2$ from the eigenvalue-weighted fraction
\[
\mathrm{mass}_V(G_s):=\frac{\Tr(P_VG_s)}{\Tr(G_s)}
=\frac{\sum_j\lambda_j\|P_Vv_j\|_2^2}{\sum_j\lambda_j},
\qquad
A_{\mathrm{sub}}(W):=\frac{\|P_VW\|_F^2}{\|W\|_F^2}.
\]
The first fraction requires $\Tr(G_s)>0$ and the second $W\ne0$.
The head-fit losses $L_{V,\mathrm{opt}}$ and $L_{\mathrm{AGOP\text{-}opt}}$
refit the head after projecting $W$ onto $V$ or the student's selected AGOP
subspace, respectively, as in~\eqref{eq:LAGOPopt}.
Empirical means, standard deviations and correlations below are taken over
the stated training-step window; aggregation over seeds is specified separately.

\subsection{Plateau-search grid: ReLU, GELU, and SiLU teachers}\label{app:grid}

We sweep $(p, r_t, r_s, \eta)$ at fixed seed $5$, full batch, $n
= 15{,}000$, and $8{,}000$--$10{,}000$ training steps. The ReLU sweep
covers $p \in \{100, 150, 200\}$, $r_t \in \{4, 6, 8, 10, 12\}$, $r_s
\in \{30, 50, 80, 100\}$, $\eta \in \{1.0, 1.25, 1.5, 2.0, 2.5\}$;
the smooth-teacher sweep adds large-$r_t$ amplification configs ($r_t
\in \{12, 16, 20\}$) at the activation's EoS learning rate ($\eta
\approx 0.7$ for GELU, $\eta \approx 0.85$ for SiLU). The cleanest
\emph{rising-loss-with-learning} regimes come from the smooth-teacher
amplification configs: SiLU $r_t = 20$, $\eta = 0.85$ exhibits a loss
that climbs from $1.92$ to $2.06$ while $\cos^2_{\min}$ grows by
$+0.43$ over the second half of training; GELU $r_t = 16$, $\eta =
0.85$ shows the analogous picture on a smaller amplitude. The
cleanest-by-score configs from each sweep populate
Figures~\ref{fig:relu-top6}, \ref{fig:silu-top6},
and~\ref{fig:gelu-top6}.
Table~\ref{tab:all33} lists all $33$ configurations with their outcomes.

\begin{table}[h]
\centering
\caption{\textbf{All $33$ one-hidden-layer teacher configurations} ($p=100$, $n=15{,}000$, $8{,}000$ steps, $5$ seeds each: $5,11,23,37,41$; medians over seeds).
$1-\rho_2$: second-half period-$2$ signature (values below $10^{-10}$ shown as $10^{-10}$); $L/L_{V,\mathrm{opt}}$: cycle-mean training loss over the last $30\%$ of steps divided by the head oracle;
$\Delta\cos^2_{\min}$: change from step $4{,}000$ to the end; final $\cos^2_{\min}$ at the last step.
Outcome, from the medians: \emph{not recovered} if final $\cos^2_{\min}<0.2$; otherwise \emph{aligned by midpoint} if
$\Delta\cos^2_{\min}<0.005$ and final $\cos^2_{\min}\ge0.95$; \emph{no late gain} if $\Delta\cos^2_{\min}<0.005$ otherwise;
\emph{late gain} if $\Delta\cos^2_{\min}\ge0.005$. This five-seed grid is distinct from the
single-seed search above and from the $53$-seed runs of Figures~\ref{fig:1hl-main} and~\ref{fig:theory-vs-exp}.}
\label{tab:all33}
\scriptsize
\setlength{\tabcolsep}{3.5pt}
\begin{tabular}{lrrrcrrrl}
\toprule
Teacher & $r_t$ & $r_s$ & $\eta$ & $1-\rho_2$ & $L/L_{V,\mathrm{opt}}$ & $\Delta\cos^2_{\min}$ & final $\cos^2_{\min}$ & Outcome\\
\midrule
ReLU & 4 & 50 & 2 & $4.6\!\times\!10^{-7}$ & 3045 & $-0.028$ & 0.670 & no late gain \\
ReLU & 6 & 50 & 1.5 & $9.2\!\times\!10^{-9}$ & 173 & $+0.016$ & 0.955 & late gain \\
ReLU & 8 & 30 & 1.5 & $1.0\!\times\!10^{-10}$ & 16.2 & $+0.532$ & 0.831 & late gain \\
ReLU & 8 & 50 & 1 & $1.0\!\times\!10^{-10}$ & 13.3 & $+0.289$ & 0.857 & late gain \\
ReLU & 8 & 50 & 1.5 & $1.2\!\times\!10^{-10}$ & 40.9 & $+0.189$ & 0.907 & late gain \\
ReLU & 12 & 50 & 1.5 & $1.0\!\times\!10^{-10}$ & 23.1 & $+0.077$ & 0.820 & late gain \\
ReLU & 16 & 50 & 1.5 & $1.9\!\times\!10^{-10}$ & 11.3 & $+0.079$ & 0.096 & not recovered \\
ReLU & 16 & 50 & 2 & $1.0\!\times\!10^{-10}$ & 26.7 & $+0.093$ & 0.104 & not recovered \\
ReLU & 20 & 50 & 2 & $1.8\!\times\!10^{-10}$ & 16.1 & $+0.052$ & 0.065 & not recovered \\
\midrule
GELU & 4 & 50 & 0.7 & $6.4\!\times\!10^{-6}$ & 7.8 & $+0.001$ & 0.995 & aligned by midpoint \\
GELU & 6 & 50 & 0.7 & $1.3\!\times\!10^{-7}$ & 7.5 & $+0.012$ & 0.990 & late gain \\
GELU & 8 & 30 & 0.7 & $2.6\!\times\!10^{-9}$ & 3.0 & $+0.015$ & 0.977 & late gain \\
GELU & 8 & 50 & 0.5 & $5.5\!\times\!10^{-9}$ & 4.2 & $+0.038$ & 0.967 & late gain \\
GELU & 8 & 50 & 0.7 & $1.7\!\times\!10^{-9}$ & 6.6 & $+0.066$ & 0.960 & late gain \\
GELU & 12 & 50 & 0.6 & $3.9\!\times\!10^{-10}$ & 3.3 & $+0.447$ & 0.758 & late gain \\
GELU & 12 & 50 & 0.7 & $2.3\!\times\!10^{-10}$ & 4.1 & $+0.242$ & 0.823 & late gain \\
GELU & 12 & 50 & 0.85 & $2.1\!\times\!10^{-10}$ & 4.8 & $+0.069$ & 0.886 & late gain \\
GELU & 12 & 80 & 0.7 & $1.8\!\times\!10^{-10}$ & 8.4 & $+0.405$ & 0.557 & late gain \\
GELU & 16 & 50 & 0.7 & $1.7\!\times\!10^{-10}$ & 3.0 & $+0.100$ & 0.809 & late gain \\
GELU & 16 & 50 & 0.85 & $1.5\!\times\!10^{-10}$ & 4.0 & $+0.359$ & 0.517 & late gain \\
GELU & 20 & 50 & 0.7 & $2.4\!\times\!10^{-10}$ & 1.9 & $+0.242$ & 0.556 & late gain \\
\midrule
SiLU & 4 & 50 & 0.85 & $4.9\!\times\!10^{-6}$ & 8.4 & $-0.000$ & 0.989 & aligned by midpoint \\
SiLU & 6 & 50 & 0.85 & $1.8\!\times\!10^{-6}$ & 5.6 & $+0.004$ & 0.987 & aligned by midpoint \\
SiLU & 8 & 30 & 0.85 & $7.2\!\times\!10^{-9}$ & 2.4 & $+0.017$ & 0.977 & late gain \\
SiLU & 8 & 50 & 0.6 & $1.1\!\times\!10^{-8}$ & 3.2 & $+0.015$ & 0.977 & late gain \\
SiLU & 8 & 50 & 0.85 & $1.3\!\times\!10^{-8}$ & 5.6 & $+0.016$ & 0.975 & late gain \\
SiLU & 12 & 50 & 0.7 & $5.8\!\times\!10^{-10}$ & 2.7 & $+0.279$ & 0.894 & late gain \\
SiLU & 12 & 50 & 0.85 & $6.1\!\times\!10^{-10}$ & 4.1 & $+0.312$ & 0.889 & late gain \\
SiLU & 12 & 50 & 1 & $6.6\!\times\!10^{-10}$ & 5.2 & $+0.176$ & 0.916 & late gain \\
SiLU & 12 & 80 & 0.85 & $5.3\!\times\!10^{-10}$ & 11.5 & $+0.211$ & 0.905 & late gain \\
SiLU & 16 & 50 & 0.85 & $1.9\!\times\!10^{-10}$ & 2.6 & $+0.277$ & 0.845 & late gain \\
SiLU & 16 & 50 & 1 & $1.3\!\times\!10^{-10}$ & 3.9 & $+0.376$ & 0.793 & late gain \\
SiLU & 20 & 50 & 0.85 & $1.0\!\times\!10^{-10}$ & 3.3 & $+0.384$ & 0.415 & late gain \\
\bottomrule
\end{tabular}
\end{table}

\subsection{Per-teacher top-$6$ configurations and the full ReLU
10-config grid}\label{app:per-teacher-extras}

The effective-rank curves count, among the 12 recorded leading AGOP eigenvalues,
those at least $0.05\lambda_1$ or at least $0.5\lambda_2$, respectively.
The latter threshold also defines the selected rank $\tilde r$ used by the
head-refitting diagnostic (over the spectrum used for that diagnostic).

Figure~\ref{fig:relu-top6} shows six ReLU configurations and
Figures~\ref{fig:silu-top6}--\ref{fig:gelu-top6} six SiLU and six GELU configurations,
selected on the seed-$5$ sweep with the scores of Section~\ref{app:selection} and shown with
all seeds; the selection and the column order are those of the seed-$5$ ranking.
Figure~\ref{fig:relu-grid10} reports the underlying full 10-config
ReLU plateau-search grid from which Figure~\ref{fig:relu-top6}
selects.

\begin{figure}[t]
\centering
\includegraphics[width=\linewidth]{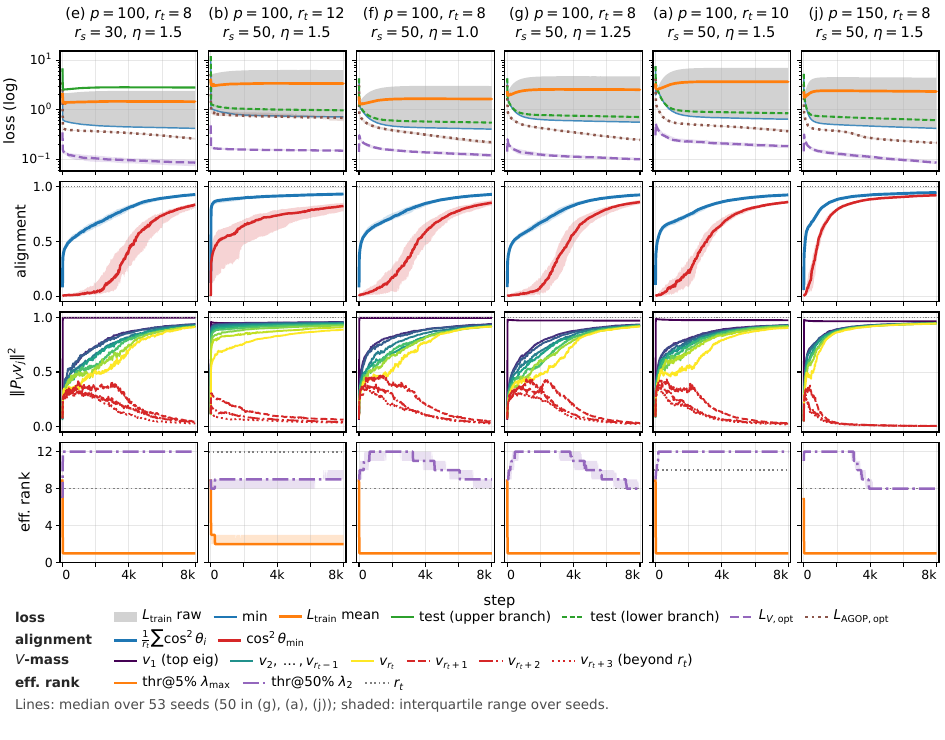}
\caption{\textbf{Top-$6$ ReLU configurations from the 10-configuration plateau-search grid:
features \ref{P:cycle} and \ref{P:align}.} Columns are configurations, lettered as in
Figure~\ref{fig:relu-grid10}; lines are medians over 53 seeds (50 for (g), (a), (j)) and shading is
the interquartile range. \textbf{Row 1} training loss (per-step range, running minimum, running mean),
test loss and the oracles $L_{V,\mathrm{opt}}$, $L_{\mathrm{AGOP\text{-}opt}}$
(Section~\ref{sec:setup}). Test loss and oracles are logged every 20 steps, so they sit on one branch
of the period-2 loss oscillation, the same for every seed; the legend names the branch. Over the last $30\%$ of
training the cycle-mean training loss is $14$--$26\times L_{V,\mathrm{opt}}$ and
$5$--$11\times L_{\mathrm{AGOP\text{-}opt}}$. \textbf{Row 2} AGOP direction-only alignment
$\overline{\cos^2}$, $\cos^2_{\min}$ (verifies \ref{P:align}). \textbf{Row 3} per-eigenvector
$V$-mass $\|P_V v_j\|^2$ for the top $r_t+3$ AGOP eigenvectors. \textbf{Row 4} AGOP effective rank
from the top 12 eigenvalues; dotted line at $r_t$.}
\label{fig:relu-top6}
\end{figure}

\begin{figure}[t]
\centering
\includegraphics[width=\linewidth]{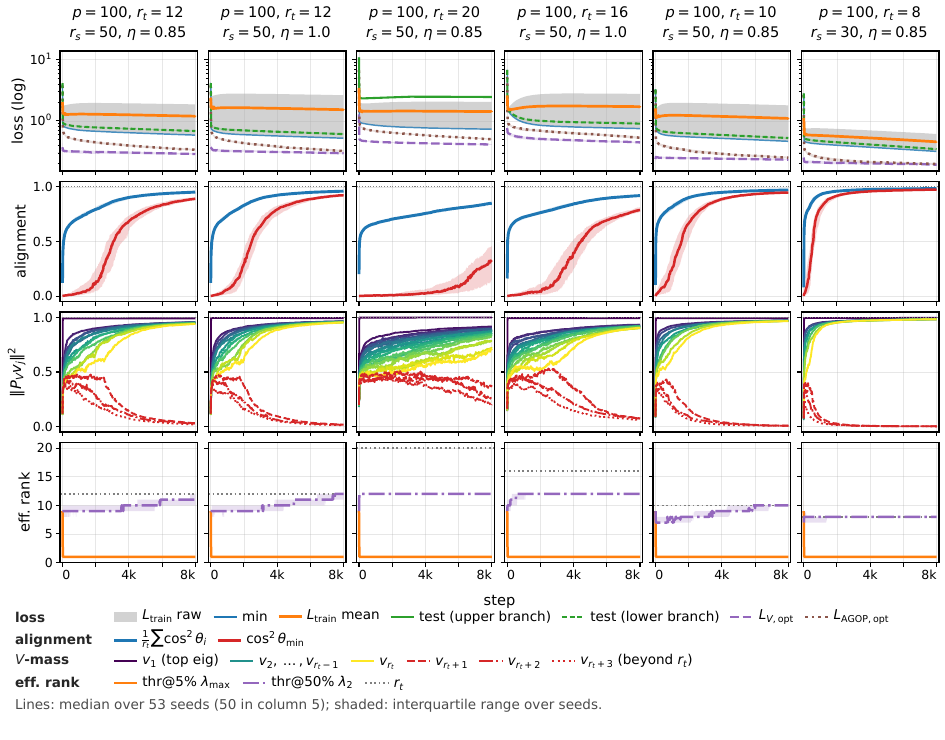}
\caption{Six SiLU configurations, selected on seed $5$ with the smooth-teacher score of
Section~\ref{app:selection}. Same four-row layout and statistics as
Figure~\ref{fig:relu-top6}: medians over 53 seeds (50 in column 5), interquartile range shaded.}
\label{fig:silu-top6}
\end{figure}

\begin{figure}[t]
\centering
\includegraphics[width=\linewidth]{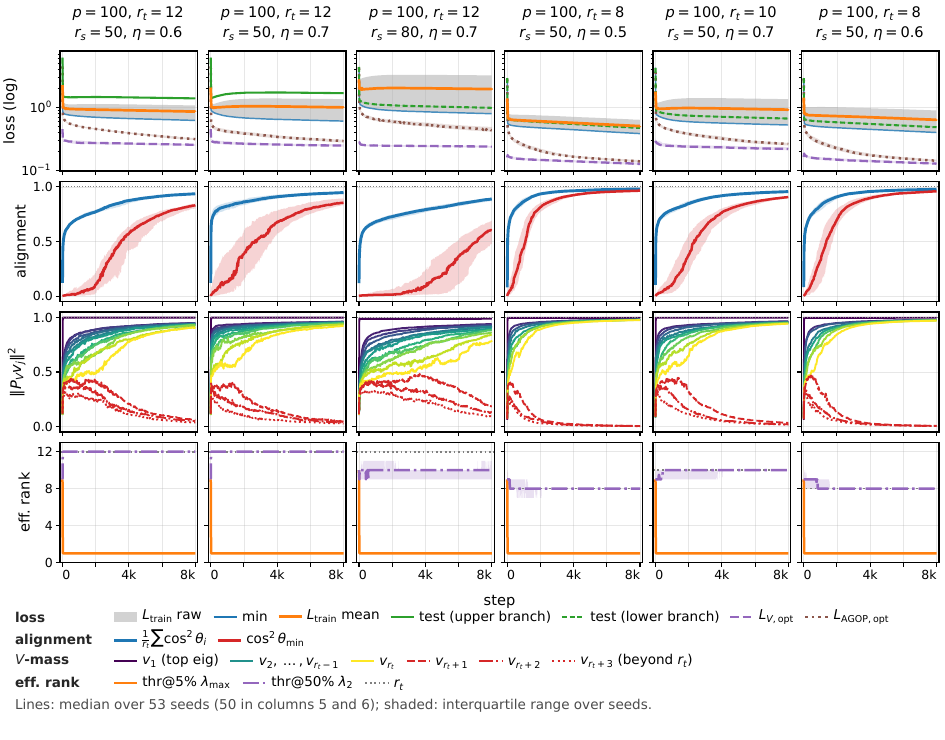}
\caption{Six GELU configurations, selected on seed $5$ with the smooth-teacher score of
Section~\ref{app:selection}. Same four-row layout and statistics as
Figure~\ref{fig:relu-top6}: medians over 53 seeds (50 in columns 5 and 6), interquartile range shaded.}
\label{fig:gelu-top6}
\end{figure}

\begin{figure}[t]
\centering
\includegraphics[width=\linewidth]{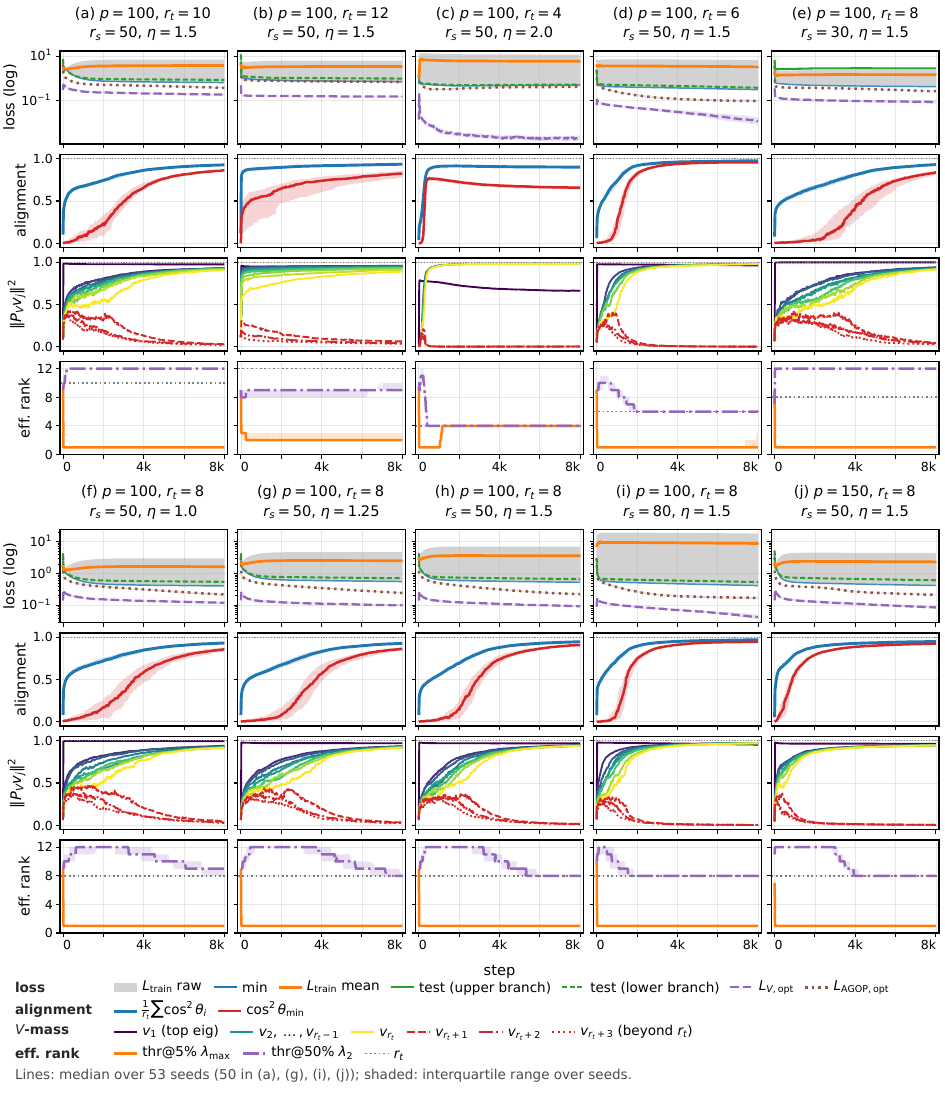}
\caption{Full 10-configuration plateau-search grid for the ReLU teacher (superset of
Figure~\ref{fig:relu-top6}). Each column is one $(p, r_t, r_s, \eta)$ configuration; rows and statistics
as in Figure~\ref{fig:relu-top6} (53 seeds; 50 in (a), (g), (i), (j)). Configurations span
$r_t \in \{4,6,8,10,12\}$, $r_s \in \{30, 50, 80\}$, $p \in \{100, 150\}$ and
$\eta \in \{1.0, 1.25, 1.5, 2.0\}$.}
\label{fig:relu-grid10}
\end{figure}

\begin{figure}[t!]
\centering
\includegraphics[width=\linewidth]{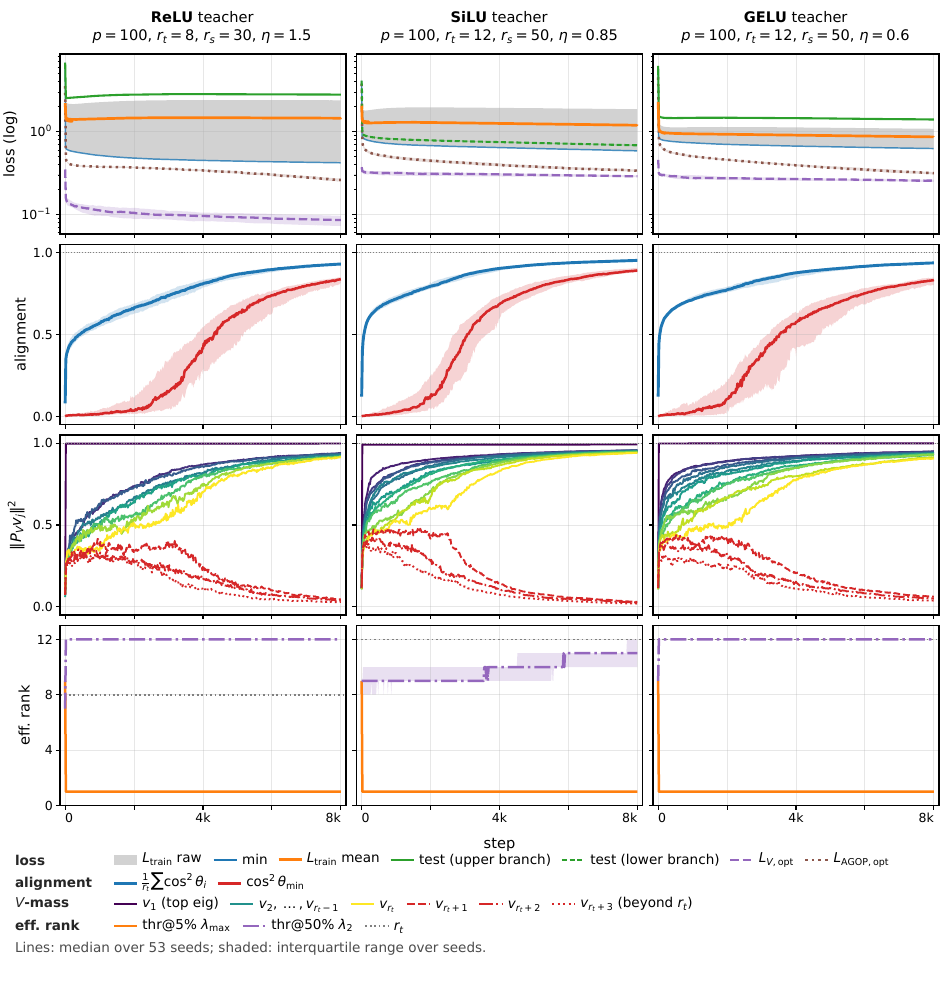}
\caption[Cleanest configuration per teacher activation: ReLU, SiLU, GELU.]{\textbf{Cleanest configuration per teacher
activation: ReLU, SiLU, GELU.} The four-row diagnostic (loss
decomposition, AGOP direction-only alignment, per-eigvec $V$-mass,
AGOP effective rank) reproduces features \ref{P:cycle} and
\ref{P:align} for all three teacher activations. Per-activation
configurations are the seed-$5$ top choices of Section~\ref{app:selection},
kept for the multi-seed runs (with all seeds, the SiLU and GELU choices rank third
among the six shown).
Per-teacher 10-config grids and 6-config rankings are in
Figures~\ref{fig:silu-top6}--\ref{fig:relu-grid10} above. Lines are medians over 53 seeds; shading is the interquartile range.}
\label{fig:three-best}
\end{figure}

\paragraph{Reading Figure~\ref{fig:1hl-main}.}
\textbf{Loss plateau with continued feature learning in the 1-HL student (M4).}
One configuration per ReLU, SiLU and GELU teacher activation; matrix Muon, $8{,}000$ steps.
\textbf{Top:} training loss (per-step range, running minimum and mean), test loss and oracles
$L_{V,\mathrm{opt}}$, $L_{\mathrm{AGOP\text{-}opt}}$.
\textbf{Bottom:} AGOP alignment $\overline{\cos^2}$, $\cos^2_{\min}$ rises during the period-$2$ loss plateau.
Medians over $53$ seeds; shaded interquartile range. Figure~\ref{fig:three-best} shows the same columns with all four rows.

\subsection{Configuration selection}\label{app:selection}

The configurations shown in Figures~\ref{fig:relu-top6}--\ref{fig:three-best} were chosen before
the multi-seed runs, on the seed-$5$ sweep described at the start of this appendix, with the
following per-trajectory scores; the multi-seed figures keep those choices and their order.
The original seed-$5$ figures span $10{,}000$ steps; the retrospective rankings below use
the $8{,}000$-step trajectories of the multi-seed replication.
For a run of $T$ steps with per-step loss $L_1,\dots,L_T$, let $\bar L$ be its running mean over windows of
$\max(3,\lfloor T/30\rfloor)$ steps, $s=\lfloor T/2\rfloor$, $L_{\mathrm{init}}$ the mean of the first
five losses and $L_p$ the mean of $\bar L$ over steps $s$ to $T$. Here $T$ counts training steps,
whereas $n$ counts training examples. The slice $[s:T]$ restricts a sequence to that window.
Both $\mathrm{std}$ and $\mathrm{sd}$ denote its empirical standard deviation. With
\begin{gather*}
\mathrm{flat}=\max\Bigl\{0,\,1-\frac{\mathrm{std}(\bar L_{[s:T]})}{L_p}\Bigr\},\qquad
\mathrm{grow}=\max\{0,\,\cos^2_{\min}(W_T)-\cos^2_{\min}(W_s)\},\\
\mathrm{drop}=\min\Bigl\{1,\max\Bigl\{0,\log_{10}\frac{L_{\mathrm{init}}}{L_p}\Bigr\}\Bigr\},
\end{gather*}
the ReLU score is $\min(\mathrm{flat},1)\cdot\min(\mathrm{grow},1)\cdot\mathrm{drop}$: a flat second
half, alignment growth after the midpoint, and loss reduction from initialisation, capped at one
decade. For SiLU and GELU it is multiplied by $e^{-|\bar k-r_t|}/(1+\mathrm{sd}(k))$, where $k_t$ is the
AGOP effective rank at the $50\%$ threshold (the number of eigenvalues at least half the second
largest) over the second half, with mean $\bar k$ and standard deviation $\mathrm{sd}(k)$. The scores
reward a flat loss, alignment growth and some loss reduction; they do not involve the head-fit
oracle. Ranked instead by the median score over all seeds, the six ReLU configurations are the same,
in the order (e), (f), (b), (a), (g), (j) rather than the plotted (e), (b), (f), (g), (a), (j);
among the six SiLU and the six GELU configurations shown, the configurations of
Figure~\ref{fig:three-best} rank third (first are $r_t=20$, $\eta=0.85$ for SiLU and $r_t=10$,
$\eta=0.7$ for GELU).

\subsection{Theory--experiment overlay and test-by-test detail}\label{app:tests}

Figure~\ref{fig:theory-vs-exp} overlays theory and experiment for the $15$ multi-activation
configurations; the tests below give the per-configuration evidence.
All use $p=100$ and $r_s=50$. ReLU (blue): $(r_t,\eta)=(4,2)$ circle, $(6,1.5)$ down-triangle,
$(8,1.5)$ square, $(12,1.5)$ diamond, $(16,2)$ star. GELU, $\eta=0.7$ (purple):
$r_t=4,8,12,16$ as square, down-triangle, diamond, star; GELU $r_t=16$, $\eta=0.85$: magenta plus.
SiLU, $\eta=0.85$ (brown): $r_t=4,8,12,16$ with the same shapes; SiLU $r_t=20$, $\eta=0.85$: olive plus.

\paragraph{Reading Figure~\ref{fig:theory-vs-exp}.}
\textbf{Theory--experiment overlay: 15 ReLU, GELU and SiLU teacher configurations.}
Medians over 53 seeds; interquartile ranges as error bars or shading.
\textbf{(a)} Second-half per-step loss period-$2$ signature $1-\rho_2$: at most about $6.3\times10^{-6}$, with values below $10^{-10}$ floored for display (supports \ref{P:cycle}).
\textbf{(b)} Final top-$r_t$ AGOP $V$-mass $\sum_{j\le r_t}\|P_V v_j\|^2$ against its maximum $r_t$
(AGOP alignment; this is not the gradient leakage of Proposition~\ref{prop:polar-identities}(iii)):
within $7\%$ for $r_t\le12$, except ReLU $r_t{=}4$, $\eta{=}2$ ($10\%$); $6$--$28\%$ below for $r_t\ge16$.
\textbf{(c)} Plateau loss (cycle mean, last $30\%$ of steps) vs.\ head oracle $L_{V,\mathrm{opt}}$:
$23$--$3000\times$ above diagonal for ReLU, $2.7$--$8.4\times$ for GELU and SiLU.
\textbf{(d)} Late-half $\cos^2_{\min}$ secant vs.\ $r_t/M$: log-log slope $1.8$ across the 12 configurations with
late-half gain $>0.005$; three saturated $r_t{=}4$ configurations are open markers.
This descriptive raw-gain fit, not the normalised rate $v_k/e_k$, does not test Proposition~\ref{claim:drift}.
\textbf{(e)} Off-$V$ AGOP-mass proxy $(1-\mathrm{mass}_V)\cdot M$ vs.\ $\eta^2(M-r_t)$,
with a unit-slope reference line whose prefactor is the median ratio
(approximately $0.09$). The slope is fixed, not fitted. This comparison
does not measure the fixed-head loss residual in Conjecture~\ref{claim:plateau}.
\textbf{(f)} $\cos^2_{\min}$ (solid), $\overline{\cos^2}$ (dashed): four selected trajectories (illustrates \ref{P:align}).

\paragraph{Test A: period-$2$ loss signature.}
Over the second half of training, the seed median of $1-\rho_2$, with
$\rho_2 := -\mathrm{Corr}(\Delta L_t, \Delta L_{t+1})$ and
$\Delta L_t:=L_{t+1}-L_t$, is at most about $6.3\times10^{-6}$
in all $15$ configurations of Figure~\ref{fig:theory-vs-exp}(a); values below $10^{-10}$
are floored for display. The loss
increments therefore alternate almost perfectly; this does not establish
an exact two-cycle of the weights (Section~\ref{sec:1NN-phenomenon}).

\paragraph{Test B: V/V$^\perp$ polar mass split
(Proposition~\ref{prop:polar-identities}).}
Computed by taking the polar update on the saved final $W$ and splitting
into $\|P_V\Polar(g)\|_F^2$ and $\|P_\perp\Polar(g)\|_F^2$, where
$g$ is the gradient reconstructed as specified in the table caption.

\begin{table}[h]
\caption{Polar mass split on three long-run ReLU checkpoints. The total equals $M$ to
numerical precision (Proposition~\ref{prop:polar-identities}(i)); the $V$-mass is approximately $2$--$5\%$
below $r_t$. The gradient is recomputed from the saved weights with a unit output head, and the
leakage scalars of Proposition~\ref{prop:polar-identities}(iii) were not measured separately.}
\label{tab:polar-mass}
\centering
\begingroup
\small
\renewcommand{\arraystretch}{1.2}
\setlength{\tabcolsep}{5pt}
\begin{tabular}{lrrrrrr}
\toprule
& \multicolumn{2}{c}{$V$-mass} & \multicolumn{2}{c}{$V^\perp$-mass} & \multicolumn{2}{c}{Total mass} \\
\cmidrule(lr){2-3}\cmidrule(lr){4-5}\cmidrule(lr){6-7}
Config $(\sigma_t,r_t,r_s)$ & Predicted & Measured & Predicted & Measured & $M$ & Measured \\
\midrule
ReLU, $r_t=4$, $r_s=50$ & 4 & 3.79 & 46 & 46.21 & 50 & 50.00 \\
ReLU, $r_t=8$, $r_s=50$ & 8 & 7.86 & 42 & 42.14 & 50 & 50.00 \\
ReLU, $r_t=12$, $r_s=50$ & 12 & 11.76 & 38 & 38.24 & 50 & 50.00 \\
\bottomrule
\end{tabular}

\endgroup
\end{table}

\paragraph{Test C: plateau height vs head-fit limit.}
For a trajectory of $T$ steps, measure
$L_{\mathrm{plateau}} := \mathrm{mean}(L_{[\lfloor0.7T\rfloor:T]})$ and
compare to the head-fit oracle $L_{V,\mathrm{opt}}$, which optimises
the head amplitudes after projecting $W$ onto $V$. Because the head is refitted, this comparison
does not test Conjecture~\ref{claim:plateau}, which keeps it fixed.

\begin{table}[h]
\caption{Plateau height vs head-fit oracle. The plateau loss sits
$10\times$--$36\times$ above $L_{V,\mathrm{opt}}$, the loss after projecting $W$
onto $V$ and refitting the head: the representation supports a much better
predictor than the trained network realises (Test~E).}
\centering
\makebox[\linewidth][c]{\small\begin{tabular}{lrrr}
\toprule
$(\sigma_t, p, r_t, r_s, \eta)$ & $L_{\mathrm{plateau}}$ & $L_{V,\mathrm{opt}}$ & ratio $L_{\mathrm{plateau}} / L_{V,\mathrm{opt}}$ \\
\midrule
ReLU, $(100, 8, 50, 1.0)$ & 1.70 & 0.18 & 9.6 \\
ReLU, $(100, 8, 50, 1.5)$ & 3.74 & 0.10 & 36.0 \\
ReLU, $(100, 12, 50, 1.5)$ & 3.48 & 0.16 & 21.5 \\
\bottomrule
\end{tabular}
}
\end{table}

\paragraph{Test D: drift growth rate.}
Measure the empirical growth rate of $\cos^2_{\min}$ over the second
half of training, and compare to the naive scale $\eta\,r_t/M$ (per
step). In Table~\ref{tab:drift}, the measured rate is the finite secant
$[\cos^2_{\min}(W_{t_b})-\cos^2_{\min}(W_{t_a})]/(t_b-t_a)$
from the first logged checkpoint in the second half, $t_a$, to the final
checkpoint $t_b$, rather than a continuous-time derivative.

\begin{table}[h]
\caption{Drift rate scaling. The measured per-step rate is
$O(10^{-4})$ across configs, while the comparison scale $\eta r_t / M$ is
$O(10^{-1})$; the table reports their ratio (Test~G). This raw rate is not the normalised rate of
Proposition~\ref{claim:drift}, whose conditions we have not verified
here. The table covers $r_t \in \{6,8,10,12\}$ at fixed $\eta = 1.5$.}
\label{tab:drift}
\centering
\makebox[\linewidth][c]{\small\begin{tabular}{lrrr}
\toprule
$(\sigma_t, p, r_t, r_s, \eta)$ & predicted $\eta r_t/M$ & measured secant & ratio \\
\midrule
ReLU, $(100, 6, 50, 1.5)$ & $0.18$ & $2.0\times10^{-5}$ & $+1.13\times10^{-4}$ \\
ReLU, $(100, 8, 50, 1.5)$ & $0.24$ & $2.2\times10^{-4}$ & $+9.16\times10^{-4}$ \\
ReLU, $(100, 10, 50, 1.5)$ & $0.30$ & $2.6\times10^{-4}$ & $+8.62\times10^{-4}$ \\
ReLU, $(100, 12, 50, 1.5)$ & $0.36$ & $4.0\times10^{-5}$ & $+1.10\times10^{-4}$ \\
\bottomrule
\end{tabular}
}
\end{table}

\paragraph{Test E: $L_{\mathrm{AGOP\text{-}opt}}$ below the training loss.}
On the long-run ReLU configs in Table~\ref{tab:nfa},
$L_{\mathrm{AGOP\text{-}opt}}$ lies well below $L_{\mathrm{train}}$ over
most of the plateau, though not at every checkpoint; on the six cleanest
ReLU configurations of Figure~\ref{fig:relu-top6} the seed-median ratio of the
cycle-mean training loss to $L_{\mathrm{AGOP\text{-}opt}}$ over the last $30\%$ of
training is about $5$--$11$, and the same ordering is visible in the loss panels of
Figures~\ref{fig:silu-top6} and~\ref{fig:gelu-top6}. The diagnostic uses no
teacher information. Because it projects $W$ before refitting the head, it
shows that the learned directions support a better predictor, not that the
head alone limits the loss.

\paragraph{Test F: weight mass versus AGOP mass.}\label{app:nfa}
Across the long-run ReLU configs, top-$r_t$ AGOP eigvecs align with
$V$ ($\overline{\cos^2} \in [0.85, 0.96]$) and the AGOP eigenvalue mass
concentrates ($\mathrm{mass}_V \in [0.83, 0.90]$), but the weight Frobenius
mass fraction $A_{\mathrm{sub}}$ stays in $[0.29, 0.50]$. If the weight Gram matrix
$WW^\top$ were proportional to the AGOP, the simplest form of the Neural Feature Ansatz, the two
mass fractions would be equal; they differ by a large margin. Power-law variants of the ansatz
(e.g.\ with the square root of the AGOP) are not tested here.

\begin{table}[h]
\caption{Weight mass versus AGOP mass. Direction-only alignment $\overline{\cos^2}$ lies in
$[0.85, 0.96]$ and eigenvalue-weighted AGOP mass-in-$V$ in $[0.83, 0.90]$;
the weight Frobenius mass-in-$V$ $A_{\mathrm{sub}}$
caps at $\approx 0.50$ across $r_t \in \{6,8,10,12\}$ at fixed
$\eta = 1.5$.}
\label{tab:nfa}
\centering
\makebox[\linewidth][c]{\small\begin{tabular}{lccc}
\toprule
$(\sigma_t, p, r_t, r_s, \eta)$ & $\overline{\cos^2}$ & $\mathrm{mass}_V(G_s)$ & $A_{\mathrm{sub}}$ \\
\midrule
ReLU, $(100, 6, 50, 1.5)$ & 0.96 & 0.90 & 0.35 \\
ReLU, $(100, 8, 50, 1.5)$ & 0.86 & 0.85 & 0.29 \\
ReLU, $(100, 10, 50, 1.5)$ & 0.85 & 0.83 & 0.29 \\
ReLU, $(100, 12, 50, 1.5)$ & 0.86 & 0.84 & 0.50 \\
\bottomrule
\end{tabular}
}
\end{table}

\paragraph{Test G: comparison with the naive per-step scale.}
The directly measured statistic is the late-window alignment secant of Test D.
Its ratio to $\eta r_t/M$ lies between $1.10\times10^{-4}$ and $9.16\times10^{-4}$
on the four configurations in Table~\ref{tab:drift}. Thus this scale substantially
exceeds the observed alignment drift. This comparison does not measure a one-step
weight-update signal-to-noise ratio or verify the conditional drift hypotheses.

\subsection{How theory and experiment align: a discussion}\label{app:align-discussion}

Combining the long-run ReLU tables in Section~\ref{app:tests} with the
top-$6$ figures for ReLU, SiLU, and GELU
(Figures~\ref{fig:relu-top6}, \ref{fig:silu-top6}, \ref{fig:gelu-top6})
lets us distinguish exact identities from empirical comparisons. The qualitative
loss/alignment separation occurs across the three teacher activations; the polar-mass and
weight-versus-AGOP-mass tables are restricted to ReLU.

\paragraph{Exact identities (machine precision).}
The polar Frobenius identity (Proposition~\ref{prop:polar-identities})
holds to machine precision in every run. It is an algebraic fact about
the polar map, independent of the activation, the data distribution and
the learning rate. The period-$2$ loss signature (Test A) is an
empirical observation, not an algebraic fact.

\paragraph{Structural identities.}
The polar Frobenius identity and the $V/V^\perp$ decomposition are exact
(Proposition~\ref{prop:polar-identities}). The measured $V$-mass of the polar update (Test B,
Table~\ref{tab:polar-mass}) is approximately $2$--$5\%$ below $r_t$ on three ReLU checkpoints; since the
leakage scalars were not measured, we do not attribute this gap quantitatively. The AGOP $V$-mass
of Figure~\ref{fig:theory-vs-exp}(b) is a separate, dynamical quantity.

\paragraph{Drift rate and plateau height.}
The measured drift rate (Test D) is about $10^{-3}$ of $\eta r_t/M$
per step (Table~\ref{tab:drift}; see Test~G for the range of ratios);
Proposition~\ref{claim:drift} gives sufficient conditions for a
power-law normalised rate, which we have not verified along these
trajectories. The plateau height (Test C) sits $10$--$36\times$ above the head-fit
oracle $L_{V,\mathrm{opt}}$; the fixed-head decomposition of Conjecture~\ref{claim:plateau}
and its off-$V$ prefactor remain untested.

\paragraph{Diagnostic gap.}
Over the last $30\%$ of training the seed-median ratio of the cycle-mean training
loss to $L_{\mathrm{AGOP\text{-}opt}}$ is about $5$--$11$ on the six cleanest ReLU
configurations (Test E): the learned directions support a better
predictor than the trained network realises. The mass gap of Test F is similarly clear:
$\overline{\cos^2}$ saturates above $0.85$ while $A_{\mathrm{sub}}$
caps in the range $[0.29, 0.50]$ across the long-run ReLU configs in
Table~\ref{tab:nfa}. The per-eigenvector $V$-mass rows of Figures~\ref{fig:silu-top6}
and~\ref{fig:gelu-top6} show AGOP directional alignment for the smooth teachers;
they do not measure weight-mass concentration.

\paragraph{Comparison across teacher activations.}
Smooth teachers (GELU, SiLU) exhibit the qualitative loss/alignment separation at tested
learning rates roughly half those of the ReLU examples (GELU: $\eta \approx 0.7$;
SiLU: $\eta \approx 0.85$; ReLU: $\eta \approx 1.5$). Training loss
\emph{rising} during alignment growth is especially visible for SiLU $r_t = 20$
(Figure~\ref{fig:silu-top6}). The polar-mass and weight-versus-AGOP-mass tables here
are restricted to ReLU and do not establish the same quantitative comparisons for
the smooth teachers.

\section{Deep Gaussian teacher--student experiments}\label{app:deep}

This appendix specifies the deep-network experiments: data, teacher, student, optimiser,
metrics, the scoring protocol, the configurations, and the depth results.

\paragraph{Notation.}
Here $p$ is the input dimension, $r_t$ the teacher-subspace dimension, $m$ the
hidden width, $L$ the number of hidden layers and $T$ the number of training steps.
We write $W=(W_1,\ldots,W_L)$ for all hidden matrices, reserving $L_t$ for the
scalar training loss at step $t$. The training, test and AGOP-subset sizes are
$n$, $n_{\mathrm{test}}$ and $n_m$, respectively.
The norms $\|\cdot\|_1$, $\|\cdot\|_2$ and $\|\cdot\|_F$ denote vector
$\ell_1$, Euclidean and matrix Frobenius norms; $I_p$ is the identity matrix.
The learning rates $\eta_1$ and $\eta_{\mathrm{deep}}$ are coefficients of the
polar factor, not the Frobenius lengths of the resulting updates.

\subsection{Data, teacher and student}\label{app:deep-model}

\paragraph{Data and teacher.}
Inputs are $x\sim\mathcal N(0,I_p)$ and labels are the noise-free outputs of a fixed teacher.
Two seeds fix all randomness: the \emph{problem seed} $s$ fixes the teacher, the training and
test inputs and the AGOP subset (generators seeded with $10000+s$, $20000+s$, $30000+s$ and
$50000+s$), and the \emph{initialisation seed} $s'$ fixes the student's initial weights
(generator seeded with $40000+s'$). The
teacher subspace is $V=\linspan(U)$, where $U\in\R^{p\times r_t}$ has orthonormal columns
(QR factor of a Gaussian matrix). The default teacher has two layers,
\begin{equation}
f^\star(x)\;=\;\sum_{j=1}^{2r_t}\Bigl[\mathrm{ReLU}\bigl(\mathrm{ReLU}(x^\top U)\,[\,Q,\,-Q\,]\bigr)\Bigr]_j
\;=\;\bigl\|Q^\top\mathrm{ReLU}(U^\top x)\bigr\|_1 ,
\label{eq:deep-teacher}
\end{equation}
with $Q\in\R^{r_t\times r_t}$ a random orthogonal matrix, so its input-relevant subspace is
again $V$. The labels are rescaled by one constant per problem so that their root mean square
on the training inputs equals that of the one-layer teacher
$f^\star_1(x)=\sum_{j\le r_t}\mathrm{ReLU}(u_j^\top x)$, which one breadth arm uses instead.

\paragraph{Student.}
The student is a residual ReLU network of depth $L$ and width $m$, without biases:
\begin{equation}
h_1=\mathrm{ReLU}(x^\top W_1),\qquad
h_\ell=h_{\ell-1}+\alpha\,\mathrm{ReLU}(h_{\ell-1}W_\ell)\quad(\ell=2,\dots,L),\qquad
f_W(x)=h_L\,a,
\label{eq:deep-student}
\end{equation}
with $W_1\in\R^{p\times m}$, $W_\ell\in\R^{m\times m}$ and residual scale $\alpha=1/(L-1)$
unless stated otherwise. The head $a=(c_a/\sqrt p)\,\mathbf 1_m$ is frozen. At
initialisation $W_1$ has i.i.d.\ $\mathcal N(0,1/p)$ entries and each $W_\ell$, $\ell\ge2$,
i.i.d.\ $\mathcal N(0,1/m)$ entries, all multiplied by a gain $g$. One arm, used in Appendix~\ref{app:freeze}, replaces the
residual network by a plain MLP, $h_\ell=\mathrm{ReLU}(h_{\ell-1}W_\ell)$, with He
initialisation ($g=\sqrt2$), $\eta_1=0.6$ and $\eta_{\mathrm{deep}}=0.02$.

\paragraph{Optimiser.}
Every hidden matrix is trained by the exact matrix-Muon step on the full-batch loss
$\mathcal L(W)=\frac1n\sum_{i\le n}\bigl(f_W(x_i)-y_i\bigr)^2$,
\begin{equation}
W_1\leftarrow W_1-\eta_1\,\Polar\bigl(\nabla_{W_1}\mathcal L\bigr),\qquad
W_\ell\leftarrow W_\ell-\eta_{\mathrm{deep}}\,\Polar\bigl(\nabla_{W_\ell}\mathcal L\bigr)\quad(\ell\ge2),
\label{eq:deep-muon}
\end{equation}
where $\Polar(G)=U_GV_G^\top$ is formed from the thin SVD $G=U_G\Sigma V_G^\top$. There is
no momentum, Newton--Schulz iteration, weight decay or learning-rate schedule; arithmetic is
float64. Runs last $T=1000$ steps.

\subsection{Metrics and scoring protocol}\label{app:deep-protocol}

\paragraph{Reading Figure~\ref{fig:deep-depth}.}
\textbf{Deep residual students under exact-polar Muon: cycling loss, rising alignment.}
Columns: Table~\ref{tab:deep-depth}'s depth arms and D8 on fresh seeds.
\textbf{Row 1:} training loss: per-step range between period-$2$ branches (grey), running minimum, the lower branch (blue),
$20$-step running mean (orange).
\textbf{Row 2:} top-$r_t$ input-AGOP eigenspace alignment with the teacher subspace,
$\overline{\cos^2}$ and $\cos^2_{\min}$, averaged over each two-step checkpoint pair.
Median lines with interquartile shading over $100$ runs per column ($20$ problems $\times$ $5$ initialisations).
The dotted line starts the scoring window. Across columns, the median
per-run cycle-mean loss change is at most about $2.6\%$, with falling
lower branch and rising alignment measures. Table~\ref{tab:deep-results}
instead reports medians over problems of the per-problem medians; their
largest cycle-mean loss change is approximately $2.62\%$.

\paragraph{Metrics.}
The end-to-end input AGOP
$G_s=\frac1{n_m}\sum_{i\le n_m}\nabla_x f_W(x_i)\nabla_x f_W(x_i)^\top$ is computed on a
fixed random subset of $n_m$ training inputs; its top-$r_t$ eigenspace is compared with $V$
through principal angles, giving the mean alignment $\overline{\cos^2}$ and the
weakest-direction alignment $\cos^2_{\min}$ of~\eqref{eq:cos-metrics}.
Explicitly, if $B$ has orthonormal columns spanning the top-$r_t$ eigenspace
of $G_s$, these are $r_t^{-1}\|U^\top B\|_F^2$ and
$\sigma_{\min}(U^\top B)^2$, where $\sigma_{\min}$ is the smallest singular value.
Both are evaluated at
pairs of consecutive steps $(t,t+1)$, every $100$ steps, and at $(999,1000)$. The loss cycle is summarised by
$\rho_2=-\mathrm{Corr}(\Delta L_t,\Delta L_{t+1})$ with $\Delta L_t=L_{t+1}-L_t$, by the
two-step ratio $R_2=\mathrm{median}_t\,|L_{t+2}-L_t|/|L_{t+1}-L_t|$, and by the cycle mean
$\bar L_t=\tfrac12(L_t+L_{t+1})$. The held-out refit loss
replaces the frozen head by a least-squares head on the last hidden layer $h_L$, fitted on the
$n_m$ AGOP inputs (ridge $10^{-8}\operatorname{tr}(H^\top H)/m$), where
$H\in\R^{n_m\times m}$ has rows $h_L(x_i)$, and evaluates it on the
$n_{\mathrm{test}}$ test inputs.

\paragraph{Window and parity.}
All statistics use the window from step $200$ to step $1000$. Because the loss alternates
between two branches, the gain of a metric is its change between the window ends computed on
each parity, even steps $200\to1000$ and odd steps $201\to999$, and averaged over the two.

\paragraph{Strict run test.}
A run passes if, on the window, all of the following hold:
(i)~$\rho_2\ge 0.99$;
(ii)~$R_2\le 0.01$;
(iii)~the relative change of $\bar L_t$ between the window ends, divided by the window length in
hundreds of steps, is at most $0.02$ in absolute value;
(iv)~$\overline{\cos^2}$ gains at least $0.05$;
(v)~$\cos^2_{\min}$ gains a positive amount;
(vi)~both gains are positive on each parity separately.

\paragraph{Aggregation.}
The problem seed, not the run, is the unit of evidence. In the $20\times5$ campaigns a
problem succeeds if at least $4$ of its $5$ initialisations pass, and an arm is robust if at
least $15$ of its $20$ problems succeed. The $12\times4$ campaign of
Appendix~\ref{app:freeze} uses $3$ of $4$ and $9$ of $12$; for the $12\times2$ follow-ups of
Appendix~\ref{app:freeze-principle} we report run counts only.

\subsection{Configurations and seeds}\label{app:deep-configs}

Table~\ref{tab:deep-common} lists the settings shared by every run and
Table~\ref{tab:deep-depth} the depth arms. The depth campaigns ($20$ problems $\times$ $5$ initialisations for D6--D12, for D16 and for
the fresh-D8 block) and the freezing study of Appendix~\ref{app:freeze} ($12$ problems $\times$ $4$
initialisations) use disjoint blocks of problem and initialisation seeds. The follow-ups of
Appendix~\ref{app:freeze-principle} reuse the $12$ problems of Appendix~\ref{app:freeze} and two of
its four initialisations ($24$ runs per arm), so they are paired with it rather than independent.

\begin{table}[ht]
\centering\small
\caption{Settings shared by every deep run unless an arm overrides them.}
\label{tab:deep-common}
\begin{tabular}{llll}
\toprule
input dimension $p$ & $64$ & teacher rank $r_t$ & $4$\\
student width $m$ & $32$ & teacher depth & $2$~\eqref{eq:deep-teacher}\\
training inputs $n$ & $4096$ & test inputs $n_{\mathrm{test}}$ & $8192$\\
AGOP inputs $n_m$ & $2048$ & steps $T$ & $1000$\\
initial gain $g$ & $1$ & head scale $c_a$ & $1$\\
residual scale $\alpha$ & $1/(L-1)$ & architecture & residual~\eqref{eq:deep-student}\\
\bottomrule
\end{tabular}
\end{table}

\begin{table}[ht]
\centering\small
\caption{Depth arms. The fresh-D8 block reuses the D8 configuration with new seeds.}
\label{tab:deep-depth}
\begin{tabular}{lrrrrr}
\toprule
arm & $L$ & $\eta_1$ & $\eta_{\mathrm{deep}}$ & $c_a$ & $\alpha$\\
\midrule
D6  & 6  & 0.66 & $3.125\times10^{-3}$  & 0.75 & 1/5\\
D8  & 8  & 0.66 & $7.8125\times10^{-4}$ & 1    & 1/7\\
D10 & 10 & 0.70 & $3.90625\times10^{-4}$ & 1   & 1/9\\
D12 & 12 & 0.70 & $1.953125\times10^{-4}$ & 1  & 1/11\\
D16 & 16 & 0.74 & $1.953125\times10^{-4}$ & 1  & 1/15\\
\bottomrule
\end{tabular}
\end{table}

\subsection{Depth results}\label{app:deep-results}
Figure~\ref{fig:deep-depth} shows the trajectories and Table~\ref{tab:deep-results} the
scores.

\begin{table}[ht]
\centering\small
\caption{Depth campaigns under the strict protocol (Appendix~\ref{app:deep-protocol}): runs
passing, problems succeeding, and, as medians over problems of the per-problem median over
initialisations, the window gains of $\overline{\cos^2}$ and $\cos^2_{\min}$ and the relative
change over the window of the cycle-mean training loss and of the held-out refit loss.}
\label{tab:deep-results}
\begin{tabular}{lrrrrrr}
\toprule
arm & strict runs & problems & $\overline{\cos^2}$ gain & $\cos^2_{\min}$ gain & cycle mean & held-out refit\\
\midrule
D6  & 94/100 & 20/20 & $+0.120$ & $+0.384$ & $+2.62\%$ & $-6.15\%$\\
D8  & 97/100 & 20/20 & $+0.206$ & $+0.515$ & $+0.62\%$ & $-6.88\%$\\
D10 & 96/100 & 19/20 & $+0.200$ & $+0.441$ & $+1.10\%$ & $-5.38\%$\\
D12 & 95/100 & 20/20 & $+0.193$ & $+0.436$ & $+1.55\%$ & $-5.20\%$\\
D16 & 96/100 & 19/20 & $+0.187$ & $+0.418$ & $+1.75\%$ & $-4.60\%$\\
Fresh D8 & 91/100 & 20/20 & $+0.153$ & $+0.396$ & $+1.24\%$ & $-6.58\%$\\
\bottomrule
\end{tabular}
\end{table}

\paragraph{Scope of these runs.}
The step sizes are on two time scales ($\eta_1\approx0.7$, $\eta_{\mathrm{deep}}\le3\times10^{-3}$)
and the residual branches are scaled by $1/(L-1)$, so we describe these runs as feature
learning in a deep trainable network, not as every layer learning the teacher directions.
A period-$2$ loss is not a period-$2$ parameter orbit: the two loss branches keep drifting
inside the window while the cycle mean stays nearly flat. The polar step uses every column
of the thin SVD, and deep-layer gradients can be numerically rank-deficient; runs reproduce
exactly in the pinned software stack. The data are synthetic Gaussian, training is full
batch with the exact polar factor, no momentum, a frozen head and fixed step sizes.

\section{Freezing the deep layers switches the plateau on}\label{app:freeze}

In the deep runs of Appendix~\ref{app:deep} the deeper layers take steps $200$ to $3{,}800$ times
smaller than the first layer (Table~\ref{tab:deep-depth}). This appendix turns that observation
into an experiment. Keeping everything else fixed, we change how
much of the network between $W_1$ and the head may train, and ask when the plateau with continued
feature learning appears:

\begin{itemize}
\item with the deep layers frozen or training slowly, the cycle-mean loss stays flat while the
alignment keeps rising, and the learned alignment sits in $W_1$;
\item with the deep layers training at a rate comparable to the first layer, the alignment still
rises, but the loss decreases through the window and part of the alignment moves into the deep
layers;
\item freezing the deep layers in the middle of a run makes a decreasing loss flat, and unfreezing
them makes a flat loss decrease, while the alignment keeps rising in both cases;
\item a single trainable layer directly after $W_1$ is enough to remove the plateau, while a single
trainable layer next to the head is not.
\end{itemize}

\subsection{Setup}\label{app:freeze-setup}

All runs use the problems, student, head, optimiser, metrics and scoring protocol of
Appendix~\ref{app:deep} and one $12\times4$ seed block ($12$ problems $\times$ $4$
initialisations, $48$ runs per arm); only the per-layer step sizes change. There are $1{,}872$ runs, none of
which diverged.

\paragraph{Notation.}
The first hidden matrix is $W_1$ and the deeper matrices are $W_2,\ldots,W_L$;
$\eta_1$ and $\eta_{\mathrm{deep}}$ multiply their polar updates. The width is $m$,
the total number of training steps is $T$, and $\bar L_t=(L_t+L_{t+1})/2$
is the cycle mean of the training loss $L_t$.
For a teacher basis $U$ and a leading input-AGOP basis $B$, both with orthonormal
columns, $\overline{\cos^2}=r_t^{-1}\|U^\top B\|_F^2$ and
$\cos^2_{\min}=\sigma_{\min}(U^\top B)^2$.
Writing $\widehat{\mathbb P}_n$ for the uniform distribution on the training
pairs, the ``offset'' is the square of the mean residual,
$\bigl(\E_{(x,y)\sim\widehat{\mathbb P}_n}[f_W(x)-y]\bigr)^2$;
it is distinct from the MSE
$L_t=\E_{(x,y)\sim\widehat{\mathbb P}_n}[(f_{W_t}(x)-y)^2]$.

\paragraph{Step-size ladder.} At depth $8$ we keep $\eta_1=0.66$ and set
$\eta_{\mathrm{deep}}\in\{0,\allowbreak\,7.8\times10^{-4},\allowbreak\,3.1\times10^{-3},\allowbreak\,0.0125,\allowbreak\,0.02,\allowbreak\,0.05,\allowbreak\,0.2,\allowbreak\,0.66\}$,
from frozen deep layers through the D8 value of Table~\ref{tab:deep-depth} ($\eta_1/845$) to the
same step for every layer. The same ladder is run for the plain network of
Appendix~\ref{app:deep-model} ($\eta_1=0.6$, He initialisation). At depths $6$ and $16$ we run four points: frozen, the
value of Table~\ref{tab:deep-depth}, $\eta_1/50$ and $\eta_1$.

\paragraph{Partial freezing.} At depth $8$ the first layer and a block of $j$ of the seven deep
layers train with $\eta_1=0.66$ and the other deep layers are frozen, for $j\in\{1,2,4,6\}$, with
the block either next to $W_1$ ($W_2,\dots,W_{j+1}$) or next to the head ($W_{9-j},\dots,W_8$).
The end points $j=0$ and $j=7$ are the ladder arms with $\eta_{\mathrm{deep}}=0$ and
$\eta_{\mathrm{deep}}=\eta_1$.

\paragraph{Switching during a run.} At depth $8$ with $T=1500$, the deep layers train with
$\eta_{\mathrm{on}}$ until step $500$ and are frozen afterwards, or are frozen until step $500$ and
train with $\eta_{\mathrm{on}}$ afterwards, for $\eta_{\mathrm{on}}\in\{0.0125,\,0.66\}$. The
controls keep the deep layers frozen, or training with $\eta_{\mathrm{on}}$, for all $1500$ steps.
The first layer always trains with $\eta_1=0.66$. The switching runs follow the same protocol
step for step (with a constant rate they reproduce the ladder runs' losses exactly) and score
two windows with the strict test: steps $200\to500$ before the switch and
$700\to1500$ after it.

\paragraph{Flatness and localisation.} The strict test of Appendix~\ref{app:deep-protocol} allows
the cycle-mean loss to change by up to $2\%$ per hundred steps, about $\pm16\%$ over the window.
We therefore also count runs that pass it with the cycle mean changing by at most $5\%$. To see
where the learned alignment lives, we recompute $\cos^2_{\min}$ at the end of the window with the
deep layers reset to their initial values and report the fraction of the trained network's
$\cos^2_{\min}$ that remains.

\subsection{Results}\label{app:freeze-results}

\begin{figure}[t]
\centering
\includegraphics[width=\linewidth]{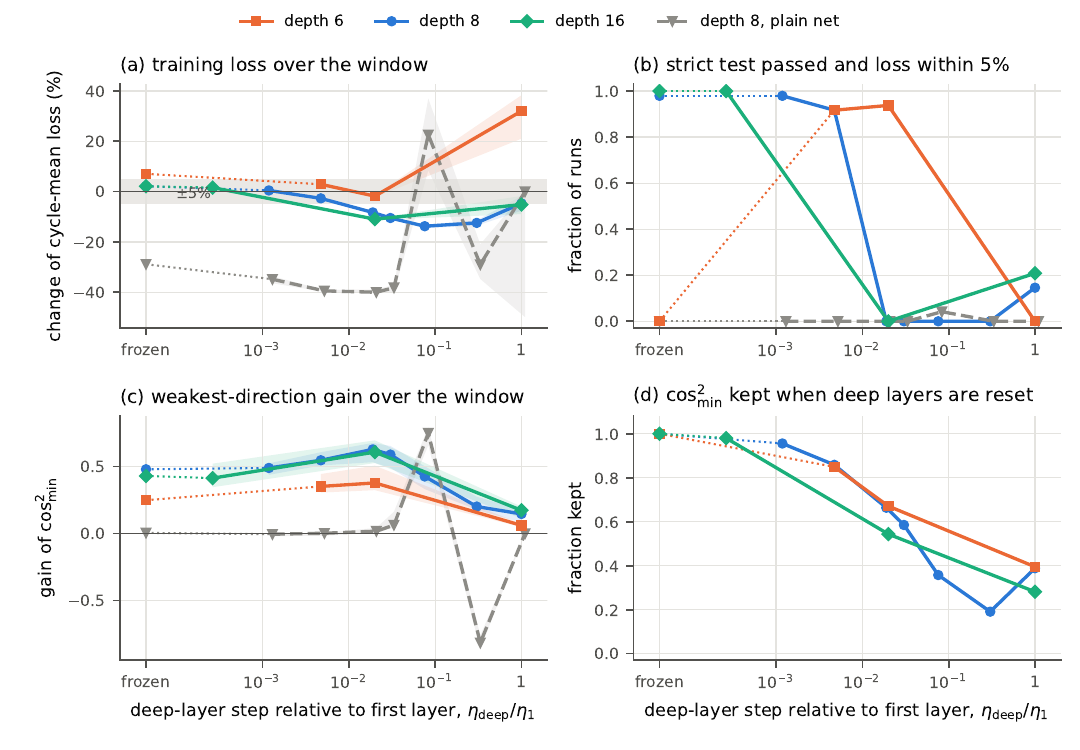}
\caption{\textbf{Slowing the deep layers makes the loss flat; training them makes it fall.}
Residual students of depth $6$, $8$ and $16$ and the depth-$8$ plain network, against the
deep-layer step relative to the first layer's. \textbf{(a)}~Change of the cycle-mean training loss
over the window (grey band: $\pm5\%$). \textbf{(b)}~Fraction of runs that pass the strict test with
the cycle mean within $5\%$. \textbf{(c)}~Window gain of $\cos^2_{\min}$. \textbf{(d)}~Fraction of
the final $\cos^2_{\min}$ that remains when the deep layers are reset to their initialisation.
Medians over problems of the per-problem median over initialisations; shading is the interquartile
range over problems; the frozen point is joined to the others by a dotted segment.}
\label{fig:freeze-dose}
\end{figure}

\paragraph{Step size.} Figure~\ref{fig:freeze-dose} gives the ladders. At depth $8$, with the deep layers frozen, at the D8 step ($\eta_1/845$) or at
$\eta_1/211$, $142$ of $144$ runs pass the strict test and $138$ keep the cycle mean within $5\%$;
the weakest direction gains $+0.48$ to $+0.55$, and resetting the deep layers keeps $86$--$100\%$
of the final $\cos^2_{\min}$, so the alignment sits in $W_1$. From $\eta_1/53$ to $\eta_1/3.3$ the alignment
still rises (gains $+0.20$ to $+0.63$), but the cycle mean falls by $8$--$14\%$ over the window, no
run stays within $5\%$, the held-out refit loss falls by $22$--$43\%$, and the deep layers carry a
growing part of the alignment (only $19$--$66\%$ remains when they are reset). The strict test
still passes most of these runs because of its $\pm16\%$ tolerance. With the same step for every
layer ($\eta_1$) most of the alignment forms before the window; the cycle mean falls by a median
$4.7\%$, $15$ of $48$ runs stay within $5\%$, $19$ pass the strict test and $7$ do both. Depth $16$
behaves like depth $8$: frozen or at the D16 step, all $96$ runs pass and stay within $5\%$; at
$\eta_1/50$ none stays within $5\%$ and the cycle mean falls by $10.9\%$. At depth $6$ the cycle
mean drifts up by $7\%$ when the deep layers are frozen entirely, mostly through the offset part
$\bigl(\E_{\widehat{\mathbb P}_n}[f_W(x)-y]\bigr)^2$, and the loss is flattest at small non-zero steps ($\eta_1/211$ and
$\eta_1/50$: $89$ of $96$ runs within $5\%$); the same step for every layer makes the depth-$6$
loss rise by $32\%$. No step makes the plain network robust across problems: at deep-layer step $0.05$, $7$ of $48$
runs pass the strict test and $2$ of them also stay within $5\%$ (median loss change $+22.6\%$),
and with its deep layers frozen the cycle mean falls by $29\%$. Figure~\ref{fig:freeze-ladder-traj} shows the depth-$8$ trajectories.

\begin{figure}[t]
\centering
\includegraphics[width=\linewidth, trim=0 0 0 30bp, clip]{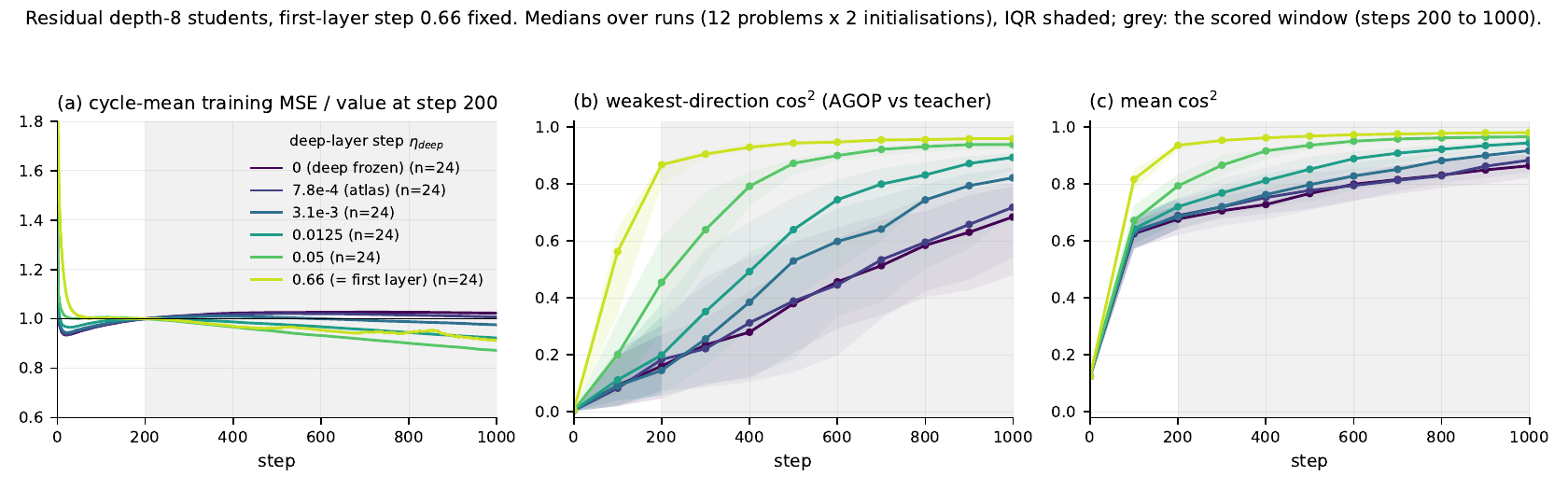}
\caption{\textbf{The deep-layer step ladder at depth $8$ over time.} Residual students, first-layer step
$0.66$ fixed, deep-layer step from $0$ (frozen) to $0.66$ (the same step for every layer); medians over
$24$ runs ($12$ problems $\times$ $2$ initialisations), interquartile range shaded; grey: the scored
window. \textbf{(a)}~Cycle-mean training loss divided by its value at step $200$.
\textbf{(b)}~$\cos^2_{\min}$. \textbf{(c)}~$\overline{\cos^2}$. With the deep layers frozen or slow
(up to $\eta_1/211$) the loss stays flat while the weakest direction aligns; from $\eta_1/53$ on the
loss falls through the window and the alignment forms earlier.}
\label{fig:freeze-ladder-traj}
\end{figure}

\paragraph{Reading Figure~\ref{fig:freeze-switch}.}
The main-text figure shows depth-$8$ students trained for $T=1500$ steps, with the deep
layers switched at step $500$ (dotted line). The first layer always trains with
$\eta_1=0.66$. Colour gives the current state of the deep layers: blue is frozen and
orange is training. Dashed curves change state at step $500$ and are drawn only after
the switch, since before it they coincide with the control they start from.
Panels (a,b) use deep-layer step $0.0125$ while training, and (c,d) use $0.66$.
Panels (a,c) show cycle-mean training loss divided by its last value before the switch;
(b,d) show $\cos^2_{\min}$. Curves are medians over
$48$ runs; the shaded bands mark the two scoring windows, not uncertainty intervals.

\paragraph{Switching.} Figure~\ref{fig:freeze-switch} switches the deep layers
during a run. With $\eta_{\mathrm{on}}=0.0125$ the loss responds at once. Freezing
the deep layers at step $500$ stops the decrease: over steps $700$--$1500$ the cycle mean changes
by $-3.0\%$ and stays within $5\%$ in all $48$ runs, against $-9.7\%$ and none of $48$ in the
control that keeps training them, while $\cos^2_{\min}$ rises by the same amount as in the control
($+0.12$). Unfreezing them at step $500$ does the opposite: the flat loss starts to fall
($-10.8\%$, no run within $5\%$) while the alignment keeps rising ($+0.18$). Only $19$ of the $48$
runs frozen at step $500$ pass the strict test, because by then the features are nearly aligned
and $\overline{\cos^2}$ gains less than the test's threshold of $0.05$ ($+0.04$). With
$\eta_{\mathrm{on}}=\eta_1$ the features are aligned before the switch ($\cos^2_{\min}\approx0.95$).
Freezing then removes the part of the loss oscillation driven by the deep layers' own steps, which
lowers the loss by about
$15\%$ within $100$ steps, after which the loss stays flat ($40$ of $48$ runs within $5\%$) with no
further alignment gain; unfreezing lowers the loss by about $20\%$ within $100$ steps while
$\cos^2_{\min}$ jumps from about $0.45$ to $0.85$.

\begin{figure}[t]
\centering
\includegraphics[width=\linewidth]{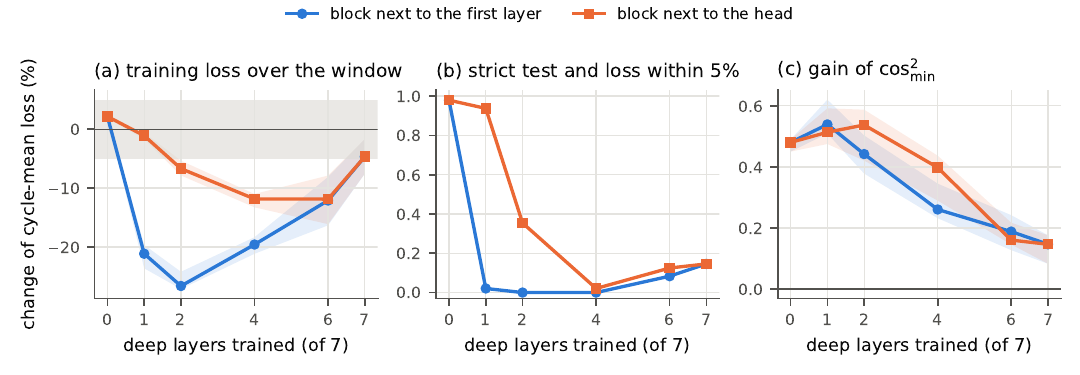}
\caption{\textbf{Which deep layers may train.} Depth $8$; the first layer and $j$ deep layers train
with $\eta_1=0.66$ and the others are frozen; the trained block sits next to $W_1$ (blue) or next
to the head (orange). \textbf{(a)}~Window change of the cycle-mean loss (grey band: $\pm5\%$).
\textbf{(b)}~Fraction of runs passing the strict test with the loss within $5\%$.
\textbf{(c)}~Window gain of $\cos^2_{\min}$. Medians over problems; shading is the interquartile
range.}
\label{fig:freeze-partial}
\end{figure}

\paragraph{Partial freezing.} Which deep layers train matters more than how many
(Figure~\ref{fig:freeze-partial}). Training only the deep layer next
to the head, with the first layer's step, keeps the plateau: $47$ of $48$ runs pass the strict test,
$45$ stay within $5\%$, and $\cos^2_{\min}$ gains $+0.51$. Training only the deep layer next to
$W_1$ removes it: the cycle mean falls by $21\%$ and $1$ of $48$ runs stays within $5\%$, even
though the alignment still sits mostly in $W_1$ ($90\%$ remains when the deep layers are reset).
Larger trained blocks remove the plateau from either end: for $j\ge4$ at most $7$ of $48$ runs stay
within $5\%$.

\subsection{Mechanism}\label{app:freeze-mechanism}

With the deep layers frozen the student is $x\mapsto g\bigl(\mathrm{ReLU}(W_1^\top x)\bigr)$ for a
fixed map $g$ (the residual stack at its initialisation followed by the frozen head), and only
$W_1$ trains, with the large polar step. This is the one-hidden-layer setting of the main text with
a fixed nonlinear readout in place of a fixed linear one, and it behaves the same way. The loss
locks into the period-$2$ cycle, whose high branch is $83$--$92\%$ offset (a shift of the mean
prediction); the cycle mean stays flat; and the polar step keeps rotating the dominant input
directions of $W_1$ towards the teacher subspace. When the deep layers also train, they keep
fitting the targets with whatever features $W_1$ currently provides, so the loss decreases, and
they come to carry part of the alignment (Figure~\ref{fig:freeze-dose}d). The partial-freezing arms
place the effect: the layer directly after $W_1$ is enough to turn the plateau into a steady
decrease, even while the alignment stays in $W_1$, and the layer next to the head is not. The
plateau with continued feature learning is therefore a property of the first layer trained against
a fixed, or slowly moving, downstream map; freezing the downstream layers isolates it, and training
them, starting with the layer right after $W_1$, hides it.

\paragraph{Scope.} All of this is for residual MLP students of a Gaussian teacher, trained full
batch with the exact polar step and a frozen head. At depth $6$ the flattest loss needs slow but
non-zero deep steps rather than frozen ones. The plain network does not show the effect at the
first-layer step of the plain arm, even with its deep layers frozen; we did not tune
that step for it. We have not established the effect for other architectures or data.

\section{A freezing principle for the plateau with feature learning}\label{app:freeze-principle}

Appendix~\ref{app:freeze} showed that, in deep residual students, freezing the layers after $W_1$
switches the plateau with continued feature learning on and training them switches it off. Here we
state the observation as a working principle and test it in three more ways: at other depths and widths, after other training
schemes, and under Muon as it is used in practice (Newton--Schulz and momentum).

\paragraph{Notation.}
For the deep-network experiments, $W_1$ is the first hidden matrix, $L$ is depth,
$m$ is width, $\eta_1$ is the coefficient of the first-layer polar update, and
$T_0$ is the number of pretraining steps. The fixed downstream map $g$ includes
all later layers and the head. The teacher basis $U$ and the leading input-AGOP
basis $B$ have orthonormal columns; the reported alignment metrics are
$\overline{\cos^2}=r_t^{-1}\|U^\top B\|_F^2$ and
$\cos^2_{\min}=\sigma_{\min}(U^\top B)^2$.
Cycle means, scoring windows and held-out head refits are as defined in
Appendix~\ref{app:deep-protocol}. The small-head experiment in
Section~\ref{app:fp-ns} uses the separate mean-field notation stated there.

\paragraph{Working principle.} Write the student as $x\mapsto g\bigl(\mathrm{ReLU}(W_1^\top
x)\bigr)$, where $g$ collects every layer after the first and the frozen head.
\begin{enumerate}
\item[(i)] \emph{Isolation.} With $g$ held fixed, normalised (polar) steps on $W_1$ keep improving the
input-AGOP alignment with the teacher subspace, while the training loss stays on the floor set
by the $W_1$ step: the period-$2$ cycle of Appendix~\ref{app:deep}, whose cycle mean stays flat.
\item[(ii)] \emph{Conversion.} Letting $g$ train, with a step of about $\eta_1/50$ or more, turns that
feature progress into loss progress: the loss falls, and $g$ comes to carry part of the alignment. The
layer directly after $W_1$ matters most (Figure~\ref{fig:freeze-partial}).

\end{enumerate}
The plateau is thus the loss floor of the normalised $W_1$ step, which better features do not lower
while $g$ is fixed; training $g$, or training $W_1$ alone by gradient descent or Adam, lowers the loss. The plateau is not needed for the feature learning itself (Section~\ref{app:fp-retrofit}).

\subsection{Depth and width}\label{app:fp-depth}

The depth-$8$ and depth-$16$ ladders of Appendix~\ref{app:freeze} extend to depths $4$ and $24$
($24$ runs per arm, first-layer steps $0.66$ and $0.74$). With the
deep layers frozen, $24$ of $24$ and $23$ of $24$ runs pass the strict test with the cycle mean
within $5\%$ ($+2\%$ in both). Over the window $\cos^2_{\min}$ rises from $0.18$ to $0.71$ at
depth $4$ and from $0.22$ to $0.60$ at depth $24$; with the deep layers frozen, this change is due to
$W_1$ alone. With the deep layers at
$\eta_1/50$ the alignment still rises ($0.26\to0.89$ and $0.26\to0.88$), but the cycle mean falls by
$7\%$ and $10\%$, $2$ and $0$ of $24$ runs stay within $5\%$, and the held-out refit loss falls by
$23\%$ and $26\%$ against $6\%$ and $4\%$ with the deep layers frozen. Depth $6$ remains the
exception noted in Appendix~\ref{app:freeze}, with the flattest loss at small non-zero deep steps.

Width is less forgiving at the depth-$8$ first-layer step, which we did not retune. At width $64$
and $\eta_1=0.66$ the frozen student still learns ($\cos^2_{\min}$ $0.04\to0.87$, $24$ of $24$
strict), but the cycle mean falls by $10\%$, so no run stays within $5\%$. At width $64$ with
$\eta_1=1$, and at width $128$ with $\eta_1=0.66$, the weakest teacher direction is never learned
($\cos^2_{\min}\le0.01$). The plain (non-residual) network fails at every first-layer step from
$0.05$ to $1$, with the cycle mean falling by $27$--$41\%$.

\subsection{Training only $W_1$ after other schemes}\label{app:fp-retrofit}

\begin{figure}[t]
\centering
\includegraphics[width=\linewidth]{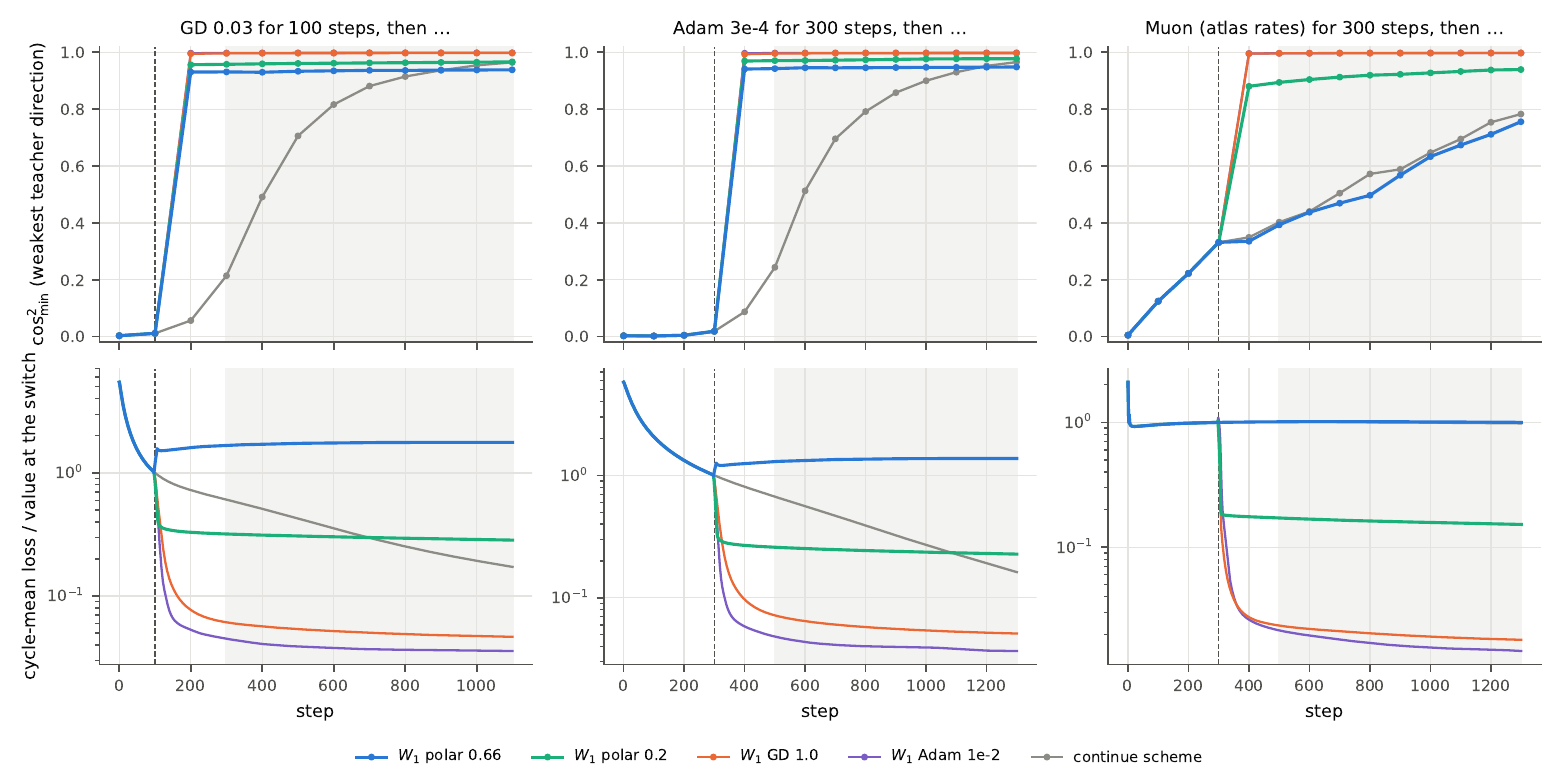}
\caption{\textbf{Training only $W_1$ after partial pretraining.} The depth-$8$ residual student is
trained with GD ($0.03$, $100$ steps), Adam ($3\times10^{-4}$, $300$ steps) or Muon at the rates of
Appendix~\ref{app:deep} ($300$ steps). At the switch (dashed) the deep layers are frozen and only
$W_1$ trains, with a polar step of $0.66$ or $0.2$, GD or Adam, or the scheme continues on every
layer. \textbf{Top}~$\cos^2_{\min}$. \textbf{Bottom}~cycle-mean training loss divided by its value
at the switch. Medians over problems of the per-problem median over initialisations; grey: the
scored window.}
\label{fig:fp-retrofit}
\end{figure}

We train the whole depth-$8$ student with a scheme $S$ for $T_0$ steps and then train only $W_1$,
with the deep layers frozen, for $1000$ steps (Figure~\ref{fig:fp-retrofit}). $S$ is GD or Adam at a fast or a slow rate (tuned on two development
problems), or Muon with one step for every layer or with the two rates of
Appendix~\ref{app:deep}; $T_0\in\{100,300,1000\}$.

\paragraph{The features.} After $100$ or $300$ steps of slow GD or Adam, $\cos^2_{\min}$ is
$0.00$--$0.21$ at the switch. Two hundred steps later, $15$ of the $16$ combinations of these four
starts with the four $W_1$ optimisers (polar $0.66$, polar $0.2$, GD, Adam) have reached
$0.93$--$1.00$; the exception is polar $0.66$ after $100$ steps of Adam ($0.45$, and $0.73$ at the
end of the window). Continuing $S$ on the whole network reaches only $0.02$--$0.71$ in the same
time, and it takes about $700$--$1000$ steps to catch up. Training $W_1$ against a fixed downstream
map is therefore a fast feature learner in its own right; the comparison also changes the optimiser
acting on $W_1$, so it does not isolate freezing as the cause of the speed-up.

\paragraph{The loss.} The optimiser on $W_1$ decides what the loss does. A polar step of $0.66$
moves the loss to the floor set by the step and holds it there: after slow GD or Adam the cycle
mean first rises by a factor $1.3$--$2.3$ (after $100$ steps of Adam it falls by about $20\%$
instead) and then changes by $+1\%$ to $+11\%$ over the window; after fast GD or Adam, when the loss is
already small, it rises by a factor $7$ to $2{,}000$. GD and Adam on $W_1$ keep lowering the loss
($21$--$30\%$ over the window after $100$ or $300$ steps of slow GD or Adam; W$_1$-only Adam after
$1000$ steps lowers it by $10$--$11\%$).

\paragraph{When the plateau carries learning.} A long plateau during which the features keep
improving appears only when the deep layers are still close to their initialisation: after Muon
at the two rates of Appendix~\ref{app:deep} for $T_0=100$ or $300$ ($24$ and $23$ of $24$ runs
strict and within $5\%$, $\cos^2_{\min}$ $0.30\to0.66$ and $0.39\to0.76$ over the window), or with
the deep layers frozen from initialisation ($23$ of $24$, $0.17\to0.60$). After $100$ steps of GD
or Adam the learning is mostly finished before the scored window begins and the loss is then flat,
with little left to learn; the exception is polar $0.66$ after $100$ steps of Adam, where
$\cos^2_{\min}$ still rises from $0.45$ to $0.73$ while the cycle mean rises by $10.5\%$ and no run
stays within $5\%$. The one diverging branch is $W_1$-only GD after $1000$ steps of Muon with
one step for every layer ($12$ of $24$ runs).

\subsection{Newton--Schulz and momentum: two regimes}\label{app:fp-ns}

\begin{figure}[t]
\centering
\includegraphics[width=\linewidth]{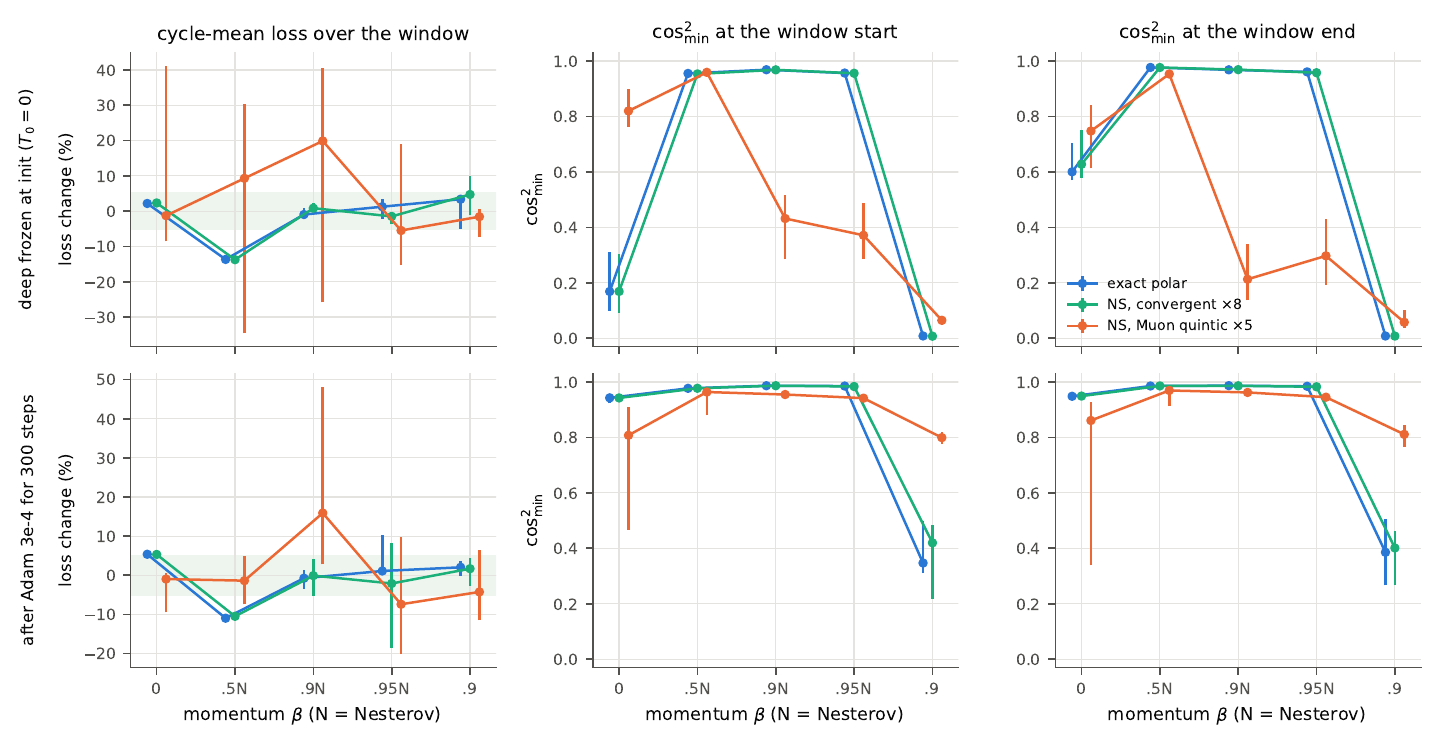}
\caption{\textbf{Muon as used in practice, depth $8$.} Only $W_1$ trains (step $0.66$, deep layers
frozen), from initialisation (top) or after $300$ steps of Adam (bottom). Colours: the exact polar
factor, a convergent quintic iteration ($8$ steps; \eqref{eq:fp-convergent-quintic}) and Muon's quintic ($5$ steps);
$x$-axis: momentum (EMA buffer as in Muon; N: Nesterov). \textbf{Left}~cycle-mean loss change over
the window. \textbf{Middle, right}~$\cos^2_{\min}$ at the start and at the end of the window.
Medians over problems; bars: interquartile range.}
\label{fig:fp-ns}
\end{figure}

The convergent comparison uses $X_0=G/(\|G\|_F+10^{-7})$ and
eight iterations of the quintic map
\begin{equation}
X_{j+1}=2X_j-\frac32(X_jX_j^\top)X_j
+\frac12(X_jX_j^\top)^2X_j.
\label{eq:fp-convergent-quintic}
\end{equation}
Tall matrices are transposed before iteration, and the result is
transposed back. This specifies the coefficients $(2,-1.5,0.5)$ of
the convergent comparison, distinct from Muon's quintic approximation.

Whether the plateau survives Muon's approximations depends on what holds the loss on its floor.

\paragraph{Small-head notation.}
For this experiment only, $d$ is the input dimension, $r$ the teacher rank,
$D=d-r$, and $m$ the number of hidden units. The balanced ReLU teacher and
student are those of Appendix~\ref{mfl:sec:setting}: the fixed head coefficient
is $\alpha/m$, the intercept is refitted, and
$\kappa=\sigma_0/(\bar\eta\sqrt d)$ for initial columns
$N(0,\sigma_0^2I_d/d)$ and polar step $\eta=\bar\eta\sqrt m$.
Here $\alpha$ denotes head scale, not the residual-branch scale of the deep
network. For orthonormal teacher directions $u_1,\ldots,u_r$, let
$\mathcal U=\mathrm{span}(u_1,\ldots,u_r)$ be their subspace,
with projector $P_U$. If $P_r$ projects onto the leading
$r$ AGOP eigenvectors, $A_r=r^{-1}\Tr(P_UP_r)$.
Write $\bar u=r^{-1/2}\sum_{i=1}^r u_i$;
$\lambda_c$ and $\lambda_\perp$ are the mean-field AGOP eigenvalues on the
contrast space $\mathcal U\cap\bar u^\perp$ and the orthogonal complement $\mathcal U^\perp$,
respectively (the common head factor cancels in their ratio).
Neuron teacher energy is $\|P_Uw_j\|_2^2/\|w_j\|_2^2$; its median is over neurons.
In the momentum experiments, $\beta$ is the exponential-moving-average
coefficient and ``Nesterov'' specifies the look-ahead buffer combination.

\paragraph{Small fixed head: the plateau survives.} In the one-hidden-layer student with a small
fixed head (mean-field scale, a balanced two-direction ReLU teacher, population loss;
$d\in\{64,402\}$, $3$ seeds each), the loss plateau is produced by the head scale, and the
orthogonaliser mainly rescales how fast each direction moves. Muon's Newton--Schulz quintic keeps
the plateau with feature learning at the level of the AGOP and the teacher subspace: the AGOP aligns with the teacher subspace ($A_r\ge0.9$) within
$120$--$280$ steps for the exact polar factor, for the quintic with $2$ or $5$ steps and for $7$
steps of a convergent quintic. The speed of the loss is set by the iteration's gain on the mean
direction ($0.70$ for an odd and $1.1$ for an even number of Muon steps), and the specialisation of
individual neurons by its gain on the remaining directions, so that more steps are needed at larger
width. In the population, momentum leaves the alignment and the plateau unchanged and delays
specialisation slightly; with minibatches it shortens the plateau and speeds up specialisation. Figures~\ref{fig:ns-vs-polar}, \ref{fig:ns-traj} and~\ref{fig:ns-momentum} show these
results.

\begin{figure}[t]
\centering
\includegraphics[width=\linewidth]{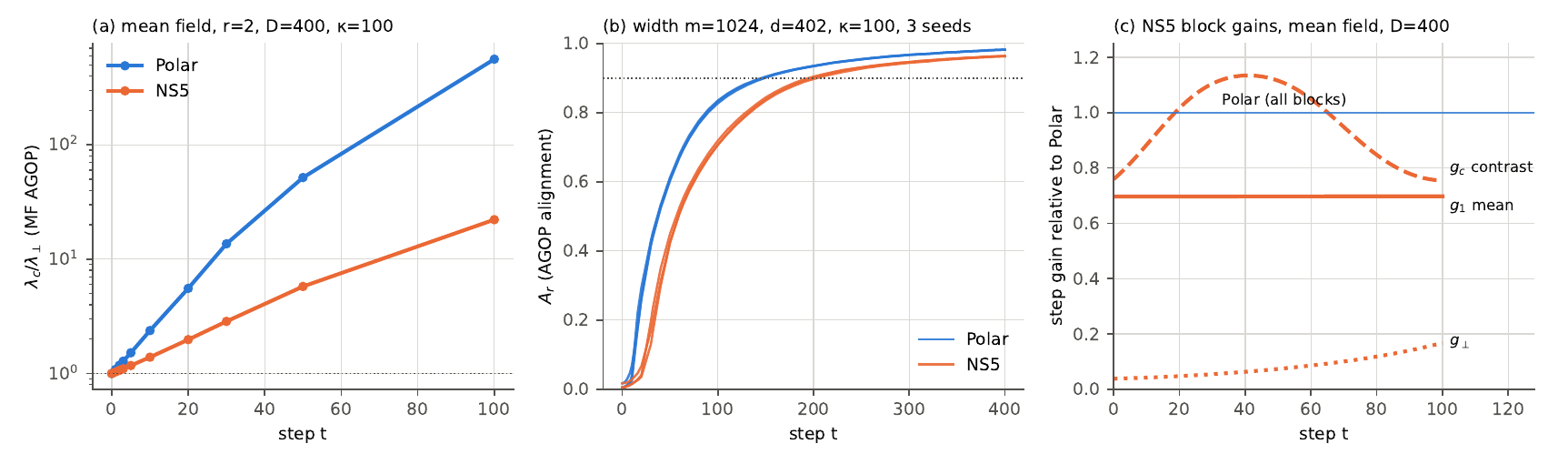}
\caption{\textbf{Newton--Schulz (NS5) against exact Polar in the small-head setting} ($r=2$,
$\kappa=100$). \textbf{(a)}~Mean-field AGOP ratio $\lambda_c/\lambda_\perp$ at $D=400$ (computed):
both optimisers put the teacher subspace on top from the first step, and the gap grows more slowly
under NS5. \textbf{(b)}~AGOP alignment $A_r$ at $d=402$, $m=1024$, $\alpha=10^{-3}$, $3$ seeds.
\textbf{(c)}~NS5 step gains on the mean, contrast and $\perp$ blocks relative to Polar, which has gain $1$
on every block.}
\label{fig:ns-vs-polar}
\end{figure}

\begin{figure}[t]
\centering
\includegraphics[width=\linewidth]{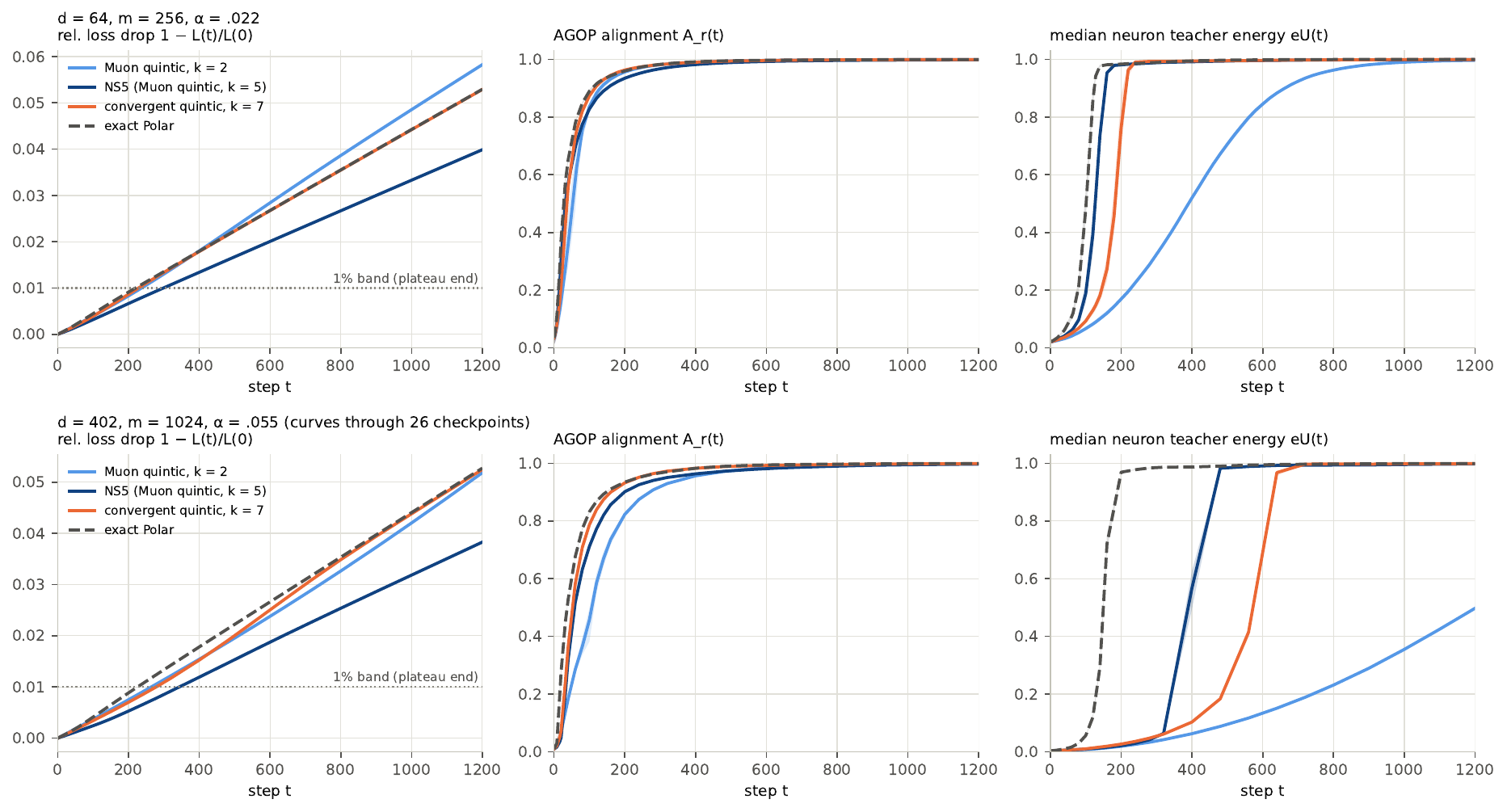}
\caption{\textbf{Training trajectories for four orthogonalisers} (small-head setting, population loss,
no momentum; $d=64$ and $d=402$): Muon's quintic with $2$ and $5$ steps, $7$ steps of a convergent
quintic, and exact Polar. Shown are $1-L(t)/L(0)$, the AGOP alignment $A_r(t)$ and the median neuron
teacher energy; lines are means over $3$ seeds and bands the seed range. The AGOP aligns early
($A_r\ge0.9$ by step $120$--$280$), while the neurons specialise at very different times.}
\label{fig:ns-traj}
\end{figure}

\begin{figure}[t]
\centering
\includegraphics[width=\linewidth]{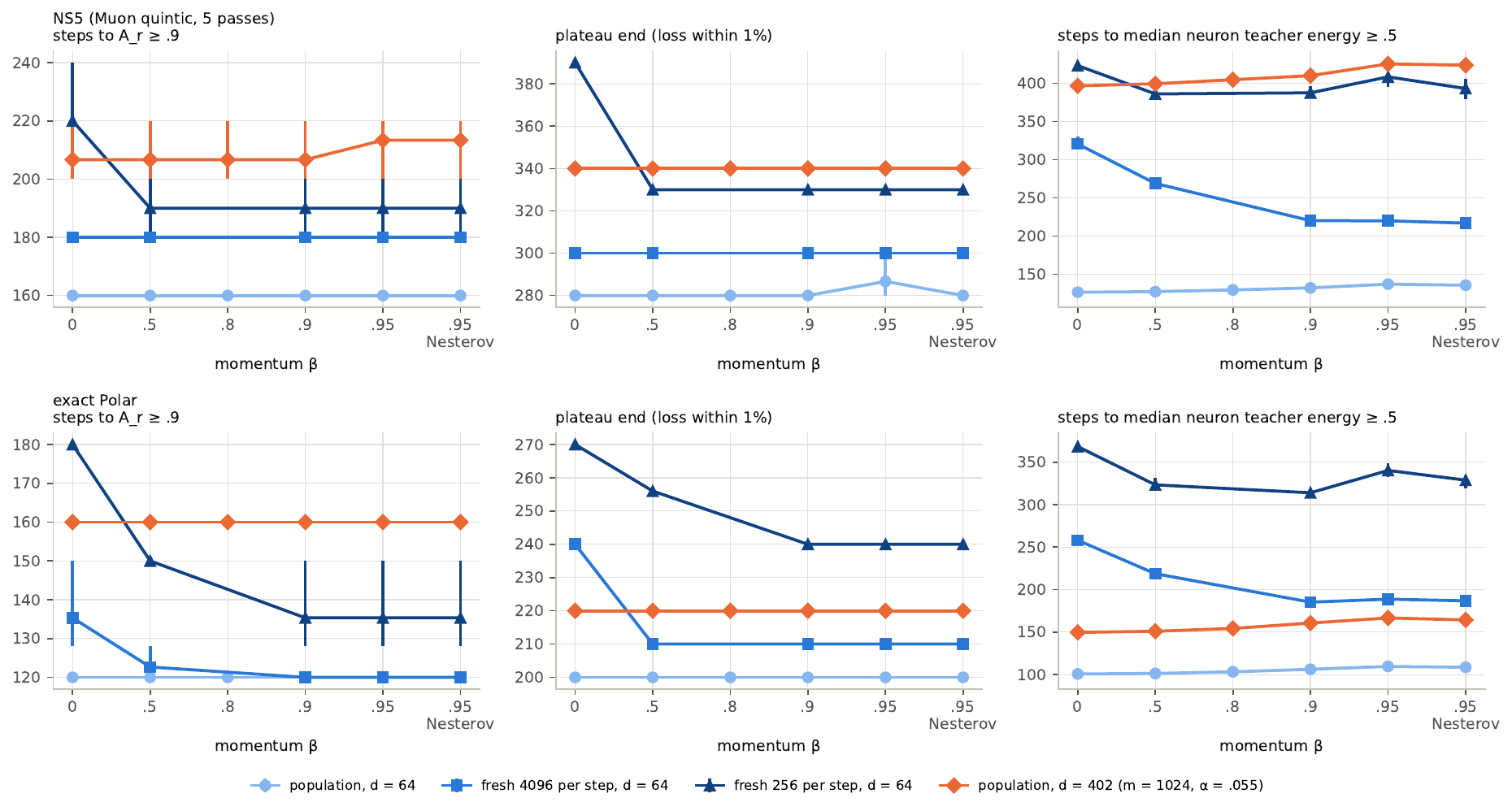}
\caption{\textbf{Momentum for NS5 and exact Polar} (small-head setting; $3$ seeds, bars show the seed
range). Rows: NS5 and Polar. Columns: time to $A_r\ge0.9$, end of the loss plateau, and time to a
median neuron teacher energy of $0.5$, against the momentum $\beta\in\{0,0.5,0.8,0.9,0.95\}$ and
$0.95$ with Nesterov; population loss at $d=64$ and $d=402$, and fresh minibatches of $4096$ or $256$
samples at $d=64$. In the population, momentum leaves the alignment and the plateau unchanged and
delays specialisation slightly; with minibatch noise it shortens the plateau and speeds up
specialisation.}
\label{fig:ns-momentum}
\end{figure}

\paragraph{Large polar step: the plateau needs an accurate orthogonaliser.} In the deep student
the floor is the period-$2$ loss oscillation of the large polar step on $W_1$, and anything that
changes that step changes the oscillation (Figure~\ref{fig:fp-ns}). From frozen deep layers, a
convergent Newton--Schulz iteration, which sits within $2\times10^{-9}$ of the exact polar factor,
reproduces it: $23$ of $24$ runs strict and within $5\%$ for both, $\cos^2_{\min}$ $0.17\to0.60$
(exact) and $0.17\to0.63$ (Newton--Schulz). Muon's quintic stays $0.18$--$0.21$ from the polar
factor, breaks the clean cycle, and no run passes. Nesterov momentum ($0.5$--$0.95$) learns the
features before the plateau forms: $\cos^2_{\min}$ is already $0.95$--$0.97$ at the start of the
window, after which the loss either keeps falling ($-14\%$ at $0.5$) or stays flat with nothing
left to learn ($17$--$20$ of $24$ runs flat at $0.9$ and $0.95$). Plain EMA momentum at $0.9$ never
aligns the weakest direction ($0.01$), and Muon's own update, the quintic with Nesterov momentum
$0.95$, stalls ($0.37\to0.30$).

The two regimes agree with the principle: in both, the features are learned by the orthogonalised
first-layer step against a fixed readout. They differ in what keeps the loss flat, and that
determines how much approximation the plateau tolerates.

\subsection{Other architectures and scope}\label{app:fp-scope}

We also ran a transformer on sequences of Gaussian tokens with a planted rank-$4$ teacher, and two
vision transformers and a CNN on the $8\times8$ optical-digit images of \citet{alpaydin1998optical}
(Table~\ref{tab:fp-arch}). For the transformer, alignment is measured against the planted rank-$4$
teacher subspace; for the digit-image models, against a fixed rank-$8$ input-AGOP reference subspace
learned beforehand. For those digit models, the principal-angle definitions above
use the rank-$8$ reference basis in place of $U$ and $r_t=8$.
After removing LayerNorm (LN) and, when $T_0>0$, partial AdamW pretraining, only the first layer is
trained with exact polar steps; one setting per model was chosen on two development seeds and run on
$10$ fresh initialisations. The loss sits on a clean period-$2$ cycle with a flat mean in $8$--$10$ of $10$ runs for the
transformer, the $4$-patch vision transformer and the CNN, and the centred error is flat as well.
Only the transformer keeps learning on that cycle, and only weakly (mean $\cos^2$ $+0.009$ over the
window, $2$ of $10$ runs pass the joint criterion defined below); Muon's quintic on the first layer removes the clean
cycle in all three models and Nesterov momentum removes it in the transformer, and training every
matrix removes the plateau. We therefore claim the principle for
residual MLP students of a Gaussian teacher (depths $4$ to $24$), for the one-hidden-layer
small-head setting, and in weak form for the transformer, and not for plain MLPs, for wider MLPs at an
untuned first-layer step, or for the vision transformers and the CNN.

\begin{table}[h]
\centering\small
\caption{Settings of the architecture runs. Every run minimises the MSE through a frozen head with
orthonormal rows, full batch. Stage~1 trains all non-head parameters with AdamW (learning rate
$10^{-3}$, zero weight decay) for $T_0$ steps and is omitted when $T_0=0$. Stage~2 trains only
the first-layer (input-embedding) weight by $W\leftarrow W-\eta\Polar(G)$,
where $G=\nabla_W L_{\mathrm{train}}$, for $1600$ steps,
scored on steps $200$--$1600$. The update coefficient is
$\eta=\delta\|W_{\mathrm{start}}\|_F/\sqrt M$, where $W_{\mathrm{start}}$ is the stage-2 initial
weight and $M=\min(n_{\mathrm{rows}},n_{\mathrm{cols}})$ after flattening any convolutional
kernel to a matrix: $M=32,16,64,9$ for the four rows below, respectively.
Transformer data: sequences of $8$ Gaussian tokens in $\R^{32}$; the teacher applies one attention
layer and a ReLU layer of width $16$ to the tokens' projections onto a rank-$4$ subspace, averages over
tokens and outputs $4$ standardised targets; $1024$ training sequences, and ten fresh planted problems
in the confirmation runs. Digit data: $500$ training and $500$ test images, one-hot targets.}
\label{tab:fp-arch}
\begin{tabular}{@{}llllrr@{}}
\toprule
model & data & architecture & LN removed & $T_0$ & $\delta$\\
\midrule
transformer & Gaussian tokens & 1 block, width 64, 4 heads & output & 0 & 1.0\\
ViT, four $4\times4$ patches & digits & 1 block, width 64, 4 heads & all & 500 & 0.03\\
ViT, one $8\times8$ patch & digits & 1 block, width 64, 4 heads & all & 500 & 1.0\\
CNN & digits & \shortstack[l]{3 conv.\ layers,\\channels 16/32/64} & output & 200 & 0.1\\
\bottomrule
\end{tabular}
\end{table}

\paragraph{Architecture-specific success criterion.}\label{app:fp-criterion}
These runs use a separate joint test from the strict residual-network
test of Appendix~\ref{app:deep-protocol}. Time is measured from the start
of stage~2, and the scoring window is $s=200$ to $e=1600$, inclusive.
For a model with $K$ outputs and training residuals $r_i(t)=f_t(x_i)-y_i\in\mathbb R^K$, let
$L_t=(nK)^{-1}\sum_i\|r_i(t)\|_2^2$ and
$E_t=L_t-K^{-1}\|n^{-1}\sum_i r_i(t)\|_2^2$ be the training MSE and
centred training error. Define $d_t=L_{t+1}-L_t$ and the adjacent-pair
means $C_t=(L_t+L_{t+1})/2$ and $H_t=(E_t+E_{t+1})/2$,
for $s\le t<e$. The loss-cycle signature requires
\begin{equation}
\rho_2=-\operatorname{Corr}_{t=s}^{e-2}(d_t,d_{t+1})\ge0.99,
\qquad
R_2=\operatorname{median}_{t=s}^{e-2}
\frac{|L_{t+2}-L_t|}{\max\{|d_t|,\epsilon\}}\le0.01,
\label{eq:fp-arch-cycle}
\end{equation}
and a normalized phase gap
\begin{equation}
\frac{\operatorname{median}_{s\le t<e}|d_t|}
     {\max\{\operatorname{median}_{s\le t<e}C_t,10^{-30}\}}
\ge0.01,
\qquad
\epsilon=10^{-15}\max\{\operatorname{median}_{s\le t<e}|d_t|,1\}.
\label{eq:fp-arch-gap}
\end{equation}
We set $\rho_2=0$ if either increment sequence has zero variance.
Flat cycle-mean loss and flat centred error mean, respectively,
\begin{equation}
\left|\frac{C_{e-1}-C_s}{\max\{C_s,10^{-30}\}}\right|\le0.05,
\qquad
\left|\frac{H_{e-1}-H_s}{\max\{H_s,10^{-30}\}}\right|\le0.15.
\label{eq:fp-arch-flatness}
\end{equation}
These are endpoint changes over the complete window, with no division
by its duration. Let $B_t$ have orthonormal columns spanning the leading
rank-$r_t$ input-AGOP eigenspace, and let $U$ be the orthonormal planted
teacher basis for the transformer
($r_t=4$) or the fixed reference basis for the digit models ($r_t=8$).
The alignment measures are
$a_t=\|U^\top B_t\|_F^2/r_t$ and
$b_t=\sigma_{\min}(U^\top B_t)^2$.
For either $q=a$ or $q=b$, define the even and odd gains by
$g_q^{\rm even}=q_e-q_s$ and
$g_q^{\rm odd}=q_{e-1}-q_{s+1}$, with
$g_q=(g_q^{\rm even}+g_q^{\rm odd})/2$.
A run passes all criteria if it has the loss-cycle signature, both
flatness bounds, $g_a\ge0.05$, $g_b>0$, and
$g_a^{\rm even},g_a^{\rm odd},g_b^{\rm even},g_b^{\rm odd}>0$.

\paragraph{Seeds and fixed reference.}
The development initialisation seeds are $1,2$ for the transformer and $26090971,5807101$
for the digit models; confirmation uses $9350101$--$9350110$, paired with transformer problem
seeds $12$--$21$. Digit splits use seeds $26090961$, $26090962$ and $26090963$ for the
test split, the $128$-image diagnostic split and the training subset, respectively.
The fixed digit reference is stored as a $64\times8$ basis in
\texttt{ref\_basis\_v4.b64}; the decoded basis has SHA-256 prefix
\texttt{cefa0672cd516a63} and is reused unchanged across runs.

\clearpage
\section{The mean-field theorem: complete proof}\label{app:mf}
\providecommand{\Prob}{\mathbb{P}}
\providecommand{\relu}{\operatorname{ReLU}}
\providecommand{\tr}{\operatorname{tr}}
\providecommand{\spn}{\operatorname{span}}
\newcommand{\mfsgn}{\operatorname{sign}}
\providecommand{\Var}{\operatorname{Var}}
\providecommand{\Cov}{\operatorname{Cov}}
\providecommand{\one}{\mathbf 1}
\providecommand{\ind}[1]{\mathbf 1\{#1\}}
\providecommand{\norm}[1]{\lVert #1\rVert}
\providecommand{\opnorm}[1]{\lVert #1\rVert_{\mathrm{op}}}
\providecommand{\fro}[1]{\lVert #1\rVert_{F}}
\providecommand{\abs}[1]{\lvert #1\rvert}
\providecommand{\ip}[2]{\langle #1,#2\rangle}
\providecommand{\eb}{\bar e}
\providecommand{\Pone}{P_1}
\providecommand{\Pc}{P_c}
\providecommand{\cz}{c_0}
\providecommand{\kk}{\mathsf k}
\providecommand{\etab}{\bar\eta}
\providecommand{\hatw}{\hat w}
\providecommand{\Dp}{D}
\providecommand{\AGOP}{M}
\providecommand{\Ar}{A_r}
\providecommand{\tF}{t_F}
\providecommand{\tP}{t_P}
\providecommand{\astar}{\alpha^{*}}
\makeatletter
\providecommand{\step}{}
\renewcommand{\step}[1]{%
  \if@inlabel
    \leavevmode
  \else
    \par\addvspace{4pt plus 1pt}%
    \ifdefined\Needspace\Needspace{3\baselineskip}\fi
    \noindent
  \fi
  \emph{#1}\enspace\ignorespaces}
\@ifundefined{secsummary}{%
  \newenvironment{secsummary}{%
    \par\begingroup\normalfont\setlength{\parindent}{0pt}%
    \ignorespaces}{\par\endgroup}}{%
  \renewenvironment{secsummary}{%
    \par\begingroup\normalfont\setlength{\parindent}{0pt}%
    \ignorespaces}{\par\endgroup}}
\makeatother
\newcommand{\mfproofpresentation}{%
  \let\mfbaseproof\proof
  \def\proof{\ifdefined\Needspace\Needspace{5\baselineskip}\fi\mfbaseproof}%
  \setlength{\jot}{5pt}%
  \setlist[enumerate]{leftmargin=2em,labelsep=.5em,itemsep=3pt,topsep=4pt,parsep=2pt}%
  \setlist[itemize]{leftmargin=1.5em,labelsep=.5em,itemsep=3pt,topsep=4pt,parsep=2pt}}

\noindent\emph{Notation.} This appendix is self-contained and uses the notation of the proof: the input
dimension is $d$ ($p$ in the main text), the width is $m$ ($r_s$), the teacher subspace is $U$ ($V$ in
the main text), and $V$ denotes the target variance $\operatorname{Var}(Y)$ ($\sigma_Y^2$ in
Section~\ref{sec:1NN-mf}). Theorem~\ref{thm:mf} is Theorem~\ref{mfl:intro-thm:main} below.
For a vector $z$, $\abs z$ is its Euclidean norm; for a deterministic matrix $A$, $\norm A$ is its operator norm.
Expectations are written $\E[\,\cdot\,]$: a subscript $X$ denotes the Gaussian input,
and a subscript $\nu$ or $w\sim\nu$ denotes the particle law; omitted subscripts use the locally specified law.
Repeated particles $w,w'\sim\nu$ are independent unless stated otherwise.
For a random scalar or vector $Z$, $\norm Z_{L^2(\nu)}=(\E_\nu[\abs Z^2])^{1/2}$; the reduced dynamics abbreviates this to $\norm Z$.
We write $\ind{B}$ for the indicator of an event $B$, $a_+:=\max(a,0)$ and $a_-:=\max(-a,0)$.
Both $\log$ and $\ln$ denote the natural logarithm. The Euler gamma function is
$\Gamma(a):=\int_0^\infty t^{a-1}e^{-t}\,dt$ for $a>0$.

\noindent\emph{Computer-assisted step.} Region N uses parameter choices obtained by computational
search and recorded here as exact rational numbers. These choices are verified against the analytically
derived certificate inequalities using outward-rounded interval arithmetic
(Section~\ref{mfl:mff-sec:regionN}). The exact inputs and enclosure formulas below specify a finite
calculation that can be checked independently of the search procedure.

\begingroup
\mfproofpresentation
\let\section\subsection
\let\subsection\subsubsection
\let\subsubsection\paragraph
\section{Main result}\label{mfl:sec:intro}
We consider a two-layer ReLU network trained by the literal Polar (Muon) update from an ordinary Gaussian initialization, in
the mean-field limit, on a balanced ReLU teacher with $r\ge2$ orthonormal directions. On a single trajectory the loss stays
flat to within $1\%$ over a window of $\Theta(\kappa)$ updates. Inside that window, the top-$r$ eigenspace of the network's
average gradient outer product (AGOP) becomes exactly the teacher subspace, whereas initially the AGOP is isotropic. Later, the
loss drops by a fixed fraction of the target variance. Section~\ref{mfl:sec:intro} states the result (Theorem~\ref{mfl:intro-thm:main});
Sections~\ref{mfl:sec:pc}--\ref{mfl:sec:main} contain the complete proof.

\subsection{Setting}\label{mfl:sec:setting}
\paragraph{Data and teacher.} $X\sim N(0,I_d)$. Let $u_1,\dots,u_r\in\R^d$ be orthonormal ($r\ge2$), $U:=\spn\{u_i\}$,
$P_U$ its orthogonal projector and $\Dp:=d-r$. Without loss of generality $u_i=e_i$. The teacher is the balanced ReLU sum
\begin{equation}\label{mfl:eq:teacher}
Y=\gamma\sum_{i=1}^r\relu(u_i^\top X),\qquad \gamma:=\sqrt{\frac{2\pi}{r(r+\pi-1)}},
\end{equation}
normalized so that $\E[Y^2]=1$. Its variance is
\begin{equation}\label{mfl:read-l1-variance}
V:=\Var Y=\gamma^2r\kk_1=(\pi-1)/(r+\pi-1),
\end{equation}
where $\kk_1:=(\pi-1)/(2\pi)$.

We write $\eb:=r^{-1/2}\sum_iu_i$ for the mean teacher direction, $\Pone:=\eb\eb^\top$, $\Pc:=P_U-\Pone$ (the
\emph{contrast} projector) and $\cz:=\gamma/(2\pi)$. For unit vectors $a,b$ with $a^\top b=\rho$,
\begin{equation}\label{mfl:read-l1-kernel}
\begin{aligned}
\Cov\big(\relu(a^\top X),\relu(b^\top X)\big)
&=\kk(\rho):=\frac{\sqrt{1-\rho^2}+\rho(\pi-\arccos\rho)-1}{2\pi},\\
\kk(1)&=\kk_1.
\end{aligned}
\end{equation}
The kernel in \eqref{mfl:read-l1-kernel} will express the population loss in terms of particle directions.

\paragraph{Student and loss.} For $W=[w_1,\dots,w_m]\in\R^{d\times m}$ the student is
\[f_W(x)=b+\frac{\alpha}{m}\sum_{j=1}^m\relu(w_j^\top x),\]
with a fixed head $\alpha/m$ ($\alpha>0$) and no hidden biases. The output intercept is refitted at every step:
\[
b=b(W):=\E[Y-(\alpha/m)\sum_j\relu(w_j^\top X)].
\]
Thus the residual $R:=Y-f_W(X)$ is centred. The population loss is
$L(W):=\E[(Y-f_W(X))^2]$, and $\hatw:=w/\abs w$.

\paragraph{Force and the Polar step.} The force is
\begin{equation}\label{mfl:eq:force}
F(W):=\E\big[R\,X\,\ind{W^\top X>0}^\top\big]\in\R^{d\times m},\qquad F_j=\E\big[R\,X\,\ind{w_j^\top X>0}\big],
\end{equation}
so that $\nabla_WL=-(2\alpha/m)F$. For a compact singular value decomposition $F=A\Sigma B^\top$, $\Polar(F):=AB^\top$,
with $\Polar(0):=0$. The hidden weights move by
\begin{equation}\label{mfl:eq:polar}
W(t+1)=W(t)+\eta\,\Polar\big(F(W(t))\big),\qquad t=0,1,2,\dots
\end{equation}
For a unit vector $u$ and a unit vector $\hatw$ at angle $\theta$, the two Gaussian identities are
\begin{equation}\label{mfl:read-l1-gated-moments}
\begin{aligned}
\E[\relu(u^\top X)X\ind{\hatw^\top X>0}]&=\frac{\pi-\theta}{2\pi}u+\frac{\sin\theta}{2\pi}\hatw,\\
\E[X\ind{\hatw^\top X>0}]&=\hatw/\sqrt{2\pi}.
\end{aligned}
\end{equation}
With $\rho_i:=u_i^\top\hatw$ and $\theta_i:=\arccos\rho_i$, centring the teacher in
\eqref{mfl:read-l1-gated-moments} gives its force on a unit with direction $\hatw$:
\begin{equation}\label{mfl:eq:teacherforce}
\begin{aligned}
g(\hatw)&=\cz\sum_{i=1}^r\Big[(\pi-\theta_i)\,u_i-(1-\sin\theta_i)\,\hatw\Big]\\
&=\cz\sum_{i=1}^r\Big[\big(\tfrac\pi2+\arcsin\rho_i\big)u_i-\big(1-\sqrt{1-\rho_i^2}\big)\hatw\Big].
\end{aligned}
\end{equation}
With $\theta_{jk}$ the angle between $w_j$ and $w_k$, the full force is
\begin{equation}\label{mfl:eq:fullforce}
F_j=g(\hatw_j)-\frac{\alpha}{2\pi m}\sum_{k=1}^m\abs{w_k}\Big[(\pi-\theta_{jk})\hatw_k-(1-\sin\theta_{jk})\hatw_j\Big].
\end{equation}
The second term is the \emph{self-force}.

\paragraph{AGOP and alignment.} The AGOP of the network is
\[\AGOP(W):=\E\big[\nabla_xf_W(X)\nabla_xf_W(X)^\top\big]=\Big(\frac{\alpha}{m}\Big)^2WHW^\top,\qquad
H_{jk}:=\frac{\pi-\theta_{jk}}{2\pi}.\]
A \emph{leading-$r$ selector} is any orthogonal projector $P_r$ onto an $r$-dimensional span of eigenvectors for the $r$
largest eigenvalues of $\AGOP$. The alignment with the teacher subspace is $\Ar(W):=\frac1r\tr(P_UP_r)$. Every statement
about $\Ar$ below holds for every leading selector; this follows once we prove the strict gap $\lambda_r>\lambda_{r+1}$,
which makes $P_r$ unique.

\paragraph{Initialization and mean-field scaling.} The columns $w_j(0)$ are IID $N(0,\sigma_0^2I_d/d)$. We take
$\eta=\etab\sqrt m$ with $\etab$ fixed and use the free scale parameter
\begin{equation}\label{mfl:eq:kappa}
\kappa:=\frac{\sigma_0}{\etab\sqrt d}=\frac{\sigma_0}{\eta}\sqrt{\frac md}.
\end{equation}
For example, $\sigma_0=1$, $\eta=.02$ and $m=16d$ give $\kappa=200$. Ordinary initialization means $\kappa\gg1$.

\paragraph{Mean-field dynamics.} A particle law $\nu$ on $\R^d$ acts on each of its particles through the force
$F(w;\nu)$. This is \eqref{mfl:eq:fullforce} with the average over $k$ replaced by $\E_{w'\sim\nu}[\cdot]$ (Section~\ref{mfl:pc-sec:def}).
Put $\Sigma_\nu:=\E_{w\sim\nu}[F(w;\nu)F(w;\nu)^\top]$. The mean-field (MF) Polar dynamics moves every particle by
\begin{equation}\label{mfl:eq:mfmap}
w\mapsto w+\etab\,\Sigma_{\nu_t}^{\dagger1/2}F(w;\nu_t),\qquad \nu_0=N(0,\sigma_0^2I_d/d),
\end{equation}
where $\Sigma^{\dagger1/2}:=(\Sigma^\dagger)^{1/2}$ is the square root of the Moore--Penrose inverse. When $\Sigma_\nu\succ0$
it equals $\Sigma_\nu^{-1/2}$. The MF loss $L(\nu)$ and the MF AGOP $\alpha^2\mathcal M(\nu)$, with
$\mathcal M(\nu):=\E_{w,w'\sim\nu}[H(w,w')\,ww'^\top]$, are defined in the same way.

At every width $m$ the update \eqref{mfl:eq:polar} is exactly the
map \eqref{mfl:eq:mfmap} applied to the empirical measure of the columns; the factors $\sqrt m$ cancel, and no condition such as
$m\ge d$ is needed (Lemma~\ref{mfl:pc-lem:identity}). The width-$m$ loss and AGOP are the MF ones evaluated at that measure.

On any
horizon along which $\Sigma_{\nu_t}\succ0$ (in the theorem below, $[0,t_F]$), the width-$m$ dynamics converges to the MF one as
$m\to\infty$ at fixed $\etab$ (Proposition~\ref{mfl:pc-prop:limit}(b)). Theorem~\ref{mfl:intro-thm:main} is a statement about the MF dynamics itself.

\ifdefined\Needspace\Needspace{10\baselineskip}\fi
\subsection{Main theorem}
\begin{theorem}[Plateau, then feature learning, then loss drop]\label{mfl:intro-thm:main}
Let $r\ge2$, $d\ge64r$ and $\Dp=d-r\ge400$, and let $\tau_F\in(0,.101)$, $\kappa_{\min}<\infty$ and $\delta>0$ be the numbers of
Theorem~\ref{mfl:mff-thm:MFF}; they depend only on $(r,\Dp)$ and are described below. Let $\kappa\ge\kappa_{\min}$,
$t_F:=\lfloor\tau_F\kappa\rfloor$ (so $t_F\ge2$), let $t_P\ge t_F$ be any integer, and put
\begin{equation}\label{mfl:read-l1-head-window}
\astar(t_P):=\frac{\bar y\,\sqrt V}{\etab\,(\kappa+t_P)},\qquad \bar y=.0084587115\ldots,
\end{equation}
where $\bar y$ is the positive root of $1.17933\,\bar y+.34085\,\bar y^2=.01$. There is a number $a_0=a_0(r,d,\kappa)>0$ with
the following properties. Consider the MF trajectory \eqref{mfl:eq:mfmap} started from $\nu_0=N(0,\sigma_0^2I_d/d)$,
$\sigma_0=\kappa\etab\sqrt d$, with head scale $\alpha>0$.
\begin{enumerate}
\item[(P)] \emph{Plateau.} If $\alpha\le\astar(t_P)$, then $L(t)\le1.01\,L(s)$ and $L(t)\ge.99\,V$ for all $0\le s,t\le t_P$.
\item[(F)] \emph{Feature learning.} $\mathcal M(\nu_0)$ is a multiple of $I_d$. If $\alpha\etab\le a_0$, then
$\Sigma_{\nu_t}\succ0$ for $t\le t_F$, and
\begin{equation}\label{mfl:read-l1-feature-gap}
\begin{aligned}
\mathcal M(\nu_{t_F})&=\lambda_1\Pone+\lambda_c\Pc+\lambda_\perp(I_d-P_U),\\
\min(\lambda_1,\lambda_c)&\ge\big(1+\tfrac\delta2\big)\,\lambda_\perp.
\end{aligned}
\end{equation}
By \eqref{mfl:read-l1-feature-gap}, at $t_F$ the top-$r$ eigenspace of the MF AGOP is exactly $U$, $\Ar=1$ for every leading selector, and
$\lambda_r\ge(1+\delta/2)\,\lambda_{r+1}$.
\item[(R)] \emph{Later loss drop.} If $\alpha\le\astar(t_P)$, then $L(0)-L(t_2)\ge.03\,V$ at
$t_2:=\lceil.05086\,\sqrt V/(\alpha\etab)\rceil$, that is, once the clock $\alpha\etab t/\sqrt V$ has reached $.05086$. At
$\alpha=\astar(t_P)$ this gives $t_2\le6.013\,(\kappa+t_P)+1$; for instance $t_2\le6.62\,\kappa+1$ when $t_P\le.1007\,\kappa$.
\end{enumerate}
In particular, with $\astar(t_P)$ from \eqref{mfl:read-l1-head-window}, (P), (F) and (R) hold together for every
$\alpha\in\big(0,\min\{\astar(t_P),a_0/\etab\}\big]$. The loss comparisons in (P) and (R) use the variance
\eqref{mfl:read-l1-variance}.
\end{theorem}

\paragraph{The constants.} Two regions cover every $(r,\Dp)$ with $\Dp\ge\max(400,63r)$; the values are those of
Theorem~\ref{mfl:mff-thm:MFF}.
\begin{itemize}
\item \emph{Region N}: $2\le r\le80$ and $\Dp\le D_A(r)$, with $D_A(r)$ the right end of the last $D$-interval of $r$ in Table~\ref{mfl:app-tab:regionN}. The constants are read,
row by row, from the $80$ certificates of Table~\ref{mfl:app-tab:regionN}. They satisfy $\tau_F\in[.013740,.088861)$,
$\kappa_{\min}\in[91,777]$ and $1+\delta\ge1.006223$ ($\ge1.012711$ for $r\ge3$, $\ge1.079513$ for $r\ge10$).
\item \emph{Region A}: all other $\Dp$, which includes every $\Dp\ge63r$ when $r\ge81$. The constants are in closed form
(Lemma~\ref{mfl:mff-lem:A}): $\tau_F\le.100667\sqrt{r^2-1}/M$ and $\kappa_{\min}=4/\tau_F$, and $1+\delta$ is the smaller of
two explicit bounds, $\lambda_c/\lambda_\perp\ge\exp(.0176\sqrt{r^2-1}/M)$ and a bound $\lambda_1/\lambda_\perp>1$, at $t_F$ for the $\alpha=0$
dynamics. Here $M=M(r,\Dp)\approx r+2\sqrt{r\ln\Dp}+2\ln\Dp$.
\end{itemize}
For example, at $r=2$, $400\le\Dp\le14\,930$ and $\kappa=1000$: $t_F=13$; and with $t_P=t_F$ and $\alpha=\astar(t_P)$, the loss
has dropped by $.03V$ at $t_2=6\,091$.

\subsection{Remarks}
\begin{enumerate}[label=(\roman*)]
\item \emph{The finite computational step.} (P) is proved in closed form, and (R) uses an explicit logarithm integral and finite rational endpoint bounds. In (F), region N is covered by $80$ explicit inequality certificates verified with outward interval arithmetic (Section~\ref{mfl:mff-sec:regionN}). Every certificate input and enclosure rule is included in this proof.
\item \emph{The head scale in (F).} (F) is first proved for the $\alpha=0$ dynamics (Theorem~\ref{mfl:mff-thm:MFF}) and then
transferred to small $\alpha>0$ by continuity of the MF map in the head scale (Proposition~\ref{mfl:pc-prop:limit}(a)); the
threshold $a_0$ is not explicit.
\item \emph{The window is long in steps, not in movement.} Each MF step moves the particles by at most $\etab\sqrt d$ in
RMS (Lemma~\ref{mfl:tools-lem:growth}(a)), so over $[0,t_F]$ the RMS displacement is at most $\etab\sqrt d\,t_F\le\tau_F\sigma_0<.101\,\sigma_0$.
The count $t_F$ is large because the step is small compared with the initialization.
\item \emph{The initial alignment.} The initial MF AGOP is isotropic, so its top-$r$ space is not determined. A finite-width initialization comparison additionally requires a specified leading selector when eigenvalues tie; no
selector-independent, all-width Haar assertion is used in this theorem.
\item \emph{Scope.} The data are population Gaussian; the teacher is a balanced sum of ReLUs on orthonormal directions; the
head is fixed, common to all units and small; the intercept is refitted exactly. (F) is proved at the single time $t_F$. The
certified margin $1+\delta$ is small, whereas the observed one is much larger. The width-$m$ dynamics converges to the MF one on
$[0,t_F+1]$ (Proposition~\ref{mfl:pc-prop:limit}(b)); no rate is claimed.
\end{enumerate}

\subsection{Structure of the proof}\label{mfl:sec:outline}
Section~\ref{mfl:sec:pc} sets up the MF map and reduces (F) to two scalar inequalities; Section~\ref{mfl:sec:tools} proves (P) and (R);
Sections~\ref{mfl:sec:rb}--\ref{mfl:sec:eval} prove (F) for the head-free dynamics; Section~\ref{mfl:sec:main} assembles the proof and
transfers (F) to small heads. Every section opens with a short statement of what it proves and where it is used, and
Table~\ref{mfl:intro-tab:map} lists where each part of Theorem~\ref{mfl:intro-thm:main} is proved.

\begin{table}[ht]
\caption{Where the parts of Theorem~\ref{mfl:intro-thm:main} are proved.}
\label{mfl:intro-tab:map}
\centering\small
\renewcommand{\arraystretch}{1.15}
\begin{tabular}{@{}>{\raggedright\arraybackslash}p{.2\textwidth}>{\raggedright\arraybackslash}p{.24\textwidth}>{\raggedright\arraybackslash}p{.5\textwidth}@{}}
\toprule
part of Theorem~\ref{mfl:intro-thm:main} & proved in & main ingredients\\
\midrule
trajectory & Section~\ref{mfl:sec:pc} & Lemmas~\ref{mfl:pc-lem:identity}, \ref{mfl:pc-lem:scaling}, \ref{mfl:pc-lem:PC3}; Proposition~\ref{mfl:mf-prop:symmetry}\\
(P) & Theorem~\ref{mfl:plat-thm:PB} & Lemmas~\ref{mfl:tools-lem:growth}, \ref{mfl:tools-lem:sandwich}, \ref{mfl:tools-lem:mfcapture}\\
(F) at $\alpha=0$ & Theorem~\ref{mfl:mff-thm:MFF} & Propositions~\ref{mfl:mf-prop:updates}, \ref{mfl:mf-prop:agop}, \ref{mfl:mff-prop:cert};
Lemmas~\ref{mfl:rb-lem:K}, \ref{mfl:rb-lem:M}, \ref{mfl:rb-lem:C}, \ref{mfl:mff-lem:MON}, \ref{mfl:mff-lem:floor}, \ref{mfl:mff-lem:A};
Proposition~\ref{mfl:mff-prop:regionN}\\
(F) at small $\alpha>0$ & Section~\ref{mfl:sec:main}, Step~3 & Lemma~\ref{mfl:pc-lem:PC2}; Proposition~\ref{mfl:pc-prop:limit}(a)\\
(F) at $t=0$ & Section~\ref{mfl:sec:main}, Step~4 & Proposition~\ref{mfl:mf-prop:agop}(4)\\
(R) & Corollary~\ref{mfl:loss-cor:MF} & Theorem~\ref{mfl:loss-thm:RLB}; Lemmas~\ref{mfl:tools-lem:growth}, \ref{mfl:tools-lem:mfcapture}\\
\bottomrule
\end{tabular}
\end{table}

\section{The mean-field map}\label{mfl:sec:pc}
\begin{secsummary}
\emph{What this section proves.} The width-$m$ Polar update is exactly a map $\Phi$ on particle laws applied to the
empirical measure of the columns (Lemma~\ref{mfl:pc-lem:identity}), and up to a dilation the trajectory depends on the head
$\alpha$ and the step $\etab$ only through $\alpha\etab$ (Lemma~\ref{mfl:pc-lem:scaling}). Along the trajectory the laws are absolutely continuous and invariant under
$S_r\times O(\Dp)$ (Sections~\ref{mfl:pc-sec:reg}--\ref{mfl:sec:sym}). At $\alpha=0$ a particle reduces to a point
$(\omega,\zeta)\in\R^r\times\R$ with explicit updates (Section~\ref{mfl:sec:mf}), and the MF AGOP has a three-block form whose
top-$r$ space is $U$ as soon as two scalar inequalities hold (Proposition~\ref{mfl:mf-prop:agop}). Finally the trajectory is
continuous in the head scale at $\alpha=0$ and in the width (Section~\ref{mfl:pc-sec:cont}).

\emph{Where it is used.} Everywhere: Sections~\ref{mfl:pc-sec:def}--\ref{mfl:sec:sym} underlie Sections~\ref{mfl:sec:tools}--\ref{mfl:sec:eval};
Section~\ref{mfl:sec:mf} and Section~\ref{mfl:sec:agop} are the language of the certificate; Section~\ref{mfl:pc-sec:cont} transfers (F) from $\alpha=0$
to small $\alpha>0$ in Section~\ref{mfl:sec:main}.
\end{secsummary}

\subsection{Definitions and the exact identity}\label{mfl:pc-sec:def}
$\mathcal P_p(\R^d)$ is the set of Borel probability measures with finite $p$-th moment, and $W_p$ is the $p$-Wasserstein
distance ($W_1\le W_2$). For $W=[w_1,\dots,w_m]$ put $\mu_W:=m^{-1}\sum_j\delta_{w_j}$, where $\delta_w$ is the point mass at $w$.
For a measurable map $T$, its pushforward is $(T_\#\mu)(B):=\mu(T^{-1}(B))$ for Borel sets $B$.
For a unit vector $\hatw$ and
$w'\in\R^d$ at angle $\theta=\theta(\hatw,w')$, let
\[\Upsilon(\hatw,w'):=\frac1{2\pi}\Big[(\pi-\theta)\,w'-(1-\sin\theta)\abs{w'}\,\hatw\Big],\qquad \Upsilon(\hatw,0):=0.\]
For $\mu\in\mathcal P_1(\R^d)$ and $w\ne0$ define
\[S(\hatw;\mu):=\int\Upsilon(\hatw,w')\,d\mu(w'),\qquad F(w;\mu):=g(\hatw)-\alpha S(\hatw;\mu),\qquad F(0;\mu):=0,\]
\begin{gather*}\Sigma(\mu):=\int F(w;\mu)F(w;\mu)^\top d\mu(w),\\
T_\mu(w):=w+\etab\,\Sigma(\mu)^{\dagger1/2}F(w;\mu),\\
\Phi(\mu):=(T_\mu)_\#\mu,\end{gather*}
where $A^{\dagger1/2}:=(A^\dagger)^{1/2}$ and $\dagger$ is the Moore--Penrose inverse. We also write $\Sigma_\mu$ for
$\Sigma(\mu)$. The MF trajectory is $\nu_0:=N(0,\sigma_0^2I_d/d)$ and $\nu_{t+1}:=\Phi(\nu_t)$; this is \eqref{mfl:eq:mfmap}. When the
dependence on the head scale matters we write $F_\alpha$, $\Sigma_\alpha$, $T_{\mu,\alpha}$, $\Phi_\alpha$ and
$\nu_\alpha(t)$, and for the step size, $\Phi_{\alpha,\etab}$. With $H(w,w'):=(\pi-\theta(\hatw,\hat w'))/(2\pi)$ and the kernel $\kk$
of Section~\ref{mfl:sec:setting}, the MF loss and the MF AGOP (up to the factor $\alpha^2$) are
\begin{gather*}
\mathcal L(\mu):=V-2\alpha\gamma\int\abs w\sum_{i=1}^r\kk(u_i^\top\hatw)\,d\mu
+\alpha^2\iint\abs w\abs{w'}\kk(\cos\theta(\hatw,\hat w'))\,d\mu\,d\mu,\\
\mathcal M(\mu):=\iint H(w,w')\,ww'^\top d\mu\,d\mu,
\end{gather*}
every integrand being $0$ when $w=0$ or $w'=0$. We write $L(\mu)$ for $\mathcal L(\mu)$ and $L(t)$ for $\mathcal L(\nu_t)$.
Equivalently $F(w;\mu)=\E_X[R_\mu(X)X\ind{w^\top X>0}]$, where $R_\mu:=Y-\E[Y]-(g_\mu-\E[g_\mu])$ is the centred residual of the
MF network output $g_\mu(x):=\alpha\int\relu(w'^\top x)\,d\mu(w')$ (Lemma~\ref{mfl:pc-lem:identity}(a) below). Note that
$F(w;\mu)$ depends on $w\ne0$ only through $\hatw$.

\begin{lemma}[Exact identity]\label{mfl:pc-lem:identity}
For every $m\ge1$, $\alpha\ge0$ and $W\in\R^{d\times m}$ (zero columns allowed), with $\eta=\etab\sqrt m$ and
$W^+:=W+\eta\Polar(F(W))$:
\begin{enumerate}[label=(\alph*)]
\item column $j$ of the force \eqref{mfl:eq:force} is $F(w_j;\mu_W)$, the self-term $k=j$ included, so $F(W)F(W)^\top=m\Sigma(\mu_W)$;
\item $\Polar(F)=(FF^\top)^{\dagger1/2}F$ for every real matrix $F$;
\item $w_j^+=T_{\mu_W}(w_j)$ for every $j$, hence $\mu_{W^+}=\Phi(\mu_W)$;
\item $L(W)=\mathcal L(\mu_W)$ and $\AGOP(W)=\alpha^2\mathcal M(\mu_W)$.
\end{enumerate}
Consequently $\mu_{W(t)}=\Phi^t(\mu_{W(0)})$ along \eqref{mfl:eq:polar}, for every $m$.
\end{lemma}
\begin{proof}
\step{Part (a): Identify the force.}
If $w_j=0$ the gate vanishes and $F_j=0=F(0;\mu_W)$. Let $w_j\ne0$. With the refitted intercept,
\[
R=(Y-\E[Y])-\frac\alpha m\sum_k(\relu(w_k^\top X)-\E[\relu(w_k^\top X)]).
\]
By the exact formulas of Section~\ref{mfl:sec:setting}, the teacher contribution is
\[
\begin{aligned}
&\E[YX\ind{\hatw^\top X>0}]-\E[Y]\,\E[X\ind{\hatw^\top X>0}]\\
&\qquad=\gamma\sum_i\big[\frac{\pi-\theta_i}{2\pi}u_i+\frac{\sin\theta_i}{2\pi}\hatw\big]
-\frac{\gamma r}{2\pi}\hatw=g(\hatw),
\end{aligned}
\]
using $\E[Y]=\gamma r/\sqrt{2\pi}$.

For $w'\ne0$ at angle $\theta$ to $\hatw$, the same formulas and homogeneity give
\[
\begin{aligned}
\E[\relu(w'^\top X)X\ind{\hatw^\top X>0}]&=\frac{\pi-\theta}{2\pi}w'+\frac{\sin\theta}{2\pi}\abs{w'}\hatw,\\
\E[\relu(w'^\top X)]\,\E[X\ind{\hatw^\top X>0}]&=\frac{\abs{w'}}{2\pi}\hatw.
\end{aligned}
\]
Their difference is $\Upsilon(\hatw,w')$, and a zero column contributes $0=\Upsilon(\hatw,0)$. Therefore
\begin{equation}\label{mfl:read-l2-empirical-force}
\begin{aligned}
F_j&=g(\hatw_j)-\frac\alpha m\sum_k\Upsilon(\hatw_j,w_k)=F(w_j;\mu_W),\\
FF^\top&=\sum_jF_jF_j^\top=m\Sigma(\mu_W).
\end{aligned}
\end{equation}
The same computation with $\mu$ in place of $\mu_W$ gives
$F(w;\mu)=\E[R_\mu X\ind{w^\top X>0}]$.

\step{Part (b): Express the polar factor.}
For a compact SVD $F=A\Sigma_FB^\top$,
\begin{equation}\label{mfl:read-l2-polar-factor}
(FF^\top)^{\dagger1/2}F=A\Sigma_F^{-1}A^\top A\Sigma_FB^\top=AB^\top.
\end{equation}
Both sides vanish if $F=0$.

\step{Part (c): Cancel the width factors.}
By \eqref{mfl:read-l2-empirical-force} and \eqref{mfl:read-l2-polar-factor},
\[
\Polar(F)=(m\Sigma(\mu_W))^{\dagger1/2}F=m^{-1/2}\Sigma(\mu_W)^{\dagger1/2}F.
\]
Thus column $j$ of $\eta\Polar(F)$ is $\etab\Sigma(\mu_W)^{\dagger1/2}F(w_j;\mu_W)$:
the factors $\sqrt m$ cancel exactly, and $\mu_{W^+}=(T_{\mu_W})_\#\mu_W$.

\step{Part (d): Identify the loss and AGOP.}
The refitted loss is $L(W)=\Var(Y-\frac\alpha m\sum_j\relu(w_j^\top X))$, with covariances
\[
\begin{aligned}
\Cov(\relu(u_i^\top X),\relu(w_j^\top X))&=\abs{w_j}\kk(u_i^\top\hatw_j),\\
\Cov(\relu(w_j^\top X),\relu(w_k^\top X))&=\abs{w_j}\abs{w_k}\kk(\cos\theta_{jk}).
\end{aligned}
\]
This is $\mathcal L(\mu_W)$. For the AGOP,
\[
\nabla_xf_W=\frac\alpha m\sum_j\ind{w_j^\top x>0}w_j,
\qquad \E[\ind{w_j^\top X>0}\ind{w_k^\top X>0}]=H_{jk}.
\]
Hence
\[
\AGOP(W)=(\frac\alpha m)^2\sum_{j,k}H_{jk}w_jw_k^\top=\alpha^2\mathcal M(\mu_W).
\]
Since $H(w,w)=\frac12$, the diagonal terms need no separate treatment.
\end{proof}

\begin{lemma}[Scaling]\label{mfl:pc-lem:scaling}
For $c>0$ let $D_c(w):=cw$. For every $\mu\in\mathcal P_1$, $\alpha\ge0$, $\etab>0$ and $w\ne0$,
\begin{gather*}F_\alpha\big(cw;(D_c)_\#\mu\big)=F_{c\alpha}(w;\mu),\\
\Sigma_\alpha\big((D_c)_\#\mu\big)=\Sigma_{c\alpha}(\mu),\\
\Phi_{\alpha,\etab}\big((D_c)_\#\mu\big)=(D_c)_\#\Phi_{c\alpha,\etab/c}(\mu),\end{gather*}
and $\mathcal M((D_c)_\#\mu)=c^2\mathcal M(\mu)$. Consequently, with $\kappa=\sigma_0/(\etab\sqrt d)$, the MF trajectory with head
$\alpha$ and step $\etab$ from $N(0,\sigma_0^2I_d/d)$ is the image under $D_{\etab}$ of the trajectory with head $\alpha\etab$
and step $1$ from $N(0,\kappa^2I_d)$. The two trajectories have the same whitening matrices $\Sigma$, and their MF AGOPs
$\mathcal M$ differ by the factor $\etab^2$, which changes no eigenspace and no eigenvalue ratio.
\end{lemma}
\begin{proof}
For $c>0$, $\Upsilon(\hatw,cw')=c\,\Upsilon(\hatw,w')$ (the angle is unchanged), so $S(\hatw;(D_c)_\#\mu)=c\,S(\hatw;\mu)$; and
$cw$ has the direction $\hatw$. This gives the identity for $F$, and the one for $\Sigma$ follows by the change of variables
$w=cw''$.

For the particle map, the same identities give
\begin{equation}\label{mfl:read-l2-scaled-map}
T_{(D_c)_\#\mu}(cw)=cw+\etab\,\Sigma_{c\alpha}(\mu)^{\dagger1/2}F_{c\alpha}(w;\mu)=c\,T'(w),
\end{equation}
where $T'$ is the map of $\mu$ with head $c\alpha$ and step $\etab/c$. Pushing $\mu$ forward in
\eqref{mfl:read-l2-scaled-map} gives the identity for $\Phi$. Since $H(cw,cw')=H(w,w')$,
$\mathcal M((D_c)_\#\mu)=c^2\mathcal M(\mu)$.

For the consequence take $c=\etab$:
\[
N(0,\sigma_0^2I_d/d)=(D_{\etab})_\#N(0,\kappa^2I_d),
\]
and induct on $t$.
\end{proof}

\subsection{Regularity}\label{mfl:pc-sec:reg}
\begin{lemma}[Regularity of the force]\label{mfl:pc-lem:PC1}
\begin{enumerate}[label=(\alph*)]
\item $g$ is continuous on $S^{d-1}$ and $C^\infty$ on $S^{d-1}\setminus\{\pm u_1,\dots,\pm u_r\}$.
\item For every unit $\hatw$ and all $w_1',w_2'$: $\abs{\Upsilon(\hatw,w_1')-\Upsilon(\hatw,w_2')}\le2\sqrt d\abs{w_1'-w_2'}$ and
$\abs{\Upsilon(\hatw,w')}\le\abs{w'}$. Hence $\sup_{\hatw}\abs{S(\hatw;\mu)-S(\hatw;\mu')}\le2\sqrt dW_1(\mu,\mu')\le2\sqrt dW_2(\mu,\mu')$.
\item For $\mu\in\mathcal P_1$, $w\mapsto S(\hatw;\mu)$ is $C^1$ on $\R^d\setminus\{0\}$. Consequently $w\mapsto F(w;\mu)$ is
continuous on $\R^d\setminus\{0\}$ and $C^1$ on the open set $\{w\ne0:\hatw\notin\{\pm u_i\}\}$, whose complement is null.
\end{enumerate}
\end{lemma}
\begin{proof}
\step{Part (a): Regularity of the teacher force.}
By \eqref{mfl:eq:teacherforce}, $g(\hatw)$ is built from $\hatw$ and the functions $\arcsin\rho_i$ and $\sqrt{1-\rho_i^2}$ of
$\rho_i=u_i^\top\hatw$; these are continuous on $[-1,1]$ and $C^\infty$ on $(-1,1)$, and $\abs{\rho_i}=1$ iff $\hatw=\pm u_i$.

\step{Part (b): Dependence on the particle law.}
By the proof of Lemma~\ref{mfl:pc-lem:identity}(a), $\Upsilon(\hatw,w')=\E[\phi_{w'}(X)X\ind{\hatw^\top X>0}]$ with
$\phi_{w'}:=\relu(w'^\top X)-\E[\relu(w'^\top X)]$ (also for $w'=0$). Put $\psi:=\relu(w_1'^\top X)-\relu(w_2'^\top X)$. As $\relu$ is
$1$-Lipschitz, $(\E[\psi^2])^{1/2}\le\abs{w_1'-w_2'}$ and $\abs{\E[\psi]}\le\abs{w_1'-w_2'}$. With $\ind\cdot:=\ind{\hatw^\top X>0}$,
\[
\Upsilon(\hatw,w_1')-\Upsilon(\hatw,w_2')=\E[\psi X\ind\cdot]-\E[\psi]\,\E[X\ind\cdot].
\]
The two terms satisfy
\[
\begin{aligned}
\abs{\E[\psi X\ind\cdot]}&\le(\E[\psi^2])^{1/2}(\E[\abs X^2])^{1/2}\le\sqrt d\abs{w_1'-w_2'},\\
\abs{\E[\psi]}\abs{\E[X\ind\cdot]}&=\abs{\E[\psi]}/\sqrt{2\pi}\le\abs{w_1'-w_2'}.
\end{aligned}
\]
Also
$\abs{\Upsilon(\hatw,w')}\le\frac{\pi+1}{2\pi}\abs{w'}\le\abs{w'}$ from the definition. For the last claim integrate against an optimal
$W_1$ coupling.

\step{Part (c): Differentiate the gated expectation.}
We first claim that if $\phi$ is continuous with $\abs{\phi(x)}\le C(1+\abs x)$, then
$h(w):=\E[\phi(X)X\ind{w^\top X>0}]$ is $C^1$ on $\R^d\setminus\{0\}$, with
\begin{equation}\label{mfl:read-l2-gate-derivative}
\begin{aligned}
\partial_vh(w)&=\abs w^{-1}\int_{\hatw^\perp}(v^\top y)\phi(y)y\,\gamma_d(y)\,d\sigma(y),\\
\abs{\partial_vh(w)}&\le C_dC\abs v/\abs w.
\end{aligned}
\end{equation}
Here $\sigma$ is Lebesgue measure on $\hatw^\perp$, and $\gamma_d$ is the $N(0,I_d)$ density.

To prove the claim, fix a smooth
nondecreasing $\vartheta$ with $\vartheta=0$ on $(-\infty,0]$ and $\vartheta=1$ on $[1,\infty)$, and put
$h_\epsilon(w):=\E[\phi(X)X\vartheta(w^\top X/\epsilon)]$. Differentiating under the expectation gives
\[
\begin{aligned}
\partial_vh_\epsilon(w)&=\int_\R\epsilon^{-1}\vartheta'(\abs wz/\epsilon)G_w(z)\,dz,\\
G_w(z)&:=\int_{\hatw^\perp}\phi(x)x(v^\top x)\gamma_d(x)\,d\sigma(y),\qquad x=z\hatw+y.
\end{aligned}
\]
The function $G_w(z)$ is jointly continuous in $(z,w)$
(dominated convergence), and the weights $\epsilon^{-1}\vartheta'(\abs wz/\epsilon)dz$ have mass $1/\abs w$ on $z\in[0,\epsilon/\abs w]$. So
$\partial_vh_\epsilon\to G_w(0)/\abs w$ and $h_\epsilon\to h$ (the difference is at most $\E[\abs{\phi(X)}\abs X\ind{0<w^\top X<\epsilon}]$),
uniformly on compact subsets of $\R^d\setminus\{0\}$. Hence $h\in C^1$ with the derivative in \eqref{mfl:read-l2-gate-derivative}.

Apply \eqref{mfl:read-l2-gate-derivative} with $\phi=\phi_{w'}$ ($C=2\abs{w'}$):
$\partial_v\Upsilon(w/\abs w,w')$ is continuous in $w\ne0$ and bounded by $2C_d\abs{w'}\abs v/\abs w$,
so for $\mu\in\mathcal P_1$ dominated convergence allows differentiation under $\int d\mu(w')$. With (a) this gives the
statements about $F$; the excluded set is a finite union of rays.
\end{proof}

\begin{lemma}[Absolute continuity along the MF trajectory]\label{mfl:pc-lem:PC3}
Let $\mu\ll\mathrm{Leb}$ and $T(w)=w+b(\hatw)$ ($w\ne0$), $T(0)=0$, with $b$ bounded and of class $C^1$ on an open
$S'\subseteq S^{d-1}$ of full measure. Then $T_\#\mu\ll\mathrm{Leb}$. Consequently, for every $\alpha\ge0$, every $\nu_\alpha(t)$ is
absolutely continuous and lies in $\mathcal P_2$; in particular $\nu_\alpha(t)(\{0\})=0$.
\end{lemma}
\begin{proof}
\step{Step 1 (Locate the singular set).}
On the open set $\Omega:=\{w\ne0:\hatw\in S'\}$ (complement null),
\begin{equation}\label{mfl:read-l2-radial-jacobian}
DT(w)=I+B(\hatw)/\abs w,\qquad B(\hatw):=Db(\hatw)(I-\hatw\hatw^\top).
\end{equation}
For fixed $\hatw$, $s\mapsto\det(I+sB(\hatw))$ is a polynomial equal to $1$ at $s=0$, so it has finitely many zeros; by polar
coordinates and Fubini applied to \eqref{mfl:read-l2-radial-jacobian}, $Z:=\{w\in\Omega:\det DT=0\}$ is null (and closed in $\Omega$).

\step{Step 2 (Preserve null sets).}
On $\Omega\setminus Z$, $T$ is a local
$C^1$-diffeomorphism; cover it by countably many open sets $U_n$ on which $T$ is a diffeomorphism onto its image (Lindel\"of).
For a null set $N$, each $T^{-1}(N)\cap U_n$ is null. Hence
\[
T^{-1}(N)\subseteq(\Omega^c\cup Z)\cup\bigcup_n(T^{-1}(N)\cap U_n)
\]
is null and $T_\#\mu(N)=0$.

\step{Step 3 (Apply the argument along the trajectory).}
Induct on $t$ from $\nu_0\ll\mathrm{Leb}$: $T_{\nu_t}(w)=w+b(\hatw)$ with
$b=\etab\Sigma_{\nu_t}^{\dagger1/2}F(\cdot;\nu_t)$, which is bounded and $C^1$ on $S^{d-1}\setminus\{\pm u_i\}$ by
Lemma~\ref{mfl:pc-lem:PC1}. Second moments stay finite because $b$ is bounded. No invertibility of $\Sigma$ is needed.
\end{proof}

\subsection{Symmetry}\label{mfl:sec:sym}
\paragraph{Coordinates.} From now on $u_i=e_i$, so $w=(w_U,w_\perp)\in\R^r\times\R^{\Dp}$, and $P_\perp:=I_d-P_U$. We identify
vectors of $U$ with their coordinate vectors in $\R^r$; then $\eb=\one/\sqrt r$ and $\Pc=I_r-\eb\eb^\top$ on $U$. For a random
variable $Z$ (scalar or vector) under a particle law $\nu$, $\norm{Z}:=(\E[\abs{Z}^2])^{1/2}$ is its $L^2(\nu)$ norm.

\paragraph{The symmetry group.} $\mathcal G\subset O(d)$ is the group generated by the permutation matrices of the $r$
coordinates of $U$ (acting as the identity on $U^\perp$) and by $O(\Dp)$ acting on $U^\perp$ (as the identity on $U$). Thus
$\mathcal G\cong S_r\times O(\Dp)$. Every $Q\in\mathcal G$ commutes with $\Pone$, $\Pc$ and $P_\perp$. The law of $X$ is
$\mathcal G$-invariant and the teacher is $\mathcal G$-invariant: $Y(QX)=Y(X)$.

\begin{proposition}[Symmetry and block form of $\Sigma_\nu$]\label{mfl:mf-prop:symmetry}
Let $\alpha\ge0$ and let $F(\cdot;\nu)$, $\Sigma_\nu$ and $T_\nu$ be as in Section~\ref{mfl:pc-sec:def}.
\begin{enumerate}[label=(\alph*)]
\item $g(Q\hatw)=Q\,g(\hatw)$ for every $Q\in\mathcal G$ and every unit vector $\hatw$. If $\nu$ is $\mathcal G$-invariant, then
$F(Qw;\nu)=QF(w;\nu)$ and $Q\Sigma_\nu Q^\top=\Sigma_\nu$ for every $Q\in\mathcal G$.
\item Let $\nu$ be $\mathcal G$-invariant with $\nu(\{0\})=0$. Then, with $F=F(w;\nu)$ and $w\sim\nu$,
\[\Sigma_\nu=a_1\Pone+a_c\Pc+a_\perp P_\perp,\qquad a_1:=\E[(\eb^\top F)^2],\quad a_c:=\frac{\E[\abs{\Pc F}^2]}{r-1},\quad
a_\perp:=\frac{\E[\abs{P_\perp F}^2]}{\Dp}.\]
\item If $\nu$ is $\mathcal G$-invariant, then $T_\nu(Qw)=QT_\nu(w)$ for all $Q\in\mathcal G$ and $w$, and $(T_\nu)_\#\nu$ is
$\mathcal G$-invariant. Consequently every law $\nu_t$ of the MF trajectory \eqref{mfl:eq:mfmap} is $\mathcal G$-invariant, for every
$\alpha\ge0$.
\end{enumerate}
\end{proposition}
\begin{proof}
\step{Part (a): Equivariance of the force.}
Write $g(\hatw)=\cz\big[\sum_i(\pi-\theta_i)u_i-s\,\hatw\big]$ with $\rho_i=u_i^\top\hatw$, $\theta_i=\arccos\rho_i$ and
$s:=\sum_i(1-\sin\theta_i)$ \eqref{mfl:eq:teacherforce}. Let $Q$ be a permutation of the coordinates of $U$, say $Qu_i=u_{\pi(i)}$.
Then $u_{\pi(i)}^\top Q\hatw=u_i^\top\hatw$, so the angles are permuted, $s$ is unchanged, and
$g(Q\hatw)=\cz[\sum_i(\pi-\theta_i)u_{\pi(i)}-sQ\hatw]=Qg(\hatw)$. Let $Q=I_r\oplus O$ with $O\in O(\Dp)$. Then all $\rho_i$ and $s$
are unchanged, $Qu_i=u_i$, and again $g(Q\hatw)=Qg(\hatw)$. These elements generate $\mathcal G$.

For the self-force,
$\Upsilon(Q\hatw,Qw')=Q\Upsilon(\hatw,w')$ because $Q$ preserves angles and norms, so by invariance of $\nu$,
\[
S(Q\hatw;\nu)=\int\Upsilon(Q\hatw,Qw'')\,d\nu(w'')=QS(\hatw;\nu).
\]
Hence $F(Qw;\nu)=QF(w;\nu)$, and invariance of $\nu$ gives
\begin{equation}\label{mfl:read-l2-covariance-symmetry}
Q\Sigma_\nu Q^\top=\E[F(Qw;\nu)F(Qw;\nu)^\top]=\Sigma_\nu.
\end{equation}

\step{Part (b): Identify the three blocks.}
By \eqref{mfl:read-l2-covariance-symmetry}, $\Sigma_\nu$ commutes with $\mathcal G$.
\begin{itemize}
\item \emph{Cross block.} $R:=I_r\oplus(-I_{\Dp})\in\mathcal G$ gives $P_U\Sigma_\nu P_\perp=P_UR\Sigma_\nu RP_\perp=-P_U\Sigma_\nu P_\perp$,
hence $P_U\Sigma_\nu P_\perp=0$.
\item \emph{$\perp$ block.} Its restriction $S_\perp$ to $U^\perp$ commutes with every reflection $I_{\Dp}-2ee^\top$ ($\abs e=1$),
hence with $ee^\top$. So every unit vector $e$ is an eigenvector of $S_\perp$, and $S_\perp=a_\perp I_{\Dp}$ with
$a_\perp=\tr S_\perp/\Dp$.
\item \emph{$U$ block.} Its restriction $S_U$ commutes with every permutation matrix, so $(S_U)_{ij}=(S_U)_{\pi(i)\pi(j)}$ for all
$\pi\in S_r$. Hence all diagonal entries are equal, say to $x$, and (since $r\ge2$) all off-diagonal entries are equal, say
to $y$. So $S_U=(x-y)\Pc+(x-y+ry)\Pone$, and $a_1=\eb^\top S_U\eb$, $a_c=\tr(\Pc S_U)/(r-1)$.
\end{itemize}

\step{Part (c): Propagate the symmetry.}
By (a), $\Sigma_\nu$ commutes with every $Q\in\mathcal G$, hence so do $\Sigma_\nu^\dagger$ and $\Sigma_\nu^{\dagger1/2}$ (functions of
$\Sigma_\nu$). Thus for $w\ne0$,
\[
T_\nu(Qw)=Qw+\etab\Sigma_\nu^{\dagger1/2}QF(w;\nu)=QT_\nu(w),
\]
and $T_\nu(0)=0$. For a Borel set $A$, equivariance gives $T_\nu^{-1}(QA)=QT_\nu^{-1}(A)$, hence
\[
((T_\nu)_\#\nu)(QA)=\nu(QT_\nu^{-1}(A))=((T_\nu)_\#\nu)(A).
\]
Since
$\nu_0$ is $O(d)$-invariant and $\nu_t(\{0\})=0$ for every $t$ (Lemma~\ref{mfl:pc-lem:PC3}), induction on $t$ concludes.
\end{proof}

\subsection{The reduced dynamics at \texorpdfstring{$\alpha=0$}{alpha = 0}}\label{mfl:sec:mf}
\paragraph{Conventions for Sections~\ref{mfl:sec:mf}--\ref{mfl:sec:agop} and Sections~\ref{mfl:sec:rb}--\ref{mfl:sec:eval}.}
\begin{itemize}
\item \emph{$\alpha=0$.} With the refitted intercept and $\alpha=0$ the residual is $Y-\E[Y]$, and the force on a particle with
direction $\hatw$ is the teacher force $g(\hatw)$ of \eqref{mfl:eq:teacherforce}. Small $\alpha>0$ is reached by continuity
(Section~\ref{mfl:pc-sec:cont} and Section~\ref{mfl:sec:main}).
\item \emph{Units $\etab=1$.} By Lemma~\ref{mfl:pc-lem:scaling} (at $\alpha=0$ the head plays no role), the trajectory with step
$\etab$ is the image under $w\mapsto\etab w$ of the trajectory with step $1$ from $N(0,\kappa^2I_d)$, $\kappa$ as in
\eqref{mfl:eq:kappa}. We work with the latter, that is, we write $w$ for $w/\etab$; then $\nu_0=N(0,\kappa^2I_d)$.
\end{itemize}

\begin{definition}[MF dynamics at $\alpha=0$]\label{mfl:mf-def:mf}
For a law $\nu$ on $\R^d$ with $\nu(\{0\})=0$ put $\Sigma_\nu:=\E_{w\sim\nu}[g(\hatw)g(\hatw)^\top]$ and
$T_\nu(w):=w+\Sigma_\nu^{\dagger1/2}g(\hatw)$ ($w\ne0$), $T_\nu(0):=0$. The $\alpha=0$ trajectory is $\nu_0=N(0,\kappa^2I_d)$,
$\nu_{t+1}:=(T_{\nu_t})_\#\nu_t$. This is the map $\Phi$ of Section~\ref{mfl:pc-sec:def} at $\alpha=0$ in units $\etab=1$. By
Lemma~\ref{mfl:pc-lem:PC4} below, $\Sigma_{\nu_t}\succ0$ for every $t$, so $\Sigma_{\nu_t}^{\dagger1/2}=\Sigma_{\nu_t}^{-1/2}$. The object of (F) is
\[\mathcal M(\nu):=\E_{w,w'}[H(w,w')\,ww'^\top],\qquad H(w,w'):=\frac{\pi-\angle(w,w')}{2\pi},\]
with $w,w'\sim\nu$ independent. At $\alpha=0$, $\mathcal M$ is a surrogate: the actual network AGOP is zero.
For $\alpha>0$, the actual AGOP is $\alpha^2\etab^2\mathcal M$, evaluated at the corresponding positive-head law in unit-step coordinates; the positive factor preserves eigenspaces and eigenvalue ratios.
\end{definition}

\begin{definition}[Reduced state]\label{mfl:mf-def:reduced}
Let $\nu$ be $\mathcal G$-invariant with $\nu(\{0\})=0$. A \emph{reduced representation} of $w\sim\nu$ is a triple
$(\omega,\zeta,n)\in\R^r\times\R\times S^{\Dp-1}$ such that
\[w_U=\omega,\qquad w_\perp=\sqrt{\Dp}\,\zeta\,n,\]
with $n$ uniform on $S^{\Dp-1}$ and independent of $(\omega,\zeta)$. The sign of $\zeta$ is part of the state. The
\emph{reduced law} is the law of $(\omega,\zeta)$. For each particle we write
\[\abs w=(\abs\omega^2+\Dp\zeta^2)^{1/2},\quad \rho:=\omega/\abs w=\hatw_U,\quad \rho_\infty:=\max_i\abs{\rho_i},\quad
\omega_1:=\eb^\top\omega,\quad \xi:=\Pc\omega,\]
\[s:=s(\rho)=\sum_{i=1}^r\big(1-\sqrt{1-\rho_i^2}\big),\qquad b:=s/\abs w,\]
\[f(\rho):=\sum_{i=1}^r\Big(\frac\pi2+\arcsin\rho_i\Big)-s\sum_{i=1}^r\rho_i,\qquad h(\rho):=\Pc\big(\arcsin\rho-s\rho\big),\]
where $\arcsin$ acts coordinatewise. The \emph{normalizations} of $\nu$ are the $L^2(\nu)$ quantities
\[\norm f,\qquad \bar a:=\frac{\E[\abs h^2]}{r-1},\qquad N:=\norm{b\zeta}=\Big(\frac{\E[s^2(1-\abs\rho^2)]}{\Dp}\Big)^{1/2},\]
and its \emph{contrast scale} is $\sigma_c:=(\E[\abs\xi^2]/(r-1))^{1/2}$.
\end{definition}
The identity $N=\norm{b\zeta}=(\E[s^2(1-\abs\rho^2)]/\Dp)^{1/2}$ holds because $\abs w^2(1-\abs\rho^2)=\Dp\zeta^2$.

\begin{proposition}[Exact reduced dynamics]\label{mfl:mf-prop:updates}
Let $\nu$ be $\mathcal G$-invariant with $\nu(\{0\})=0$. A reduced representation exists: take
$\zeta:=\abs{w_\perp}/\sqrt{\Dp}$ and $n:=w_\perp/\abs{w_\perp}$ (an independent uniform vector if $w_\perp=0$). The existence claim
uses only the invariance of $\nu$; in (a)--(c), $\alpha=0$.
\begin{enumerate}[label=(\alph*)]
\item Per particle, $g_U=\cz\big(\frac\pi2+\arcsin\rho_i-s\rho_i\big)_{i\le r}$, $g_\perp=-\cz s\,\hatw_\perp$, and
$s\in[\abs\rho^2/2,\abs\rho^2]$. Moreover $\eb^\top g=\cz f/\sqrt r$, $\Pc g=\cz h$ and $g_\perp=-\cz\sqrt{\Dp}\,b\zeta\,n$. Hence
$a_1=\cz^2\norm f^2/r$, $a_c=\cz^2\bar a$, $a_\perp=\cz^2N^2$, and $\Sigma_\nu\succ0$ iff $\bar a>0$ and $N>0$.
\item Let $\Sigma_\nu\succ0$. Then $T_\nu(w)=(\omega^+,\sqrt{\Dp}\,\zeta^+n)$ with the \emph{same} $n$, where
\[\omega_1^+=\omega_1+\delta\omega_1,\ \ \delta\omega_1:=\frac{f(\rho)}{\norm f};\qquad
\xi^+=\xi+\delta\xi,\ \ \delta\xi:=\frac{h(\rho)}{\sqrt{\bar a}};\]
\[\zeta^+=\zeta+\delta\zeta=\zeta\Big(1-\frac bN\Big),\qquad\delta\zeta:=-\frac{b\zeta}N.\]
The constant $\cz$ cancels. Consequently $(\omega^+,\zeta^+,n)$ is a reduced representation of $(T_\nu)_\#\nu$. The reduced law
evolves autonomously, by a map that depends on the law only through $\norm f$, $\bar a$ and $N$.
\item (Whitening.) $\E[\delta\omega\,\delta\omega^\top]=I_r$, i.e.\ $\E[\delta\omega_1^2]=1$, $\E[\abs{\delta\xi}^2]=r-1$ and
$\E[\delta\omega_1\delta\xi]=0$. Also $\E[(\delta\zeta)^2]=1$.
\end{enumerate}
\end{proposition}
\begin{proof}
First establish the reduced representation. Let $Q=I_r\oplus O$ with $O\in O(\Dp)$. Invariance of $\nu$ gives $(\omega,\zeta,On)\overset d=(\omega,\zeta,n)$ on
$\{w_\perp\ne0\}$; on $\{w_\perp=0\}$ this holds because the auxiliary vector is uniform and independent. So the conditional
law of $n$ given $(\omega,\zeta)$ is $O(\Dp)$-invariant, hence uniform.

\step{Part (a): Resolve the force into its blocks.}
The coordinate formula for $g_U$ is \eqref{mfl:eq:teacherforce} with $\hatw_U=\rho$, and $g_\perp=-\cz s\hatw_\perp$. For
$x\in[-1,1]$, $1-\sqrt{1-x^2}=x^2/(1+\sqrt{1-x^2})\in[x^2/2,x^2]$; summing gives the bounds on $s$. Next,
\[
\eb^\top g=\cz r^{-1/2}\sum_i(\frac\pi2+\arcsin\rho_i-s\rho_i)=\cz f/\sqrt r,
\]
and $\Pc g=\cz h$ because $\Pc$ annihilates the
constant vector $\frac\pi2\one$. Also $\hatw_\perp=\sqrt{\Dp}\zeta n/\abs w$, so $g_\perp=-\cz\sqrt{\Dp}(s/\abs w)\zeta n=-\cz\sqrt{\Dp}b\zeta n$.
The formulas for $a_1,a_c,a_\perp$ follow from Proposition~\ref{mfl:mf-prop:symmetry}(b) and $\abs{g_\perp}^2=\cz^2\Dp b^2\zeta^2$. Finally
$a_1>0$ always, because $f>0$ (Lemma~\ref{mfl:mf-lem:signs}(i)).

\step{Part (b): Read off the reduced update.}
By Proposition~\ref{mfl:mf-prop:symmetry}(b),
\[
T_\nu(w)=w+a_1^{-1/2}(\eb^\top g)\eb+a_c^{-1/2}\Pc g+a_\perp^{-1/2}g_\perp.
\]
With (a),
\[
\delta\omega_1=a_1^{-1/2}\cz f/\sqrt r=f/\norm f,
\qquad \delta\xi=a_c^{-1/2}\cz h=h/\sqrt{\bar a}.
\]
The $\perp$ part is
\[
\sqrt{\Dp}\zeta n-a_\perp^{-1/2}\cz\sqrt{\Dp}b\zeta n=\sqrt{\Dp}\,n\,\zeta(1-b/N).
\]
The new triple has the same $n$, which is still uniform
and independent of $(\omega^+,\zeta^+)$, because the latter is a function of $(\omega,\zeta)$ and of the three normalizations.

\step{Part (c): Verify whitening.}
The scalar normalizations give
\[
\E[\delta\omega_1^2]=\E[f^2]/\norm f^2=1,
\qquad \E[\abs{\delta\xi}^2]=\E[\abs h^2]/\bar a=r-1.
\]
Next, $\E[(\eb^\top g)\Pc g]=\Pc\Sigma_\nu\eb=a_1\Pc\eb=0$, so $\E[\delta\omega_1\delta\xi]=0$.
Also $\E[\delta\xi\delta\xi^\top]=\Pc\Sigma_\nu\Pc/a_c=\Pc$, so
\[
\E[\delta\omega\delta\omega^\top]=\Pone+\Pc=I_r.
\]
Finally $\E[(\delta\zeta)^2]=\E[b^2\zeta^2]/N^2=1$.

\end{proof}
At $\alpha=0$ the reduction applies from $t=0$; the initial reduced law is given in Section~\ref{mfl:rb-sec:setup}.

\begin{lemma}[Per-particle signs]\label{mfl:mf-lem:signs}
Let $B_r:=\{\rho\in\R^r:\abs\rho\le1\}$, $f_{\min}:=\min_{B_r}f$, $f_{\max}:=\max_{B_r}f$ and $c_1:=f_{\min}/f_{\max}$. For every
$\rho\in B_r$:
\begin{enumerate}[label=(\roman*)]
\item $\frac\pi2(r-\sqrt r)\le f_{\min}\le f(\rho)\le f_{\max}\le\frac\pi2(r+\sqrt r)$. Hence $f_{\min}>0$, $c_1\ge(\sqrt r-1)/(\sqrt r+1)$,
and in every step of Proposition~\ref{mfl:mf-prop:updates}(b), $\delta\omega_1=f/\norm f\in[c_1,1/c_1]$ for every particle.
\item With $\bar\rho:=r^{-1}\sum_i\rho_i$, for every particle
\[\xi\cdot\Pc g_U=\cz\abs w\Big[\sum_{i=1}^r(\rho_i-\bar\rho)(\arcsin\rho_i-\rho_i)+(1-s)\abs{\Pc\rho}^2\Big]\ \ge\ \cz(1-s)\frac{\abs\xi^2}{\abs w}\ \ge0.\]
Equivalently $\xi\cdot h\ge(1-s)\abs\xi^2/\abs w\ge0$.
\item $0\le s\le\abs\rho^2\le1$, hence $b\in[0,1/\abs w]$.
\end{enumerate}
\end{lemma}
\begin{proof}
\step{Part (i): Bounds on the mean force.}
The function $f$ is continuous on the compact set $B_r$, so $f_{\min}$ and $f_{\max}$ exist. Fix $s\in[0,1]$ and put $\psi_s(x):=\arcsin x-sx$, an
odd function on $[-1,1]$. For $x\in[0,1]$, $x\le\arcsin x\le\frac\pi2x$: the first inequality holds because $\arcsin x-x$ is
nondecreasing and vanishes at $0$, the second by convexity of $\arcsin$ on $[0,1]$. Hence $0\le(1-s)x\le\psi_s(x)\le\frac\pi2x$
there, and by oddness $\abs{\psi_s(x)}\le\frac\pi2\abs x$ on $[-1,1]$. Since $f(\rho)=\frac{r\pi}2+\sum_i\psi_{s(\rho)}(\rho_i)$ and
$s(\rho)\in[0,1]$ by (iii),
\[\Big|f(\rho)-\frac{r\pi}2\Big|\le\frac\pi2\sum_i\abs{\rho_i}\le\frac\pi2\sqrt r\abs\rho\le\frac\pi2\sqrt r.\]
This gives the bounds, $f_{\min}\ge\frac\pi2(r-\sqrt r)>0$ for $r\ge2$, and $c_1\ge(\sqrt r-1)/(\sqrt r+1)$. Since
$f_{\min}\le\norm f\le f_{\max}$, we get $\delta\omega_1\in[f_{\min}/f_{\max},f_{\max}/f_{\min}]=[c_1,1/c_1]$.

\step{Part (ii): Sign of the contrast force.}
We have $\xi=\abs w\Pc\rho$ and $\Pc g_U=\cz h=\cz\Pc(\arcsin\rho-s\rho)$. Since $\Pc$ is an orthogonal projector,
\[\xi\cdot\Pc g_U=\cz\abs w\big[\Pc\rho\cdot\arcsin\rho-s\abs{\Pc\rho}^2\big],\qquad
\Pc\rho\cdot\arcsin\rho=\sum_i(\rho_i-\bar\rho)(\arcsin\rho_i-\rho_i)+\abs{\Pc\rho}^2.\]
This is the identity. The sum is nonnegative by Chebyshev's sum inequality: for any reals $x_i,y_i$ with
$\bar x:=r^{-1}\sum_ix_i$,
\begin{equation}\label{mfl:read-l2-chebyshev-sum}
\sum_i(x_i-\bar x)y_i=\frac1{2r}\sum_{i,j}(x_i-x_j)(y_i-y_j).
\end{equation}
In \eqref{mfl:read-l2-chebyshev-sum}, take $x_i=\rho_i$, $y_i=\arcsin\rho_i-\rho_i$. Every
term is $\ge0$, because $x\mapsto\arcsin x-x$ is nondecreasing. Finally $1-s\ge0$ by (iii), and $\abs{\Pc\rho}^2=\abs\xi^2/\abs w^2$.

\step{Part (iii): Bounds on the radial coefficient.}
Proposition~\ref{mfl:mf-prop:updates}(a) gives $s\in[\abs\rho^2/2,\abs\rho^2]$, and $\abs\rho=\abs{\hatw_U}\le1$.
\end{proof}

\begin{lemma}[Nondegenerate whitening at $\alpha=0$]\label{mfl:pc-lem:PC4}
For $\alpha=0$, every step $\etab>0$, every $\sigma_0>0$ and every $t\ge0$, $\nu_t$ is invariant under the group $\mathcal G\cong S_r\times O(\Dp)$ of
Proposition~\ref{mfl:mf-prop:symmetry}, absolutely continuous, and $\Sigma_{\nu_t}\succ0$.
\end{lemma}
\begin{proof}
Induction on $t$; $\nu_0$ is $O(d)$-invariant and absolutely continuous. Let $\nu:=\nu_t$ be $\mathcal G$-invariant and
absolutely continuous. At $\alpha=0$, $F=g(\hatw)$, and Proposition~\ref{mfl:mf-prop:symmetry}(b) gives
$\Sigma_\nu=a_1\Pone+a_c\Pc+a_\perp P_\perp$ with $a_1=\E[(\eb^\top g)^2]$, $a_c=\E[\abs{\Pc g}^2]/(r-1)$ and $a_\perp=\E[\abs{P_\perp g}^2]/\Dp$.

\step{Step 1 (The mean block is positive).}
We have $\eb^\top g=\cz f(\rho)/\sqrt r\ge\cz f_{\min}/\sqrt r>0$ everywhere
(Lemma~\ref{mfl:mf-lem:signs}(i)), so $a_1>0$.

\step{Step 2 (The perpendicular block is positive).}
The force $P_\perp g=-\cz sP_\perp\hatw$ with $s\ge\abs\rho^2/2$ is nonzero unless $w_U=0$ or $w_\perp=0$, a null set;
so $a_\perp>0$.

\step{Step 3 (The contrast block is positive).}
By Lemma~\ref{mfl:mf-lem:signs}(ii), $\Pc\rho\cdot\Pc g_U\ge\cz(1-s)\abs{\Pc\rho}^2$, and $s<1$ unless $\hatw=\pm u_i$: since
$x\mapsto1-\sqrt{1-x}$ is strictly convex and vanishes at $0$,
\[
s\le1-\sqrt{1-\abs\rho^2}\le1,
\]
with equality in both only if
$\abs\rho=1$ and at most one $\rho_i\ne0$. So $\Pc g\ne0$ off the null set $\{\Pc w_U=0\}\cup\{\hatw=\pm u_i\}$, and $a_c>0$.

The three positive blocks give $\Sigma_\nu\succ0$; $\nu_{t+1}$ is $\mathcal G$-invariant (Proposition~\ref{mfl:mf-prop:symmetry}(c)) and absolutely continuous
(Lemma~\ref{mfl:pc-lem:PC3}).
\end{proof}

\subsection{The MF AGOP and the criterion for (F)}\label{mfl:sec:agop}
For a $\mathcal G$-invariant $\nu$ with $\nu(\{0\})=0$ and $\E[\abs w^2]<\infty$ put
\[\mathsf v(X):=\E_{w\sim\nu}\big[\ind{w^\top X>0}\,w\big],\qquad
\mathcal M(\nu)=\E_X\big[\mathsf v(X)\mathsf v(X)^\top\big]=\E_{w,w'}\big[H(w,w')\,ww'^\top\big],\]
where the second equality is Fubini with independent $w,w'\sim\nu$ and $\Prob_X(w^\top X>0,w'^\top X>0)=(\pi-\angle(w,w'))/(2\pi)$.
The two statistics of the criterion are
\[\Gamma_U:=\frac{\E[\abs\xi^2/\abs w]}{r-1},\qquad\Gamma_\perp:=\frac{\E[\abs\zeta]}{\sqrt{\Dp-1/2}}.\]

\begin{proposition}[MF AGOP: block form and the (F) criterion]\label{mfl:mf-prop:agop}
Let $\nu$ be $\mathcal G$-invariant with $\nu(\{0\})=0$ and $\E[\abs w^2]<\infty$, with a reduced representation $(\omega,\zeta,n)$.
\begin{enumerate}[label=(\arabic*)]
\setcounter{enumi}{-1}
\item (Block form.) $\mathcal M(\nu)=\lambda_1\Pone+\lambda_c\Pc+\lambda_\perp P_\perp$, with
\begin{gather*}\lambda_1=\E_X\big[(\eb^\top\mathsf v(X))^2\big],\\
\lambda_c=\frac{\E_X\big[\abs{\E_w[(\ind{w^\top X>0}-\frac12)\,\xi]}^2\big]}{r-1},\\
\lambda_\perp=\frac{\E_X[\abs{P_\perp\mathsf v(X)}^2]}{\Dp}.\end{gather*}
Consequently $\Ar(\mathcal M(\nu))=1$ with a strict gap $\lambda_r>\lambda_{r+1}$ if and only if $\min(\lambda_1,\lambda_c)>\lambda_\perp$. In
that case the top-$r$ space is $U$ and $\lambda_r/\lambda_{r+1}=\min(\lambda_1,\lambda_c)/\lambda_\perp$.
\item $\lambda_c\ge\Gamma_U^2/(2\pi)$.
\item $\lambda_\perp\le\Gamma_\perp^2/(2\pi)$.
\item Consequently $\Ar=1$ with a strict gap as soon as $\Gamma_U>\Gamma_\perp$ and $\lambda_1>\lambda_\perp$.
\item If $\nu$ is $O(d)$-invariant (for instance $\nu=\nu_0$), then $\mathcal M(\nu)$ is a multiple of $I_d$.
\end{enumerate}
\end{proposition}
\begin{proof}
\step{Part (0): Block form and eigenspace criterion.}
For $Q\in\mathcal G$, substitute $w=Qw'$ with $w'\sim\nu$. This gives
\[
\mathsf v(QX)=\E_{w'}[\ind{w'^\top X>0}Qw']=Q\mathsf v(X).
\]
Since $X\overset d=QX$,
\begin{equation}\label{mfl:read-l2-agop-symmetry}
Q\mathcal M(\nu)Q^\top=\E_X[\mathsf v(QX)\mathsf v(QX)^\top]=\mathcal M(\nu).
\end{equation}
As in the proof of
Proposition~\ref{mfl:mf-prop:symmetry}(b), a symmetric matrix commuting with $\mathcal G$ has the form $\lambda_1\Pone+\lambda_c\Pc+\lambda_\perp P_\perp$,
with
\[
\begin{aligned}
\lambda_1&=\eb^\top\mathcal M\eb,\\
\lambda_c&=\tr(\Pc\mathcal M)/(r-1)=\E_X[\abs{\Pc\mathsf v}^2]/(r-1),\\
\lambda_\perp&=\tr(P_\perp\mathcal M)/\Dp.
\end{aligned}
\]
Next, $\Pc\mathsf v(X)=\E_w[\ind{w^\top X>0}\xi]$, and $\E_w[\xi]=0$: the vector $\E[\omega]$ is permutation-invariant, hence a multiple of
$\one$, and $\E[\xi]=\Pc\E[\omega]=0$. So $\frac12\E_w[\xi]$ may be subtracted.

For the criterion: if $\min(\lambda_1,\lambda_c)>\lambda_\perp$, the $r$ largest eigenvalues (with multiplicity) are $\lambda_1$ (on $\eb$)
and $\lambda_c$ (on the contrast space, multiplicity $r-1$). Their span is $U$, so $P_r=P_U$ is unique, $\Ar=1$, and
$\lambda_r=\min(\lambda_1,\lambda_c)$, $\lambda_{r+1}=\lambda_\perp$.

Conversely, suppose $\lambda_r>\lambda_{r+1}$ and $\Ar=1$ for the unique selector.
Then $\tr(P_UP_r)=r$ with both projectors of rank $r$, so $P_r=P_U$; the eigenvalues on $U$ are therefore the $r$ largest, and
$\lambda_{r+1}=\lambda_\perp<\min(\lambda_1,\lambda_c)$.

\step{Part (1): Lower-bound the contrast eigenvalue.}
Put $A(X):=\Pc\mathsf v(X)$ and use the test function $\Pc X_U$, where $X_U:=P_UX$. By the exact formula of
Section~\ref{mfl:sec:setting}, $\E_X[X_U\ind{w^\top X>0}]=\hatw_U/\sqrt{2\pi}=\omega/(\sqrt{2\pi}\abs w)$, hence by Fubini
\begin{equation}\label{mfl:read-l2-contrast-test}
\begin{aligned}
\E_X[A(X)\cdot\Pc X_U]
&=\E_w\Big[\xi\cdot\frac{\Pc\omega}{\sqrt{2\pi}\abs w}\Big]\\
&=\frac{\E[\abs\xi^2/\abs w]}{\sqrt{2\pi}}=\frac{(r-1)\Gamma_U}{\sqrt{2\pi}}.
\end{aligned}
\end{equation}
Since $\E_X[\abs{\Pc X_U}^2]=\tr\Pc=r-1$, Cauchy--Schwarz applied to \eqref{mfl:read-l2-contrast-test} gives
\[
(r-1)\lambda_c=\E_X[\abs A^2]\ge(r-1)^2\Gamma_U^2/(2\pi(r-1)),
\]
that is,
$\lambda_c\ge\Gamma_U^2/(2\pi)$.

\step{Part (2): Upper-bound the perpendicular eigenvalue.}
Write $X_\perp:=P_\perp X$, $\hat X_\perp:=X_\perp/\abs{X_\perp}$ ($X_\perp\ne0$ a.s.) and $q:=n^\top\hat X_\perp$. Since
$w^\top X=\omega^\top X_U+\sqrt{\Dp}\zeta\abs{X_\perp}q$ and $n$ is independent of $(\omega,\zeta)$,
\[P_\perp\mathsf v(X)=\E_{(\omega,\zeta)}\Big[\sqrt{\Dp}\,\zeta\;\E_n\big[\ind{\omega^\top X_U+\sqrt{\Dp}\zeta\abs{X_\perp}q>0}\,n\big]\Big].\]
Fix $X$ and $(\omega,\zeta)$ and decompose $n=q\hat X_\perp+n'$. Given $q$, the vector $n'$ is symmetric in $\hat X_\perp^\perp$, while the
gate depends on $n$ only through $q$. Hence $\E_n[\ind{\cdot}\,n]=\hat X_\perp\,\psi$ with $\psi:=\E_q[\ind{\cdot}\,q]$ and
$\abs\psi\le\E[q_+]=\tfrac12\E[\abs q]$, because $-\E[q_-]\le\E[\mathbf 1_Eq]\le\E[q_+]$ for any event $E$ and $q$ is symmetric.
Therefore, for every $X$,
\begin{equation}\label{mfl:read-l2-perpendicular-test}
\begin{aligned}
\abs{P_\perp\mathsf v(X)}&\le\sqrt{\Dp}\,\E[\abs\zeta]\,\E[\abs q]/2,\\
\lambda_\perp&\le(\E[\abs\zeta])^2(\E[\abs q])^2/4.
\end{aligned}
\end{equation}

To bound \eqref{mfl:read-l2-perpendicular-test}, use that $q$ is the
first coordinate of a uniform point of $S^{\Dp-1}$, so $\E[\abs q]=\Gamma(\frac{\Dp}2)/(\sqrt\pi\,\Gamma(\frac{\Dp+1}2))$. Kershaw's form of
Gautschi's inequality, $\Gamma(x+1)/\Gamma(x+\frac12)>(x+\frac14)^{1/2}$ for $x>0$, applied with $x=\frac{\Dp-1}2$ gives
$\E[\abs q]\le(2/(\pi(\Dp-\frac12)))^{1/2}$ for $\Dp\ge2$. Substituting in \eqref{mfl:read-l2-perpendicular-test} gives
\[
\lambda_\perp\le(\E[\abs\zeta])^2/(2\pi(\Dp-\frac12))=\Gamma_\perp^2/(2\pi).
\]

\step{Part (3): Combine the scalar bounds.}
By (1) and (2),
\[
\lambda_c\ge\Gamma_U^2/(2\pi)>\Gamma_\perp^2/(2\pi)\ge\lambda_\perp.
\]
Together with $\lambda_1>\lambda_\perp$, (0) applies.

\step{Part (4): Isotropy at initialization.}
As in \eqref{mfl:read-l2-agop-symmetry}, $\mathcal M(\nu)$ commutes with every $Q\in O(d)$;
commuting with every reflection, it is a multiple of $I_d$.
\end{proof}

\subsection{Continuity in the head scale and in the width}\label{mfl:pc-sec:cont}
Here $\etab>0$ is fixed and the head scale $\alpha\ge0$ varies; $F_\alpha$, $\Sigma_\alpha$, $T_{\mu,\alpha}$ and $\Phi_\alpha$ are the
objects of Section~\ref{mfl:pc-sec:def} at head scale $\alpha$.

\begin{lemma}[Continuity of $\Phi$ in the law and the head scale]\label{mfl:pc-lem:PC2}
Let $\alpha_k\to\alpha\ge0$, let $\mu_0\in\mathcal P_2$ with $\mu_0(\{0\})=0$ and $\Sigma_\alpha(\mu_0)\succ0$, and let
$W_2(\mu_k,\mu_0)\to0$. Then $\Sigma_{\alpha_k}(\mu_k)\to\Sigma_\alpha(\mu_0)$, $\sup_w\abs{T_{\mu_k,\alpha_k}(w)-T_{\mu_0,\alpha}(w)}\to0$ and
$W_2(\Phi_{\alpha_k}(\mu_k),\Phi_\alpha(\mu_0))\to0$.
\end{lemma}
\begin{proof}
Write $F_k:=F_{\alpha_k}(\cdot;\mu_k)$, $F_0:=F_\alpha(\cdot;\mu_0)$, $T_k:=T_{\mu_k,\alpha_k}$ and $T_0:=T_{\mu_0,\alpha}$.

\step{Step 1 (The force converges uniformly).}
First separate the change in head scale from the change in law:
\[
F_k-F_0=-(\alpha_k-\alpha)\,S(\hatw;\mu_k)-\alpha\big(S(\hatw;\mu_k)-S(\hatw;\mu_0)\big).
\]
Lemma~\ref{mfl:pc-lem:PC1}(b) gives
\begin{equation}\label{mfl:read-l2-force-convergence}
\sup_w\abs{F_k(w)-F_0(w)}\le\abs{\alpha_k-\alpha}\int\abs{w'}\,d\mu_k(w')+2\alpha\sqrt d\,W_2(\mu_k,\mu_0)\ \to\ 0,
\end{equation}
because $\int\abs{w'}\,d\mu_k\le(\int\abs{w'}^2d\mu_k)^{1/2}$ stays bounded ($W_2$-convergence implies convergence of second
moments). The $F_k$ are uniformly bounded: $\abs{F_k}\le\sup\abs g+\alpha_k\int\abs{w'}\,d\mu_k$.

\step{Step 2 (The whitening converges).}
Decompose the covariance difference as
\begin{equation}\label{mfl:read-l2-whitening-convergence}
\begin{aligned}
\Sigma_{\alpha_k}(\mu_k)-\Sigma_\alpha(\mu_0)
&=\int(F_kF_k^\top-F_0F_0^\top)\,d\mu_k\\
&\quad+\big(\int F_0F_0^\top d\mu_k-\int F_0F_0^\top d\mu_0\big).
\end{aligned}
\end{equation}
The first term tends to $0$ by \eqref{mfl:read-l2-force-convergence}, the second by the portmanteau
theorem, as $F_0F_0^\top$ is bounded and continuous off the $\mu_0$-null set $\{0\}$ (Lemma~\ref{mfl:pc-lem:PC1}(a),(c)).

\step{Step 3 (The maps converge).}
By \eqref{mfl:read-l2-whitening-convergence}, $\Sigma_{\alpha_k}(\mu_k)\succ0$ eventually and
$\Sigma_{\alpha_k}(\mu_k)^{-1/2}\to\Sigma_\alpha(\mu_0)^{-1/2}$.
Together with \eqref{mfl:read-l2-force-convergence}, this gives $T_k\to T_0$ uniformly
($T_{\mu,\alpha}(0)=0$ for every $\mu$ and $\alpha$).

\step{Step 4 (The laws converge).}
Coupling $T_k(w)$ with $T_0(w)$, $w\sim\mu_k$, gives
\begin{equation}\label{mfl:read-l2-pushforward-coupling}
W_2(\Phi_{\alpha_k}(\mu_k),\Phi_\alpha(\mu_0))\le\sup\abs{T_k-T_0}+W_2((T_0)_\#\mu_k,(T_0)_\#\mu_0).
\end{equation}
It remains to control the second term in \eqref{mfl:read-l2-pushforward-coupling}.
Now $T_0=\mathrm{id}+b_0$ with $b_0$ bounded and continuous $\mu_0$-a.e., so $(T_0)_\#\mu_k\to(T_0)_\#\mu_0$ weakly
(continuous mapping theorem). In
\[
\abs{T_0(w)}^2=\abs w^2+2w^\top b_0+\abs{b_0}^2,
\]
the integral of the first term converges,
the second is continuous $\mu_0$-a.e.\ with linear growth, hence uniformly integrable under $(\mu_k)$, and the third is bounded
and continuous $\mu_0$-a.e. Weak convergence together with convergence of second moments is $W_2$-convergence.
\end{proof}

\begin{proposition}[Small head and large width]\label{mfl:pc-prop:limit}
Fix $\etab,\sigma_0>0$ and $T\in\mathbb N$, and let $\nu_\alpha(t)$ be the MF trajectory from $\nu_0=N(0,\sigma_0^2I_d/d)$ at
head scale $\alpha\ge0$.
\begin{enumerate}[label=(\alph*)]
\item (Small head.) As $\alpha\to0$, for every $t\le T$: $W_2(\nu_\alpha(t),\nu_0(t))\to0$,
$\Sigma_\alpha(\nu_\alpha(t))\to\Sigma_0(\nu_0(t))\succ0$ and $\mathcal M(\nu_\alpha(t))\to\mathcal M(\nu_0(t))$. In particular there is
$\alpha_T>0$ such that $\Sigma_\alpha(\nu_\alpha(t))\succ0$ for all $t\le T$ and $0\le\alpha\le\alpha_T$.
\item (Large width.) Fix $\alpha\ge0$ and assume $\Sigma_\alpha(\nu_\alpha(t))\succ0$ for every $t<T$. Let $(w_j(0))_{j\ge1}$ be IID
$N(0,\sigma_0^2I_d/d)$, $W^{(m)}(0):=[w_1(0),\dots,w_m(0)]$, and run \eqref{mfl:eq:polar} at width $m$ with $\eta=\etab\sqrt m$ and
head $\alpha$. Then almost surely, for every $t\le T$: $W_2(\mu_{W^{(m)}(t)},\nu_\alpha(t))\to0$, $L(W^{(m)}(t))\to L(\nu_\alpha(t))$ and
$\AGOP(W^{(m)}(t))\to\alpha^2\mathcal M(\nu_\alpha(t))$. If moreover $\alpha>0$ and $\lambda_r>\lambda_{r+1}$ for $\mathcal M(\nu_\alpha(t))$, then
the top-$r$ eigenspace of $\AGOP(W^{(m)}(t))$ converges to that of $\mathcal M(\nu_\alpha(t))$.
\end{enumerate}
\end{proposition}
\begin{proof}
The common ingredient is continuity of $\mathcal L$ and $\mathcal M$. Their integrands extend continuously by $0$ to $w=0$ or
$w'=0$ and grow at most quadratically, and $W_2(\mu_k\otimes\mu_k,\nu\otimes\nu)\to0$ when $W_2(\mu_k,\nu)\to0$. Hence $\mathcal L$
and $\mathcal M$ are continuous on $(\mathcal P_2,W_2)$.

\step{Part (a): Small head.}
Let $\alpha_k\to0$. By Lemmas~\ref{mfl:pc-lem:PC3} and~\ref{mfl:pc-lem:PC4}, every $\nu_0(t)$ lies in $\mathcal P_2$, has no
atom at $0$ and satisfies $\Sigma_0(\nu_0(t))\succ0$, and every $\nu_{\alpha_k}(t)$ lies in $\mathcal P_2$. We show
$W_2(\nu_{\alpha_k}(t),\nu_0(t))\to0$ by induction on $t\le T$. At $t=0$ the laws coincide. If it holds at some $t<T$,
Lemma~\ref{mfl:pc-lem:PC2} with $\mu_k=\nu_{\alpha_k}(t)$, $\mu_0=\nu_0(t)$ and $\alpha=0$ gives
\[
W_2(\nu_{\alpha_k}(t+1),\nu_0(t+1))\to0.
\]
The same lemma at every $t\le T$, followed by continuity of $\mathcal M$, gives
\[
\begin{aligned}
\Sigma_{\alpha_k}(\nu_{\alpha_k}(t))&\to\Sigma_0(\nu_0(t)),\\
\mathcal M(\nu_{\alpha_k}(t))&\to\mathcal M(\nu_0(t)).
\end{aligned}
\]
The sequence $\alpha_k\to0$ was arbitrary, so the limits hold as $\alpha\to0$. Finally
the positive definite matrices form an open set and $t$ ranges over a finite set, which gives $\alpha_T$.

\step{Part (b): Large width.}
By Lemma~\ref{mfl:pc-lem:identity}, $\mu_{W^{(m)}(t)}=\Phi_\alpha^t(\mu_{W^{(m)}(0)})$.
Varadarajan's theorem and the strong law for $m^{-1}\sum_j\abs{w_j(0)}^2$ give
\begin{equation}\label{mfl:read-l2-empirical-initial-limit}
W_2(\mu_{W^{(m)}(0)},\nu_0)\to0
\end{equation}
almost surely. By Lemma~\ref{mfl:pc-lem:PC3}, $\nu_\alpha(s)\in\mathcal P_2$ and
$\nu_\alpha(s)(\{0\})=0$; with $\Sigma_\alpha(\nu_\alpha(s))\succ0$ for $s<T$, Lemma~\ref{mfl:pc-lem:PC2} (with $\alpha_k=\alpha$) applies at each $\nu_\alpha(s)$, $s<T$, and induction on $t$
gives $W_2$-continuity of $\Phi_\alpha^t$ at $\nu_0$ for $t\le T$.
Apply this continuity to \eqref{mfl:read-l2-empirical-initial-limit}. With Lemma~\ref{mfl:pc-lem:identity}(d) and the continuity of $\mathcal L$ and $\mathcal M$,
this gives the convergence of the loss and of the AGOP.

The last claim is the Davis--Kahan
theorem: with a strict gap, the spectral projector onto the top-$r$ eigenspace is continuous at $\mathcal M(\nu_\alpha(t))$ (for every
choice of top-$r$ eigenvectors of the approximants), and the factor $\alpha^2>0$ changes no eigenspace.
\end{proof}

\section{The plateau and the loss drop}\label{mfl:sec:tools}
\begin{secsummary}
\emph{What this section proves.} (P) (Theorem~\ref{mfl:plat-thm:PB}) and (R) (Corollary~\ref{mfl:loss-cor:MF}), for every head scale
$0<\alpha\le\astar(t_P)$. Only three scalar statistics of the particle law enter: the output scales $y$, $y_U$ and the mean
mode $x$ (Definition~\ref{mfl:tools-def:scales}). The loss is sandwiched as $1-2\sqrt{\kk_1}\,y_U\le L/V\le1-c_xx+\kk_1y^2$
(Lemma~\ref{mfl:tools-lem:sandwich}); one whitened step raises $y$ and $y_U$ by at most the clock increment
(Lemma~\ref{mfl:tools-lem:growth}); and the mean mode increases on the small-output window used below (Lemma~\ref{mfl:tools-lem:mfcapture}).

\emph{Where it is used.} Steps~1 and~5 of the proof of Theorem~\ref{mfl:intro-thm:main} (Section~\ref{mfl:sec:main}).

\emph{Standing hypothesis.} Apart from Lemmas~\ref{mfl:tools-lem:growth} and~\ref{mfl:tools-lem:mfcapture}, nothing in
Sections~\ref{mfl:tools-sec:scales}--\ref{mfl:tools-sec:mean} involves the dynamics: each statement holds for every particle law $\nu$ with
finite second moment and $\nu(\{0\})=0$, so that $\hatw$ is defined $\nu$-almost surely. Along the MF trajectory this holds at
every time (Lemma~\ref{mfl:pc-lem:PC3}).
\end{secsummary}

\subsection{Output scales and one MF step}\label{mfl:tools-sec:scales}
We retain the teacher amplitude $\gamma$, variance $V=\Var(Y)$, mean direction $\eb$ and constant
$\cz=\gamma/(2\pi)$ from Section~\ref{mfl:sec:setting}.
Here $\norm A$ for a matrix is the operator norm; all expectations carrying $\nu$ average over $w\sim\nu$,
while moments of functions of $X$ average over the Gaussian input.
For an input function $h$, $\norm{h}_2:=(\E_X[\abs{h(X)}^2])^{1/2}$.
$E_U:=[u_1,\dots,u_r]\in\R^{d\times r}$ and, for $w\in\R^d$, $w_U:=E_U^\top w\in\R^r$ (teacher coordinates) and
$w_\perp:=(I-P_U)w$. For $w\ne0$ put $\rho=\rho(w):=E_U^\top\hatw\in\R^r$, so $\abs\rho=\abs{w_U}/\abs w\le1$. For $v\in\R^r$ with
$\abs v\le1$ let
\[s(v):=\sum_{i=1}^r\big(1-\sqrt{1-v_i^2}\big),\qquad \tfrac12\abs v^2\le s(v)\le\abs v^2\le1,\]
since $t/2\le1-\sqrt{1-t}\le t$ on $[0,1]$. The constant $\kk_1=(\pi-1)/(2\pi)=0.3408450\ldots$ appears throughout.
Because $V=\gamma^2r\kk_1$,
\begin{equation}\label{mfl:tools-eq:gammaV}
\frac{\gamma\sqrt r}{\sqrt V}=\frac1{\sqrt{\kk_1}},\qquad \frac{\sqrt V\sqrt r}{\gamma}=r\sqrt{\kk_1}.
\end{equation}

\begin{definition}[Output scales and clock]\label{mfl:tools-def:scales}
For a law $\nu$ on $\R^d$ with finite second moment,
\begin{gather*}y(\nu):=\frac{\alpha\norm{\E_\nu[ww^\top]}^{1/2}}{\sqrt V},\\
y_U(\nu):=\frac{\alpha\norm{\E_\nu[w_Uw_U^\top]}^{1/2}}{\sqrt V},\\
x(\nu):=\frac{\alpha\,\E_\nu[\eb^\top w]}{\sqrt V}\quad\text{(the mean mode)}.\end{gather*}
Along the MF trajectory we write $y(t),y_U(t),x(t)$, $y_0:=y(0)$ and $y_{U0}:=y_U(0)$. The \emph{clock} is
$T(t):=\alpha\etab t/\sqrt V$, and $\Delta T:=T(1)=\alpha\etab/\sqrt V$.
\end{definition}
Always $\abs x\le y_U\le y$, since $\abs{\E_\nu[\eb^\top w]}\le(\E_\nu[(\eb^\top w)^2])^{1/2}\le\norm{\E_\nu[w_Uw_U^\top]}^{1/2}$. At width $m$ these
are the familiar quantities $\alpha\norm W/\sqrt{mV}$, $\alpha\norm{E_U^\top W}/\sqrt{mV}$ and $\alpha\,\eb^\top W\one_m/(m\sqrt V)$ evaluated at
$\mu_W$.

\begin{lemma}[One MF step]\label{mfl:tools-lem:growth}
Let $\nu$ have finite second moment, put $\delta w:=\etab\,\Sigma_\nu^{\dagger1/2}F(w;\nu)$ and $\nu^+:=\Phi(\nu)$, and let $P$ be the
orthogonal projector onto the range of $\Sigma_\nu$.
\begin{enumerate}[label=(\alph*)]
\item $\E_\nu[\delta w\,\delta w^\top]=\etab^2P$; in particular $\E_\nu[\abs{\delta w}^2]=\etab^2\rank\Sigma_\nu\le\etab^2d$.
\item $\norm{\E_{\nu^+}[ww^\top]}^{1/2}\le\norm{\E_\nu[ww^\top]}^{1/2}+\etab$, and the same holds for $w_U$. Hence along the MF trajectory
$y(t)\le y_0+T(t)$ and $y_U(t)\le y_{U0}+T(t)$.
\item $x(\nu^+)-x(\nu)=\Delta T\cdot p_{11}(\nu)$ with $p_{11}(\nu):=\E_\nu\big[\eb^\top\Sigma_\nu^{\dagger1/2}F(w;\nu)\big]$, and
$\abs{p_{11}}\le1$. Along the trajectory we write $p_{11}(t):=p_{11}(\nu_t)$.
\end{enumerate}
\end{lemma}
\begin{proof}
\step{Part (a): The increment covariance.}
The whitening identity is
\begin{equation}\label{mfl:read-l3-increment-covariance}
\E_\nu[\delta w\delta w^\top]=\etab^2\Sigma_\nu^{\dagger1/2}\Sigma_\nu\Sigma_\nu^{\dagger1/2}=\etab^2P.
\end{equation}
Taking the trace proves the stated RMS bound.

\step{Part (b): Growth of the output scales.}
For a unit $a\in\R^d$, $(\E_{\nu^+}[(a^\top w)^2])^{1/2}$ is the $L^2(\nu)$ norm of $a^\top w+a^\top\delta w$.
By Minkowski's inequality and \eqref{mfl:read-l3-increment-covariance}, it is at most
\[
\norm{a^\top w}_{L^2(\nu)}+\etab\abs{Pa}\le\norm{a^\top w}_{L^2(\nu)}+\etab.
\]
Take the supremum over
unit $a\in\R^d$ (respectively $a\in U$), and multiply by $\alpha/\sqrt V$; induct on $t$.
\step{Part (c): Growth of the mean mode.}
By definition,
\[
x(\nu^+)-x(\nu)=\alpha\,\E_\nu[\eb^\top\delta w]/\sqrt V=\Delta T\,p_{11}(\nu).
\]
Cauchy--Schwarz and \eqref{mfl:read-l3-increment-covariance} give
\[
\abs{p_{11}}\le(\E_\nu[(\eb^\top\Sigma_\nu^{\dagger1/2}F)^2])^{1/2}=\abs{P\eb}\le1.\qedhere
\]
\end{proof}

\subsection{The loss sandwich}\label{mfl:tools-sec:loss}
\begin{lemma}[Kernel expansion]\label{mfl:tools-lem:kernel}
$\kk(\rho)=\frac\rho4+\frac{k_2(\rho)}{2\pi}$ with $k_2(\rho):=\sqrt{1-\rho^2}+\rho\arcsin\rho-1=\sum_{n\ge1}b_n\rho^{2n}$, where every
$b_n>0$, $b_1=\frac12$ and $\sum_nb_n=k_2(1)=\frac\pi2-1$. Hence
\[\tfrac12\rho^2\le k_2(\rho)\le(\tfrac\pi2-1)\rho^2\qquad(\abs\rho\le1),\]
and $\kk(\rho)=\sum_{k\ge1}q_k\rho^k$ with $q_1=\frac14$, $q_{2n}=b_n/(2\pi)$ and $q_k=0$ otherwise. So all $q_k\ge0$ and
$\sum_kq_k=\kk_1$.
\end{lemma}
\begin{proof}
Insert $\pi-\arccos\rho=\frac\pi2+\arcsin\rho$ in the formula for $\kk$ of Section~\ref{mfl:sec:setting}. Next,
\[
k_2'(\rho)=\arcsin\rho,\qquad
k_2''(\rho)=(1-\rho^2)^{-1/2}=\sum_{n\ge0}\binom{2n}n4^{-n}\rho^{2n}.
\]
Integrating twice from $k_2(0)=k_2'(0)=0$ gives
\[
b_{n+1}=\binom{2n}n4^{-n}/((2n+1)(2n+2))>0,
\]
so $b_1=\frac12$.

The series converges
absolutely on $[-1,1]$, with sum $k_2(1)=\frac\pi2-1$ at $\rho=1$. Since $k_2(\rho)/\rho^2=\sum_nb_n\rho^{2n-2}$ is
nondecreasing in $\abs\rho$, it lies between $b_1=\frac12$ and $k_2(1)=\frac\pi2-1$. The coefficients $q_k$ are read
off, and $\sum_kq_k=\kk(1)=\kk_1$.
\end{proof}

\begin{lemma}[Loss decomposition]\label{mfl:tools-lem:lossdecomp}
Put $g_\nu(x):=\alpha\E_{w\sim\nu}[\relu(w^\top x)]$ as in Section~\ref{mfl:pc-sec:def}. Then $L(\nu)=\Var(Y-g_\nu(X))$ and
\begin{gather*}
L(\nu)=V-2\alpha\Phi(\nu)+\alpha^2Q(\nu),\qquad \Phi(\nu):=\E_\nu[\Cov(Y,\relu(w^\top X))]=\Phi_{\mathrm{lin}}(\nu)+\Phi_{\mathrm{nl}}(\nu),\\
\Phi_{\mathrm{lin}}(\nu):=\tfrac\gamma4\sqrt r\,\E_\nu[\eb^\top w],\qquad
\Phi_{\mathrm{nl}}(\nu):=\tfrac\gamma{2\pi}\E_\nu\Big[\abs w\sum_ik_2(\rho_i(w))\Big]\ge0,\\
Q(\nu):=\Var(\E_\nu[\relu(w^\top X)])=\E_{w,w'\sim\nu}[\abs w\abs{w'}\kk(\hatw^\top\hatw')].
\end{gather*}
In particular, exactly,
\begin{equation}\label{mfl:tools-eq:linx}
\frac{2\alpha\Phi_{\mathrm{lin}}(\nu)}V=c_x\,x(\nu),\qquad c_x:=\frac{\gamma\sqrt r}{2\sqrt V}=\frac1{2\sqrt{\kk_1}}=0.8564292\ldots
\end{equation}
\end{lemma}
\begin{proof}
$\mathcal L(\nu)=\Var(Y-g_\nu(X))$ is the definition of $\mathcal L$ in Section~\ref{mfl:pc-sec:def} written out with the covariances
$\Cov(Y,\relu(w^\top X))=\gamma\abs w\sum_i\kk(\rho_i)$ and $\Cov(\relu(w^\top X),\relu(w'^\top X))=\abs w\abs{w'}\kk(\hatw^\top\hatw')$
(homogeneity and the kernel $\kk$; Fubini's theorem applies by finite second moments). Expanding the variance gives the
first display. Insert $\kk(\rho)=\rho/4+k_2(\rho)/(2\pi)$ and use $\abs w\sum_i\rho_i=\sum_iu_i^\top w=\sqrt r\,\eb^\top w$.
Nonnegativity of $\Phi_{\mathrm{nl}}$ is $k_2\ge0$. For \eqref{mfl:tools-eq:linx},
$2\alpha\frac\gamma4\sqrt r\,\E_\nu[\eb^\top w]/V=\frac{\gamma\sqrt r}{2\sqrt V}x(\nu)$, and $\gamma\sqrt r/\sqrt V=1/\sqrt{\kk_1}$ by
\eqref{mfl:tools-eq:gammaV}.
\end{proof}

\begin{lemma}[Quadratic term]\label{mfl:tools-lem:Q}
For every law $\nu$ with finite second moment, $Q(\nu)\le\kk_1\norm{\E_\nu[ww^\top]}$.
\end{lemma}
\begin{proof}
By Lemma~\ref{mfl:tools-lem:kernel}, $Q(\nu)=\sum_kq_k\,\E_{w,w'}[(w^\top w')(\hatw^\top\hatw')^{k-1}]$. Let $\phi(\hatw):=\hatw^{\otimes(k-1)}$ with
coordinates $\phi_\iota$, so that $(\hatw^\top\hatw')^{k-1}=\sum_\iota\phi_\iota(\hatw)\phi_\iota(\hatw')$ and $\sum_\iota\phi_\iota(\hatw)^2=1$. Then
\begin{equation}\label{mfl:read-l3-tensor-bound}
\begin{aligned}
\E_{w,w'}[(w^\top w')(\hatw^\top\hatw')^{k-1}]
&=\sum_\iota\abs{\E_\nu[\phi_\iota(\hatw)w]}^2\\
&\le\sum_\iota\norm{\E_\nu[ww^\top]}\,\E_\nu[\phi_\iota(\hatw)^2]=\norm{\E_\nu[ww^\top]},
\end{aligned}
\end{equation}
because $\abs{\E[fw]}^2=\sup_{\abs a=1}(\E[f\,a^\top w])^2\le\E[f^2]\norm{\E[ww^\top]}$ by Cauchy--Schwarz.
Sum \eqref{mfl:read-l3-tensor-bound} over $k$, using $q_k\ge0$ and
$\sum_kq_k=\kk_1$.
\end{proof}

\begin{lemma}[Covariance bounds]\label{mfl:tools-lem:cov}
$\Phi(\nu)\le\kk_1\gamma\sqrt r\,\norm{\E_\nu[w_Uw_U^\top]}^{1/2}$ and
\[0\le\Phi_{\mathrm{nl}}(\nu)\le(\kk_1-\tfrac14)\gamma\,\E_\nu[\abs{w_U}^2/\abs w]\le(\kk_1-\tfrac14)\gamma\sqrt r\,
\norm{\E_\nu[w_Uw_U^\top]}^{1/2}.\]
\end{lemma}
\begin{proof}
Note $\kk_1-\frac14=\frac{\pi/2-1}{2\pi}$. By Lemma~\ref{mfl:tools-lem:kernel}, $\sum_ik_2(\rho_i)\le(\frac\pi2-1)\abs\rho^2$ and
$\abs w\abs\rho^2=\abs{w_U}^2/\abs w\le\abs{w_U}$. Moreover $\E_\nu[\abs{w_U}]\le(\tr\E_\nu[w_Uw_U^\top])^{1/2}\le\sqrt r\,\norm{\E_\nu[w_Uw_U^\top]}^{1/2}$.
This gives the bound on $\Phi_{\mathrm{nl}}$. With $\abs{\E_\nu[\eb^\top w]}\le\norm{\E_\nu[w_Uw_U^\top]}^{1/2}$,
$\Phi\le\frac\gamma4\sqrt r\,\norm{\E_\nu[w_Uw_U^\top]}^{1/2}+(\kk_1-\frac14)\gamma\sqrt r\,\norm{\E_\nu[w_Uw_U^\top]}^{1/2}$.
\end{proof}

\begin{lemma}[Loss sandwich]\label{mfl:tools-lem:sandwich}
For every law $\nu$ with finite second moment and $\nu(\{0\})=0$, with $y,y_U,x$ from Definition~\ref{mfl:tools-def:scales}:
\begin{enumerate}[label=(\alph*)]
\item $1-2\sqrt{\kk_1}\,y_U\ \le\ L/V\ \le\ 1-c_x\,x+\kk_1\,y^2$;
\item $L/V\ge1-c_x\,x-e_{\mathrm{nl}}$, where $e_{\mathrm{nl}}:=2\alpha\Phi_{\mathrm{nl}}(\nu)/V$ satisfies $0\le e_{\mathrm{nl}}\le c_{\mathrm{nl}}\,y_U$ with
$c_{\mathrm{nl}}:=2(\kk_1-\frac14)/\sqrt{\kk_1}$;
\item for the Gaussian law $\nu_0=N(0,\sigma_0^2I_d/d)$: $e_{\mathrm{nl}}(\nu_0)\le c_{\mathrm{nl}}\sqrt{r/d}\;y_0'$ with
$y_0':=\alpha\sigma_0/(\sqrt d\sqrt V)$.
\end{enumerate}
Numerically $2\sqrt{\kk_1}=1.1676387\ldots$, $c_x=0.8564292\ldots$ and $c_{\mathrm{nl}}=0.3112094\ldots$.
\end{lemma}
\begin{proof}
Throughout, use Lemma~\ref{mfl:tools-lem:lossdecomp}.

\step{Part (a): The two-sided loss bound.}
For the lower bound, drop $\alpha^2Q\ge0$ and use Lemma~\ref{mfl:tools-lem:cov}. By \eqref{mfl:tools-eq:gammaV},
\[
2\alpha\kk_1\gamma\sqrt r\,\norm{\E[w_Uw_U^\top]}^{1/2}/V
=2\kk_1\frac{\gamma\sqrt r}{\sqrt V}y_U=2\sqrt{\kk_1}y_U.
\]
For the upper bound, $\Phi\ge\Phi_{\mathrm{lin}}$ with \eqref{mfl:tools-eq:linx}, and
$\alpha^2Q/V\le\kk_1y^2$ by Lemma~\ref{mfl:tools-lem:Q}.

\step{Part (b): Keep the nonlinear covariance explicit.}
Drop $\alpha^2Q\ge0$ and write $\Phi=\Phi_{\mathrm{lin}}+\Phi_{\mathrm{nl}}$ with \eqref{mfl:tools-eq:linx}.
By Lemma~\ref{mfl:tools-lem:cov},
\[
\begin{aligned}
e_{\mathrm{nl}}&\le2\alpha(\kk_1-\frac14)\gamma\sqrt r\,\norm{\E[w_Uw_U^\top]}^{1/2}/V\\
&=2(\kk_1-\frac14)\frac{\gamma\sqrt r}{\sqrt V}y_U=c_{\mathrm{nl}}y_U.
\end{aligned}
\]

\step{Part (c): Use the Gaussian initial law.}
Under $\nu_0$, $w=\abs w\hatw$ with $\hatw$ uniform on the sphere and independent of $\abs w$. Thus
\begin{equation}\label{mfl:read-l3-gaussian-covariance}
\E[\abs{w_U}^2/\abs w]=\E[\abs w]\cdot\E[\abs{E_U^\top\hatw}^2]=(r/d)\E[\abs w]\le\sigma_0r/d.
\end{equation}
Substituting \eqref{mfl:read-l3-gaussian-covariance} into the covariance bound gives
\[
e_{\mathrm{nl}}\le2\alpha(\kk_1-\frac14)\gamma\sigma_0r/(dV)
=2(\kk_1-\frac14)\frac{\gamma\sqrt r}{\sqrt V}\sqrt{\frac rd}\,
\frac{\alpha\sigma_0}{\sqrt d\sqrt V}.\qedhere
\]
\end{proof}

\subsection{The self-force and the mean mode}\label{mfl:tools-sec:mean}
\begin{lemma}[Self-force]\label{mfl:tools-lem:selfforce}
For every law $\nu$ with finite second moment and every $w\ne0$,
\begin{gather*}F(w;\nu)-g(\hatw)=-\E\big[(g_\nu(X)-\E[g_\nu])\,X\,\ind{w^\top X>0}\big],\\
\abs{F(w;\nu)-g(\hatw)}\le\sqrt{\kk_1/2}\;y(\nu)\sqrt V.\end{gather*}
Here $\sqrt{\kk_1/2}=0.4128226\ldots$ In units of $\cz/\sqrt r$, the natural unit of the mean force (Lemma~\ref{mfl:tools-lem:psi}),
the bound reads
\begin{equation}\label{mfl:tools-eq:e}
\frac{\sqrt r}{\cz}\sqrt{\kk_1/2}\,y\sqrt V=e(y):=\sqrt2\,\pi\kk_1\,r\,y\approx1.5143347\,r\,y.
\end{equation}
The displayed decimal coefficient is rounded upward, so using it in a small-output threshold is conservative.
\end{lemma}
\begin{proof}
The identity is $F(w;\nu)=\E[R_\nu X\ind{w^\top X>0}]$ with $R_\nu=(Y-\E[Y])-(g_\nu-\E[g_\nu])$ (Lemma~\ref{mfl:pc-lem:identity}(a) and
its proof), whose teacher part is $g(\hatw)$. For a unit $a\in\R^d$, Cauchy--Schwarz gives
\begin{equation}\label{mfl:read-l3-self-force-test}
\abs{a^\top(F-g)}\le\norm{g_\nu-\E[g_\nu]}_2\,(\E[(a^\top X)^2\ind{w^\top X>0}])^{1/2}.
\end{equation}
The second factor is $\sqrt{1/2}$ by the symmetry $X\mapsto-X$, and Lemma~\ref{mfl:tools-lem:Q} gives
\[
\norm{g_\nu-\E[g_\nu]}_2^2=\alpha^2Q(\nu)\le\alpha^2\kk_1\norm{\E_\nu[ww^\top]}=\kk_1y^2V.
\]
Taking the supremum over unit $a$ in \eqref{mfl:read-l3-self-force-test} proves the norm bound.
For \eqref{mfl:tools-eq:e}, use $\sqrt V\sqrt r/\gamma=r\sqrt{\kk_1}$ from \eqref{mfl:tools-eq:gammaV} and $\cz=\gamma/(2\pi)$.
\end{proof}

For $\rho\in\R^r$ with $\abs\rho\le1$ let
\begin{gather*}
f(\rho):=\sum_{i=1}^r\Big(\frac\pi2+\arcsin\rho_i\Big)-s(\rho)\sum_{i=1}^r\rho_i,\\
\Psi=\Psi(r):=\frac\pi2-1+\frac{4}{3\sqrt6}\sqrt r,
\end{gather*}
and $\underline f:=\frac{r\pi}2-\Psi$, $\overline f:=\frac{r\pi}2+\Psi$. ($f$ is also used in Section~\ref{mfl:sec:mf}.)

\begin{lemma}[The mean force]\label{mfl:tools-lem:psi}
(a) $\eb^\top g(\hatw)=\frac{\cz}{\sqrt r}f(\rho(w))$. (b) $\abs{f(\rho)-\frac{r\pi}2}\le\Psi(r)$ for every $\abs\rho\le1$; so
$0<\underline f\le f\le\overline f$ for $r\ge2$.
\end{lemma}
\begin{proof}
\step{Part (a): Project the teacher force.}
By \eqref{mfl:eq:teacherforce}, $g(\hatw)=\cz\big[E_U(\frac\pi2\one_r+\arcsin\rho)-s(\rho)\hatw\big]$, and $\eb^\top E_U=\one_r^\top/\sqrt r$,
$\eb^\top\hatw=\sum_i\rho_i/\sqrt r$.
\step{Part (b): Bound the deviation from the central value.}
Let $h(t):=\arcsin t-t=\sum_{n\ge1}a_nt^{2n+1}$ with $a_n\ge0$. Since $h(t)/t^3$ is nondecreasing in $\abs t$,
$\abs{h(t)}\le h(1)\abs t^3=(\frac\pi2-1)\abs t^3$. Write $a:=\abs\rho\le1$. Then
\[
f(\rho)-\frac{r\pi}2=(1-s(\rho))\sum_i\rho_i+\sum_ih(\rho_i),
\]
with $0\le1-s\le1-a^2/2$, $\abs{\sum_i\rho_i}\le\sqrt r\,a$ and $\sum_i\abs{\rho_i}^3\le a^3$. Hence
\begin{equation}\label{mfl:read-l3-mean-force-deviation}
\abs{f(\rho)-\frac{r\pi}2}\le\sqrt r\,a(1-a^2/2)+(\frac\pi2-1)a^3.
\end{equation}
The first term on the right of \eqref{mfl:read-l3-mean-force-deviation} is at most
$(\frac23)^{3/2}\sqrt r$ (maximum at $a=\sqrt{2/3}$) and the second at most $\frac\pi2-1$. Finally $\underline f>0$ for $r\ge2$:
$\frac{r\pi}2-\Psi(r)$ is increasing in $r$ and equals $1.80\ldots$ at $r=2$.
\end{proof}

\begin{lemma}[MF mean-mode increment]\label{mfl:tools-lem:mfcapture}
Let $\mathcal G\cong S_r\times O(\Dp)$ be the group of Proposition~\ref{mfl:mf-prop:symmetry}, let $\nu$ be $\mathcal G$-invariant with
finite second moment and $\nu(\{0\})=0$, and put $y:=y(\nu)$, $e:=e(y)$ as in \eqref{mfl:tools-eq:e} and $u:=\Psi/\underline f$. Every
law $\nu_t$ of the MF trajectory satisfies these hypotheses, for every $\alpha\ge0$ (Proposition~\ref{mfl:mf-prop:symmetry}(c) and
Lemma~\ref{mfl:pc-lem:PC3}). The mean-mode increment of a particle is $\delta\omega_1(w):=\eb^\top(T_\nu(w)-w)=\etab\,\eb^\top\Sigma_\nu^{\dagger1/2}F(w;\nu)$.
\begin{enumerate}[label=(\alph*)]
\item $\eb$ is an eigenvector of $\Sigma_\nu$ with eigenvalue $\lambda_1:=\E_\nu[(\eb^\top F)^2]$. If $e<\underline f$, then $\lambda_1>0$
and every particle has
\[\frac{\delta\omega_1(w)}{\etab}=\frac{\eb^\top F(w;\nu)}{\sqrt{\lambda_1}}\ \ge\ \frac{\underline f-e}{\overline f+e}\ >0.\]
\item If $e\le\underline f$, then
\[p_{11}(\nu)=\frac{\E_\nu[\delta\omega_1]}{\etab}\ \ge\ c(y):=\frac{1-e(y)/\underline f}{\sqrt{1+u^2}+e(y)/\underline f}\ \ge0.\]
Set $p_{\mathrm{MF}}(y):=c(y)$ if $e(y)\le\underline f$ and $p_{\mathrm{MF}}(y):=-1$ otherwise. Then $p_{\mathrm{MF}}$ is nonincreasing and
$p_{11}(\nu)\ge p_{\mathrm{MF}}(y(\nu))$ for every such $\nu$.
\item $p_{\mathrm{MF}}(y;r)$ is nondecreasing in $r$. A sufficient condition for $e(y)<\underline f$ is $y<\underline f/(1.5143347\,r)$; the right side is
$.5946\ldots$ at $r=2$ and increases with $r$.
\end{enumerate}
\end{lemma}
\begin{proof}
\step{Part (a): Positivity of every mean increment.}
By $\mathcal G$-invariance, $Q\Sigma_\nu Q^\top=\Sigma_\nu$ for $Q\in\mathcal G$ (Proposition~\ref{mfl:mf-prop:symmetry}(a)), so
$Q\Sigma_\nu\eb=\Sigma_\nu Q\eb=\Sigma_\nu\eb$. The $\mathcal G$-fixed vectors are exactly $\spn\{\eb\}$: permutations fix only multiples
of $\one_r$ in $U$, and $-I_{U^\perp}\in\mathcal G$ fixes no nonzero vector of $U^\perp$. Therefore
\[
\Sigma_\nu\eb=(\eb^\top\Sigma_\nu\eb)\eb=\lambda_1\eb,
\]
and $\Sigma_\nu^{\dagger1/2}\eb=\lambda_1^{-1/2}\eb$ if $\lambda_1>0$ (and $=0$ if $\lambda_1=0$).

Write the mean force in its natural units:
\begin{equation}\label{mfl:read-l3-normalized-mean-force}
a(w):=\frac{\sqrt r}{\cz}\eb^\top F(w;\nu)=f(\rho(w))+\delta a(w).
\end{equation}
By Lemmas~\ref{mfl:tools-lem:psi}(a) and~\ref{mfl:tools-lem:selfforce},
$\abs{\delta a}\le e$, and $f\in[\underline f,\overline f]$ by Lemma~\ref{mfl:tools-lem:psi}(b). So $a\in[\underline f-e,\overline f+e]$. If $e<\underline f$ then $a>0$
everywhere, and
\[
\lambda_1=\frac{\cz^2}r\E_\nu[a^2]>0,
\qquad \sqrt{\lambda_1}\le\frac{\cz}{\sqrt r}(\overline f+e).
\]

\step{Part (b): Lower-bound the average increment.}
If $e=\underline f$, then $a\ge0$, so $\delta\omega_1\ge0$ and the bound $c=0$ holds.
Let $e<\underline f$. With $a$ from \eqref{mfl:read-l3-normalized-mean-force},
$\E_\nu[\delta\omega_1]/\etab=\E_\nu[a]/\norm a_{L^2(\nu)}$.
Let $M:=\E_\nu[f]\ge\underline f$ and
$\mathrm{sd}_\nu(f):=(\E_\nu[(f-\E_\nu[f])^2])^{1/2}$.
Since $f$ takes values in an interval of length $2\Psi$, $\mathrm{sd}_\nu(f)\le\Psi$;
we omit the subscript $\nu$ below. Also $\E_\nu[a]\ge M-e$ and, by Minkowski's inequality,
\[
\norm a_{L^2(\nu)}\le\norm f_{L^2(\nu)}+e=\sqrt{M^2+\mathrm{sd}(f)^2}+e.
\]
Hence
\begin{equation}\label{mfl:read-l3-average-increment}
\frac{\E_\nu[a]}{\norm a_{L^2(\nu)}}\ge\frac{1-e/M}{\sqrt{1+\mathrm{sd}(f)^2/M^2}+e/M}\ge\frac{1-e/\underline f}{\sqrt{1+u^2}+e/\underline f},
\end{equation}
because $(1-\epsilon)/(\sqrt{1+v^2}+\epsilon)$ is decreasing in $\epsilon\in[0,1]$ and in $v\ge0$.
The right side of \eqref{mfl:read-l3-average-increment} is $c(y)$. This function is decreasing in $y$ and
$c\ge0$ on $\{e\le\underline f\}$, so $p_{\mathrm{MF}}$ is nonincreasing; $p_{11}\ge-1$ always (Lemma~\ref{mfl:tools-lem:growth}(c)).
\step{Part (c): Monotonicity in the teacher rank.}
The two parameters are
\[
e/\underline f=\sqrt2\,\pi\kk_1\,y/(\frac\pi2-\Psi/r),
\qquad u=(\Psi/r)/(\frac\pi2-\Psi/r).
\]
Both decrease in $r$, because
$\Psi/r=(\frac\pi2-1)/r+(\frac23)^{3/2}/\sqrt r$ does, and $c$ decreases in $e/\underline f$ and in $u$. The set $\{e\le\underline f\}$ grows with
$r$ for the same reason. At $r=2$, $\underline f/(2\cdot1.5143347)\approx1.800997/3.028669\approx.59465$.
\end{proof}

\subsection{The plateau}\label{mfl:sec:plat}
By the loss sandwich (Lemma~\ref{mfl:tools-lem:sandwich}), $L/V-1$ is controlled by the output scales $y,y_U,x$. By
Lemma~\ref{mfl:tools-lem:growth}, these grow at most by the clock $T(t)$, which is linear in $\alpha t$. Two further facts give
the constant: in MF the initial scales are known exactly, and the mean mode $x$ does not decrease on the plateau window, so the linear term of the
loss can only lower it.

\begin{lemma}[Initial scales]\label{mfl:plat-lem:init}
For $\nu_0=N(0,\sigma_0^2I_d/d)$: $\E_{\nu_0}[ww^\top]=(\sigma_0^2/d)I_d$ and $\E_{\nu_0}[w]=0$. So, exactly,
\[y_0=y_{U0}=y_0':=\frac{\alpha\sigma_0}{\sqrt d\sqrt V}=\frac{\alpha\etab\kappa}{\sqrt V},\qquad x(0)=0.\]
\end{lemma}
\begin{proof}
Immediate from Definition~\ref{mfl:tools-def:scales} and $\sigma_0=\kappa\etab\sqrt d$ \eqref{mfl:eq:kappa}.
\end{proof}

\begin{lemma}[Mean-mode monotonicity]\label{mfl:plat-lem:mono}
If $y(t)<\underline f/(1.5143347\,r)$ (which is approximately $.5946$ at $r=2$ and increases with $r$), then every particle's $\eb^\top w$ increases
at step $t$, and $x(t+1)>x(t)$.
\end{lemma}
\begin{proof}
Lemma~\ref{mfl:tools-lem:mfcapture}(a) gives $\delta\omega_1>0$ for every particle, so $p_{11}(t)>0$ and
$x(t+1)-x(t)=\Delta T\,p_{11}(t)>0$ by Lemma~\ref{mfl:tools-lem:growth}(c).
\end{proof}

\begin{theorem}[Plateau]\label{mfl:plat-thm:PB}
Let $r\ge2$, $d>r$, $\kappa>0$ and $t_P\in\mathbb N$. Let $\bar y$ be the positive root of $1.17933\,\bar y+.34085\,\bar y^2=.01$, so
$\bar y=.0084587115\ldots$, and
\[\astar(t_P):=\frac{\bar y\sqrt V}{\etab(\kappa+t_P)}.\]
For every $\alpha\in(0,\astar(t_P)]$, the MF trajectory \eqref{mfl:eq:mfmap} from $\nu_0=N(0,\sigma_0^2I_d/d)$ satisfies
\[\max_{s,t\le t_P}\frac{L(t)}{L(s)}\le1.01\qquad\text{and}\qquad L(t)\ge.99V\quad(0\le t\le t_P),\]
and its mean mode is nondecreasing on $[0,t_P+1]$, with $x\ge0$ there.
\end{theorem}
\begin{proof}
\step{Step 1 (Control the output scales throughout the window).}
By Lemma~\ref{mfl:plat-lem:init}, $y_0=y_{U0}=\alpha\etab\kappa/\sqrt V$ and $x(0)=0$, and $T(t_P)=\alpha\etab t_P/\sqrt V$. Put
\begin{equation}\label{mfl:read-l3-window-scale}
Y:=y_0+T(t_P)=y_{U0}+T(t_P)=\alpha\etab(\kappa+t_P)/\sqrt V\le\bar y.
\end{equation}
By Lemma~\ref{mfl:tools-lem:growth}(b), $y(t)\le Y\le\bar y<.5946$
and $y_U(t)\le Y$ for $t\le t_P$. So by Lemma~\ref{mfl:plat-lem:mono} the mean mode increases at every step $t\le t_P$, and
$x(t)\ge x(0)=0$ for $t\le t_P+1$.

\step{Step 2 (Apply the loss sandwich).}
By Lemma~\ref{mfl:tools-lem:sandwich}(a), for $s,t\le t_P$,
\begin{equation}\label{mfl:read-l3-window-loss-bounds}
\begin{aligned}
\frac{L(t)}V&\le1-c_xx(t)+\kk_1y(t)^2\le1+\kk_1Y^2,\\
\frac{L(s)}V&\ge1-2\sqrt{\kk_1}\,y_U(s)\ge1-2\sqrt{\kk_1}\,Y.
\end{aligned}
\end{equation}
Now $2\sqrt{\kk_1}Y\le1.1676388\,\bar y<.00988$, which gives $L\ge.99V$.

\step{Step 3 (Verify the one-percent ratio).}
Using \eqref{mfl:read-l3-window-loss-bounds}, the required ratio is bounded by
\begin{equation}\label{mfl:read-l3-plateau-ratio}
\frac{L(t)}{L(s)}\le\frac{1+\kk_1Y^2}{1-2\sqrt{\kk_1}Y}\le1.01\iff\kk_1Y^2+1.01\cdot2\sqrt{\kk_1}\,Y\le.01.
\end{equation}
Since $\kk_1=.3408451\ldots\le.34085$ and $1.01\cdot2\sqrt{\kk_1}=1.1793151\ldots\le1.17933$, the last inequality in
\eqref{mfl:read-l3-plateau-ratio} follows from $1.17933\,Y+.34085\,Y^2\le.01$.
This holds because the left side is increasing in $Y\ge0$ and $Y\le\bar y$ by \eqref{mfl:read-l3-window-scale}.
\end{proof}

\begin{remark}[The constant $\bar y$]\label{mfl:plat-rem:ybarround}
The window enters only through $\kappa+t_P$: $\astar(t_P)=\bar y\sqrt V/(\etab\kappa(1+\tau_P))$ with $\tau_P:=t_P/\kappa$, and
$\bar y$ does not depend on $t_P$, $r$ or $d$. The decimal $.008459$, the root rounded to nearest, is \emph{not} admissible:
$1.17933\,(.008459)+.34085\,(.008459)^2=.0100003\ldots>.01$. The plateau estimate uses the exact positive root of the displayed quadratic; its decimal expansion is descriptive and is not substituted for that root.
\end{remark}

\subsection{The loss drop}\label{mfl:sec:loss}
The mechanism is the mean mode. By Lemma~\ref{mfl:tools-lem:mfcapture}, $x$ grows by at least $p_{\mathrm{MF}}(y)\,\Delta T$ per step.
This lowers the upper loss bound linearly, while the only penalty, $\kk_1y^2$, is quadratic in the output scale. The linear
part of the initial covariance cancels exactly.

\begin{theorem}[Loss-drop bound]\label{mfl:loss-thm:RLB}
Consider the MF trajectory \eqref{mfl:eq:mfmap} from $\nu_0=N(0,\sigma_0^2I_d/d)$ with $\alpha>0$, and let $t\in\mathbb N$ and $T:=T(t)$. Then
\[\frac{L(0)-L(t)}V\ \ge\ \mathcal D(T):=c_x\int_0^Tp_{\mathrm{MF}}(y_0+u)\,du-\kk_1(y_0+T)^2-e_{\mathrm{nl}},\]
with $p_{\mathrm{MF}}$ from Lemma~\ref{mfl:tools-lem:mfcapture}(b) and $0\le e_{\mathrm{nl}}\le c_{\mathrm{nl}}\sqrt{r/d}\,y_0$ (Lemma~\ref{mfl:tools-lem:sandwich}(c)).
\end{theorem}
\begin{proof}
\step{Step 1 (Cancel the initial linear covariance).}
Lemma~\ref{mfl:tools-lem:sandwich}(b) at time $0$ and (a) at time $t$ give $L(0)/V\ge1-c_xx(0)-e_{\mathrm{nl}}$ and
$L(t)/V\le1-c_xx(t)+\kk_1y(t)^2$. Subtracting,
\begin{equation}\label{mfl:read-l3-loss-difference}
\frac{L(0)-L(t)}V\ge c_x\big(x(t)-x(0)\big)-\kk_1y(t)^2-e_{\mathrm{nl}}.
\end{equation}
The linear part of the initial covariance has cancelled exactly by \eqref{mfl:tools-eq:linx}.

\step{Step 2 (Integrate the mean-mode lower bound).}
By Lemma~\ref{mfl:tools-lem:growth}(c) and
Lemma~\ref{mfl:tools-lem:mfcapture}(b), which applies at every step because every $\nu_s$ is $\mathcal G$-invariant with
$\nu_s(\{0\})=0$,
\begin{equation}\label{mfl:read-l3-mean-integral}
\begin{aligned}x(t)-x(0)&=\Delta T\sum_{s<t}p_{11}(s)\ \ge\ \Delta T\sum_{s<t}p_{\mathrm{MF}}(y(s))\\
&\ge\ \Delta T\sum_{s<t}p_{\mathrm{MF}}(y_0+s\Delta T)\ \ge\
\int_0^{t\Delta T}p_{\mathrm{MF}}(y_0+u)\,du.\end{aligned}
\end{equation}
The second inequality uses $y(s)\le y_0+s\Delta T$ (Lemma~\ref{mfl:tools-lem:growth}(b)) and that $p_{\mathrm{MF}}$ is nonincreasing; the
third is the left Riemann sum of a nonincreasing function.
Substitute \eqref{mfl:read-l3-mean-integral} into \eqref{mfl:read-l3-loss-difference}, using $y(t)\le y_0+T$ and $c_x>0$.
\end{proof}

\begin{corollary}[Loss drop]\label{mfl:loss-cor:MF}
Let $r\ge2$, $d\ge64r$, $\kappa\ge1$, $t_P\in\mathbb N$ and $\alpha\in(0,\astar(t_P)]$, with $\astar$ and $\bar y$ as in
Theorem~\ref{mfl:plat-thm:PB}. Put $T_2^*:=.05086$. Then
\[L(0)-L(t_2)\ge.03\,V\qquad\text{at}\qquad t_2:=\Big\lceil\frac{T_2^*\sqrt V}{\alpha\etab}\Big\rceil,\]
that is, once the clock $T(t)$ has reached $T_2^*$. At $\alpha=\astar(t_P)$ this gives $t_2\le C_R(\kappa+t_P)+1$ with
$C_R:=T_2^*/\bar y=6.0127\ldots\le6.013$; in particular $t_2\le6.62\,\kappa+1$ whenever $t_P\le.1007\,\kappa$. For
$0<\alpha<\astar(t_P)$, $t_2=\lceil(\astar(t_P)/\alpha)\,C_R(\kappa+t_P)\rceil$.
\end{corollary}
\begin{proof}
\step{Step 1 (Reduce to the largest initial scale and smallest rank).}
Write $y_0'=y_0=\alpha\etab\kappa/\sqrt V$ (Lemma~\ref{mfl:plat-lem:init}); $y_0'\le\bar y$ because $\alpha\le\astar(t_P)\le\bar y\sqrt V/(\etab\kappa)$.
Since $\sqrt{r/d}\le\frac18$, Theorem~\ref{mfl:loss-thm:RLB} gives $(L(0)-L(t))/V\ge\mathcal D(T(t);y_0',r)$, where
\[\mathcal D(T;y_0',r):=c_x\int_0^Tp_{\mathrm{MF}}(y_0'+u;r)\,du-\kk_1(y_0'+T)^2-\frac{c_{\mathrm{nl}}}8\,y_0'.\]
This function is nondecreasing in $r$, because $p_{\mathrm{MF}}(\cdot;r)$ is (Lemma~\ref{mfl:tools-lem:mfcapture}(c)), and pointwise
nonincreasing in $y_0'$, because $p_{\mathrm{MF}}$ is nonincreasing and the last two terms decrease in $y_0'$. Hence
\begin{equation}\label{mfl:read-l3-worst-case-loss-drop}
\mathcal D(T;y_0',r)\ge\mathcal D^*(T):=\mathcal D(T;\bar y,2).
\end{equation}
The function $\mathcal D^*$ is concave (the integral of a nonincreasing
function is concave, and $-\kk_1(\bar y+T)^2$ is concave), so its superlevel set $\{\mathcal D^*\ge.03\}$ is an interval.
It therefore suffices to certify two endpoints for the lower bound in \eqref{mfl:read-l3-worst-case-loss-drop}.

\step{Step 2 (Bound the mean increment by a rational function).}
We certify the endpoints using rational bounds. Put
\begin{equation}\label{mfl:read-l3-rational-increment}
\begin{gathered}
Y:=.008458712,\qquad a_+:=1.681668,\qquad b_+:=1.246628,
\\ p_-(y):=\frac{1-a_+y}{b_+ +a_+y}.
\end{gathered}
\end{equation}
The defining quadratic for $\bar y$ gives $\bar y<Y$.
At $r=2$, write $a=(e(y)/y)/\underline f$ and
$b=\sqrt{1+(\Psi/\underline f)^2}$, so that
$p_{\mathrm{MF}}(y;2)=(1-ay)/(b+ay)$ on the interval used here.
The bounds
\[
\begin{gathered}
\sqrt2<1.414214,\qquad
\pi>3.14159265358979323846264338327950288,\\
\Psi(2)\le .570797+.544332\sqrt2,\qquad
e(y)/y\le3.028672
\end{gathered}
\]
imply $a\le a_+$ and $b\le b_+$.
These last two inequalities follow by rational cross-multiplication
(and squaring for the second). Since $a_+(Y+.1)<1$, the numerator
stays positive and $p_{\mathrm{MF}}(y;2)\ge p_-(y)$ for
$0\le y\le Y+.1$.

\step{Step 3 (Integrate and bound the logarithm).}
Integrating the rational function in \eqref{mfl:read-l3-rational-increment} and using
$c_x\ge.856429$, $\kk_1\le.340846$ and $c_{\mathrm{nl}}/8\le.0389013$, we obtain, for $0\le T\le.1$,
\begin{equation}\label{mfl:read-l3-logarithmic-loss-bound}
\begin{aligned}
\mathcal D^*(T)&\ge .856429\left[
\frac{1+b_+}{a_+}\log\frac{b_+ +a_+(Y+T)}{b_+ +a_+Y}-T\right]
\\&\quad-.340846(Y+T)^2-.0389013Y.
\end{aligned}
\end{equation}
For $z(T):=a_+T/[2(b_+ +a_+Y)+a_+T]$, let
\begin{equation}\label{mfl:read-l3-rational-loss-bound}
\begin{aligned}
\ell_{12}(T)&:=2\sum_{j=0}^{11}\frac{z(T)^{2j+1}}{2j+1},\\
Q(T)&:=.856429\left[\frac{1+b_+}{a_+}\ell_{12}(T)-T\right]
\\&\quad-.340846(Y+T)^2-.0389013Y.
\end{aligned}
\end{equation}
The positive-term expansion
\[
\log((1+z)/(1-z))=2\sum_{j\ge0}z^{2j+1}/(2j+1),
\]
valid for $0\le z<1$, gives $\mathcal D^*(T)\ge Q(T)$ when substituted in
\eqref{mfl:read-l3-logarithmic-loss-bound}.

\step{Step 4 (Certify the two endpoints).}
All quantities defining $Q$ at rational $T$ are rational, and direct
cross-multiplication in \eqref{mfl:read-l3-rational-loss-bound} gives
\begin{equation}\label{mfl:read-l3-certified-endpoints}
Q(.05086)>.030007,\qquad Q(.1)>.053264.
\end{equation}
By \eqref{mfl:read-l3-certified-endpoints} and concavity, $\mathcal D^*(T)\ge.03$ throughout
$[.05086,.1]$.

\step{Step 5 (Convert the clock bound to an integer time).}
By the definition of $t_2$,
\[
T(t_2)\in[T_2^*,T_2^*+\Delta T],\qquad
\Delta T=\alpha\etab/\sqrt V\le\bar y/(\kappa+t_P)<.0085.
\]
Thus $T(t_2)$ lies in the certified interval and $L(0)-L(t_2)\ge.03V$.
At $\alpha=\astar(t_P)$, $\sqrt V/(\alpha\etab)=(\kappa+t_P)/\bar y$, which gives the bound on $t_2$; for
$t_P\le.1007\kappa$, $6.0128\cdot1.1007<6.62$. The case $\alpha<\astar(t_P)$ is the same identity.
\end{proof}

\section{The \texorpdfstring{$\alpha=0$}{alpha = 0} trajectory from the Gaussian initialization}\label{mfl:sec:rb}
\begin{secsummary}
\emph{What this section proves.} Facts about the $\alpha=0$ trajectory from $\nu_0=N(0,\kappa^2I_d)$ that hold uniformly in
$r\ge2$, $\Dp$ and $\kappa$: the initial statistics (Proposition~\ref{mfl:rb-prop:init}); per-particle monotonicity, in particular
$\abs\xi$ never decreases (Lemma~\ref{mfl:rb-lem:facts}); the $\perp$-statistic $\Gamma_\perp$ grows only through rare overshoots
(the kick budget, Lemma~\ref{mfl:rb-lem:K}); the mean mode grows linearly and lifts $\lambda_1$ (Lemma~\ref{mfl:rb-lem:M}); and an exact
three-part decomposition of the one-step change of $\Phi=(r-1)\Gamma_U$ (Lemma~\ref{mfl:rb-lem:C}).

\emph{Where it is used.} These are the inputs (F1)--(F7) of the certificate (Section~\ref{mfl:mff-sec:inputs}); Lemma~\ref{mfl:rb-lem:M}
and Lemma~\ref{mfl:rb-lem:facts}(e) are also used in Step~2 of Section~\ref{mfl:sec:main}. The conventions of Section~\ref{mfl:sec:mf} are in force.
\end{secsummary}

\subsection{Setup and initial law}\label{mfl:rb-sec:setup}
Recall $D=d-r$, $\Pone=\eb\eb^\top$ and $\Pc=P_U-\Pone$. In the reduced coordinates of
Definition~\ref{mfl:mf-def:reduced}, $w=(\omega,\sqrt D\,\zeta n)$ with $n\in S^{D-1}$,
$\omega_1=\eb^\top\omega$ and $\xi=\Pc\omega$; teacher-space vectors and projectors are identified with their
$r$-dimensional coordinate representations. Unsubscripted expectations and $L^2$ norms below use the particle law.
In units $\etab=1$ the initial law is $\nu_0=N(0,\kappa^2I_d)$. Here $\chi_k$ denotes a chi random variable with $k$ degrees of
freedom, so $\E[\chi_k]=\sqrt2\,\Gamma(\frac{k+1}2)/\Gamma(\frac k2)$, and we put $v:=\chi_{\Dp}/\sqrt{\Dp}$. Tilde units are
$\tilde\omega:=\omega/\kappa$, $\tilde\zeta:=\zeta/\kappa$ and $\tau:=t/\kappa$.

\begin{lemma}[Initial reduced law]\label{mfl:rb-lem:init}
$\nu_0$ is $O(d)$-invariant, hence $\mathcal G$-invariant, so the reduction of Section~\ref{mfl:sec:mf} applies from $t=0$. A reduced
representation of $w(0)$ is
\[\omega(0)=\kappa z,\qquad\zeta(0)=\kappa v,\qquad n,\]
with $z\sim N(0,I_r)$, $v=\chi_{\Dp}/\sqrt{\Dp}$ and $n$ uniform on $S^{\Dp-1}$, all independent. Equivalently $w(0)=\kappa\mathsf g$ with
$\mathsf g\sim N(0,I_d)$. At $\alpha=0$, every $\nu_t$ is $\mathcal G$-invariant (Proposition~\ref{mfl:mf-prop:symmetry}(c)) and
$\Sigma_{\nu_t}\succ0$ for every $t\ge0$ (Lemma~\ref{mfl:pc-lem:PC4}). So Propositions~\ref{mfl:mf-prop:updates} and~\ref{mfl:mf-prop:agop} hold at every $t\ge0$.
\end{lemma}
\begin{proof}
$w_U(0)=\kappa z$ and $w_\perp(0)=\kappa\mathsf g_\perp$ with $\mathsf g_\perp\sim N(0,I_{\Dp})$ independent of $z$. The radius $\abs{\mathsf g_\perp}=\chi_{\Dp}$
and the direction $n:=\mathsf g_\perp/\abs{\mathsf g_\perp}$ are independent, and $n$ is uniform. So $\zeta(0)=\abs{w_\perp(0)}/\sqrt{\Dp}=\kappa v$.
\end{proof}

\paragraph{Notation.} Per particle we use the quantities of Definition~\ref{mfl:mf-def:reduced}: $\abs w$, $\rho$, $s$, $b=s/\abs w$,
$\omega_1$, $\xi=\Pc\omega$, and the \emph{contrast cosine} $k_p:=\abs{\Pc\rho}=\abs\xi/\abs w$. By Proposition~\ref{mfl:mf-prop:updates}(b)
the updates are
\begin{equation}\label{mfl:read-l4-updates}
\begin{aligned}
\delta\omega_1&=\frac f{\norm f},&
\delta\xi&=\frac h{\sqrt{\bar a}},&
\bar a&=\frac{\E[\abs h^2]}{r-1},\\
\zeta^+&=\zeta\Big(1-\frac bN\Big),&&&N&=\norm{b\zeta}.
\end{aligned}
\end{equation}
Normalizations at time $t$ carry a subscript: $N_t$, $\bar a_t$, and $b_t$ for the per-particle $b$. The statistics are
\[\Phi:=\E\Big[\frac{\abs\xi^2}{\abs w}\Big]=(r-1)\Gamma_U,\qquad\Gamma_\perp:=\frac{\E[\abs\zeta]}{\sqrt{\Dp-1/2}},\qquad
\sigma_c^2:=\frac{\E[\abs\xi^2]}{r-1},\]
and $\lambda_1,\lambda_c,\lambda_\perp$ are the eigenvalues of $\mathcal M(\nu_t)$ (Proposition~\ref{mfl:mf-prop:agop}).

\begin{proposition}[Initial values]\label{mfl:rb-prop:init}
\[\Gamma_U(0)=\kappa\,\frac{\E[\chi_d]}d,\qquad\Phi(0)=\kappa(r-1)\frac{\E[\chi_d]}d,\qquad
\Gamma_\perp(0)=\frac{\kappa\,\E[v]}{\sqrt{\Dp-1/2}},\qquad\sigma_c(0)=\kappa,\]
\[R_0:=\frac{\Gamma_U(0)}{\Gamma_\perp(0)}=\frac{\E[\chi_d]}{d}\cdot\frac{\sqrt{\Dp}\sqrt{\Dp-1/2}}{\E[\chi_{\Dp}]}.\]
\end{proposition}
\begin{proof}
By Lemma~\ref{mfl:rb-lem:init}, $w(0)=\kappa\mathsf g$. Write $\mathsf g=\abs{\mathsf g}\,e$ with $e$ uniform on $S^{d-1}$ and independent of
$\abs{\mathsf g}=\chi_d$. The contrast statistic separates into a radius and a direction:
\[
\abs\xi^2/\abs w=\kappa\abs{\mathsf g}\,\abs{\Pc e_U}^2,
\qquad \E[\abs{\Pc e_U}^2]=\tr(\Pc)/d=(r-1)/d.
\]
Thus $\Phi(0)=\kappa\,\E[\chi_d]\,(r-1)/d$. The remaining initial values follow from
\[
\E[\abs{\zeta(0)}]=\kappa\E[v],
\qquad \sigma_c(0)^2=\kappa^2\E[\abs{\Pc z}^2]/(r-1)=\kappa^2.\qedhere
\]
\end{proof}

\begin{lemma}[Per-particle facts]\label{mfl:rb-lem:facts}
For every $t\ge0$:
\begin{enumerate}[label=(\alph*)]
\item $\delta\omega_1(t)\in[c_1,1/c_1]$ for every particle, with $c_1=f_{\min}/f_{\max}\ge(\sqrt r-1)/(\sqrt r+1)$;
\item $\abs{\xi(t)}$ is nondecreasing in $t$ for every particle;
\item $\E[\abs{\delta\xi(t)}^2]=r-1$, hence $\sigma_c(t)\le\kappa+t$;
\item $\E[(\delta\zeta(t))^2]=1$, hence $\norm{\zeta(t)-\zeta(0)}\le t$;
\item $\abs{\xi(t)}>0$ and $\Pc\rho\ne0$ for almost every particle, hence $\Phi(t)>0$ and $\Gamma_U(t)>0$.
\end{enumerate}
\end{lemma}
\begin{proof}
\step{Part (a): Mean increment.}
This is Lemma~\ref{mfl:mf-lem:signs}(i).

\step{Part (b): Contrast monotonicity.}
The contrast update in \eqref{mfl:read-l4-updates} and Lemma~\ref{mfl:mf-lem:signs}(ii) give
\begin{equation}\label{mfl:read-l4-contrast-monotonicity}
\abs{\xi^+}^2-\abs\xi^2=2\,\xi\cdot h/\sqrt{\bar a}+\abs{\delta\xi}^2\ge0.
\end{equation}

\step{Parts (c) and (d): Second-moment bounds.}
The identities are Proposition~\ref{mfl:mf-prop:updates}(c). By Minkowski,
\[
\begin{aligned}
\norm{\xi(t)}&\le\norm{\xi(0)}+t\sqrt{r-1}=\sqrt{r-1}(\kappa+t),\\
\norm{\zeta(t)-\zeta(0)}&\le\sum_{s<t}\norm{\delta\zeta(s)}=t.
\end{aligned}
\]

\step{Part (e): Strict positivity.}
Initially $\xi(0)=\kappa\Pc z\ne0$ almost surely. By \eqref{mfl:read-l4-contrast-monotonicity},
$\abs{\xi(t)}\ge\abs{\xi(0)}$; also $\Pc\rho=\xi/\abs w$.
\end{proof}

\subsection{The kick budget and the mean mode}
Define the \emph{overshoot mass} and the \emph{kick budget}
\[\nu_{\rm os}(t):=\nu_t\big(b_t>2N_t\big),\qquad K_\nu(t):=\frac1\kappa\sum_{s<t}\sqrt{\nu_{\rm os}(s)}.\]

\begin{lemma}[Kick budget: $\Gamma_\perp$ grows only through overshoot]\label{mfl:rb-lem:K}
Every particle with $b_t\le2N_t$ has $\abs{\zeta(t+1)}\le\abs{\zeta(t)}$. Moreover $\E[\abs{\zeta(t+1)}]\le\E[\abs{\zeta(t)}]+\sqrt{\nu_{\rm os}(t)}$,
hence $\Gamma_\perp(t)\le\Gamma_\perp(0)\big(1+K_\nu(t)/\E[v]\big)$.
\end{lemma}
\begin{proof}
\step{Step 1 (Isolate the overshoots).}
By \eqref{mfl:read-l4-updates}, $\zeta^+=\zeta(1-b/N)$ with $b\ge0$. If $b\le2N$ then $\abs{1-b/N}\le1$ and $\abs{\zeta^+}\le\abs\zeta$.
If $b>2N$ then
\[
\abs{\zeta^+}-\abs\zeta=\abs\zeta(b/N-2)\le b\abs\zeta/N.
\]
Thus the pointwise bound is
\begin{equation}\label{mfl:read-l4-overshoot-increment}
\abs{\zeta^+}-\abs\zeta\le\ind{b>2N}\,b\abs\zeta/N.
\end{equation}

\step{Step 2 (Average and sum).}
Cauchy--Schwarz bounds the numerator in \eqref{mfl:read-l4-overshoot-increment}:
\[
\E[\ind{b>2N}\,b\abs\zeta]\le\sqrt{\nu_{\rm os}}\,\norm{b\zeta}=\sqrt{\nu_{\rm os}}\,N.
\]
Summing over $s<t$ gives
\[
\E[\abs{\zeta(t)}]\le\E[\abs{\zeta(0)}]+\kappa K_\nu(t)=\kappa(\E[v]+K_\nu(t)).
\]
Divide by $\sqrt{\Dp-1/2}$ and use
$\Gamma_\perp(0)=\kappa\E[v]/\sqrt{\Dp-1/2}$ (Proposition~\ref{mfl:rb-prop:init}).
\end{proof}

\begin{lemma}[Mean mode]\label{mfl:rb-lem:M}
For every $t\ge0$, $\E[\delta\omega_1(t)]\ge2\sqrt{f_{\min}f_{\max}}/(f_{\min}+f_{\max})\ge\sqrt{1-1/r}$. Hence $\E[\omega_1(t)]\ge\sqrt{1-1/r}\;t$ and
$\lambda_1(t)\ge(\E[\omega_1(t)])^2/4$.
\end{lemma}
\begin{proof}
\step{Step 1 (The P\'olya--Szeg\H{o} bound).}
Let $Z\in[a,b]$ a.s.\ with $0<a\le b$. Then $(Z-a)(b-Z)\ge0$ gives $\E[Z^2]\le(a+b)\E[Z]-ab$.
Together with $((a+b)\E[Z]-2ab)^2\ge0$, this yields
\[
(a+b)^2(\E[Z])^2\ge4ab[(a+b)\E[Z]-ab]\ge4ab\,\E[Z^2].
\]
Equivalently,
\begin{equation}\label{mfl:read-l4-polya-szego}
\E[Z]/\sqrt{\E[Z^2]}\ge2\sqrt{ab}/(a+b).
\end{equation}

\step{Step 2 (Accumulate the mean increment).}
Apply \eqref{mfl:read-l4-polya-szego} to $Z=f(\rho)$ under $\nu_t$, with $a=f_{\min}$ and $b=f_{\max}$
(Lemma~\ref{mfl:mf-lem:signs}(i)). Then
\[
\E[\delta\omega_1]=\E[f]/\norm f\ge2\sqrt{c_1}/(1+c_1).
\]
The map $c\mapsto2\sqrt c/(1+c)$ is nondecreasing on $(0,1]$, and
$c_1\ge c_r:=(\sqrt r-1)/(\sqrt r+1)$, for which $2\sqrt{c_r}/(1+c_r)=\sqrt{(r-1)/r}$.
Since $\E[\omega_1(0)]=\kappa\,\E[\eb^\top z]=0$, we obtain
\begin{equation}\label{mfl:read-l4-mean-growth}
\E[\omega_1(t)]=\sum_{s<t}\E[\delta\omega_1(s)]\ge\sqrt{1-1/r}\,t.
\end{equation}

\step{Step 3 (Pass to the AGOP eigenvalue).}
To turn \eqref{mfl:read-l4-mean-growth} into an eigenvalue bound, test with the constant function:
\[
\begin{aligned}
\lambda_1&=\E_X[(\eb^\top\mathsf v(X))^2]\ge(\E_X[\eb^\top\mathsf v(X)])^2,\\
\E_X[\eb^\top\mathsf v(X)]&=\E_w[\Prob_X(w^\top X>0)\,\omega_1]=\frac12\E[\omega_1].\qedhere
\end{aligned}
\]
\end{proof}

\subsection{The contrast step}
Write $\varphi(x,y):=x^2/\sqrt{x^2+y^2}$ for $x,y\ge0$ (and $\varphi(0,0):=0$). For each particle put
\[y:=(\omega_1^2+\Dp\zeta^2)^{1/2},\qquad y_m:=((\omega_1+\delta\omega_1)^2+\Dp\zeta^2)^{1/2},\qquad
y^+:=((\omega_1+\delta\omega_1)^2+\Dp(\zeta^+)^2)^{1/2}.\]
Since $\abs w^2=\omega_1^2+\abs\xi^2+\Dp\zeta^2$, we have $\abs\xi^2/\abs w=\varphi(\abs\xi,y)$ and
$\abs{\xi^+}^2/\abs{w^+}=\varphi(\abs{\xi^+},y^+)$. Define the contrast, mean-mode and $\perp$ steps
\begin{gather*}dU:=\varphi(\abs{\xi^+},y)-\varphi(\abs\xi,y),\\
dM:=\varphi(\abs{\xi^+},y_m)-\varphi(\abs{\xi^+},y),\\
dP:=\varphi(\abs{\xi^+},y^+)-\varphi(\abs{\xi^+},y_m).\end{gather*}
For particles with $\Pc\rho\ne0$ put $F_\rho:=\Pc\rho\cdot h/k_p^2$ and $\hat h(k):=\min\big(k(2-k^2),1\big)$, and let
$\ell:=(1-\rho_\infty^2)^{-1/2}-1$.

\begin{lemma}[One-step decomposition of $\Phi$]\label{mfl:rb-lem:C}
For one step $t\to t+1$, $\Phi(t+1)-\Phi(t)=\E[dU+dM+dP]$ exactly, and:
\begin{enumerate}[label=(\roman*)]
\item $dU\ge\hat h(k_p)\,\Delta u$, where $\Delta u:=\hat\xi\cdot\delta\xi=F_\rho k_p/\sqrt{\bar a}\ge0$ and $\hat\xi:=\xi/\abs\xi$.
\item For every $\rho$ with $\Pc\rho\ne0$: $F_\rho\ge1-s\ge1-\abs\rho^2$, $(1-\abs\rho^2-\ell)k_p\le\abs h\le(1+\ell)k_p$, and always
$\abs h\le\frac\pi2+1$.
\item $dP\ge0$ for every particle with $\abs{\zeta^+}\le\abs\zeta$, in particular for every particle with $b\le2N$. Let
$k':=\abs{\xi^+}/(\abs{\xi^+}^2+y^2)^{1/2}$ and $y>0$. Then
\[dM\ \ge\ -k'^2\,(y_m-y)_+,\qquad (y_m-y)_+\le\min\Big(\delta\omega_1,\ \frac{\delta\omega_1(\abs{\omega_1}+\delta\omega_1/2)}y\Big).\]
\end{enumerate}
\end{lemma}
\begin{proof}
\step{Step 1 (Telescope the three changes).}
The three steps telescope:
\[
dU+dM+dP=\abs{\xi^+}^2/\abs{w^+}-\abs\xi^2/\abs w.
\]
Since $\nu_{t+1}=(T_{\nu_t})_\#\nu_t$,
taking $\nu_t$-expectations gives $\Phi(t+1)-\Phi(t)$.

\step{Step 2 (Record the derivatives).}
With $k:=x/\sqrt{x^2+y^2}\in[0,1]$,
\begin{equation}\label{mfl:read-l4-varphi-derivatives}
\partial_x\varphi=k(2-k^2)\ge0,\qquad
\partial_y\varphi=-k^2y/\sqrt{x^2+y^2}\in[-k^2,0].
\end{equation}
For fixed $y$, $k$ is nondecreasing in $x$; for fixed $x$, $k$ is nonincreasing in $y$.
The function $k\mapsto k(2-k^2)$ increases on $[0,\sqrt{2/3}]$ and decreases on $[\sqrt{2/3},1]$ to the value $1$, so on any
interval $[k_0,1]$ its minimum is $\hat h(k_0)$.

\step{Step 3 (Bound the contrast change in part (i)).}
By Proposition~\ref{mfl:mf-prop:updates}(b), $\delta\xi=h/\sqrt{\bar a}$. Since $\hat\xi=\Pc\rho/k_p$, $\hat\xi\cdot h=F_\rho k_p$, so
$\Delta u=F_\rho k_p/\sqrt{\bar a}$, which is $\ge0$ by (ii). Next, $\abs{\xi^+}\ge\hat\xi\cdot(\xi+\delta\xi)=\abs\xi+\Delta u$, and
$\varphi$ is nondecreasing in $x$ by \eqref{mfl:read-l4-varphi-derivatives}, so
\[
dU\ge\int_{\abs\xi}^{\abs\xi+\Delta u}\partial_x\varphi(x,y)\,dx.
\]
Along the segment the cosine $k$
increases from $k(\abs\xi,y)=\abs\xi/\abs w=k_p$ and stays $\le1$, so $\partial_x\varphi\ge\hat h(k_p)$ there.

\step{Step 4 (Bound the contrast force in part (ii)).}
Dividing the identity of Lemma~\ref{mfl:mf-lem:signs}(ii) by $\cz\abs w$ gives
\[
\Pc\rho\cdot h=\sum_i(\rho_i-\bar\rho)(\arcsin\rho_i-\rho_i)+(1-s)k_p^2\ge(1-s)k_p^2,
\]
so $F_\rho\ge1-s\ge1-\abs\rho^2$.

For the norm bound, put $\tilde\phi(x):=\arcsin x-x$ and $e:=\Pc\tilde\phi(\rho)$ (coordinatewise). Then
\begin{equation}\label{mfl:read-l4-contrast-remainder}
h=(1-s)\Pc\rho+e.
\end{equation}
On $[-\rho_\infty,\rho_\infty]$, $0\le\tilde\phi'(x)=(1-x^2)^{-1/2}-1\le\ell$, so $\tilde\phi$ is $\ell$-Lipschitz there.
Using $\abs{\Pc x}^2=\frac1{2r}\sum_{i,j}(x_i-x_j)^2$ gives $\abs e\le\ell k_p$. Therefore
\eqref{mfl:read-l4-contrast-remainder} implies
\[
(1-s-\ell)k_p\le\abs h\le(1-s+\ell)k_p,
\]
and
$0\le s\le\abs\rho^2$ gives the two bounds. Finally $\Pc$ is a contraction, $\abs{\arcsin x}\le\frac\pi2\abs x$ and $s\le\abs\rho^2\le1$,
so $\abs h\le\abs{\arcsin\rho}+s\abs\rho\le\frac\pi2+1$.

\step{Step 5 (Control the remaining changes in part (iii)).}
If $\abs{\zeta^+}\le\abs\zeta$ then $y^+\le y_m$. By \eqref{mfl:read-l4-varphi-derivatives}, $\varphi$ is nonincreasing in $y$, so $dP\ge0$;
by Lemma~\ref{mfl:rb-lem:K} this holds whenever $b\le2N$.

For $dM$, fix $x:=\abs{\xi^+}$. If $y_m\le y$ then $dM\ge0$. If $y_m>y$, then
\[
\abs{\partial_y\varphi(x,y')}\le k(x,y')^2\le k(x,y)^2=k'^2\qquad (y'\ge y),
\]
so $dM\ge-k'^2(y_m-y)$. Finally, $y_m\le y+\delta\omega_1$ by the triangle inequality ($\delta\omega_1>0$); if $y_m>y$,
\[
y_m-y=\delta\omega_1(2\omega_1+\delta\omega_1)/(y_m+y)
\le\delta\omega_1(2\abs{\omega_1}+\delta\omega_1)/(2y).\qedhere
\]
\end{proof}
(Only the one-sided bound on $dM$ is used, and only it is claimed: when $y_m<y$, $dM$ can be positive and larger than
$k'^2\delta\omega_1(\abs{\omega_1}+\delta\omega_1/2)/y$.)

\section{The good-set certificate at \texorpdfstring{$\alpha=0$}{alpha = 0}}\label{mfl:sec:mff}
\begin{secsummary}
\emph{What this section proves.} Theorem~\ref{mfl:mff-thm:MFF}: for the $\alpha=0$ dynamics, at $t_F=\lfloor\tau_F\kappa\rfloor$ the
top-$r$ space of the MF AGOP is exactly $U$, with an explicit eigengap, for all $r\ge2$, $d\ge64r$, $\Dp\ge400$ and
$\kappa\ge\kappa_{\min}(r,\Dp)$. This section states the theorem and proves its analytic core, Proposition~\ref{mfl:mff-prop:cert}
(``CERT''): a static good set $G$ of initial conditions is fixed; each of its particles carries explicit envelopes; three
super-solution inequalities close an induction over the window; and on $G$ the contrast statistic $\Phi$ grows at a rate that
the $\perp$-statistic cannot match. Section~\ref{mfl:sec:eval} evaluates the certificate: in closed form for large $\Dp$ (region A)
and by $80$ interval-arithmetic certificates otherwise (region N).

\emph{Where it is used.} Step~2 of the proof of Theorem~\ref{mfl:intro-thm:main} (Section~\ref{mfl:sec:main}).
\end{secsummary}

The certificate budget $\Psi$ in slot \ref{mfl:mff-C7} is a local quantity, distinct from
$\Psi(r)$, the mean-force deviation bound of Lemma~\ref{mfl:tools-lem:psi}.

\subsection{Statement}\label{mfl:mff-sec:statement}
\begin{theorem}[Alignment at $t_F$, $\alpha=0$]\label{mfl:mff-thm:MFF}
Consider the $\alpha=0$ MF dynamics of Definition~\ref{mfl:mf-def:mf} from $\nu_0=N(0,\kappa^2I_d)$ (units $\etab=1$). Let $r\ge2$, $d\ge64r$ and $D=d-r\ge400$. There are explicit numbers
$\tau_F(r,D)\in(0,0.101]$, $\kappa_{\min}(r,D)<\infty$ and $\delta(r,D)>0$ such that, for every $\kappa\ge\kappa_{\min}(r,D)$,
with $t_F:=\lfloor\tau_F\kappa\rfloor$ (so $t_F\ge2$),
\[\lambda_c(t_F)\ge(1+\delta)\,\lambda_\perp(t_F)\qquad\text{and}\qquad\lambda_1(t_F)\ge(1+\delta)\,\lambda_\perp(t_F).\]
Hence $\Ar(t_F)=1$ for every leading-$r$ selector, with the strict gap $\lambda_r/\lambda_{r+1}\ge1+\delta$. The constants
come from two regions, which together cover every $(r,D)$ with $D\ge\max(400,63r)$.
\begin{enumerate}[label=(\alph*)]
\item \emph{Region N}: $2\le r\le80$ and $\max(400,63r)\le D\le D_A(r)$, with $D_A(r)$ the right end of the last $D$-interval of $r$ in Table~\ref{mfl:app-tab:regionN}. Then
$(\tau_F,\kappa_{\min},1+\delta)$ are the $(\tau_F,\kappa_0,\text{margin})$ of the row of
Table~\ref{mfl:app-tab:regionN} whose $D$-interval contains $D$. In particular $\tau_F\in[0.013740,0.088861]$,
$\kappa_{\min}\in[91,777]$ as tabulated (the certified real floors lie in $[90.88,776.91]$), and $1+\delta\ge1.006223$; moreover $1+\delta\ge1.012711$ for $r\ge3$ and
$1+\delta\ge1.079513$ for $r\ge10$.
\item \emph{Region A}: $D\ge D_A(r)$; for $r\ge81$ this is every $D\ge63r$. With $M=M(r,D)$ and $\bar A=\bar A(r,D)$ the
explicit numbers of Section~\ref{mfl:mff-sec:lemmaA},
\[\tau_F=\tau_1:=\frac{151}{750\,\bar A}\in[0.0802648,\ 0.100667]\cdot\frac{\sqrt{r^2-1}}M,\qquad\kappa_{\min}=\frac4{\tau_1},\]
and, at $t_F$,
\[\log\frac{\lambda_c}{\lambda_\perp}\ge0.077216\,\frac{\sqrt{r^2-1}}M-\frac{1.1(r+1)}D\ge0.0176\,\frac{\sqrt{r^2-1}}M,\]
\[\frac{\lambda_1}{\lambda_\perp}\ge\Bigg(\frac{0.0601986\,\sqrt{r^2-1}\,\sqrt{\frac\pi2(1-\frac1r)(D-\frac12)}}{1.00106\,M}\Bigg)^2>1.\]
\end{enumerate}
Here $M=(r+4)+2\sqrt{(r+4)L}+2L$ with $L=\ln(D/(r-1))+5$, so $M\approx r+2\sqrt{r\ln D}+2\ln D$. Thus
$\tau_F=\Theta(r/\log D)$ for fixed $r$ and $D\to\infty$, and $\tau_F\approx0.08$--$0.1$ for large $r$.

\end{theorem}

The proof is completed in Section~\ref{mfl:mff-sec:proof}. In region A, define $1+\delta:=\min\{\exp(0.0176\sqrt{r^2-1}/M),B_1\}$, where $B_1>1$ is the displayed lower bound for $\lambda_1/\lambda_\perp$ in (b).

\subsection{Setting, inputs and conventions}\label{mfl:mff-sec:inputs}

\paragraph{Conventions.} We work in MF units $\etab=1$. The MF AGOP has the block form
$\lambda_1\Pone\oplus\lambda_c\Pc\oplus\lambda_\perp I_D$ (Proposition~\ref{mfl:mf-prop:agop}); here $\lambda_1,\lambda_c,\lambda_\perp$ are the eigenvalues of the surrogate $\mathcal M$ at $\alpha=0$.
At positive head scale the actual AGOP includes the common positive factor $\alpha^2\etab^2$, which cancels in every ratio. \emph{In Sections~\ref{mfl:sec:mff} and~\ref{mfl:sec:eval},
$M$ denotes the squared-radius cutoff of the good set below, not the AGOP.} The letter $\nu$ denotes a mass bound; the MF law of
the particles at time $t$ is written $\nu_t$. Unless a subscript specifies otherwise, $\E[\,\cdot\,]$ and $\Prob$ use
this law (or the initial law for particle labels), and $\norm Z=(\E[\abs Z^2])^{1/2}$.
For an event $B$, $\E[Z;B]:=\E[Z\ind{B}]$ is a restricted, unnormalized expectation;
$\E[Z\mid q]$ denotes conditional expectation given $q$. In particular,
$\E[Z;B\mid q]=\E[Z\ind{B}\mid q]$.

\paragraph{Particles.} As in Section~\ref{mfl:rb-sec:setup}, a particle is labelled by its initial condition $p=(z,v,n)$:
$\omega(0)=\kappa z$ with $z\sim N(0,I_r)$; $\zeta(0)=\kappa v$ with $v=\chi_D/\sqrt D$ independent of $z$; and $n$ uniform on
$S^{D-1}$, independent of $(z,v)$. We write
\[q:=\abs z^2,\qquad z_c:=\Pc z,\qquad z_1:=\eb\cdot z,\]
and use the tilde variables $\tilde\omega=\omega/\kappa$, $\tilde\zeta=\zeta/\kappa$, $\tilde\xi=\xi/\kappa$,
$\tilde w=w/\kappa$ and $\tau=t/\kappa$. For each particle, $\abs w^2=\abs\omega^2+D\zeta^2$, $\rho=\omega/\abs w\in\R^r$
(the $U$-part of $\hatw$), $\rho_\infty:=\max_i\abs{\rho_i}$, $\omega_1:=\eb\cdot\omega$, $\xi:=\Pc\omega$, $u:=\abs\xi$,
$k_p:=\abs{\Pc\rho}=u/\abs w$ (the contrast cosine), $s:=\sum_i(1-\sqrt{1-\rho_i^2})$ and $b:=s/\abs w$.

\paragraph{Dynamics ($\alpha=0$; Proposition~\ref{mfl:mf-prop:updates}).} The exact reduced updates are
\begin{gather*}
\delta\omega_1=\frac f{\norm f},\qquad f(\rho):=\sum_i\Big(\frac\pi2+\arcsin\rho_i\Big)-s\sum_i\rho_i;\\
\delta\xi=\frac h{\sqrt a},\qquad h(\rho):=\Pc(\arcsin\rho-s\rho),\qquad a:=\frac{\E[\abs h^2]}{r-1};\\
\zeta^+=\zeta\Big(1-\frac bN\Big),\qquad N:=\norm{b\zeta}=\Big(\frac{\E[s^2(1-\abs\rho^2)]}{D}\Big)^{1/2}.
\end{gather*}
The vector $n$ of each particle never changes. A superscript $+$ denotes the next step, and $\delta Z:=Z^+-Z$.
Section~\ref{mfl:sec:rb} writes $\bar a$ for $a$.

\paragraph{Statistics.} $\Phi:=\E[u^2/\abs w]=(r-1)\Gamma_U$, $\Gamma_\perp:=\E[\abs\zeta]/\sqrt{D-\frac12}$,
$\E[v]=\E[\chi_D]/\sqrt D$ and $R_0:=\Gamma_U(0)/\Gamma_\perp(0)=(\E[\chi_d]/d)\sqrt D\sqrt{D-\frac12}/\E[\chi_D]$. We also use
$\sigma_c:=(\E[\abs\xi^2]/(r-1))^{1/2}$ and $\Gamma_1:=\E[\omega_1^2/\abs w]$.
Here $\Phi$ is the contrast statistic; the law-update map of Section~\ref{mfl:pc-sec:def}
and the teacher--student covariance of Section~\ref{mfl:tools-sec:loss} are separate, locally defined uses of this letter.

\paragraph{Facts used (all proved in Sections~\ref{mfl:sec:pc} and~\ref{mfl:sec:rb}).}
\begin{enumerate}[label=(F\arabic*)]
\item\label{mfl:mff-F1} (Lemma~\ref{mfl:mf-lem:signs}, Lemma~\ref{mfl:rb-lem:facts}.) $\abs{\xi_p}$ is nondecreasing for every particle; $\delta\omega_1>0$ for
every particle; and $\E[\abs{\delta\xi}^2]=r-1$, so $\sigma_c(t)\le\kappa+t$.
\item\label{mfl:mff-F2} (Proposition~\ref{mfl:mf-prop:agop}.) $\lambda_c\ge\Gamma_U^2/(2\pi)$ and
$\lambda_\perp\le\Gamma_\perp^2/(2\pi)$. Hence $\Ar=1$ with a strict gap, for every leading selector, iff
$\min(\lambda_1,\lambda_c)>\lambda_\perp$; then $\lambda_r/\lambda_{r+1}=\min(\lambda_1,\lambda_c)/\lambda_\perp$.
\item\label{mfl:mff-F3} (Lemma~\ref{mfl:rb-lem:K}.) With $\nu_{\rm os}(t):=\Prob(b_t>2N_t)$ and
$K_\nu(t):=\kappa^{-1}\sum_{s<t}\sqrt{\nu_{\rm os}(s)}$: $\Gamma_\perp(t)\le\Gamma_\perp(0)(1+K_\nu(t)/\E[v])$. Every particle
with $b\le2N$ has $\abs{\zeta^+}\le\abs\zeta$.
\item\label{mfl:mff-F4} (Lemma~\ref{mfl:rb-lem:M}.) $\E[\omega_1(t)]\ge\sqrt{1-1/r}\;t$.
\item\label{mfl:mff-F5} (Lemma~\ref{mfl:rb-lem:C}.) Write $F(u,W):=u^2/\sqrt{u^2+W^2}$ with $W^2:=\omega_1^2+D\zeta^2$, so
that $u^2/\abs w=F(u,W)$. The per-particle increment of $u^2/\abs w$ is exactly $dU+dM+dP$, with
$dU:=F(u^+,W)-F(u,W)$, $dM:=F(u^+,W_m)-F(u^+,W)$, $W_m^2:=(\omega_1+\delta\omega_1)^2+D\zeta^2$, and
$dP:=F(u^+,W^+)-F(u^+,W_m)$. Moreover:
\begin{itemize}
\item $dU\ge\hat h(k_p)F_\rho k_p/\sqrt a$, where $\hat h(k):=\min(k(2-k^2),1)$ and $F_\rho:=\Pc\rho\cdot h/k_p^2$;
\item $F_\rho\ge1-s\ge1-\abs\rho^2$;
\item $(1-\abs\rho^2-\ell)k_p\le\abs h\le(1+\ell)k_p$ with $\ell:=(1-\rho_\infty^2)^{-1/2}-1$, and $\abs h\le\pi/2+1$ always;
\item $dP\ge0$ whenever $\abs{\zeta^+}\le\abs\zeta$;
\item $dM\ge-k'^2(W_m-W)_+$, where $k':=u^+/\sqrt{u^{+2}+W^2}$ and\\
$(W_m-W)_+\le\min\big(\delta\omega_1,\ \delta\omega_1(\abs{\omega_1}+\delta\omega_1/2)/W\big)$.
\end{itemize}
\item\label{mfl:mff-F6} $\abs{f/(r\pi/2)-1}\le e_f(\abs\rho):=\abs\rho(1+2\abs\rho^2/\pi)/\sqrt r$. Indeed
$\abs{\sum_i\arcsin\rho_i}\le\frac\pi2\sum_i\abs{\rho_i}\le\frac\pi2\sqrt r\abs\rho$ and
$\abs{s\sum_i\rho_i}\le\abs\rho^2\sqrt r\abs\rho$, because $\abs{\arcsin x}\le\frac\pi2\abs x$ and $s\le\abs\rho^2$.
\item\label{mfl:mff-F7} (Proposition~\ref{mfl:rb-prop:init}.) $\Phi(0)=\kappa(r-1)\E[\chi_d]/d$ and $\Gamma_\perp(0)=\kappa\,\E[v]/\sqrt{D-\frac12}$.
\end{enumerate}

\subsection{Auxiliary lemmas}\label{mfl:mff-sec:aux}

\begin{lemma}[Exact split of $\lambda_1$ and the pairing bound]\label{mfl:mff-lem:L1split}
Let the particle law be $O(D)\times S_r$-invariant with $n$ uniform (for instance $\nu_t$). Put
$\varepsilon_p(X):=\ind{w_p^\top X>0}-\frac12$ and $A_1(X):=\E_p[\varepsilon_p(X)\,\omega_{1p}]$, where
$\omega_{1p}=\eb\cdot\omega_p$. Then
\[\lambda_1=\tfrac14(\E[\omega_1])^2+\E_X\big[A_1(X)^2\big]\ \ge\ \tfrac14(\E[\omega_1])^2+\frac{\Gamma_1^2}{2\pi},\qquad
\Gamma_1=\E\big[\omega_1^2/\abs w\big].\]
Consequently, along the MF trajectory, at every $t\ge0$ with $\tau=t/\kappa$,
\begin{equation}\label{mfl:mff-eq:L1ratio}
\frac{\lambda_1}{\lambda_\perp}\ge\frac\pi2\Big(1-\frac1r\Big)\frac{\tau^2(D-\frac12)}{(\E[v]+K_\nu)^2}
+\Big(\frac{\Gamma_1\sqrt{D-\frac12}}{\kappa\,(\E[v]+K_\nu)}\Big)^2 .
\end{equation}
The pairing bound for the cross term is part (iii) of Proposition~\ref{mfl:mff-prop:cert}.

\end{lemma}
\begin{proof}
Let $v(X):=\E_p[\ind{w_p^\top X>0}w_p]$, so that the MF AGOP is $\E_X[v(X)v(X)^\top]$ and, by the block form,
$\lambda_1=\E_X[(\eb\cdot v(X))^2]$. Now $\eb\cdot v(X)=\frac12\E[\omega_1]+A_1(X)$. Since $\varepsilon_p(-X)=-\varepsilon_p(X)$
for almost every $X$ and $X$ is symmetric, $\E_X[A_1(X)]=0$; expanding the square gives the identity. For the bound, apply
Cauchy--Schwarz with the test function $\varphi(X):=\eb^\top X$, for which $\E_X[\varphi(X)^2]=1$:
$\E_X[A_1^2]\ge(\E_X[A_1\varphi])^2$. Since $\E_X[X\ind{w^\top X>0}]=\hatw/\sqrt{2\pi}$ and $\E_X[\varphi]=0$,
\[\E_X[A_1\varphi]=\E_p\big[\omega_{1p}\,\eb^\top\E_X[\varepsilon_p(X)X]\big]=\E_p\Big[\omega_{1p}\frac{\eb^\top\hatw_p}{\sqrt{2\pi}}\Big]
=\frac{\E[\omega_1^2/\abs w]}{\sqrt{2\pi}},\]
because $\eb^\top\hatw_p=\omega_{1p}/\abs{w_p}$. For \eqref{mfl:mff-eq:L1ratio}, divide by $\lambda_\perp\le\Gamma_\perp^2/(2\pi)$
\ref{mfl:mff-F2}, use $\Gamma_\perp(t)\le\Gamma_\perp(0)(1+K_\nu/\E[v])=\kappa(\E[v]+K_\nu)/\sqrt{D-\frac12}$ (\ref{mfl:mff-F3},
\ref{mfl:mff-F7}), and $(\E[\omega_1])^2\ge(1-1/r)t^2$ \ref{mfl:mff-F4}.
\end{proof}

\begin{lemma}[Lemma H: golden-ratio refinement]\label{mfl:mff-lem:H}
If $\abs{\rho_p}^2\le0.38<((\sqrt5-1)/2)^2=0.3819\ldots$, then
\[\hat h(k_p)\,F_\rho\,k_p\ \ge\ q_p\,k_p^2,\qquad q_p:=(2-\abs{\rho_p}^2)(1-\abs{\rho_p}^2),\]
and $2-q_p=3\abs{\rho_p}^2-\abs{\rho_p}^4$.

\end{lemma}
\begin{proof}
The function $k\mapsto k(2-k^2)$ increases on $[0,\sqrt{2/3}]$ and equals $1$ at $k=(\sqrt5-1)/2$ (the root in $(0,1)$ of
$k^3-2k+1=(k-1)(k^2+k-1)$). Since $k_p\le\abs{\rho_p}<(\sqrt5-1)/2$, we get
$\hat h(k_p)=k_p(2-k_p^2)\ge k_p(2-\abs{\rho_p}^2)$. Multiply by $F_\rho k_p$ and use $F_\rho\ge1-\abs{\rho_p}^2\ge0$
\ref{mfl:mff-F5}.
\end{proof}

For $y\in(0,1)$ put $\iota(y):=y-1-\ln y>0$, and for $n\ge1$
\[\mathcal U(n,y):=\frac{(y\,e^{1-y})^{n/2}}{\sqrt{\pi n}\,(1-y)}=\frac{e^{-n\iota(y)/2}}{\sqrt{\pi n}\,(1-y)}.\]

\begin{lemma}[Lemma V: $v$-averages and the $\chi^2$ lower tail]\label{mfl:mff-lem:V}
Let $v=\chi_D/\sqrt D$ (so $\E[v^2]=1$) and $v_{\min}<1$.
\begin{enumerate}[label=(\alph*)]
\item Let $\varphi$ be convex and nonincreasing on $(0,\infty)$ with $0\le\varphi(x)\le\bar\varphi(x)$ for $x<v_{\min}^2$. Then
$\E[\varphi(v^2);v\ge v_{\min}]\ge\varphi(1)-\E[\bar\varphi(v^2);v<v_{\min}]$. It is used with
$\varphi(x)=D^k/(c+Dx)^k$, $\bar\varphi(x)=x^{-k}$ ($k=1,2$, $c\ge0$), and with $\varphi(x)=(c+Dx)^{-1/2}$,
$\bar\varphi(x)=(Dx)^{-1/2}$.
\item For $0<y<1$ and $n\ge1$: $\Prob(\chi^2_n<ny)\le\mathcal U(n,y)$. $\mathcal U$ is decreasing in $n$ and increasing in $y$.
\item Exactly, for $k=1,2$: $\E[v^{-2k};v<v_{\min}]=D^k\Prob(\chi^2_{D-2k}<Dv_{\min}^2)/\prod_{j\le k}(D-2j)$, and
$\E[v^{-1};v<v_{\min}]=(\sqrt D/\E[\chi_{D-1}])\,\Prob(\chi^2_{D-1}<Dv_{\min}^2)$.
\item Put $y:=v_{\min}^2$ and assume $Dy/(D-4)<1$. For $v=\chi_{D'}/\sqrt{D'}$ at every $D'\ge D$:
\begin{align*}
\nu_v&:=\Prob(v<v_{\min})\le\mathcal U(D,y),\\
m_1&:=\E[v^{-2};v<v_{\min}]\le\tfrac D{D-2}\,\mathcal U\big(D-2,\tfrac{Dy}{D-2}\big),\\
m_2&:=\E[v^{-4};v<v_{\min}]\le\tfrac{D^2}{(D-2)(D-4)}\,\mathcal U\big(D-4,\tfrac{Dy}{D-4}\big),\\
m_h&:=\E[v^{-1};v<v_{\min}]\le\sqrt{\tfrac D{D-3/2}}\;\mathcal U\big(D-1,\tfrac{Dy}{D-1}\big),
\end{align*}
and
\[D'\nu_v(D')\le\sup_{D''\ge D}\frac{\sqrt{D''}\,e^{-D''\iota(y)/2}}{\sqrt\pi(1-y)}
=\begin{cases}\sqrt D\,e^{-D\iota(y)/2}/(\sqrt\pi(1-y)),&D\iota(y)\ge1,\\ \iota(y)^{-1/2}e^{-1/2}/(\sqrt\pi(1-y)),&\text{otherwise.}\end{cases}\]
\end{enumerate}

\end{lemma}
\begin{proof}
\step{Part (a): Removing the lower tail.} Jensen's inequality and the bound on $\varphi$ give
\[
\E[\varphi(v^2)]\ge\varphi(\E[v^2])=\varphi(1),\qquad
\E[\varphi(v^2);v<v_{\min}]\le\E[\bar\varphi(v^2);v<v_{\min}].
\]
In the three uses, $\varphi$ is convex and nonincreasing, and $\varphi\le\bar\varphi$ because $c\ge0$.

\step{Part (b): A lower-tail probability bound.} With $a=n/2$ and $x=ay<a$,
$\Prob(\chi^2_n<ny)=\gamma(a,x)/\Gamma(a)$. The positive series satisfies
\begin{equation}\label{mfl:read-l5-gamma-tail}
\gamma(a,x)=x^ae^{-x}\sum_{j\ge0}\frac{x^j}{a(a+1)\cdots(a+j)}
\le\frac{x^ae^{-x}}{a(1-x/a)},
\end{equation}
since each ratio of consecutive terms is at most $x/a$. Thus \eqref{mfl:read-l5-gamma-tail} gives
$\Prob(\chi^2_n<ny)\le x^ae^{-x}/(\Gamma(a+1)(1-y))$, and Stirling's lower bound
$\Gamma(a+1)\ge\sqrt{2\pi a}(a/e)^a$ gives $(ay)^ae^{-ay}/(\sqrt{2\pi a}(a/e)^a(1-y))=\mathcal U(n,y)$. Since $\ln y<y-1$,
$ye^{1-y}<1$, so $\mathcal U$ decreases in $n$; and $ye^{1-y}$ and $1/(1-y)$ increase on $(0,1)$.

\step{Part (c): Inverse moments on the tail.} If $f_n$ is the $\chi^2_n$ density, then
\[
\begin{aligned}
x^{-k}f_D(x)
&=f_{D-2k}(x)\,2^{-k}\frac{\Gamma(\frac D2-k)}{\Gamma(\frac D2)}
=\frac{f_{D-2k}(x)}{\prod_{j\le k}(D-2j)},\\
x^{-1/2}f_D(x)
&=f_{D-1}(x)\,\frac{\Gamma(\frac{D-1}2)}{\sqrt2\,\Gamma(\frac D2)}
=\frac{f_{D-1}(x)}{\E[\chi_{D-1}]}.
\end{aligned}
\]
Use $v^{-2k}=D^k(\chi^2_D)^{-k}$ and integrate these identities over the lower tail.

\step{Part (d): Uniformity over $D'\ge D$.} Combine (b) and (c), and use $\E[\chi_{D-1}]\ge\sqrt{D-3/2}$ (Lemma~\ref{mfl:mff-lem:chi}). Every prefactor decreases in $D$.
In $\mathcal U(D-j,Dy/(D-j))$ the first argument increases and the second decreases with $D$, so each bound decreases in $D$
and the value at $D$ bounds the value at every $D'\ge D$. For the last claim, $D'\nu_v(D')\le\sqrt{D'}\,e^{-D'\iota/2}/
(\sqrt\pi(1-y))$, and $x\mapsto\sqrt x\,e^{-x\iota/2}$ increases up to $x=1/\iota$ and decreases afterwards.
\end{proof}

\begin{lemma}[Lemma D: discrete growth]\label{mfl:mff-lem:D}
For $x\ge0$, $\kappa>0$ and every integer $t\ge0$:
$\prod_{s<t}\big(1+\frac x{\kappa+s}\big)\ge\exp\big(x\ln(1+t/\kappa)-x^2t/(2\kappa^2)\big)$.

\end{lemma}
\begin{proof}
Use $\ln(1+y)\ge y-y^2/2$ for $y\ge0$, together with
\[
\sum_{s<t}(\kappa+s)^{-1}\ge\int_0^t(\kappa+s)^{-1}ds=\ln(1+t/\kappa),\qquad
\sum_{s<t}(\kappa+s)^{-2}\le t/\kappa^2.
\]
The first inequality is the left Riemann sum of a decreasing function.
\end{proof}

\begin{lemma}[Moments of $\chi$ distributions]\label{mfl:mff-lem:chi}
Let $c_k:=\E[\chi_k]/\sqrt k$.
\begin{enumerate}[label=(\alph*)]
\item $c_k$ is increasing in $k$, and $\sqrt{1-1/(2k)}\le c_k\le1$; equivalently $\E[\chi_k]\ge\sqrt{k-\frac12}$ (Kershaw).
\item $\E[\chi_k]\,\E[\chi_{k+1}]=k$, and $\E[1/\chi_k]=1/\E[\chi_{k-1}]\le(k-\frac32)^{-1/2}$ for $k\ge2$.
\item For integers $D\in[D_{\rm lo},D_{\rm hi}]$ and $d=D+r$: $\E[v]=c_D\in[c_{D_{\rm lo}},c_{D_{\rm hi}}]\subset
[\sqrt{1-1/(2D_{\rm lo})},1]$ and
\[R_0=\frac{c_d}{c_D}\sqrt{\frac{D-\frac12}d}\ \ge\ \frac{c_{D_{\rm lo}+r}}{c_{D_{\rm hi}}}\sqrt{\frac{D_{\rm lo}-\frac12}{D_{\rm lo}+r}}
\ \ge\ \sqrt{1-\frac1{2(D_{\rm lo}+r)}}\sqrt{\frac{D_{\rm lo}-\frac12}{D_{\rm lo}+r}}.\]
\end{enumerate}

\end{lemma}
\begin{proof}
\step{Part (a): Bounds and monotonicity.} $c_k\le1$ is Jensen. Kershaw's inequality $\Gamma(x+1)/\Gamma(x+s)>(x+s/2)^{1-s}$ ($x>0$, $0<s<1$), with $s=\frac12$,
$x=\frac{k-1}2$, gives $\E[\chi_k]=\sqrt2\,\Gamma(\frac{k+1}2)/\Gamma(\frac k2)>\sqrt{k-\frac12}$ for $k\ge2$; for $k=1$,
$\sqrt{2/\pi}>\sqrt{1/2}$. For monotonicity, with $\psi$ the digamma function,
$\frac d{dk}\ln c_k=\frac12[\psi(\frac k2+\frac12)-\psi(\frac k2)]-\frac1{2k}$, and for $x>0$ the series of $\psi$ gives
\begin{align*}
\psi(x+\tfrac12)-\psi(x)&=\sum_{n\ge0}\frac{1/2}{(x+n)(x+n+\frac12)}>\frac12\sum_{n\ge0}\frac1{(x+n)(x+n+1)}\\
&=\frac12\sum_{n\ge0}\Big(\frac1{x+n}-\frac1{x+n+1}\Big)=\frac1{2x}.
\end{align*}
With $x=k/2$ this gives $\frac d{dk}\ln c_k>0$.

\step{Part (b): Moment identities.} The gamma-function formulas give
\[
\E[\chi_k]\E[\chi_{k+1}]=\frac{2\Gamma(\frac k2+1)}{\Gamma(\frac k2)}=k,\qquad
\E[1/\chi_k]=\frac{\Gamma(\frac{k-1}2)}{\sqrt2\,\Gamma(\frac k2)}=\frac1{\E[\chi_{k-1}]}.
\]
Apply part (a) to the last denominator.

\step{Part (c): Bounds on a $D$-interval.} This follows from (a), since
$\sqrt{(D-\frac12)/(D+r)}$ increases in $D$.
\end{proof}

\subsection{The certificate (Proposition CERT)}\label{mfl:mff-sec:cert}

\paragraph{Certificate data.} The certificate consists of the following numbers and derived quantities.
\begin{enumerate}[label=(C\arabic*)]
\item\label{mfl:mff-C1} Parameters: $r\ge2$, an integer $D\ge5$, $M>0$, $\theta\in(0,1)$, $\tau_1>0$, and a triple
$\mathbf x=(\Omega,A,C')\in(0,\infty)^3$.
\item\label{mfl:mff-C2} $v_{\min}:=(\theta^3+3A\tau_1)^{1/3}$, assumed $<1$, and the \emph{static good set}
$G:=\{p:\ q\le M,\ v\ge v_{\min}\}$, a set of initial conditions fixed once and for all. Put $\nu_q:=\Prob(\chi^2_r>M)$, let
$\nu_v$ be any upper bound on $\Prob(v<v_{\min})$, and $\nu:=\nu_q+\nu_v\ge\Prob(G^c)$, with $0\le\nu<1$.
\item\label{mfl:mff-C3} $\rho_G^2:=\Omega^2/(\Omega^2+D\theta^2)$, assumed $\le0.38$; $\ell_G:=(1-\rho_G^2)^{-1/2}-1$;
$e_G:=e_f(\rho_G)$, assumed $<\frac12$; and $c_G:=(1-e_G)\sqrt{1-\nu}/(1+e_G)$.
\item\label{mfl:mff-C4} $I(\tau):=\int_0^\tau(v_{\min}^3-3A\sigma)^{-1/3}d\sigma
=\big(v_{\min}^2-(v_{\min}^3-3A\tau)^{2/3}\big)/(2A)$ for $0\le\tau\le\tau_1$.
\item\label{mfl:mff-C5} Envelopes, for $0\le\bar q\le M$ and $0\le\tau\le\tau_1$:
\[\hat\Omega^2(\tau;\bar q):=\Big(\sqrt{\bar q}+\frac\tau{c_G}\Big)^2e^{2C'I(\tau)},\qquad
\hat\theta(\tau;\bar q):=\Big(v_{\min}^3-3\tau A\,\frac{\hat\Omega^2(\tau;\bar q)}{\Omega^2}\Big)^{1/3},\]
\begin{gather*}\hat\rho^2(\tau;\bar q):=\frac{\hat\Omega^2}{\hat\Omega^2+D\hat\theta^2},\\
\hat\ell:=(1-\hat\rho^2)^{-1/2}-1,\\
\hat\delta_M(\tau;\bar q):=\frac1{c_G}\min\Big(1,\frac{\hat\Omega+1/(2c_G\kappa)}{\sqrt D\,\hat\theta}\Big).\end{gather*}
Under (S-$\Omega$) below, $\hat\Omega\le\Omega$, so the base of $\hat\theta$ is at least $v_{\min}^3-3\tau_1A=\theta^3>0$ and
$\hat\theta\ge\theta$. The envelope $\hat\delta_M$ is the only one that depends on $\kappa$.
\item\label{mfl:mff-C6} Slots. In the expectations below, each particle's envelopes are taken at its own $\bar q=q$ and at
$\tau=\tau_1$. Let $\mathcal N,\mathcal K_a,\underline K>0$ and $Q_2,H_2,W_2\ge0$ satisfy
\begin{gather*}
\mathcal N\le\tfrac14\E\Big[\abs{z_c}^4(1-\hat\rho^2)\frac{D^2}{(\hat\Omega^2+Dv^2)^2};G\Big],\\
\mathcal K_a\le\E\Big[(1-\hat\rho^2-\hat\ell)_+^2\frac{\abs{z_c}^2D}{\hat\Omega^2+Dv^2};G\Big],\\
\underline K\le\E\Big[\frac{\abs{z_c}^2}{\hat\Omega^2+Dv^2};G\Big],\\
Q_2^2\ge\E\big[(3\hat\rho^2-\hat\rho^4)^2\hat\rho^2;G\big],\\
H_2\ge\E\Big[\frac{\hat\rho^4}{1-\hat\rho^2};G\Big],\\
W_2^2\ge\E\big[\hat\rho^2\hat\delta_M^2;G\big].
\end{gather*}
\item\label{mfl:mff-C7} Derived constants:
\begin{gather*}
\beta_G:=\frac{(1-Q_2/(2\sqrt{\underline K}))_+}{\sqrt{1+\Psi}},\qquad\Psi:=\frac{H_2+(\pi/2+1)^2\nu}{\underline K},\\
\lambda_L:=\frac{(1+C'/(\kappa\theta))^2W_2}{2\beta_G\sqrt{r-1}},\qquad\beta':=\beta_G(1-\lambda_L);\\
e_0:=\Big[\Prob(\chi^2_{r+2}>M)\,\E[1/\chi_D]+\frac{\E[\chi_r]}r\,\nu_v\Big]\frac d{\E[\chi_d]},\\
R(\tau):=R_0(1-e_0)\frac{\exp\big(2\beta'\ln(1+\tau)-2\beta'^2\tau/\kappa\big)}{1+\tau\sqrt\nu/\E[v]}.
\end{gather*}
\item\label{mfl:mff-C8} Mean-eigenvalue slots, at $\tau\le\tau_1$ (envelopes at $\tau$): with $m_h$ as in
Lemma~\ref{mfl:mff-lem:V}(d) and $\bar\vartheta_r:=\E[\abs{\vartheta_1}]$ for $\vartheta$ uniform on $S^{r-1}$,
\begin{gather*}
F_p(\tau):=\Big(1-\frac{(\hat\Omega^2(\tau;q)-q)_+}{2(q+Dv_{\min}^2)}\Big)_+,\\
\mathrm{main}(\tau):=\E\Big[\frac qr\,F_p(\tau)\Big(\frac{\sqrt{D-\frac12}}{\sqrt{q+D}}-m_h\Big)_+;\ q\le M\Big],\\
X_p(\tau):=\frac{1+e_f(\hat\rho)}{\hat\theta}-(1-e_f(\hat\rho))\,F_p(\tau)\Big(\frac{\sqrt{D-\frac12}}{\sqrt{q+D}}-m_h\Big)_+,\\
\mathrm{cross}(\tau):=\frac{\tau\,\bar\vartheta_r\,\E[\sqrt q\,X_p(\tau);q\le M]}{\sqrt{1-\nu}\,(1-e_G)},
\end{gather*}
and $\mathcal G_1(\tau):=\mathrm{main}(\tau)-\mathrm{cross}(\tau)$.
\end{enumerate}
\emph{Self-consistency (super-solution) inequalities:}
\begin{gather*}
\text{(S-$\Omega$)}\ \ \Omega\ge\big(\sqrt M+\tau_1/c_G\big)\,e^{C'I(\tau_1)},\\
\text{(S-$A$)}\ \ A\ge\frac{\Omega^2}{(1+\sqrt{1-\rho_G^2})\sqrt{\mathcal N}},\qquad
\text{(S-$C$)}\ \ C'\ge\frac{\sqrt{r-1}\,(1+\ell_G)}{\sqrt{\mathcal K_a}}.
\end{gather*}

\begin{proposition}[Proposition CERT]\label{mfl:mff-prop:cert}
Assume the certificate data \ref{mfl:mff-C1}--\ref{mfl:mff-C8} satisfy (S-$\Omega$), (S-$A$), (S-$C$), and let $\kappa\ge A/\theta^3$.
Then for every integer $t$ with $0\le t\le\tau_1\kappa$, every $p\in G$ and $\tau=t/\kappa$:
\begin{enumerate}[label=(I\arabic*)]
\item\label{mfl:mff-I1} $\tilde\zeta_p$ is nonincreasing on $[0,t]$ and $\tilde\zeta_p(t)\ge\hat\theta(\tau;q_p)\ge\theta$; at every
step $s<t$ the particle neither flips sign nor overshoots ($0\le b_p\le N$);
\item\label{mfl:mff-I2} $\abs{\tilde\omega_p(t)}\le\hat\Omega(\tau;q_p)\le\Omega$;
\item\label{mfl:mff-I3} $\abs{\tilde\xi_p(t)}\ge\abs{z_c}$; and
$\delta\omega_{1,p}(t)\in[(1-e_f(\hat\rho))s_t,(1+e_f(\hat\rho))s_t]$ with $\hat\rho=\hat\rho(\tau;q_p)$ and
$s_t:=(r\pi/2)/\norm{f(t)}\le1/((1-e_G)\sqrt{1-\nu})$.
\end{enumerate}
Moreover, put $K_G(t):=\E[k_p(t)^2;G]$ and $\Phi_G(t):=\E[u^2/\abs w;G]$, and assume $\beta'>0$ and $0\le e_0<1$. Then:
\begin{enumerate}[label=(\roman*)]
\item\emph{($\Phi$-growth.)} $\Phi_G(t+1)\ge\Phi_G(t)\big(1+2\beta'/(\kappa+t)\big)$ whenever $t+1\le\tau_1\kappa$.
\item\emph{(Ratio.)} $\Gamma_U(t)/\Gamma_\perp(t)\ge R(\tau)$, hence $\lambda_c(t)/\lambda_\perp(t)\ge R(\tau)^2$.
\item\emph{(Mean eigenvalue.)} $\displaystyle\frac{\lambda_1(t)}{\lambda_\perp(t)}\ge\frac\pi2\Big(1-\frac1r\Big)
\frac{\tau^2(D-\frac12)}{(\E[v]+\tau\sqrt\nu)^2}+\Big(\frac{(\mathcal G_1(\tau))_+}{\E[v]+\tau\sqrt\nu}\Big)^2.$
\end{enumerate}

\end{proposition}

\begin{proof}
We first prove \ref{mfl:mff-I1}, \ref{mfl:mff-I2} and \ref{mfl:mff-I3} of Proposition~\ref{mfl:mff-prop:cert} by induction. Let $\mathcal H(t)$ be the statement that, for every $p\in G$,
$\tilde\zeta_p$ is nonincreasing on $[0,t]$ and
\[
\tilde\zeta_p(t)\ge\hat\theta(t/\kappa;q_p),\qquad
\abs{\tilde\omega_p(t)}\le\hat\Omega(t/\kappa;q_p),\qquad
\abs{\tilde\xi_p(t)}\ge\abs{z_c}.
\]
The initial values establish $\mathcal H(0)$:
\[
\tilde\zeta_p(0)=v\ge v_{\min}=\hat\theta(0;q),\qquad
\abs{\tilde\omega_p(0)}=\sqrt q=\hat\Omega(0;q),\qquad
\tilde\xi_p(0)=z_c.
\]
Fix $t$ with $t+1\le\tau_1\kappa$ and assume
$\mathcal H(s)$ for all $s\le t$. We first prove four facts at every time $s\le t$, and then $\mathcal H(t+1)$.
Note that the envelopes are nondecreasing in $\tau$ ($\hat\Omega$, $\hat\rho$) or nonincreasing ($\hat\theta$), so every bound
below that uses envelopes at $\tau=\tau_1$ holds at $s/\kappa\le\tau_1$.

\step{Step 1 (Lower bound for $N$).} Since $1-\sqrt{1-x}\ge x/2$, we have $s\ge\abs\rho^2/2$ and therefore
\[
N^2=\frac{\E[s^2(1-\abs\rho^2)]}{D}
\ge\frac{\E[\abs\rho^4(1-\abs\rho^2);G]}{4D}.
\]
On $G$, the induction hypothesis gives the two bounds
\begin{equation}\label{mfl:read-l5-rho-bounds}
\begin{aligned}
\abs{\rho_p}^2&\ge\frac{\abs{\tilde\xi_p}^2}{\abs{\tilde w_p}^2}
\ge\frac{\abs{z_c}^2}{\hat\Omega^2+Dv^2},\\
\abs{\rho_p}^2&=\frac{\abs{\tilde\omega_p}^2}{\abs{\tilde\omega_p}^2+D\tilde\zeta_p^2}
\le\hat\rho^2.
\end{aligned}
\end{equation}
For the first line, $\abs{\tilde w_p}^2=\abs{\tilde\omega_p}^2+D\tilde\zeta_p^2\le\hat\Omega^2+Dv^2$, because
$\tilde\zeta_p$ is nonnegative and nonincreasing from $v$. The second line uses $\tilde\zeta_p\ge\hat\theta$.
Substituting \eqref{mfl:read-l5-rho-bounds} into the $N$-bound and using slot \ref{mfl:mff-C6} yields
\begin{equation}\label{mfl:read-l5-N-bound}
D^3N^2\ge\mathcal N,\qquad D^{3/2}N\ge\sqrt{\mathcal N}.
\end{equation}

\step{Step 2 (Erosion).} On $G$, $\abs{\rho_p}^2\le\hat\rho^2\le\rho_G^2$ (as $\hat\Omega\le\Omega$ and $\hat\theta\ge\theta$),
so $s=\sum_i\rho_i^2/(1+\sqrt{1-\rho_i^2})\le\abs\rho^2/(1+\sqrt{1-\rho_G^2})$; and $\abs w\ge\sqrt D\abs\zeta=\kappa\sqrt D\tilde\zeta$.
The step decreases $\tilde\zeta_p$ by $\delta\tilde\zeta_p:=\tilde\zeta_pb_p/N$, and
\begin{gather*}\delta\tilde\zeta_p=\frac{\tilde\zeta_p\,s}{\abs w\,N}\le\frac{\abs{\omega}^2}{(1+\sqrt{1-\rho_G^2})\abs w^3N}\,\tilde\zeta_p
\le\frac{A_p(s)}{\kappa\tilde\zeta_p^2},\\
A_p(s):=\frac{\abs{\tilde\omega_p(s)}^2}{(1+\sqrt{1-\rho_G^2})D^{3/2}N}
\le\frac{A\abs{\tilde\omega_p(s)}^2}{\Omega^2},\end{gather*}
by \eqref{mfl:read-l5-N-bound} and (S-$A$). Since $\abs{\tilde\omega_p}\le\Omega$, $\tilde\zeta_p\ge\theta$ and $\kappa\ge A/\theta^3$:
\[
\delta\tilde\zeta_p\le\frac{\tilde\zeta_pA}{\kappa\theta^3}\le\tilde\zeta_p.
\]
Thus $b_p/N\in[0,1]$: no flip and no overshoot at step $s$.
Moreover $(\tilde\zeta-\delta)^3\ge\tilde\zeta^3-3\tilde\zeta^2\delta$ for $0\le\delta\le\tilde\zeta$, so
\begin{equation}\label{mfl:mff-eq:cube}
\tilde\zeta_p(s+1)^3\ge\tilde\zeta_p(s)^3-\frac{3A\,\hat\Omega^2(s/\kappa;q_p)}{\kappa\,\Omega^2}.
\end{equation}

\step{Step 3 (Contrast).} By \ref{mfl:mff-F5}, $\abs{\delta\xi}=\abs h/\sqrt a\le(1+\ell)k_p/\sqrt a\le(1+\ell_G)\abs\xi/(\abs w\sqrt a)$
on $G$, since $\ell\le(1-\abs\rho^2)^{-1/2}-1\le\ell_G$. For the normalization, again by \ref{mfl:mff-F5} and \eqref{mfl:read-l5-rho-bounds}, $\abs h\ge(1-\abs\rho^2-\ell)k_p\ge(1-\hat\rho^2-\hat\ell)k_p$ and $k_p^2\ge\abs{z_c}^2/(\hat\Omega^2+Dv^2)$ on $G$, so
\[(r-1)a=\E[\abs h^2]\ge\E\big[(1-\hat\rho^2-\hat\ell)_+^2k_p^2;G\big]\ge\mathcal K_a/D.\]
With $\abs w\ge\kappa\sqrt D\tilde\zeta$ and (S-$C$), this gives
$\abs{\delta\tilde\xi_p}\le C'\abs{\tilde\xi_p}/(\kappa\tilde\zeta_p)$, hence
\begin{equation}\label{mfl:read-l5-contrast-step}
\abs{\tilde\xi_p(s+1)}\le\abs{\tilde\xi_p(s)}
\Big(1+\frac{C'}{\kappa\tilde\zeta_p(s)}\Big).
\end{equation}

\step{Step 4 (Mean mode).} On $G$, \ref{mfl:mff-F6} gives
\[
\frac{f_p}{r\pi/2}\in[1-e_f(\hat\rho),1+e_f(\hat\rho)]\subset[1-e_G,1+e_G],
\]
because $e_f$ is increasing and $\abs{\rho_p}\le\hat\rho\le\rho_G$. Consequently
\[
\norm f^2\ge\E[f^2;G]\ge(r\pi/2)^2(1-e_G)^2(1-\nu),\qquad
s_s=\frac{r\pi/2}{\norm{f(s)}}\le\frac1{(1-e_G)\sqrt{1-\nu}}.
\]
The resulting increment bound is
\begin{equation}\label{mfl:read-l5-mean-step}
\delta\omega_{1,p}=s_s\,\frac{f_p}{r\pi/2}
\in[(1-e_f(\hat\rho))s_s,(1+e_f(\hat\rho))s_s]\subset(0,1/c_G].
\end{equation}
This is the second part of \ref{mfl:mff-I3} at time $s$.

\step{Step 5 (Closing the induction).} By Step 2, $\tilde\zeta_p$ is nonincreasing up to $t+1$. Summing \eqref{mfl:mff-eq:cube} over
$s=0,\dots,t$ and using that $\hat\Omega$ is nondecreasing in $\tau$ and $v\ge v_{\min}$,
\[\tilde\zeta_p(t+1)^3\ge v^3-\frac{3(t+1)A\,\hat\Omega^2((t+1)/\kappa;q_p)}{\kappa\,\Omega^2}\ge\hat\theta\big(\tfrac{t+1}\kappa;q_p\big)^3,\]
and $\hat\theta\ge\theta$. With $\hat\Omega\le\Omega$ the same sum gives $\tilde\zeta_p(s)\ge\vartheta(s/\kappa):=
(v_{\min}^3-3As/\kappa)^{1/3}$ for $s\le t$. Since $1/\vartheta$ is increasing, its left Riemann sum is below its integral,
and iterating \eqref{mfl:read-l5-contrast-step} gives
\begin{equation}\label{mfl:read-l5-contrast-envelope}
\begin{aligned}
\abs{\tilde\xi_p(t+1)}
&\le\abs{z_c}\prod_{s\le t}\Big(1+\frac{C'}{\kappa\vartheta(s/\kappa)}\Big)\\
&\le\abs{z_c}\exp\Big(C'\sum_{s\le t}\frac1{\kappa\,\vartheta(s/\kappa)}\Big)
\le\abs{z_c}\,e^{C'I((t+1)/\kappa)}.
\end{aligned}
\end{equation}
By \eqref{mfl:read-l5-mean-step} and $\tilde\omega_1(0)=z_1$, $\abs{\tilde\omega_{1,p}(t+1)}\le\abs{z_1}+(t+1)/(\kappa c_G)$. With
$\tau'=(t+1)/\kappa$, $\abs{z_1}^2+\abs{z_c}^2=q$ and $e^{2C'I}\ge1$, combine this with
\eqref{mfl:read-l5-contrast-envelope} to obtain
\begin{align*}\abs{\tilde\omega_p(t+1)}^2&=\tilde\omega_{1,p}^2+\abs{\tilde\xi_p}^2\le\big(\abs{z_1}+\tau'/c_G\big)^2e^{2C'I(\tau')}
+\abs{z_c}^2e^{2C'I(\tau')}\\
&\le\big(\sqrt q+\tau'/c_G\big)^2e^{2C'I(\tau')}=\hat\Omega^2(\tau';q),\end{align*}
and $\hat\Omega(\tau';q)\le\hat\Omega(\tau_1;M)\le\Omega$ by (S-$\Omega$). Finally $\abs{\tilde\xi_p}$ is nondecreasing
\ref{mfl:mff-F1}, so $\abs{\tilde\xi_p(t+1)}\ge\abs{z_c}$. This proves $\mathcal H(t+1)$, hence \ref{mfl:mff-I1}--\ref{mfl:mff-I3}.

\step{Part (i): $\Phi$-growth.} Fix $t$ with $t+1\le\tau_1\kappa$; everything below is at time $t$. Since $G$ is a fixed set of
particles, $\Phi_G(t+1)-\Phi_G(t)=\E[dU+dM+dP;G]$ by \ref{mfl:mff-F5}. On $G$: $dP\ge0$ (no overshoot); $dU\ge q_pk_p^2/\sqrt a$
by \ref{mfl:mff-F5} and Lemma~\ref{mfl:mff-lem:H} ($\abs{\rho_p}^2\le\rho_G^2\le0.38$); and $dM\ge-k'^2(W_m-W)_+$. For the last
term, $\delta\omega_1\le1/c_G$, $\abs{\omega_1}\le\kappa\hat\Omega$ and $W\ge\sqrt D\abs\zeta\ge\kappa\sqrt D\hat\theta$ give
$(W_m-W)_+\le\hat\delta_M$; and $u^+\le u(1+C'/(\kappa\theta))$ (Step 3) with $u^+\ge u$ \ref{mfl:mff-F1} gives
$k'\le(1+C'/(\kappa\theta))k_p$. Hence
\begin{equation}\label{mfl:read-l5-gain-loss}
\Phi_G(t+1)-\Phi_G(t)\ge\frac{\E[q_pk_p^2;G]}{\sqrt a}
-\Big(1+\frac{C'}{\kappa\theta}\Big)^2\E[k_p^2\hat\delta_M;G].
\end{equation}
We bound the gain and loss in \eqref{mfl:read-l5-gain-loss} separately.

For the numerator, Cauchy--Schwarz gives
\begin{equation}\label{mfl:read-l5-numerator}
\E[q_pk_p^2;G]=2K_G-\E[(2-q_p)k_p\cdot k_p;G]\ge2K_G-Q_2\sqrt{K_G}.
\end{equation}
Indeed,
$(2-q_p)^2k_p^2=(3\abs\rho^2-\abs\rho^4)^2k_p^2\le(3\hat\rho^2-\hat\rho^4)^2\hat\rho^2$ ($x\mapsto3x-x^2$ increases on $[0,1]$,
and $k_p\le\abs\rho\le\hat\rho$). Also $\E[q_pk_p^2;G]\ge0$.
For the normalization budget,
\[
(r-1)a=\E[\abs h^2]\le\E[(1+\ell_p)^2k_p^2;G]+(\pi/2+1)^2\nu.
\]
Here $(1+\ell_p)^2-1=\rho_\infty^2/(1-\rho_\infty^2)\le\hat\rho^2/(1-\hat\rho^2)$ and $k_p^2\le\hat\rho^2$, so
\begin{equation}\label{mfl:read-l5-budget}
(r-1)a\le K_G+H_2+(\pi/2+1)^2\nu.
\end{equation}
Combining \eqref{mfl:read-l5-numerator} and \eqref{mfl:read-l5-budget}, and using the nonnegativity of the numerator,
the gain is at least
\begin{equation}\label{mfl:read-l5-gain}
\begin{aligned}
&2\sqrt{r-1}\sqrt{K_G}\,
\frac{(1-Q_2/(2\sqrt{K_G}))_+}{\sqrt{1+(H_2+(\pi/2+1)^2\nu)/K_G}}\\
&\hspace{2em}\ge2\beta_G\sqrt{r-1}\sqrt{K_G}.
\end{aligned}
\end{equation}
Each factor is
nondecreasing in $K_G$, and $K_G\ge\underline K$ because $k_p^2\ge\abs{z_c}^2/(\hat\Omega^2+Dv^2)$ on $G$.
For the mean-mode loss, Cauchy--Schwarz and $k_p\le\hat\rho$ give
\begin{equation}\label{mfl:read-l5-loss}
\E[k_p\cdot k_p\hat\delta_M;G]
\le\sqrt{K_G}\,(\E[\hat\rho^2\hat\delta_M^2;G])^{1/2}\le\sqrt{K_G}\,W_2.
\end{equation}
Substitute \eqref{mfl:read-l5-gain} and \eqref{mfl:read-l5-loss} into \eqref{mfl:read-l5-gain-loss}. By the definition of $\beta'$,
\[
\Phi_G(t+1)-\Phi_G(t)\ge2\beta'\sqrt{r-1}\sqrt{K_G}.
\]
Finally, \ref{mfl:mff-F1} bounds the current statistic by
\[
\Phi_G=\E[u\,k_p;G]\le(\E[u^2])^{1/2}K_G^{1/2}
=\sqrt{r-1}\,\sigma_c\sqrt{K_G}\le\sqrt{r-1}(\kappa+t)\sqrt{K_G}.
\]
Since $\beta'>0$, these two inequalities yield
\begin{equation}\label{mfl:read-l5-phi-growth}
\Phi_G(t+1)-\Phi_G(t)\ge\frac{2\beta'\Phi_G(t)}{\kappa+t}.
\end{equation}

\step{Part (ii): Ratio.} Iterating \eqref{mfl:read-l5-phi-growth} and applying Lemma~\ref{mfl:mff-lem:D} with $x=2\beta'$ gives
\begin{equation}\label{mfl:read-l5-phi-product}
\begin{aligned}
\Phi(t)\ge\Phi_G(t)
&\ge\Phi_G(0)\prod_{s<t}\Big(1+\frac{2\beta'}{\kappa+s}\Big)\\
&\ge\Phi_G(0)\exp\big(2\beta'\ln(1+\tau)-2\beta'^2\tau/\kappa\big).
\end{aligned}
\end{equation}
We next compare the initial good-set statistic with the full statistic:
\begin{equation}\label{mfl:read-l5-initial-loss}
\Phi_G(0)\ge\Phi(0)(1-e_0).
\end{equation}
To prove \eqref{mfl:read-l5-initial-loss}, at $t=0$ we have $u^2/\abs w=\kappa\abs{z_c}^2/\abs g$ with
$g=(z,\sqrt Dv\,n)\sim N(0,I_d)$, and $G^c\subset\{q>M\}\cup\{v<v_{\min}\}$. On $\{q>M\}$ use $\abs g\ge\sqrt Dv=\chi_D$ and the
independence of $z$ and $v$: $\E[\abs{z_c}^2;q>M]=(r-1)\Prob(\chi^2_{r+2}>M)$, since $\abs{z_c}^2/q$ has mean $(r-1)/r$
independently of $q$ and $\E[q;q>M]=r\Prob(\chi^2_{r+2}>M)$. On $\{v<v_{\min}\}$ use $\abs g\ge\abs z$ and
$\E[\abs{z_c}^2/\abs z]=(r-1)\E[\chi_r]/r$. Dividing by $\Phi(0)=\kappa(r-1)\E[\chi_d]/d$ \ref{mfl:mff-F7} proves \eqref{mfl:read-l5-initial-loss}.

For the denominator, particles of $G$ never overshoot by \ref{mfl:mff-I1}. Thus
\[
\nu_{\rm os}(s)\le\Prob(G^c)\le\nu,\qquad K_\nu(t)\le\tau\sqrt\nu,
\]
and \ref{mfl:mff-F3} yields
\begin{equation}\label{mfl:read-l5-perp-budget}
\Gamma_\perp(t)\le\Gamma_\perp(0)\Big(1+\frac{\tau\sqrt\nu}{\E[v]}\Big).
\end{equation}
Combine \eqref{mfl:read-l5-phi-product}, \eqref{mfl:read-l5-initial-loss} and \eqref{mfl:read-l5-perp-budget}.
Since $\Gamma_U=\Phi/(r-1)$,
$\Gamma_U(t)/\Gamma_\perp(t)\ge R(\tau)\ge0$ (since $e_0<1$), and \ref{mfl:mff-F2} gives $\lambda_c/\lambda_\perp\ge(\Gamma_U/\Gamma_\perp)^2\ge R(\tau)^2$.

\step{Part (iii): Mean eigenvalue.} We bound the terms of \eqref{mfl:mff-eq:L1ratio} with $K_\nu\le\tau\sqrt\nu$ (as in (ii)); it
remains to show $\sqrt{D-\frac12}\,\Gamma_1/\kappa\ge\mathcal G_1(\tau)$. First,
$\Gamma_1/\kappa=\E[\tilde\omega_1^2/\abs{\tilde w}]\ge\E[\tilde\omega_1^2/\abs{\tilde w};G]$. Write
$\tilde\omega_{1,p}(t)=z_1+m_p(t)$ with $m_p(t):=\kappa^{-1}\sum_{s<t}\delta\omega_{1,p}(s)$. By \ref{mfl:mff-I3} and the
monotonicity of $\hat\rho$ in $\tau$, $m_p\in[(1-e_f(\hat\rho))S_t,(1+e_f(\hat\rho))S_t]$, where $\hat\rho=\hat\rho(\tau;q_p)$ and
$S_t:=\kappa^{-1}\sum_{s<t}s_s\le\bar S:=\tau/((1-e_G)\sqrt{1-\nu})$. Then $\tilde\omega_1^2\ge z_1^2+2z_1m_p$.
We treat the main term and the cross term in this lower bound separately.

For the main term, put $\abs g^2:=q+Dv^2$ and $\Delta_p:=\hat\Omega^2(\tau;q)-q$. On $G$,
$\abs{\tilde w}^2\le\hat\Omega^2+Dv^2=\abs g^2+\Delta_p$, so
\begin{equation}\label{mfl:read-l5-inverse-radius}
\begin{aligned}
\frac1{\abs{\tilde w}}
&\ge\abs g^{-1}(1+\Delta_p/\abs g^2)^{-1/2}\\
&\ge\abs g^{-1}(1-\Delta_p/(2\abs g^2))_+\ge\frac{F_p(\tau)}{\abs g},
\end{aligned}
\end{equation}
since $\abs g^2\ge q+Dv_{\min}^2$ on $G$. Given $q$, $z_1^2$
has mean $q/r$ and $v$ is independent; Lemma~\ref{mfl:mff-lem:V}(a) with $\varphi(x)=(q+Dx)^{-1/2}$, $\bar\varphi(x)=(Dx)^{-1/2}$ gives
$\E[(q+Dv^2)^{-1/2};v\ge v_{\min}]\ge(q+D)^{-1/2}-m_h/\sqrt D$. Multiplying by $\sqrt{D-\frac12}$ and using
$\sqrt{(D-\frac12)/D}\le1$, we obtain
\begin{equation}\label{mfl:read-l5-main-term}
\sqrt{D-\frac12}\,\E[z_1^2/\abs{\tilde w};G]\ge\mathrm{main}(\tau).
\end{equation}

For the cross term, use pairing: $G$ and all envelopes depend on $(q,v)$ only, and the initial law is invariant under
$z_1\mapsto-z_1$. Pair each particle $p$ with $z_1>0$ with its mirror $p'$ (same $z_c$, $v$, $n$, and $-z_1$): both lie in
$G$ and have the same envelopes. Using \eqref{mfl:read-l5-inverse-radius} for $p$ and the lower bound
$\abs{\tilde w_{p'}}\ge\sqrt D\tilde\zeta_{p'}\ge\sqrt D\hat\theta$ for its mirror gives
\[
\frac{m_p}{\abs{\tilde w_p}}\ge(1-e_f)S_t\,\frac{F_p}{\abs g},\qquad
\frac{m_{p'}}{\abs{\tilde w_{p'}}}\le\frac{(1+e_f)S_t}{\sqrt D\hat\theta}.
\]
Thus the pair contributes
\begin{equation}\label{mfl:read-l5-pairing}
\begin{aligned}
2z_1\Big(\frac{m_p}{\abs{\tilde w_p}}-\frac{m_{p'}}{\abs{\tilde w_{p'}}}\Big)&\ge-2z_1S_t\,B_p,\\
B_p&:=\frac{1+e_f}{\sqrt D\hat\theta}-(1-e_f)\frac{F_p}{\abs g}.
\end{aligned}
\end{equation}
The bracket is pointwise nonnegative, because $F_p/\abs g\le1/\abs g\le1/(\sqrt Dv_{\min})\le1/(\sqrt D\hat\theta)$
($\hat\theta\le v_{\min}$). So $S_t$ in \eqref{mfl:read-l5-pairing} may be replaced by its upper bound $\bar S$. Averaging over $v$ with
Lemma~\ref{mfl:mff-lem:V}(a) as in the main term (which only lowers the subtracted part), and multiplying by
$\sqrt{D-\frac12}\le\sqrt D$, gives $\sqrt{D-\frac12}\,\E[B_p;v\ge v_{\min}\mid q]\le X_p(\tau)$. Finally
$\E[2z_1\ind{z_1>0}\mid q]=\E[\abs{z_1}\mid q]=\sqrt q\,\bar\vartheta_r$. Hence
\begin{equation}\label{mfl:read-l5-cross-term}
\sqrt{D-\frac12}\,\E[2z_1m_p/\abs{\tilde w};G]\ge-\mathrm{cross}(\tau).
\end{equation}
Adding \eqref{mfl:read-l5-main-term} and \eqref{mfl:read-l5-cross-term} gives $\sqrt{D-\frac12}\,\Gamma_1/\kappa\ge\mathcal G_1(\tau)$, and \eqref{mfl:mff-eq:L1ratio} gives (iii).
\end{proof}

\begin{remark}[Global and per-cell versions]\label{mfl:mff-rem:versions}
Proposition~\ref{mfl:mff-prop:cert} leaves the slots \ref{mfl:mff-C6}, \ref{mfl:mff-C8} free. The \emph{global} (single-cell) version
replaces every envelope by its worst case over $G$ and every expectation by exact $\chi^2$ moments; it is used in
Lemma~\ref{mfl:mff-lem:A}. The \emph{per-cell} version evaluates the expectations as rigorous Riemann sums over $q$-cells
(Lemma~\ref{mfl:mff-lem:QC}); it is used for region N. Both are instances of Proposition~\ref{mfl:mff-prop:cert}.

\end{remark}

\section{Evaluating the certificate}\label{mfl:sec:eval}
\begin{secsummary}
\emph{What this section proves.} The hypotheses of Proposition~\ref{mfl:mff-prop:cert} hold, with a margin $>1$, for every
$(r,\Dp)$ with $\Dp\ge\max(400,63r)$; this completes the proof of Theorem~\ref{mfl:mff-thm:MFF} (Section~\ref{mfl:mff-sec:proof}).
Section~\ref{mfl:mff-sec:eval} turns the certificate into finitely many checks: the expectations become Riemann sums over $q$-cells
(Lemma~\ref{mfl:mff-lem:QC}), one set of numbers certifies a whole $\Dp$-interval (Lemma~\ref{mfl:mff-lem:MON}), and one floor
$\kappa_0$ certifies all $\kappa\ge\kappa_0$ (Lemma~\ref{mfl:mff-lem:floor}). Region A (large $\Dp$) is handled in closed form
(Lemma~\ref{mfl:mff-lem:A}); region N by $80$ certificates checked in interval arithmetic (Proposition~\ref{mfl:mff-prop:regionN}).

\emph{What is checked numerically.} Proposition~\ref{mfl:mff-prop:cert} and Lemmas~\ref{mfl:mff-lem:L1split}--\ref{mfl:mff-lem:chi},
\ref{mfl:mff-lem:QC}, \ref{mfl:mff-lem:MON} and \ref{mfl:mff-lem:floor} are proved by hand. The finite Region A endpoint checks have explicit rational bounds in Section~\ref{mfl:mff-sec:R-endpoints}. The $80$ region-N certificates are verified with outward interval enclosures, using the exact inputs and formulas printed below.
\end{secsummary}

\subsection{Cells, \texorpdfstring{$D$}{D}-intervals and the floor}\label{mfl:mff-sec:eval}

\paragraph{Notation for the evaluation.} We retain the certificate definitions of Section~\ref{mfl:mff-sec:cert}:
$D=d-r$, $q=\abs z^2$ for $z\sim N(0,I_r)$, $z_c=\Pc z$, and the independent radial variable
$v=\chi_D/\sqrt D$. The good set is $G=\{q\le M,\ v\ge v_{\min}\}$; $M$ is its squared-radius cutoff.
The hatted quantities are the envelopes in \ref{mfl:mff-C5}, and $\mathcal N,\mathcal K_a,\underline K,Q_2,H_2,W_2$
are the moment bounds in \ref{mfl:mff-C6}. The notation $\E[Z;B]=\E[Z\ind{B}]$ retains its unnormalized meaning.
The constants $m_1,m_2,m_h$ bound the lower-tail inverse moments in Lemma~\ref{mfl:mff-lem:V}(d).

\begin{lemma}[Lemma QC: $q$-cells]\label{mfl:mff-lem:QC}
Split $[0,M]$ into cells $[q_a^{(i)},q_b^{(i)}]$, $i=1,\dots,n$, with weights $w_i:=\Prob(\chi^2_r\in[q_a^{(i)},q_b^{(i)}])$ (any split is valid; region N uses the deterministic lattice partition of Section~\ref{mfl:mff-sec:regionN}). Write $\hat\Omega_i,\hat\theta_i,\hat\rho_i,
\hat\ell_i,\hat\delta_{M,i}$ for the envelopes at $(\tau_1;q_b^{(i)})$, and put $b_1:=(r-1)/r$ and
$b_2:=(r^2-1)/(r(r+2))$. With $m_1,m_2,m_h$ from Lemma~\ref{mfl:mff-lem:V}(d), the following are admissible slots:
\begin{align*}
\mathcal N&:=\tfrac{b_2}4\textstyle\sum_iw_i\,(q_a^{(i)})^2(1-\hat\rho_i^2)\big(D^2/(\hat\Omega_i^2+D)^2-m_2\big)_+,\\
\mathcal K_a&:=b_1\textstyle\sum_iw_i(1-\hat\rho_i^2-\hat\ell_i)_+^2\,q_a^{(i)}\big(D/(\hat\Omega_i^2+D)-m_1\big)_+,\\
D\underline K&:=b_1\textstyle\sum_iw_i\,q_a^{(i)}\big(D/(\hat\Omega_i^2+D)-m_1\big)_+,\\
Q_2^2&:=\textstyle\sum_iw_i(3\hat\rho_i^2-\hat\rho_i^4)^2\hat\rho_i^2,\\
H_2&:=\textstyle\sum_iw_i\,\hat\rho_i^4/(1-\hat\rho_i^2),\\
W_2^2&:=\textstyle\sum_iw_i\,\hat\rho_i^2\hat\delta_{M,i}^2.
\end{align*}
Likewise, at time $\tau$ (envelopes at $(\tau;q_b^{(i)})$), with
$F_i:=\big(1-(\hat\Omega^2(\tau;q_b^{(i)})-q_a^{(i)})_+/(2(q_a^{(i)}+Dv_{\min}^2))\big)_+$ and
$L_i:=F_i\big(\sqrt{D-\frac12}/\sqrt{q_b^{(i)}+D}-m_h\big)_+$,
\begin{gather*}
\mathrm{main}(\tau)\ge\sum_iw_i\frac{q_a^{(i)}}r\,L_i,\qquad
\mathrm{cross}(\tau)\le\frac{\tau\bar\vartheta_r}{\sqrt{1-\nu}\,(1-e_G)}\sum_iw_i\sqrt{q_b^{(i)}}\,X_i,\\
X_i:=\frac{1+e_f(\hat\rho)}{\hat\theta}-(1-e_f(\hat\rho))L_i,
\end{gather*}
with $e_f(\hat\rho)$ and $\hat\theta$ at $(\tau;q_b^{(i)})$. These are rigorous lower (resp.\ upper) Riemann sums for every $n$.

\end{lemma}
\begin{proof}
Every envelope is monotone in $\bar q$: $\hat\Omega$, $\hat\rho$ and $\hat\delta_M$ increase and $\hat\theta$ decreases. Given
$q$, $\abs{z_c}^2/q$ has the Beta$(\frac{r-1}2,\frac12)$ law, independent of $q$, with first two moments $b_1$ and $b_2$; and
$v$ is independent of $z$. Thus
\[
\E[\abs{z_c}^4\ind{q\in{\rm cell}_i}]
=b_2\E[q^2\ind{{\rm cell}_i}]\ge b_2(q_a^{(i)})^2w_i,
\]
with analogous bounds for $\abs{z_c}^2$ and $z_1^2$ ($\E[z_1^2\mid q]=q/r$).

In each cell, bound the integrand below at the worst corner: use $q_a$ in polynomial factors, and the envelopes at $q_b$
in denominators and in $1-\hat\rho^2$. Lemma~\ref{mfl:mff-lem:V}(a),(d) removes the $v$-average:
\begin{equation}\label{mfl:read-l6-cell-v-averages}
\begin{aligned}
\E[D/(\hat\Omega^2+Dv^2);v\ge v_{\min}]&\ge D/(\hat\Omega^2+D)-m_1,\\
\E[D^2/(\hat\Omega^2+Dv^2)^2;v\ge v_{\min}]&\ge D^2/(\hat\Omega^2+D)^2-m_2.
\end{aligned}
\end{equation}
These cellwise moment bounds and \eqref{mfl:read-l6-cell-v-averages} give the three lower slots.

For the upper slots, the integrands of $Q_2^2$, $H_2$ and $W_2^2$ are increasing in $\hat\rho^2$ and $\hat\delta_M$ and do not depend on $v$; bound them above by their values at $q_b$ and
drop the restriction $v\ge v_{\min}$. In $\mathrm{cross}$, $X_p\ge0$, so $\sqrt q\,X_p\le\sqrt{q_b}\,X_i$.
\end{proof}

In region N, $n=32000$ except in row $2b$, where $n=64000$; any partition gives a valid bound.

\begin{lemma}[Lemma MON: one certificate per $D$-interval]\label{mfl:mff-lem:MON}
Fix $r$, integers $D_{\rm hi}\ge D_{\rm lo}\ge5$ and the numbers $(M,\theta,\tau_1,\mathbf x)$ with
$D_{\rm lo}v_{\min}^2/(D_{\rm lo}-4)<1$, and use the $q$-cells of
Lemma~\ref{mfl:mff-lem:QC} (independent of $D$). Evaluate every slot of \ref{mfl:mff-C3}--\ref{mfl:mff-C8} at $D=D_{\rm lo}$, with
$\nu_v,m_1,m_2,m_h$ the bounds of Lemma~\ref{mfl:mff-lem:V}(d) at $D_{\rm lo}$, except in two places, which use $D_{\rm hi}$:
\begin{enumerate}[label=(\alph*)]
\item in $\Psi$, the budget term $(\pi/2+1)^2\nu/\underline K=(\pi/2+1)^2D\nu/(D\underline K)$ uses
$D\nu\le\nu_q D_{\rm hi}+\sup_{D'\ge D_{\rm lo}}D'\nu_v(D')$;
\item in the $\lambda_1$ bound (iii), $\E[v]$ is replaced by its upper bound $c_{D_{\rm hi}}\le1$;
\end{enumerate}
and $\E[v]$ in $R(\tau)$ is replaced by its lower bound $c_{D_{\rm lo}}$, and $R_0$ by the lower bound of
Lemma~\ref{mfl:mff-lem:chi}(c). If (S-$\Omega$), (S-$A$), (S-$C$) hold for these values, then for every integer
$D\in[D_{\rm lo},D_{\rm hi}]$ they hold at $D$. Subject to the certificate-data domains, conclusions \ref{mfl:mff-I1}, \ref{mfl:mff-I2} and \ref{mfl:mff-I3} of Proposition~\ref{mfl:mff-prop:cert} hold at $D$; conclusions (i)--(iii) additionally require $\beta'>0$ and $0\le e_0<1$ for the conservative slot values computed in this way.

\end{lemma}
\begin{proof}
The parameters $(M,\theta,\tau_1,\mathbf x)$, $v_{\min}$, $G$ (as a set of $(q,v)$), $I(\tau)$ and the cell weights do not
depend on $D$. We check the five groups of slots in turn.

\step{Step 1 (Super-solution map).} $\rho_G^2=\Omega^2/(\Omega^2+D\theta^2)$ decreases in $D$, hence so do $\ell_G$ and
$e_G=e_f(\rho_G)$. The bound $\nu(D)=\nu_q+\mathcal U(D,v_{\min}^2)$ decreases, so $c_G(D)=(1-e_G)\sqrt{1-\nu}/(1+e_G)$
increases. Therefore $\hat\Omega^2=(\sqrt{\bar q}+\tau/c_G)^2e^{2C'I}$ decreases in $D$, and
$\hat\theta^3=v_{\min}^3-3\tau A\hat\Omega^2/\Omega^2$ increases (this is the only, implicit, way $\hat\theta$ depends on $D$:
through $c_G(D)$). So $\hat\rho^2$ and $\hat\ell$ decrease. The factors $D/(\hat\Omega^2+D)-m_1$ and
$D^2/(\hat\Omega^2+D)^2-m_2$ increase, where $m_1,m_2$ are the $D_{\rm lo}$ bounds, valid at every $D\ge D_{\rm lo}$ by
Lemma~\ref{mfl:mff-lem:V}(d). Hence $\mathcal N$ and $\mathcal K_a$ of Lemma~\ref{mfl:mff-lem:QC} increase, and the right-hand
sides of (S-$\Omega$), (S-$A$), (S-$C$) decrease in $D$: a super-solution at $D_{\rm lo}$ is one at every $D\ge D_{\rm lo}$.
The side conditions $\rho_G^2\le0.38$, $e_G<\frac12$, $v_{\min}<1$ are preserved.
\step{Step 2 (Budget ratios).} Write them with $D$-scaled quantities. $D\underline K$ increases. For the $Q_2$ slot, along
$x(D)=a/(a+bD)$ (the explicit $D$-dependence of $\hat\rho^2$ at fixed $a=\hat\Omega^2$, $b=\hat\theta^2$) one has
$dx/dD=-x(1-x)/D$ and
\begin{equation}\label{mfl:read-l6-budget-monotonicity}
\frac d{dD}\ln\big(Dx^3(3-x)^2\big)
=\frac1D\Big[1-\Big(3-\frac{2x}{3-x}\Big)(1-x)\Big]<0\quad(0\le x\le\tfrac12),
\end{equation}
because the bracketed product is at least $1.3$ there; and the implicit dependence ($\hat\Omega\downarrow$, $\hat\theta\uparrow$)
only lowers $x$, while $Dx^3(3-x)^2$ increases in $x$ on $[0,\frac12]$. Together with \eqref{mfl:read-l6-budget-monotonicity}, this shows that $DQ_2^2$ decreases. Likewise
$\frac d{dD}\ln(Dx^2/(1-x))=(x-1)/D<0$, so $DH_2$ decreases. Hence $Q_2/\sqrt{\underline K}$ and $H_2/\underline K$ are
largest at $D_{\rm lo}$. The only increasing term is $D\nu_q$ in $\Psi$, which is taken at $D_{\rm hi}$; $D\nu_v(D)$ is
bounded by its supremum over $D\ge D_{\rm lo}$.
\step{Step 3 (Loss).} $\hat\delta_M$ decreases in $D$ ($c_G\uparrow$, $\hat\Omega\downarrow$, $\hat\theta\uparrow$, $\sqrt D\uparrow$),
and so do $\hat\rho$ and $W_2$. Hence $\beta_G$ increases and $\lambda_L$ decreases, so $\beta'$ increases.
\step{Step 4 (Ratio).} $R_0$ is bounded below by Lemma~\ref{mfl:mff-lem:chi}(c). In $e_0$, Lemma~\ref{mfl:mff-lem:chi}(b) gives
\begin{gather*}
\Prob(\chi^2_{r+2}>M)\,\E[1/\chi_D]\frac d{\E[\chi_d]}\le\Prob(\chi^2_{r+2}>M)\sqrt{\frac d{D-\frac32}}\sqrt{\frac d{d-\frac12}},\\
\nu_v\frac d{\E[\chi_d]}\le\frac{e^{-D\iota(y)/2}\sqrt{1+r/D}\sqrt{d/(d-\frac12)}}{\sqrt\pi(1-y)},
\end{gather*}
and both right-hand sides decrease in $D$. Also $\nu$ decreases, $\E[v]\ge c_{D_{\rm lo}}$ in $R(\tau)$, and $\beta'$ increases.
\step{Step 5 (Mean eigenvalue).} $D-\frac12$ increases; $\E[v]\le c_{D_{\rm hi}}\le1$ in the denominators (this is the second slot that
uses $D_{\rm hi}$); $\sqrt{D-\frac12}/\sqrt{q+D}$ and $F_p$ increase in $D$, while $X_p$, $\nu$ and $e_G$ decrease. So
$\mathcal G_1$ increases in $D$.
\end{proof}

\begin{lemma}[The floor $\kappa_0$ and the certified margin]\label{mfl:mff-lem:floor}
Let a certificate on $[D_{\rm lo},D_{\rm hi}]$ satisfy the hypotheses of Lemma~\ref{mfl:mff-lem:MON}, set $\tau_F:=\tau_1$, let
$\kappa_0\ge\max(A/\theta^3,2/\tau_1,2)$ (the certificates use this maximum, rounded up), and put $[\tau_a,\tau_b]:=[\tau_F-1/\kappa_0,\tau_F]$. Evaluate $\beta'$ and
$\hat\delta_M$ at $\kappa=\kappa_0$, and suppose $\beta'>0$ and $0\le e_0<1$. Define
\begin{align}
R_{\rm low}&:=R_0(1-e_0)\frac{\exp\big(2\beta'\ln(1+\tau_a)-2\beta'^2\tau_b/\kappa_0\big)}{1+\tau_b\sqrt\nu/(\E[v])_{\rm lo}},
\label{mfl:read-l6-floor-ratio}\\
\Lambda_{\rm low}&:=\frac\pi2\Big(1-\frac1r\Big)\frac{\tau_a^2(D_{\rm lo}-\frac12)}{((\E[v])_{\rm hi}+\tau_a\sqrt\nu)^2}
+\Big(\frac{\mathcal G_1(\tau_b)_+}{(\E[v])_{\rm hi}+\tau_b\sqrt\nu}\Big)^2,
\label{mfl:read-l6-floor-mean}
\end{align}
with the slot values of Lemma~\ref{mfl:mff-lem:MON} ($(\E[v])_{\rm lo}=c_{D_{\rm lo}}$, $(\E[v])_{\rm hi}=c_{D_{\rm hi}}$ or $1$), and
${\rm margin}:=\min(R_{\rm low}^2,\Lambda_{\rm low})$. Then for every integer $D\in[D_{\rm lo},D_{\rm hi}]$ and every
$\kappa\ge\kappa_0$, at $t_F:=\lfloor\tau_F\kappa\rfloor\ge2$: $\lambda_c(t_F)/\lambda_\perp(t_F)\ge{\rm margin}$ and
$\lambda_1(t_F)/\lambda_\perp(t_F)\ge{\rm margin}$.

\end{lemma}
\begin{proof}
$\kappa\ge\kappa_0\ge A/\theta^3$, so Proposition~\ref{mfl:mff-prop:cert} applies at every such $D$ (Lemma~\ref{mfl:mff-lem:MON}),
and $\tau_F\kappa\ge2$.

Every $\kappa$-dependent factor improves as $\kappa$ grows: $e^{-2\beta'^2\tau/\kappa}$,
$(1+C'/(\kappa\theta))^2$ in $\lambda_L$, and $1/(2c_G\kappa)$ in $\hat\delta_M$.
Thus $\beta'(\kappa)\ge\beta'(\kappa_0)$. At $t=t_F$, the floor is accounted for by
\[
\tau=\frac{t_F}{\kappa}\in(\tau_F-1/\kappa,\tau_F]\subset[\tau_a,\tau_b],\qquad t_F\le\tau_1\kappa.
\]
Applying Lemma~\ref{mfl:mff-lem:D} with $x=2\beta'(\kappa_0)$ in part (ii) of the proof of
Proposition~\ref{mfl:mff-prop:cert} gives $\Gamma_U/\Gamma_\perp\ge R_{\rm low}$, with $R_{\rm low}$ as in
\eqref{mfl:read-l6-floor-ratio}.

In part (iii), the mean part increases in $\tau$, while $\mathcal G_1(\tau)$ decreases in $\tau$:
$\hat\Omega$ and $\hat\rho$ increase and $\hat\theta$ decreases, so $F_p$ decreases, $X_p$ increases and
$\mathrm{cross}$ increases. Evaluate these two terms at $\tau_a$ and $\tau_b$, respectively, to obtain
$\Lambda_{\rm low}$ in \eqref{mfl:read-l6-floor-mean}. Conclude with \ref{mfl:mff-F2}.
\end{proof}

\subsection{Region A: the closed-form certificate}\label{mfl:mff-sec:lemmaA}

\paragraph{Parameter rule.} For $r\ge2$ and $D\ge63r$ put $\theta:=\frac12$, $\bar v:=\frac9{10}$,
\[L:=\ln\frac D{r-1}+5,\qquad M:=(r+4)+2\sqrt{(r+4)L}+2L,\qquad \bar\Omega^2:=2M,\qquad\varepsilon:=\frac{2M}{2M+D/4},\]
$\ell:=(1-\varepsilon)^{-1/2}-1$ and $\nu_0:=e^{-L}=e^{-5}(r-1)/D$. Let $\nu_v,m_1,m_2,m_h$ be the bounds of
Lemma~\ref{mfl:mff-lem:V}(d) at this $D$ and $y=\bar v^2=0.81$, and
\begin{gather*}
\bar C':=\frac{1+\ell}{(1-\varepsilon-\ell)\sqrt{(1-\nu_0)\big(\frac D{\bar\Omega^2+D}-m_1\big)}},\qquad
\bar A:=\frac{\bar\Omega^2}{(1+\sqrt{1-\varepsilon})\sqrt{\bar{\mathcal N}}},\\
\bar{\mathcal N}:=\frac{(1-\varepsilon)(r^2-1)(1-\nu_0)}4\Big(\frac{D^2}{(\bar\Omega^2+D)^2}-m_2\Big).
\end{gather*}
Finally $\tau_1:=(\bar v^3-\theta^3)/(3\bar A)=151/(750\bar A)=0.201333\ldots/\bar A$ (so that $v_{\min}=\bar v$),
$\tau_F:=\tau_1$ and $\kappa_A:=4/\tau_1$.

\paragraph{Conditions.}
\begin{itemize}
\item[(R1)] $D\ge1000$;
\item[(R2)] $D\ge32M$;
\item[(R3)] $0.06015\,\sqrt{r^2-1}\,\sqrt{\frac\pi2(1-\frac1r)(D-\frac12)}\big/M\ge1.0011$.
\end{itemize}

\begin{lemma}[Validity range of the rule]\label{mfl:mff-lem:R}
\begin{enumerate}[label=(\alph*)]
\item (R1)--(R3) are upward closed in $D$ (at fixed $r$).
\item For every $r\ge81$, (R1)--(R3) hold at $D=63r$; put $D_A(r):=63r$.
\item For $2\le r\le80$, (R1)--(R3) hold at $D=D_A(r)$ (Appendix~\ref{mfl:sec:app}).
\end{enumerate}
Hence (R1)--(R3) hold for all $D\ge D_A(r)$, for every $r\ge2$.

\end{lemma}
\begin{proof}
\step{Part (a): Upward closure.} $L'(D)=1/D$, so $D\,dM/dD=\sqrt{(r+4)/L}+2$. Since $L\ge5$, $\sqrt{(r+4)/L}<r+4$ and $2<2L$, hence $M>D\,dM/dD$ and $D/M$
increases. Also $2(\sqrt{(r+4)/L}+2)<M$ (as $2\sqrt{(r+4)/5}<r+4$ and $4<2L$), so $D/M^2$ increases. (R1) is trivially upward
closed, (R2) is $D/M\ge32$, and the left side of (R3) is a constant times $\sqrt{(D-\frac12)/D}\sqrt{D/M^2}$.

\step{Part (b): Ranks $r\ge81$.} At $D=63r$ with $r\ge81$, $L=\ln(63r/(r-1))+5\le\ln(63\cdot81/80)+5\le9.1556$. (R2) reads
$31r\ge128+64L+64\sqrt{(r+4)L}$; since $32(4+2L)\le714$ and $64\sqrt L\le193.7$, it suffices that
$31r\ge714+193.7\sqrt{r+4}$. This holds at $r=81$, since $\sqrt{85}<9.22$ gives $714+193.7\sqrt{85}<2500<2511$, and the difference increases in $r$, since
$193.7/(2\sqrt{85})\le31$. (R1): $63r\ge5103$. (R3): by (R2), $M\le63r/32$, so its left side is at least
$0.06015\cdot\frac{32}{63}\frac{\sqrt{r^2-1}}r\sqrt{\frac\pi2(1-\frac1r)(63r-\frac12)}$, which is $\ge2.7$ at $r=81$ and
increasing in $r$.

\step{Part (c): Ranks $2\le r\le80$.} This follows from the finite rational endpoint certificate in Section~\ref{mfl:mff-sec:R-endpoints}. The numerical inequalities in (b) also follow from the elementary bounds specified there.
\end{proof}

\begin{lemma}[Lemma A: closed-form region A]\label{mfl:mff-lem:A}
Assume $d\ge64r$ and (R1)--(R3). Then for every $\kappa\ge\kappa_A$, at $t_F=\lfloor\tau_1\kappa\rfloor\ge4$:
\begin{align*}
\log\frac{\lambda_c}{\lambda_\perp}&\ge0.077216\,\frac{\sqrt{r^2-1}}M-\frac{1.1(r+1)}D\ \ge\ 0.0176\,\frac{\sqrt{r^2-1}}M>0,\\
\frac{\lambda_1}{\lambda_\perp}&\ge\Bigg(\frac{0.0601986\,\sqrt{r^2-1}\,\sqrt{\frac\pi2(1-\frac1r)(D-\frac12)}}{1.00106\,M}\Bigg)^2>1,
\end{align*}
and $\tau_1\in[0.0802648,\,0.100667]\cdot\sqrt{r^2-1}/M$.

\end{lemma}
\begin{proof}
We apply Proposition~\ref{mfl:mff-prop:cert} with one cell and worst-case envelopes (the global version of
Remark~\ref{mfl:mff-rem:versions}), with $M$, $\theta=\frac12$, $\tau_1$ and $\mathbf x=(\bar\Omega,\bar A,\bar C')$, at the given $D$
(no $D$-interval is needed). Every terminating decimal used below as an upper or lower bound denotes the corresponding exact rational number. The following inequalities establish the required directions of these bounds.

\step{Step 1 (A-1: Tails and $\varepsilon$).} Laurent--Massart gives $\Prob(\chi^2_k-k\ge2\sqrt{kx}+2x)\le e^{-x}$; with $k=r+4$,
$x=L$ this is $\Prob(\chi^2_{r+4}>M)\le e^{-L}=\nu_0$, and by stochastic order the same bound holds for $\chi^2_r$ and
$\chi^2_{r+2}$. Since $D\ge63r\ge63(r-1)$, $\nu_0\le e^{-5}/63\le1.06952\cdot10^{-4}$. By (R2), $\varepsilon\le2M/(2M+8M)=0.2$;
hence $\ell\le0.118034$, $1-\varepsilon-\ell\ge0.681965$, and $e_f(\rho_G)\sqrt r=\sqrt\varepsilon(1+2\varepsilon/\pi)\le0.504155$,
where $\rho_G^2=\bar\Omega^2/(\bar\Omega^2+D\theta^2)=\varepsilon$. Also $L\ge\ln63+5\ge9.14313$, so $M\ge r+22.2862$.

\step{Step 2 (A-2: $v$-tails).} By Lemma~\ref{mfl:mff-lem:V}(d) at $D=1000$, all bounds decrease in $D$ and
$1/\iota(0.81)\le48.2602<1000$. Thus, for every $D\ge1000$,
\[
\nu_v\le2.973\cdot10^{-6},\qquad m_1\le3.710\cdot10^{-6},\qquad
m_2\le4.629\cdot10^{-6},\qquad D\nu_v\le2.973\cdot10^{-3}.
\]
In particular,
\begin{equation}\label{mfl:read-l6-A-mass}
\nu\le\nu_0+\nu_v\le1.0993\cdot10^{-4},\qquad \sqrt\nu\le0.0104845.
\end{equation}

\step{Step 3 (A-3: Constants).} $c_G\ge0.474366$ (the worst case is $r=2$, as $e_G\le0.504155/\sqrt r$). (R2) gives
$D/(\bar\Omega^2+D)\ge16/17$, hence
\[
\bar C'\le1.68998,\qquad
\frac{2M}{\sqrt{r^2-1}}\le\bar A\le\frac{2.50837\,M}{\sqrt{r^2-1}}.
\]
The lower bound on $\bar A$ uses that the square root in its denominator is at most $1$ and
$1+\sqrt{1-\varepsilon}\le2$. Therefore
\begin{equation}\label{mfl:read-l6-A-time}
\tau_1\in[0.0802648,0.100667]\frac{\sqrt{r^2-1}}M,
\qquad \frac{\sqrt{r^2-1}}M<1.
\end{equation}

\step{Step 4 (A-4: Super-solution).} With one cell and $\hat\Omega\le\bar\Omega$, $\hat\rho^2\le\varepsilon$, choose
\begin{equation}\label{mfl:read-l6-A-slots}
\begin{aligned}
\mathcal N&:=\frac{(1-\varepsilon)(r^2-1)(1-\nu_0)}4
\Big(\frac{D^2}{(\bar\Omega^2+D)^2}-m_2\Big),\\
\mathcal K_a&:=(1-\varepsilon-\ell)^2(r-1)(1-\nu_0)
\Big(\frac D{\bar\Omega^2+D}-m_1\Big).
\end{aligned}
\end{equation}
These slots are admissible because
\[
\E[\abs{z_c}^4;q\le M]=(r^2-1)\Prob(\chi^2_{r+4}\le M),\qquad
\E[\abs{z_c}^2;q\le M]=(r-1)\Prob(\chi^2_{r+2}\le M),
\]
and both probabilities are at least $1-\nu_0$. With \eqref{mfl:read-l6-A-slots}, (S-$A$) and (S-$C$) hold with equality
by the definitions of $\bar A$ and $\bar C'$.

It remains to check (S-$\Omega$). We have
$I(\tau_1)=(\bar v^2-\theta^2)/(2\bar A)=0.28/\bar A\le0.14\sqrt{r^2-1}/M\le0.14$, so
\[\frac{(\sqrt M+\tau_1/c_G)e^{\bar C'I(\tau_1)}}{\sqrt M}\le\Big(1+\frac{0.100667}{0.474366\sqrt{24.2862}}\Big)e^{1.68998\cdot0.14}
\le1.32149<\sqrt2=\frac{\bar\Omega}{\sqrt M}.\]
Also $\rho_G^2=\varepsilon\le0.2<0.38$, $1-\varepsilon-\ell>0.05$, $e_G<\frac12$, $v_{\min}=0.9<1$, and
$\kappa_A=4/\tau_1=(3000/151)\bar A\ge8\bar A=\bar A/\theta^3$.

\step{Step 5 (A-5: Gain and loss).} Take
\[
\underline K:=(r-1)(1-\nu_0)\Big(\frac1{\bar\Omega^2+D}-\frac{m_1}D\Big)
\ge\frac{0.9410(r-1)}D.
\]
In the single-cell setting
we bound the two budget slots of part (i) of the proof of Proposition~\ref{mfl:mff-prop:cert} pointwise instead of through
$Q_2$ and $H_2$ (the pointwise bounds are used inside that proof, before its Cauchy--Schwarz step): on $G$, $\abs{\rho_p}^2\le\varepsilon$ gives $2-q_p\le3\varepsilon-\varepsilon^2$ and
$(1+\ell_p)^2-1\le\varepsilon/(1-\varepsilon)$. Consequently
\[
\E[q_pk_p^2;G]\ge(2-3\varepsilon+\varepsilon^2)K_G,\qquad
(r-1)a\le\frac{K_G}{1-\varepsilon}+(\pi/2+1)^2\nu.
\]
The gain is therefore at least $2\beta_G\sqrt{r-1}\sqrt{K_G}$ with
\begin{equation}\label{mfl:read-l6-A-gain}
\beta_G:=\frac{(1-1.5\varepsilon+0.5\varepsilon^2)\sqrt{1-\varepsilon}}{\sqrt{1+(1-\varepsilon)\Psi_0}}
\ \ge\ \frac{(1-1.5\varepsilon+0.5\varepsilon^2)\sqrt{1-\varepsilon}}{\sqrt{1+\Psi_0}},
\end{equation}
where $\Psi_0:=(\pi/2+1)^2\nu/\underline K$.
Now $(\pi/2+1)^2\le6.61$, and $D\nu\le D\nu_0+D\nu_v\le e^{-5}(r-1)+2.973\cdot10^{-3}$, so
$\Psi_0\le6.61(e^{-5}+D\nu_v)/((1-\nu_0)(16/17-m_1))\le0.068207$. Since
$(1-1.5\varepsilon+0.5\varepsilon^2)\sqrt{1-\varepsilon}$ decreases in $\varepsilon$,
\eqref{mfl:read-l6-A-gain} gives $\beta_G\ge0.643987/1.033541\ge0.623088$.

For the loss, $\kappa\ge\kappa_A$ gives $\kappa\theta\ge2/\tau_1\ge19.8675$, so
\[
\Big(1+\frac{\bar C'}{\kappa\theta}\Big)^2\le1.17736,\qquad
\hat\delta_M\le\frac{2\sqrt{\bar\Omega^2/D}+1/(c_G\kappa\sqrt D)}{c_G}
\le\frac{0.5+0.00168}{c_G}\le1.05758,
\]
because
$2\sqrt{\bar\Omega^2/D}=2\sqrt{2M/D}\le\frac12$ by (R2), $\kappa\ge4/0.100667$ and $D\ge1000$. With
$W_2:=\sqrt\varepsilon\cdot1.05758$, this yields
\begin{equation}\label{mfl:read-l6-A-net-growth}
\lambda_L\sqrt{r-1}\le0.446844,\qquad
\beta'=\beta_G(1-\lambda_L)\ge0.344664.
\end{equation}
The worst case is $r=2$.

\step{Step 6 (A-6: Ratio).} First control the initial ratio. Lemma~\ref{mfl:mff-lem:chi} gives
\[
R_0^2\ge\Big(1-\frac1{2d}\Big)\frac{D-\frac12}d
=\frac{(2d-1)(2D-1)}{4d^2}\ge\frac{D-1}{D+r}=1-x,
\]
where $x:=(r+1)/(D+r)$; the middle inequality is equivalent to $2r+1\ge0$.
By (R1) and $D\ge63r$,
\[
x\le\max(16/1015,17/1024)\le0.016602
\]
(the maximum is at $r=16$). Therefore
\begin{equation}\label{mfl:read-l6-A-initial-ratio}
-\log R_0^2\le\frac{x}{1-x}\le1.0169x\le\frac{1.0169(r+1)}D.
\end{equation}

Next control the initial good-set loss $e_0$. By Lemma~\ref{mfl:mff-lem:chi}(b), its first part is at most
\[
\nu_0\sqrt{\frac d{D-\frac32}}\sqrt{\frac d{d-\frac12}}\le1.0090\,\nu_0.
\]
Indeed,
\[
\frac d{D-\frac32}=1+\frac{r+\frac32}{D-\frac32}
\le1+\frac{D/63+\frac32}{D-\frac32}.
\]
The last expression decreases in $D$, so it is at most
$1+(1000/63+\frac32)/998.5\le1.017400$. Also $d/(d-\frac12)\le1000/999.5\le1.000501$.
Thus the product of the square roots is at most $(1.017400\cdot1.000501)^{1/2}\le1.00892$.

The second part satisfies
\[
e_0^v:=(\E[\chi_r]/r)\nu_vd/\E[\chi_d]
\le\frac{e^{-D\iota(0.81)/2}\sqrt{1+1/63}\sqrt{1000/999.5}}{\sqrt2\sqrt\pi\cdot0.19}.
\]
Here $De_0^v\le0.06702$ at $D=1000$ and decreases in $D$. Combining the two parts gives
\[
e_0\le1.0090\cdot1.06952\cdot10^{-4}+0.06702/1000\le4\cdot10^{-4},
\]
and hence
\begin{equation}\label{mfl:read-l6-A-initial-loss}
-2\log(1-e_0)\le2.001\,e_0\le\frac{0.01361(r-1)}D+\frac{0.135}D.
\end{equation}
Since $0.135+0.01361(r-1)\le0.0831(r+1)$ for $r\ge2$ and $1.0169+0.0831\le1.1$,
the three $D$-terms in \eqref{mfl:read-l6-A-initial-ratio} and \eqref{mfl:read-l6-A-initial-loss}
sum to at most $1.1(r+1)/D$.

We now include the discrete growth and the floor. With $\kappa\ge\kappa_A=4/\tau_1$,
\[
\tau_a:=\tau_1-1/\kappa\ge\frac34\tau_1,\qquad
\frac{4\beta'^2\tau_1}{\kappa}\le\beta'^2\tau_1^2.
\]
Apply Lemma~\ref{mfl:mff-lem:floor} with
$\kappa_0=\kappa_A\ge\max(\bar A/\theta^3,2/\tau_1,2)$ and use
\[
\ln(1+\tau_a)\ge\tau_a-\tau_a^2/2\ge\frac34\tau_1-\frac9{32}\tau_1^2.
\]
Together with the preceding initial-loss bounds, this yields
\begin{equation}\label{mfl:read-l6-A-log-growth}
2\log R\ge\tau_1\Big[4\beta'\Big(\frac34-\frac{9\tau_1}{32}\Big)
-\beta'^2\tau_1-\frac{2\sqrt\nu}{\E[v]}\Big]-\frac{1.1(r+1)}D.
\end{equation}
The bracket increases in $\beta'$ and decreases in $\tau_1$. At $\beta'=0.344664$, $\tau_1=0.100667$,
$\sqrt\nu\le0.0104845$ and $\E[v]\ge\sqrt{1-1/2000}\ge0.9997499$, it is $\ge0.962025$.
These values are supplied by \eqref{mfl:read-l6-A-net-growth}, \eqref{mfl:read-l6-A-time} and \eqref{mfl:read-l6-A-mass}.
Using the lower time bound in \eqref{mfl:read-l6-A-time} and
$0.962025\cdot0.0802648\ge0.077216$, \eqref{mfl:read-l6-A-log-growth} becomes
\begin{equation}\label{mfl:read-l6-A-contrast-ratio}
\log(\lambda_c/\lambda_\perp)\ge2\log R
\ge0.077216\frac{\sqrt{r^2-1}}M-\frac{1.1(r+1)}D.
\end{equation}
Finally, (R2) gives
\[
\frac{1.1(r+1)}D\le\frac{1.1}{32}\sqrt{\frac{r+1}{r-1}}\,\frac{\sqrt{r^2-1}}M,
\qquad \sqrt{\frac{r+1}{r-1}}\le\sqrt3.
\]
Thus the lower bound in \eqref{mfl:read-l6-A-contrast-ratio} is at least
\[
\Big(0.077216-\frac{1.1}{32}\sqrt3\Big)\frac{\sqrt{r^2-1}}M
\ge0.0176\frac{\sqrt{r^2-1}}M.
\]
In particular, this bound is positive if $D>14.25M\sqrt{(r+1)/(r-1)}$, which (R2) implies since $14.25\sqrt3<32$.

\step{Step 7 (A-7: Mean eigenvalue).} The mean part of Proposition~\ref{mfl:mff-prop:cert}(iii) suffices on its own.
The time and mass bounds \eqref{mfl:read-l6-A-time} and \eqref{mfl:read-l6-A-mass} give
\[
\tau_a\ge\frac34\tau_1\ge0.0601986\,\frac{\sqrt{r^2-1}}M,\qquad
\E[v]+\tau_a\sqrt\nu\le1+0.100667\cdot0.0104845\le1.00106,
\]
where $\frac34\cdot0.0802648=0.0601986$. These give the displayed eigenvalue bound.
Condition (R3), with $0.0601986\ge0.06015$ and $1.0011>1.00106$, makes it $>1$. Finally
$t_F\ge\lfloor4\rfloor=4$.
\end{proof}

\subsection{Region N: an explicit finite certificate}\label{mfl:mff-sec:regionN}

This part is a computer-assisted proof of finitely many inequalities. All of its real inputs, its finite sums and its target inequalities are specified here; the existence or output of a particular program is not a hypothesis.
Region N consists of $2\le r\le80$ and $\max(400,63r)\le D\le D_A(r)$, where $D_A(r)$ is the final endpoint for rank $r$ in Table~\ref{mfl:app-tab:regionN}. The two intervals for $r=2$ are indexed by $2a,2b$; the other rows are indexed by $r$.
Tables~\ref{mfl:app-tab:inputs-partition-1}--\ref{mfl:app-tab:inputs-supersolution-41} specify $(M,\theta,\tau_1,\Omega,A,C')$ as dyadic rationals. Set $\tau_F=\tau_1$ and take the integer $\kappa_0$ and the rational six-decimal margin $\mu$ from Table~\ref{mfl:app-tab:regionN}.

\paragraph{A deterministic partition.}
For a row with integer $K=2^{20}M$ and refinement index $j$, set $n=32000\,2^j$ and
\begin{equation}\label{mfl:read-l6-lattice}
 k_i=\left\lfloor\frac{iK}{n}\right\rfloor,\qquad q_i=2^{-20}k_i\quad(0\le i\le n).
\end{equation}
Thus $q_0=0$, $q_n=M$, and every boundary is specified by integer arithmetic. The exact weights are
\[
 w_i=F_r(q_i)-F_r(q_{i-1}),\qquad F_r(q)=\Pr(\chi_r^2\le q).
\]
Use the cells $[q_{i-1},q_i]$ from \eqref{mfl:read-l6-lattice} in Lemma~\ref{mfl:mff-lem:QC}. There are $64000$ cells only in row $2b$ and $32000$ in every other row. These uniform lattice cells replace the approximate equal-probability cells used to discover the original candidates; no quantile calculation is needed in this certificate.

\paragraph{Exact scalar choices.}
For each row let $D_-=D_{\rm lo}$, $D_+=D_{\rm hi}$, $d_-=D_-+r$, and $y=v_{\min}^2$ with $v_{\min}$ from (C2). Choose
\[
 \nu_q=1-F_r(M),\qquad \nu_v=\mathcal U(D_-,y),\qquad \nu=\nu_q+\nu_v.
\]
Use the explicit bounds $m_1,m_2,m_h$ of Lemma~\ref{mfl:mff-lem:V}(d), and use its displayed supremum for $\sup_{D\ge D_-}D\mathcal U(D,y)$. Evaluate all cell slots in Lemma~\ref{mfl:mff-lem:QC} at $D_-$, with the $D\nu$ replacement in Lemma~\ref{mfl:mff-lem:MON}. Use
\[
 (\E[v])_{\rm lo}=\sqrt{1-\frac1{2D_-}},\quad (\E[v])_{\rm hi}=1,\quad
 R_{0,\rm lo}=\sqrt{1-\frac1{2d_-}}\sqrt{\frac{D_--1/2}{d_-}}.
\]
In every lower bound involving $1-e_0$, replace $e_0$ by the following explicit upper bound:
\begin{equation}\label{mfl:read-l6-E0}
\begin{aligned}
 E_0={}&(1-F_{r+2}(M))\sqrt{\frac{d_-}{D_--3/2}}\sqrt{\frac{d_-}{d_--1/2}}\\
 &+\frac{e^{-D_-\iota(y)/2}}{\sqrt{\pi r}(1-y)}
 \sqrt{1+\frac r{D_-}}\sqrt{\frac{d_-}{d_--1/2}}.
\end{aligned}
\end{equation}
The bound \eqref{mfl:read-l6-E0} is proved in Lemma~\ref{mfl:mff-lem:MON}, using $\E[\chi_r]/r\le r^{-1/2}$. The row's integer $\kappa_0$ is used throughout the loss slot and the floor calculation, so $\tau_a=\tau_F-1/\kappa_0$ and $\tau_b=\tau_F$. All remaining quantities are now fixed by (C3)--(C8) and Lemmas~\ref{mfl:mff-lem:QC} and~\ref{mfl:mff-lem:floor}.

\paragraph{Explicit probability enclosures.}\label{mfl:mff-sec:IA}
There is no numerical quadrature over an unspecified probability law. For $a=r/2$, $x=q/2$, set
\[
 t_0=1,\qquad t_{m+1}=\frac{x}{a+m+1}t_m,\qquad
 S_N=\sum_{m=0}^N t_m,\qquad z_N=\frac{x}{a+N+1}.
\]
With any integer $N\ge\lceil2x\rceil+40$ (the computation uses the common row value $N=\lceil M\rceil+40$) one has $z_N<1/2$, and the positive series gives
\begin{equation}\label{mfl:read-l6-cdf-enclosure}
 \frac{e^{-x}x^a}{\Gamma(a+1)}S_N
 \ \le F_r(q)\le\
 \frac{e^{-x}x^a}{\Gamma(a+1)}
 \left(S_N+\frac{t_Nz_N}{1-z_N}\right).
\end{equation}
At $q=0$ use $F_r(0)=0$. For $r=2m$, $\Gamma(a+1)=m!$; for $r=2m+1$,
$\Gamma(a+1)=(2m+2)!\sqrt\pi/(4^{m+1}(m+1)!)$. Thus all cell masses are differences of the explicit positive-series enclosures \eqref{mfl:read-l6-cdf-enclosure}. Use any integer $N\ge\lceil2x\rceil+60$ at the single tail endpoint $q=M$ when evaluating $1-F_r(M)$ and $1-F_{r+2}(M)$.

The finite sums of Lemma~\ref{mfl:mff-lem:QC} are enclosed using these weights and outward interval arithmetic.
The computation uses binary64 intervals for cell arrays: after each correctly rounded elementary arithmetic operation or square root,
move each endpoint one representable number outward. Cubic-root endpoints are accepted only when outward cubing proves the requisite inequalities.
For a sum of $n$ nonnegative interval terms, inflate the rounded lower and upper sums by
\[
1-2n\,2^{-53}\qquad\text{and}\qquad1+2n\,2^{-53},
\]
respectively, and round outward; here $n\le64000$.

Scalar exponentials, logarithms, square roots and $\pi$ are enclosed at 120-bit precision: each interval contains the exact real value,
with its endpoints rounded outward, rather than merely giving a 120-bit point approximation.
The following rational formulas allow independent implementation of these enclosures.

For the exponential, take $z\ge0$ and $N+2>z$:
\[
 E_N(z):=\sum_{k=0}^N\frac{z^k}{k!}\le e^z\le E_N(z)+\frac{z^{N+1}}{(N+1)!}\frac1{1-z/(N+2)};
\]
Use reciprocals for negative arguments.

For the logarithm, take $x>0$, write $x=2^m u$ with $1\le u<2$ and set $s=(u-1)/(u+1)$. Then
\[
 0\le\log u-2\sum_{k=0}^{N-1}\frac{s^{2k+1}}{2k+1}
 \le\frac{2s^{2N+1}}{(2N+1)(1-s^2)},\qquad \log x=m\log2+\log u,
\]
where $\log2$ uses the same series at $s=1/3$.

For $\pi$, use that, for $0<z<1$, consecutive partial sums of $\sum_{k\ge0}(-1)^kz^{2k+1}/(2k+1)$ enclose $\arctan z$, and
$\pi=16\arctan(1/5)-4\arctan(1/239)$.

Positive square and cube roots are enclosed by rational bisection with the corresponding power comparisons. The factorial formulas above enclose every gamma value needed here; in particular
$\bar\vartheta_r=\Gamma(r/2)/(\sqrt\pi\,\Gamma((r+1)/2))$.

Increase truncation orders and bisection precision until the desired scalar interval width is reached. These are explicit alternatives for independently reproducing the scalar enclosures, not additional assumptions on a software library. When evaluating the growth exponential, fix a certified lower endpoint of $\beta'$ and apply Lemma~\ref{mfl:mff-lem:D} to that value, rather than independently optimizing its positive and negative occurrences.

\begin{proposition}[The finite region-N check]\label{mfl:mff-prop:regionN}
For every exact input row, the finite calculations just specified verify the certificate-data domains, (S-$\Omega$), (S-$A$), (S-$C$), and
\begin{equation}\label{mfl:read-l6-N-targets}
\begin{gathered}
 \kappa_0\ge\max(A/\theta^3,2/\tau_1,2),\qquad \beta'>0,\qquad 0\le E_0<1,\\
 R_{\rm low}^2\ge\mu>1,\qquad \Lambda_{\rm low}\ge\mu.
\end{gathered}
\end{equation}
The intervals in Table~\ref{mfl:app-tab:regionN} cover each required integer interval without gaps. Consequently the row proves both eigenvalue ratios are at least its displayed $\mu$ for every $D\in[D_{\rm lo},D_{\rm hi}]$ and every $\kappa\ge\kappa_0$.
\end{proposition}
\begin{proof}[Computer-assisted verification and deduction]
\step{Step 1 (Enclose the cell slots).} Substitute the exact integer tables into the finite expressions above.
Enclose each cell mass with \eqref{mfl:read-l6-cdf-enclosure}, then form the lower slots
$\mathcal N,\mathcal K_a,\underline K$ and the upper slots $Q_2,H_2,W_2$ of Lemma~\ref{mfl:mff-lem:QC}.

\step{Step 2 (Verify each certificate).} Evaluate the three super-solution right sides and compare their upper endpoints
with the exact $(\Omega,A,C')$. Form the remaining scalar and mean-eigenvalue bounds with the conservative $D$ choices above.
The resulting outward enclosures certify all inequalities in \eqref{mfl:read-l6-N-targets}; the aggregate bounds are recorded
in the numerical summary below, and every row's lower margin is recorded in Table~\ref{mfl:app-tab:regionN}.
This is the finite computer-assisted step, not an analytic replacement for it.

\step{Step 3 (Cover the full parameter range).} For coverage, the first left endpoint for each $r$ equals $\max(400,63r)$, consecutive endpoints satisfy $D_{{\rm lo},i+1}=D_{{\rm hi},i}+1$, and the last right endpoint defines $D_A(r)$. Finally apply Proposition~\ref{mfl:mff-prop:cert} and Lemmas~\ref{mfl:mff-lem:MON} and~\ref{mfl:mff-lem:floor}. These analytic results turn the finite checks into the claimed continuum of $\kappa$ and integer intervals of $D$.
\end{proof}

\paragraph{Enclosure summary.}
The following bounds hold simultaneously over all $80$ rows. Each decimal bound is rounded outward. The three super-solution ratios mean the enclosed right side divided by its exact left side.
\[
\begin{array}{c|c@{\qquad}c|c}
\text{quantity}&\text{upper bound}&\text{quantity}&\text{lower bound}\\\hline
\max\mathrm{rhs}_{\Omega}/\Omega&0.999999903&\min\beta'&0.415294\\
\max\mathrm{rhs}_{A}/A&0.999997516&\min(1-\rho_G^2-\ell_G)&0.353108\\
\max\mathrm{rhs}_{C'}/C'&0.999990967&\min_i X_i&0.032598\\
\max v_{\min}&0.968938&\min\{R_{\rm low}^2,\Lambda_{\rm low}\}&1.006240\\
\max D_-v_{\min}^2/(D_--4)&0.939091&&\\
\max\rho_G^2&0.378463&&\\
\max e_G&0.472221&&\\
\max E_0&0.002185&&
\end{array}
\]
In addition $\nu\le0.000327<1/2$, all three lower slots are strictly positive, every published row margin is no larger than its individually checked bound, and every printed integer floor exceeds its required real floor. The stronger common lower bound in this display is not needed: the conservative row-specific margins in Table~\ref{mfl:app-tab:regionN} are retained.

\subsection{\texorpdfstring{Proof of Theorem~\ref{mfl:mff-thm:MFF}}{Proof of the alpha = 0 theorem}}\label{mfl:mff-sec:proof}

\begin{proof}[Proof of Theorem~\ref{mfl:mff-thm:MFF}]
Let $r\ge2$ and $D\ge\max(400,63r)$ (i.e.\ $d\ge64r$ and $D\ge400$).
\begin{itemize}
\item If $D\ge D_A(r)$ (which includes every such $D$ when $r\ge81$), then (R1)--(R3) hold by Lemma~\ref{mfl:mff-lem:R}, and
Lemma~\ref{mfl:mff-lem:A} applies with $\tau_F=\tau_1$ and $\kappa_{\min}=\kappa_A=4/\tau_1$. It gives $\log(\lambda_c/\lambda_\perp)
\ge0.0176\sqrt{r^2-1}/M>0$ and $\lambda_1/\lambda_\perp>1$ at $t_F$; set $1+\delta:=\min\{\exp(0.0176\sqrt{r^2-1}/M),B_1\}$, where $B_1>1$ is the displayed lower bound for $\lambda_1/\lambda_\perp$ in Lemma~\ref{mfl:mff-lem:A}.
\item Otherwise $2\le r\le80$ and $\max(400,63r)\le D\le D_A(r)$. By Proposition~\ref{mfl:mff-prop:regionN}, $D$ lies in the
$D$-interval of one of the $80$ rows. Proposition~\ref{mfl:mff-prop:cert}, Lemma~\ref{mfl:mff-lem:QC}, Lemma~\ref{mfl:mff-lem:MON} and
Lemma~\ref{mfl:mff-lem:floor} then give, for every $\kappa\ge\kappa_0$ of that row, $\lambda_c/\lambda_\perp\ge{\rm margin}$ and
$\lambda_1/\lambda_\perp\ge{\rm margin}$ at $t_F=\lfloor\tau_F\kappa\rfloor$, with ${\rm margin}>1$ the IA value of the row. Set
$\tau_F$, $\kappa_{\min}:=\kappa_0$ and $1+\delta:={\rm margin}$ from that row.
\end{itemize}
In both cases $\min(\lambda_1,\lambda_c)\ge(1+\delta)\lambda_\perp>\lambda_\perp$ at $t_F$, so by \ref{mfl:mff-F2}
(Proposition~\ref{mfl:mf-prop:agop}) the top-$r$ eigenspace of the MF AGOP is $U$, it is unique, $\Ar(t_F)=1$ for every leading
selector, and $\lambda_r/\lambda_{r+1}=\min(\lambda_1,\lambda_c)/\lambda_\perp\ge1+\delta$. The numerical ranges in (a) are
read off Table~\ref{mfl:app-tab:regionN}; those in (b) are Lemma~\ref{mfl:mff-lem:A}.
\end{proof}

\section{\texorpdfstring{Proof of Theorem~\ref{mfl:intro-thm:main}}{Proof of the main theorem}}\label{mfl:sec:main}
\begin{secsummary}
\emph{What this section proves.} Theorem~\ref{mfl:intro-thm:main}. Steps~1 and~5 use only $\alpha\le\astar(t_P)$. Steps~2--3
prove (F): first for the $\alpha=0$ dynamics, by the certificate, and then for $0<\alpha\etab\le a_0$, by continuity in the head
scale. Step~4 is the isotropy at $t=0$.
\end{secsummary}

\paragraph{Notation for the assembly.} We use $D=d-r$, $P_\perp=I_d-P_U$,
$\Pone=\eb\eb^\top$ and $\Pc=P_U-\Pone$ from Sections~\ref{mfl:sec:setting} and~\ref{mfl:sec:sym}.
The law-update map $\Phi_\alpha$, force covariance $\Sigma_\alpha$ and AGOP surrogate $\mathcal M$
are those of Section~\ref{mfl:pc-sec:def}; the actual AGOP is $\alpha^2\mathcal M$ in the original coordinates.
The times $t_P,t_F$, head threshold $\astar(t_P)$ and gap parameter $\delta$ are those of Theorem~\ref{mfl:intro-thm:main}.

\begin{proof}[Proof of Theorem~\ref{mfl:intro-thm:main}]
\step{Step 0 (The trajectory).} With $\sigma_0=\kappa\etab\sqrt d$, the law $\nu_0=N(0,\sigma_0^2I_d/d)$ is the initialization of
every result below, and $\nu_{t+1}=\Phi_\alpha(\nu_t)$ with $\Phi_\alpha$ as in Section~\ref{mfl:pc-sec:def}; this is \eqref{mfl:eq:mfmap}, with the
Moore--Penrose convention, so the trajectory is defined for all $t$ and every $\alpha\ge0$.

By Lemma~\ref{mfl:pc-lem:PC3}, every
$\nu_t$ is absolutely continuous and has finite second moment, so $\nu_t(\{0\})=0$; by Proposition~\ref{mfl:mf-prop:symmetry}(c),
every $\nu_t$ is $\mathcal G$-invariant. These are the standing hypotheses of Section~\ref{mfl:sec:tools} and of
Proposition~\ref{mfl:mf-prop:agop}.

\step{Step 1 (P).} Theorem~\ref{mfl:plat-thm:PB} assumes $r\ge2$, $d>r$, $\kappa>0$, $t_P\in\mathbb N$ and $0<\alpha\le\astar(t_P)$. It
gives $\max_{s,t\le t_P}L(t)/L(s)\le1.01$ and $L(t)\ge.99V$ for $t\le t_P$, which is (P).

\step{Step 2 ((F) for the $\alpha=0$ dynamics).} By Lemma~\ref{mfl:pc-lem:scaling}, the trajectory with head $\alpha$ and step
$\etab$ is the image under $w\mapsto\etab w$ of the trajectory $\tilde\nu_a(t)$ with head $a:=\alpha\etab$ and step $1$ from
$N(0,\kappa^2I_d)$; the two have the same whitening matrices, and their MF AGOPs differ by the factor $\etab^2$. So it suffices
to prove (F) for $\tilde\nu_a$.

For $a=0$, $\tilde\nu_0(t)$ is the trajectory of Definition~\ref{mfl:mf-def:mf}. Our hypotheses
$r\ge2$, $d\ge64r$, $\Dp\ge400$ and $\kappa\ge\kappa_{\min}$ are those of Theorem~\ref{mfl:mff-thm:MFF}, which gives, at
$t_F=\lfloor\tau_F\kappa\rfloor\ge2$,
\begin{equation}\label{mfl:read-l7-zero-head-gap}
\begin{gathered}
\mathcal M(\tilde\nu_0(t_F))=\lambda_1^0\Pone+\lambda_c^0\Pc+\lambda_\perp^0P_\perp,\\
\min(\lambda_1^0,\lambda_c^0)\ge(1+\delta)\,\lambda_\perp^0.
\end{gathered}
\end{equation}
Both teacher-space block eigenvalues are strictly positive. Indeed, Lemma~\ref{mfl:rb-lem:M} gives the first inequality below,
and Proposition~\ref{mfl:mf-prop:agop}(1) with Lemma~\ref{mfl:rb-lem:facts}(e) gives the second:
\[
\lambda_1^0\ge\frac{(\E[\omega_1(t_F)])^2}4\ge\frac{(1-1/r)\,t_F^2}4>0,
\qquad \lambda_c^0\ge\frac{\Gamma_U(t_F)^2}{2\pi}>0.
\]
Thus $\min(\lambda_1^0,\lambda_c^0)>0$.

\step{Step 3 (Transfer to $0<a\le a_0$).} For a symmetric $d\times d$ matrix $A$ put
\[\ell_1(A):=\eb^\top A\eb,\qquad\ell_c(A):=\frac{\tr(\Pc A)}{r-1},\qquad\ell_\perp(A):=\frac{\tr(P_\perp A)}{\Dp},\]
\begin{equation}\label{mfl:read-l7-gap-functional}
\mathsf g(A):=\min\big(\ell_1(A),\ell_c(A)\big)-\big(1+\tfrac\delta2\big)\,\ell_\perp(A).
\end{equation}
These functions are continuous, and on a matrix of the block form $\lambda_1\Pone+\lambda_c\Pc+\lambda_\perp P_\perp$ they return
$\lambda_1$, $\lambda_c$ and $\lambda_\perp$.
The gap \eqref{mfl:read-l7-zero-head-gap} and positivity from Step~2 give
\begin{equation}\label{mfl:read-l7-positive-gap}
\mathsf g\big(\mathcal M(\tilde\nu_0(t_F))\big)
\ge\min(\lambda_1^0,\lambda_c^0)\Big(1-\frac{1+\delta/2}{1+\delta}\Big)>0.
\end{equation}
Apply Proposition~\ref{mfl:pc-prop:limit}(a) in units $\etab=1$, with $\sigma_0=\kappa\sqrt d$ and $T=t_F$.
As $a\to0$,
\[
\mathcal M(\tilde\nu_a(t_F))\to\mathcal M(\tilde\nu_0(t_F)),
\]
and there is $a_T>0$ with $\Sigma_a(\tilde\nu_a(t))\succ0$ for all $t\le t_F$ and $a\le a_T$.
By continuity, the strict inequality \eqref{mfl:read-l7-positive-gap} persists: there is $a_0\in(0,a_T]$ such that
\begin{equation}\label{mfl:read-l7-transferred-gap}
\mathsf g(\mathcal M(\tilde\nu_a(t_F)))>0\qquad\text{for every }a\in(0,a_0].
\end{equation}
The
trajectories $\tilde\nu_a$ are determined by $(r,d,\kappa,a)$, and $\delta=\delta(r,\Dp)$, so $a_0$ can be chosen as a function
of $(r,d,\kappa)$ alone.

Fix $a=\alpha\etab\in(0,a_0]$. By Step~0 and Proposition~\ref{mfl:mf-prop:agop}(0), $\mathcal M(\tilde\nu_a(t_F))$ has the block form,
and its eigenvalues are $\lambda_j=\ell_j(\mathcal M(\tilde\nu_a(t_F)))$, $j\in\{1,c,\perp\}$.
Use \eqref{mfl:read-l7-transferred-gap} in the definition \eqref{mfl:read-l7-gap-functional}, together with
$\lambda_\perp\ge0$ ($\mathcal M$ is positive semidefinite), to obtain
\[\min(\lambda_1,\lambda_c)\ >\ \big(1+\tfrac\delta2\big)\lambda_\perp\ \ge\ \lambda_\perp .\]
By Proposition~\ref{mfl:mf-prop:agop}(0), the top-$r$ eigenspace of $\mathcal M(\tilde\nu_a(t_F))$ is $U$ and is unique, $\Ar(t_F)=1$
for every leading selector, and $\lambda_r=\min(\lambda_1,\lambda_c)\ge(1+\delta/2)\lambda_{r+1}$.

The MF AGOP of the original trajectory is
\[
\alpha^2\mathcal M(\nu_{t_F})=\alpha^2\etab^2\mathcal M(\tilde\nu_a(t_F)),
\]
which has the same eigenspaces and eigenvalue ratios. Also
$\Sigma_{\nu_t}=\Sigma_a(\tilde\nu_a(t))\succ0$ for $t\le t_F$. This is (F) at $t_F$.

\step{Step 4 ((F) at $t=0$).} $\nu_0$ is $O(d)$-invariant, so $\mathcal M(\nu_0)$ is a multiple of $I_d$ by
Proposition~\ref{mfl:mf-prop:agop}(4).

\step{Step 5 (R).} Corollary~\ref{mfl:loss-cor:MF} assumes $r\ge2$, $d\ge64r$, $\kappa\ge1$, $t_P\in\mathbb N$ and
$0<\alpha\le\astar(t_P)$. Our hypotheses include these: $\kappa\ge\kappa_{\min}\ge2/\tau_F>19$, because $\kappa_0\ge2/\tau_1$ in
region N (Lemma~\ref{mfl:mff-lem:floor}) and $\kappa_A=4/\tau_1$ in region A.
The corollary gives
\[
L(0)-L(t_2)\ge.03V,\qquad
t_2=\Big\lceil\frac{T_2^*\sqrt V}{\alpha\etab}\Big\rceil,\qquad T_2^*=.05086.
\]
At $\alpha=\astar(t_P)$ it also gives $t_2\le C_R(\kappa+t_P)+1$ with $C_R\le6.013$;
for $t_P\le.1007\kappa$ this is at most $6.62\kappa+1$.
\end{proof}

\begin{remark}[Finite width]\label{mfl:main-rem:width}
Under (F), $\Sigma_{\nu_t}\succ0$ for $t\le t_F$, so Proposition~\ref{mfl:pc-prop:limit}(b) applies with $T=t_F+1$: the width-$m$
trajectory at fixed $\etab$ converges to the MF one on $[0,t_F+1]$, and its top-$r$ AGOP space at $t_F$ converges to $U$ as
$m\to\infty$. Given $\nu_0$, the MF trajectory is deterministic, so (P), (F) and (R) hold surely.
\end{remark}

\section{Finite verification details}\label{mfl:sec:repro}
This section specifies the finite inequalities used above.

\subsection{The Region A endpoints}\label{mfl:mff-sec:R-endpoints}
To prove Lemma~\ref{mfl:mff-lem:R}(c), the following finite rational certificate applies to the right
endpoints $D_A(r)$ of Table~\ref{mfl:app-tab:regionN}. Write $H_r$ for the
left side of (R3) at $D=D_A(r)$. Each row gives lower bounds valid
for every integer rank in its first column:
\[
\begin{array}{c|r|c|c}
r & \text{lower bound on }D_A(r) & D_A(r)/M\ge & H_r^2\ge\\ \hline
2\!:\!10 &7917&125.18416249&1.00220175\\
11\!:\!20&2596&34.75697688&1.00224498\\
21\!:\!30&2429&32.00130160&1.00385938\\
31\!:\!40&2905&32.00055080&1.86188294\\
41\!:\!50&3363&32.00109878&2.83732405\\
51\!:\!60&3808&32.00078823&3.89645378\\
61\!:\!70&4243&32.00079990&5.01895909\\
71\!:\!80&4671&32.00051696&6.19274707
\end{array}
\]
In particular $D_A(r)\ge1000$, $D_A(r)/M\ge32$, and
$H_r^2\ge1.00220175>1.0011^2=1.00220121$, proving (R1)--(R3).
The entries follow from three rational enclosures, described below.

\paragraph{Logarithm enclosure.} For $0\le z\le1/3$, put
\[
S(z):=2\sum_{j=0}^{31}\frac{z^{2j+1}}{2j+1},\qquad
E(z):=\frac{2z^{65}}{65(1-z^2)}.
\]
The positive logarithm series and a geometric tail bound give
$S(z)\le\log((1+z)/(1-z))\le S(z)+E(z)$.
For rational $x\ge1$, write $x=2^k y$ with $1\le y<2$ and define
\begin{equation}\label{mfl:read-lb-log-enclosure}
\mathsf L(x):=k\bigl[S(1/3)+E(1/3)\bigr]
+S\left(\frac{y-1}{y+1}\right)+E\left(\frac{y-1}{y+1}\right).
\end{equation}
Thus $\log x\le\mathsf L(x)$, with a rational right side.

\paragraph{Square-root enclosure.} For rational $v\ge0$, choose the integer $q$ characterized by
\begin{equation}\label{mfl:read-lb-root-enclosure}
q^2\le10^{30}v<(q+1)^2,\qquad \mathsf S(v):=\frac{q+1}{10^{15}}.
\end{equation}
Then $\sqrt v<\mathsf S(v)$, certified by the two rational inequalities in \eqref{mfl:read-lb-root-enclosure}.

\paragraph{Endpoint comparison.} For each of the $79$ exact endpoints in the complete certificate table
(not the eight blockwise minima above), use \eqref{mfl:read-lb-log-enclosure} and \eqref{mfl:read-lb-root-enclosure} to set
\begin{equation}\label{mfl:read-lb-endpoint-radius}
\begin{aligned}
L_r^+&:=5+\mathsf L(D_A(r)/(r-1)),\\
M_r^+&:=r+4+2\mathsf S((r+4)L_r^+)+2L_r^+.
\end{aligned}
\end{equation}
With $\pi_-:=3.14159265358979323846264338327950288<\pi$,
the third and fourth columns of the table are lower bounds on
the following rational numbers, using \eqref{mfl:read-lb-endpoint-radius}, respectively:
\[
\frac{D_A(r)}{M_r^+},\qquad
\frac{(.06015)^2(r^2-1)(\pi_-/2)(1-1/r)(D_A(r)-1/2)}{(M_r^+)^2}.
\]
All table inequalities follow by integer cross-multiplication, with
the results rounded downward to eight decimal places.

The bound on $\pi$ can itself be checked from Machin's identity
$\pi=16\arctan(1/5)-4\arctan(1/239)$ using the first $64$ terms of
the alternating power series for each arctangent and its next-term
remainder bound. This proves (c) using only the printed endpoints
and finite rational inequalities. The numerical inequalities in
(b) follow from the same elementary bounds.

\section{Complete certificate data}\label{mfl:sec:app}
The exact dyadic input tables and Table~\ref{mfl:app-tab:regionN} together specify every parameter of the region-N certificate. The former determine $M,\theta,\tau_F,\Omega,A,C'$ and the deterministic cell partition; the latter give the integer $D$-interval, integer floor $\kappa_0$ and rational lower margin for each row. The six-decimal $\tau_F$ in the results table is a display only: its exact value is $T2^{-59}$ in the input table. No externally stored parameter values are needed.

The result row $(r,[D_{\rm lo},D_{\rm hi}])$ proves both eigenvalue ratios are at least its displayed margin at $t_F=\lfloor\tau_F\kappa\rfloor$, for every integer $D$ in that interval and every $\kappa\ge\kappa_0$. The margins retained here are conservative six-decimal lower bounds; the deterministic-cell verification can yield slightly stronger bounds.
For $2\le r\le80$, define $D_A(r)$ to be the last right endpoint for $r$ in this table. For $r\ge81$, define $D_A(r)=63r$. The region-A argument uses only the stated inequalities (R1)--(R3) at these endpoints and their upward closure, not the historical method used to select them.

\begin{table}[!htbp]
\caption{The 80 region-N certificates at $\alpha=0$: $D$-interval, $\tau_F$, floor $\kappa_0$ and certified margin
$\min(\lambda_1,\lambda_c)/\lambda_\perp$ at $t_F$.}\label{mfl:app-tab:regionN}
\centering\footnotesize
\setlength{\tabcolsep}{3.3pt}
\renewcommand{\arraystretch}{1.02}
\begin{tabular}{r c c r c c r c c r c}
\toprule
$r$ & $[D_{\rm lo},D_{\rm hi}]$ & $\tau_F$ & $\kappa_0$ & margin && $r$ & $[D_{\rm lo},D_{\rm hi}]$ & $\tau_F$ & $\kappa_0$ & margin\\
\midrule
2 & $[400,14930]$ & 0.013740 & 336 & 1.006223 && 41 & $[2583,3363]$ & 0.078835 & 126 & 1.228737\\
2 & $[14931,465350]$ & 0.014026 & 777 & 1.048716 && 42 & $[2646,3408]$ & 0.079339 & 125 & 1.230167\\
3 & $[400,119688]$ & 0.014191 & 286 & 1.012711 && 43 & $[2709,3453]$ & 0.079824 & 124 & 1.231501\\
4 & $[400,55355]$ & 0.017549 & 172 & 1.025338 && 44 & $[2772,3497]$ & 0.080300 & 124 & 1.232778\\
5 & $[400,32467]$ & 0.022016 & 136 & 1.029847 && 45 & $[2835,3542]$ & 0.080758 & 123 & 1.233975\\
6 & $[400,21679]$ & 0.025939 & 119 & 1.033781 && 46 & $[2898,3587]$ & 0.081197 & 122 & 1.235096\\
7 & $[441,15709]$ & 0.027378 & 110 & 1.039982 && 47 & $[2961,3631]$ & 0.081637 & 121 & 1.236213\\
8 & $[504,12040]$ & 0.030688 & 99 & 1.048280 && 48 & $[3024,3675]$ & 0.082058 & 121 & 1.237254\\
9 & $[567,9613]$ & 0.039394 & 125 & 1.057938 && 49 & $[3087,3720]$ & 0.080968 & 122 & 1.238515\\
10 & $[630,7917]$ & 0.039459 & 127 & 1.079513 && 50 & $[3150,3764]$ & 0.081371 & 122 & 1.239536\\
11 & $[693,6681]$ & 0.045293 & 120 & 1.089914 && 51 & $[3213,3808]$ & 0.081765 & 121 & 1.240522\\
12 & $[756,5748]$ & 0.046299 & 119 & 1.101188 && 52 & $[3276,3852]$ & 0.082149 & 121 & 1.241473\\
13 & $[819,5025]$ & 0.048771 & 113 & 1.109930 && 53 & $[3339,3896]$ & 0.082525 & 120 & 1.242394\\
14 & $[882,4452]$ & 0.051069 & 109 & 1.117599 && 54 & $[3402,3939]$ & 0.082900 & 119 & 1.243321\\
15 & $[945,3989]$ & 0.053212 & 104 & 1.124371 && 55 & $[3465,3983]$ & 0.083257 & 119 & 1.244175\\
16 & $[1008,3608]$ & 0.053451 & 104 & 1.130751 && 56 & $[3528,4027]$ & 0.083605 & 118 & 1.244999\\
17 & $[1071,3290]$ & 0.055296 & 101 & 1.136427 && 57 & $[3591,4070]$ & 0.083944 & 118 & 1.245786\\
18 & $[1134,3022]$ & 0.057027 & 98 & 1.141549 && 58 & $[3654,4113]$ & 0.084282 & 117 & 1.246585\\
19 & $[1197,2793]$ & 0.058666 & 96 & 1.146243 && 59 & $[3717,4157]$ & 0.084603 & 117 & 1.247310\\
20 & $[1260,2596]$ & 0.060199 & 93 & 1.150495 && 60 & $[3780,4200]$ & 0.084923 & 116 & 1.248040\\
21 & $[1323,2429]$ & 0.061591 & 91 & 1.154202 && 61 & $[3843,4243]$ & 0.085235 & 116 & 1.248739\\
22 & $[1386,2478]$ & 0.067373 & 140 & 1.164924 && 62 & $[3906,4286]$ & 0.085546 & 115 & 1.249446\\
23 & $[1449,2526]$ & 0.062740 & 141 & 1.162169 && 63 & $[3969,4329]$ & 0.085839 & 115 & 1.250079\\
24 & $[1512,2574]$ & 0.066769 & 140 & 1.176584 && 64 & $[4032,4372]$ & 0.086132 & 114 & 1.250719\\
25 & $[1575,2622]$ & 0.070623 & 139 & 1.179868 && 65 & $[4095,4415]$ & 0.086416 & 114 & 1.251325\\
26 & $[1638,2670]$ & 0.066155 & 140 & 1.182391 && 66 & $[4158,4458]$ & 0.086699 & 114 & 1.251938\\
27 & $[1701,2717]$ & 0.069717 & 139 & 1.193370 && 67 & $[4221,4500]$ & 0.086974 & 113 & 1.252516\\
28 & $[1764,2764]$ & 0.072271 & 138 & 1.188541 && 68 & $[4284,4543]$ & 0.087239 & 113 & 1.253060\\
29 & $[1827,2811]$ & 0.068865 & 139 & 1.197385 && 69 & $[4347,4586]$ & 0.087505 & 112 & 1.253610\\
30 & $[1890,2858]$ & 0.073818 & 135 & 1.200144 && 70 & $[4410,4628]$ & 0.086480 & 114 & 1.254297\\
31 & $[1953,2905]$ & 0.074541 & 134 & 1.204589 && 71 & $[4473,4671]$ & 0.086736 & 114 & 1.254845\\
32 & $[2016,2951]$ & 0.075237 & 133 & 1.208340 && 72 & $[4536,4713]$ & 0.086992 & 113 & 1.255396\\
33 & $[2079,2997]$ & 0.075905 & 131 & 1.211545 && 73 & $[4599,4755]$ & 0.087239 & 113 & 1.255911\\
34 & $[2142,3044]$ & 0.074843 & 133 & 1.214607 && 74 & $[4662,4797]$ & 0.087487 & 112 & 1.256433\\
35 & $[2205,3090]$ & 0.075475 & 132 & 1.217328 && 75 & $[4725,4840]$ & 0.087725 & 112 & 1.256919\\
36 & $[2268,3135]$ & 0.076079 & 131 & 1.219727 && 76 & $[4788,4882]$ & 0.087963 & 112 & 1.257411\\
37 & $[2331,3181]$ & 0.076674 & 130 & 1.221887 && 77 & $[4851,4924]$ & 0.088192 & 111 & 1.257865\\
38 & $[2394,3227]$ & 0.077242 & 129 & 1.223829 && 78 & $[4914,4966]$ & 0.088421 & 111 & 1.258326\\
39 & $[2457,3272]$ & 0.077791 & 128 & 1.225601 && 79 & $[4977,5008]$ & 0.088640 & 111 & 1.258751\\
40 & $[2520,3317]$ & 0.078322 & 127 & 1.227230 && 80 & $[5040,5050]$ & 0.088860 & 111 & 1.259180\\
\bottomrule
\end{tabular}
\end{table}

\FloatBarrier
\paragraph{Exact encoding.} For each row let $M=K2^{-20}$, $\tau_1=\tau_F=T2^{-59}$, $\Omega=O2^{-50}$, $A=B2^{-51}$, and $C'=C2^{-52}$. The index $h$ defines $\theta$ by
\[
\begin{aligned}
\theta(3)&=5404319552844595\,2^{-54},\\
\theta(35)&=3152519739159347\,2^{-53},\\
\theta(4)&=3602879701896397\,2^{-53}.
\end{aligned}
\]
The cell count is $n=32000\cdot2^j$. Rows $2a,2b$ refer to the two $r=2$ intervals in the results table; all other rows are indexed by $r$. Every integer below is exact.

\begin{table}[!htbp]
\caption{Exact partition inputs for region N (rows 1--40).}\label{mfl:app-tab:inputs-partition-1}
\centering\footnotesize
\setlength{\tabcolsep}{3pt}
\begin{tabular}{r r r r r}\toprule
row & $K$ & $h$ & $j$ & $T$\\\midrule
2a & 21600665 & 35 & 0 & 7920963464610295\\
2b & 36059480 & 3 & 1 & 8085720989539134\\
3 & 29042409 & 35 & 0 & 8180717210882235\\
4 & 35333865 & 4 & 0 & 10116431691206438\\
5 & 34521219 & 4 & 0 & 12691884803174992\\
6 & 35458646 & 4 & 0 & 14952843575107892\\
7 & 38181797 & 4 & 0 & 15782428537537256\\
8 & 39344668 & 4 & 0 & 17690797072963136\\
9 & 35233202 & 35 & 0 & 22709285047772204\\
10 & 39985348 & 35 & 0 & 22746804483491268\\
11 & 41193308 & 35 & 0 & 26110149771855272\\
12 & 44290801 & 35 & 0 & 26689626722048196\\
13 & 45557481 & 35 & 0 & 28114570042194732\\
14 & 46829404 & 35 & 0 & 29439767329931008\\
15 & 48102375 & 35 & 0 & 30674718207391344\\
16 & 51380224 & 35 & 0 & 30812462728338840\\
17 & 52688846 & 35 & 0 & 31876420407381584\\
18 & 53996421 & 35 & 0 & 32873880731484156\\
19 & 55301898 & 35 & 0 & 33819093133848028\\
20 & 56606326 & 35 & 0 & 34702557992338884\\
21 & 57912852 & 35 & 0 & 35505276062688100\\
22 & 55029268 & 3 & 0 & 38838108865442336\\
23 & 58798899 & 3 & 0 & 36167659532130680\\
24 & 60316188 & 3 & 0 & 38489789387184296\\
25 & 61824040 & 3 & 0 & 40711645453042408\\
26 & 65633517 & 3 & 0 & 38136192341073864\\
27 & 67151855 & 3 & 0 & 40189166235655360\\
28 & 68661805 & 3 & 0 & 41661607666473440\\
29 & 72502738 & 3 & 0 & 39698352425382656\\
30 & 71661780 & 3 & 0 & 42553516633528128\\
31 & 73151807 & 3 & 0 & 42970444493867288\\
32 & 74635542 & 3 & 0 & 43371539650649280\\
33 & 76111937 & 3 & 0 & 43756802103874080\\
34 & 80035708 & 3 & 0 & 43144604232996312\\
35 & 81522589 & 3 & 0 & 43508756414811544\\
36 & 83004227 & 3 & 0 & 43857075893069584\\
37 & 84480622 & 3 & 0 & 44200117803475232\\
38 & 85952823 & 3 & 0 & 44527327010323696\\
39 & 87418732 & 3 & 0 & 44843981081467368\\
40 & 88880447 & 3 & 0 & 45150080016906256\\
\bottomrule\end{tabular}
\end{table}
\begin{table}[!htbp]
\caption{Exact partition inputs for region N (rows 41--80).}\label{mfl:app-tab:inputs-partition-41}
\centering\footnotesize
\setlength{\tabcolsep}{3pt}
\begin{tabular}{r r r r r}\toprule
row & $K$ & $h$ & $j$ & $T$\\\midrule
41 & 90337968 & 3 & 0 & 45445623816640352\\
42 & 91790245 & 3 & 0 & 45735890048522056\\
43 & 93238329 & 3 & 0 & 46015601144698960\\
44 & 94680121 & 3 & 0 & 46290034673023472\\
45 & 96119816 & 3 & 0 & 46553913065643208\\
46 & 97556365 & 3 & 0 & 46807236322558152\\
47 & 98987671 & 3 & 0 & 47060559579473088\\
48 & 100414783 & 3 & 0 & 47303327700683232\\
49 & 104529395 & 3 & 0 & 46675297126248288\\
50 & 105968041 & 3 & 0 & 46907510111753648\\
51 & 107402493 & 3 & 0 & 47134445529406608\\
52 & 108832751 & 3 & 0 & 47356103379207184\\
53 & 110260912 & 3 & 0 & 47572483661155360\\
54 & 111684878 & 3 & 0 & 47788863943103536\\
55 & 113105698 & 3 & 0 & 47994689089346928\\
56 & 114524422 & 3 & 0 & 48195236667737920\\
57 & 115937902 & 3 & 0 & 48390506678276520\\
58 & 117350334 & 3 & 0 & 48585776688815120\\
59 & 118759620 & 3 & 0 & 48770491563648928\\
60 & 120164712 & 3 & 0 & 48955206438482736\\
61 & 121567707 & 3 & 0 & 49134643745464144\\
62 & 122968604 & 3 & 0 & 49314081052445568\\
63 & 124366356 & 3 & 0 & 49482963223722192\\
64 & 125762011 & 3 & 0 & 49651845394998824\\
65 & 127154520 & 3 & 0 & 49815449998423056\\
66 & 128544931 & 3 & 0 & 49979054601847280\\
67 & 129932197 & 3 & 0 & 50137381637419120\\
68 & 131317366 & 3 & 0 & 50290431105138560\\
69 & 132700438 & 3 & 0 & 50443480572858008\\
70 & 137042591 & 3 & 0 & 49852392973389808\\
71 & 138436149 & 3 & 0 & 50000164873256856\\
72 & 139826561 & 3 & 0 & 50147936773123904\\
73 & 141213827 & 3 & 0 & 50290431105138560\\
74 & 142600044 & 3 & 0 & 50432925437153216\\
75 & 143984164 & 3 & 0 & 50570142201315480\\
76 & 145366188 & 3 & 0 & 50707358965477736\\
77 & 146745065 & 3 & 0 & 50839298161787600\\
78 & 148122894 & 3 & 0 & 50971237358097464\\
79 & 149498626 & 3 & 0 & 51097898986554928\\
80 & 150872260 & 3 & 0 & 51224560615012400\\
\bottomrule\end{tabular}
\end{table}
\begin{table}[!htbp]
\caption{Exact supersolution inputs for region N (rows 1--40).}\label{mfl:app-tab:inputs-supersolution-1}
\centering\footnotesize
\setlength{\tabcolsep}{3pt}
\begin{tabular}{r r r r}\toprule
row & $O$ & $B$ & $C$\\\midrule
2a & 5292723980082718 & 32428288108055624 & 5608012396498452\\
2b & 6759749330873526 & 47234666444723024 & 4567102900135213\\
3 & 6150004491384073 & 27557958287837056 & 5960833993768631\\
4 & 6813178771523468 & 24649586303920624 & 5919624845377477\\
5 & 6805068924342156 & 19562686959486248 & 5979297203968494\\
6 & 6961059657208525 & 17116075937116040 & 6087232431130972\\
7 & 7244570957128253 & 15814519020401872 & 6084414177483384\\
8 & 7397342834379401 & 14265978139366600 & 5982018411495578\\
9 & 7165248123289302 & 11994103524559812 & 6108689384936562\\
10 & 7624330648560800 & 12188339050011954 & 6107855295699327\\
11 & 7816992396520535 & 11584092604917802 & 6047903646975958\\
12 & 8114724398401507 & 11408304546243588 & 6026939544538555\\
13 & 8265067162849759 & 10878388211123606 & 5983459685048755\\
14 & 8412778197264684 & 10428066953059118 & 5946617038309694\\
15 & 8557675136672041 & 10039938774911952 & 5915045146480910\\
16 & 8842706773114639 & 10037296873231022 & 5908968387516019\\
17 & 8982767256514036 & 9724656117379244 & 5883696393520844\\
18 & 9120312371788476 & 9447363298365294 & 5861474521720000\\
19 & 9255597473357322 & 9199936591472062 & 5841814212876274\\
20 & 9388473355124492 & 8977254943204165 & 5824337885408719\\
21 & 9518708692360278 & 8775033938432237 & 5808753889297893\\
22 & 9475197000037700 & 8488557454026792 & 6097795241830543\\
23 & 9713465021536114 & 8525516794517812 & 6100380809618886\\
24 & 9901178661896920 & 8482673800498109 & 6089131885914375\\
25 & 10084705286757180 & 8443389019409702 & 6079829751118036\\
26 & 10310197713544460 & 8480045370510340 & 6082520858566669\\
27 & 10486821438934066 & 8443337268778820 & 6073497625162175\\
28 & 10646271907328426 & 8386871754573909 & 6065406584806843\\
29 & 10874580845956476 & 8445379180561290 & 6068856206736862\\
30 & 10904259663039498 & 8197305009682678 & 6047007800350284\\
31 & 11030272289731130 & 8110909709240749 & 6038778496457271\\
32 & 11154461939110860 & 8029491184468583 & 6031111719018059\\
33 & 11276785963160284 & 7952426465393715 & 6023943385245597\\
34 & 11535056041251706 & 8077705013149340 & 6029545993731846\\
35 & 11653987293091250 & 8004452736735269 & 6022722784745882\\
36 & 11771298843906270 & 7934791483279786 & 6016304083349640\\
37 & 11887327152694930 & 7868826098726293 & 6010267694267958\\
38 & 12001848803664564 & 7805896294084367 & 6004573184671171\\
39 & 12114905764571344 & 7745765307089613 & 5999189004889925\\
40 & 12226654541612578 & 7688369272956011 & 5994099997093083\\
\bottomrule\end{tabular}
\end{table}
\begin{table}[!htbp]
\caption{Exact supersolution inputs for region N (rows 41--80).}\label{mfl:app-tab:inputs-supersolution-41}
\centering\footnotesize
\setlength{\tabcolsep}{3pt}
\begin{tabular}{r r r r}\toprule
row & $O$ & $B$ & $C$\\\midrule
41 & 12337113433222418 & 7633481250065179 & 5989283289725050\\
42 & 12446397115720632 & 7581022616188417 & 5984719452569983\\
43 & 12554428182284166 & 7530683027299399 & 5980389152653382\\
44 & 12661257956522734 & 7482341945772299 & 5976273220784853\\
45 & 12766998498356294 & 7435959038899650 & 5972366854639561\\
46 & 12871597276411164 & 7391309963979724 & 5968652097336570\\
47 & 12975276608248382 & 7348516067938118 & 5965118537347517\\
48 & 13077781077239216 & 7307133485710754 & 5961749265244648\\
49 & 13310520257288940 & 7416909742886614 & 5965930815396461\\
50 & 13410953891133806 & 7376288651918253 & 5962574449157443\\
51 & 13510434200087326 & 7337061796677928 & 5959366216036499\\
52 & 13608974335688814 & 7299138910873558 & 5956297449397509\\
53 & 13706707044802002 & 7262568106318256 & 5953365894217178\\
54 & 13803700236105944 & 7227330813318137 & 5950562068819066\\
55 & 13899670972862910 & 7193028672532432 & 5947873578315527\\
56 & 13994864905308404 & 7159850747766006 & 5945300518889238\\
57 & 14089114945790088 & 7127546338099018 & 5942829274854179\\
58 & 14182846904304308 & 7096492258904211 & 5940467714917040\\
59 & 14275590210831866 & 7066127608059542 & 5938198062975710\\
60 & 14367659848247368 & 7036721999401950 & 5936020736832794\\
61 & 14458998169125878 & 7008154097169425 & 5933931755850891\\
62 & 14549797373180440 & 6980562010179316 & 5931930015212254\\
63 & 14639638500452052 & 6953473216162533 & 5930002652690999\\
64 & 14728956487954490 & 6927273706497871 & 5928154447869270\\
65 & 14817516618845280 & 6901684965020106 & 5926376715333673\\
66 & 14905570296063234 & 6876907518315706 & 5924670920206436\\
67 & 14992881188422112 & 6852671138826332 & 5923029152365483\\
68 & 15079510739209174 & 6828996676976786 & 5921450034687528\\
69 & 15165655963302758 & 6806033357063836 & 5919933626041060\\
70 & 15376809956722276 & 6898256573472885 & 5922142225437706\\
71 & 15461648514619670 & 6875227199849992 & 5920568180005206\\
72 & 15545977330830428 & 6852809223160655 & 5919051220385149\\
73 & 15629607382902882 & 6830796914691879 & 5917586011992899\\
74 & 15712847109331618 & 6809441371700104 & 5916176202553731\\
75 & 15795451327347682 & 6788490936913212 & 5914815157153569\\
76 & 15877623668380466 & 6768102119292371 & 5913503719228300\\
77 & 15959118060443858 & 6748030252218483 & 5912235937372244\\
78 & 16040243809047156 & 6728523450096594 & 5911015169338573\\
79 & 16120752682242078 & 6709340075482683 & 5909835739195862\\
80 & 16200851077071876 & 6690638106199919 & 5908698653676628\\
\bottomrule\end{tabular}
\end{table}
\FloatBarrier

\endgroup

\end{document}